%% file: main.tex
\documentclass[letterpaper,twocolumn,10pt]{article}
\usepackage{usenix}

\usepackage{tikz}
\usepackage{amsmath}
\input{math_commands.tex}

\usepackage{todonotes}
\usepackage{amssymb}
\usepackage{mathtools}
\usepackage{amsthm}

\usepackage[utf8]{inputenc} %
\usepackage[T1]{fontenc}    %
\usepackage{hyperref}       %
\usepackage{xurl}
\usepackage{booktabs}       %
\usepackage{amsfonts}       %
\usepackage{nicefrac}       %
\usepackage{microtype}      %
\usepackage{xcolor}         %

\usepackage{booktabs}
\usepackage{soul}
\usepackage[numbers]{natbib}
\usepackage{comment}

\usepackage[justification=centering]{subfig}
\usepackage{graphicx}
\usepackage{cleveref}
\usepackage{fancyhdr}

\usepackage{enumitem}

\usepackage{nopageno}

\theoremstyle{plain}
\newtheorem{theorem}{Theorem}[section]

\newtheorem{lemma}[theorem]{Lemma}
\newtheorem{corollary}[theorem]{Corollary}
\newtheorem{fact}[theorem]{Fact}
\theoremstyle{definition}
\newtheorem{definition}[theorem]{Definition}

\theoremstyle{remark}

\newif\ifarxiv
\arxivtrue

\renewcommand{\paragraph}[1]{\noindent\textbf{#1}~~}

\DeclareRobustCommand*{\authorrefmark}[1]{\raisebox{0pt}[0pt][0pt]{\textsuperscript{\footnotesize\ensuremath{\ifcase#1\or *\or \dagger\or \ddagger\or%
    \mathsection\or \mathparagraph\or \|\or **\or \dagger\dagger%
    \or \ddagger\ddagger \else\textsuperscript{\expandafter\romannumeral#1}\fi}}}}

\begin{document}

\date{}

\title{\Large \bf Gradients Look Alike: Sensitivity is Often Overestimated in DP-SGD}

\author{Anvith Thudi\authorrefmark{3}\authorrefmark{4},Hengrui Jia\authorrefmark{3}\authorrefmark{4},Casey Meehan\authorrefmark{2},Ilia Shumailov\authorrefmark{1},Nicolas Papernot\authorrefmark{3}\authorrefmark{4} \\ 
 University of Toronto\authorrefmark{3}, Vector Institute\authorrefmark{4}, University of California, San Diego\authorrefmark{2}, University of Oxford\authorrefmark{1}}
\maketitle

\begin{abstract}
\looseness=-1
Differentially private stochastic gradient descent (DP-SGD) is the canonical approach to private deep learning. While the current privacy analysis of DP-SGD is known to be tight in some settings, several empirical results suggest that models trained on common benchmark datasets leak significantly less privacy for many datapoints. Yet, despite past attempts, a rigorous explanation for why this is the case has not been reached. Is it because there exist tighter privacy upper bounds when restricted to these dataset settings, or are our attacks not strong enough for certain datapoints? In this paper, we provide the first per-instance (i.e., ``data-dependent") DP analysis of DP-SGD. Our analysis captures the intuition that points with similar neighbors in the dataset enjoy better data-dependent privacy than outliers. Formally, this is done by modifying the per-step privacy analysis of DP-SGD to introduce a dependence on the distribution of model updates computed from a training dataset. We further develop a new composition theorem to effectively use this new per-step analysis to reason about an entire training run. Put all together, our evaluation shows that this novel DP-SGD analysis allows us to now \emph{formally} show that DP-SGD leaks significantly less privacy for many datapoints (when trained on common benchmarks) than the current data-independent guarantee. This implies privacy attacks will necessarily fail against many datapoints if the adversary does not have sufficient control over the possible training datasets. \footnote{Accepted at the 33rd USENIX Security Symposium. This version contains an extended appendix.}
\end{abstract}

\input{Sections/Introduction}

\input{Sections/Background}

\input{Sections/Analysis_V2}
\input{Sections/Main_Body_Empirical_Results}

\input{Sections/Discussion}

\input{Sections/Conclusion}

\ifarxiv
\input{Sections/acknowledgements}
\else
\footnotesize
\input{Sections/acknowledgements}
\fi

\ifarxiv
\setlength{\bibsep}{0pt plus 0.3ex}
\bibliographystyle{abbrvnat}
\bibliography{references}
\else
\footnotesize
\setlength{\bibsep}{0pt plus 0.3ex}
\bibliographystyle{abbrvnat}
\bibliography{references}
\fi

\normalsize
\appendix
\newpage

\input{Sections/Appendix}

\end{document}

%% file: math_commands.tex
\usepackage{amsmath,amsfonts,bm}

\def\eqref#1{equation~\ref{#1}}

\def\1{\bm{1}}

\DeclareMathAlphabet{\mathsfit}{\encodingdefault}{\sfdefault}{m}{sl}
\SetMathAlphabet{\mathsfit}{bold}{\encodingdefault}{\sfdefault}{bx}{n}

%% file: Sections/Introduction.tex
\section{Introduction}

Differential Privacy (DP) is the standard framework for private data analysis~\citep{dwork2006calibrating}. Making an algorithm differentially private limits the success any attack can have in knowing whether any datapoint was or was not an input to the algorithm given just the outputs of the algorithm. %
To obtain this notion of indistinguishability the algorithm needs to perform a noisy analysis of the data. %
In the case of deep learning, the canonical private training algorithm is DP-SGD~\citep{abadi2016deep}, where Gaussian noise is added to the gradients computed on training examples. Much work has gone into improving the privacy analysis of DP-SGD for a given amount of noise \citep{mironov2017renyi,mironov2019r,gopi2021numerical} in an effort to minimize the impact of noise on performance.

\looseness=-1
To reiterate, this current privacy analysis for DP-SGD  is \textit{data independent}: it assumes an upper-bound on how much \emph{any} individual datapoint from \textit{any} dataset can have their privacy leaked. It is furthermore now known to be tight; there exist specific pairs of datasets and models for which a privacy attack can match the upper-bound of DP-SGD~\citep{nasr2021adversary}. Yet, when training on common benchmark datasets like CIFAR10, \citet{carlini2022privacy} empirically saw that even strong privacy attacks perform significantly worse for many datapoints than the guarantees associated with DP-SGD. %
That is, when training on real-world data, there is a gap between what our strongest attacks can achieve and what our current data-independent privacy analysis of DP-SGD can tell us. Hence the question, why is there a gap? Is it because our attacks are still too weak for many datapoints, or is it because there exist tighter privacy upper bounds when restricted to these dataset settings? Past work attempted to answer this by analyzing specific privacy attacks \citep{mahloujifar2022optimal, thudi2022bounding, guo2022bounding}, or studying a weakened notion of DP~\citep{yu2022individual}. But all past work either have bounds limited in scope or with unproven assumptions. No work has yet derived tighter indistinguishability guarantees that are specific to the data being analyzed, analogous to how DP gives indistinguishability guarantees to prevent privacy attacks (for all possible datasets).

Our work provides the first per-instance DP analysis of DP-SGD, i.e., bounds on the distinguishability of outputted models that are specific to training on a given dataset or the dataset plus a point. This analysis bridges the theoretical gap between the tight data-independent analysis of DP-SGD and what is achievable when training on common deep-learning datasets. %
These new guarantees follow from an exploration of the role of \textit{sensivity} in privacy analysis. Currently, to obtain data-independent privacy guarantees, a model trainer needs to bound how much \emph{any} individual datapoint from \textit{any} dataset can contribute to a gradient update---a quantity known as the algorithm's sensitivity. This is currently done by setting an upper-bound ahead of time which is enforced during training by clipping the gradient computed on each datapoint %
to a norm below this preset sensitivity value. However, we highlight that this overestimates the sensitivity of DP-SGD to a \emph{specific} datapoint in a \textit{given} dataset when that datapoint's update is similar to the update given by many other datapoints in this dataset. In deep learning, many mini-batches in a dataset do produce similar gradients~\citep{shumailov2021dataorder, thudi2022necessity, kong2022forgeability}, hence such a case of overestimating sensitivity is common. Our per-instance (i.e., ``data-dependent") DP analysis of DP-SGD leverages this phenomenon.

Let us first focus on a single update of DP-SGD. Intuitively, if many of the datapoints produced almost the same gradient, then with high probability we would have obtained the same updated model with or without one of these datapoints. Making this intuition rigorous, we introduce a class of distributions we call \textit{sensitivity distributions}: broadly they capture the difference between updates computed from a given mini-batch to sampling another mini-batch. From this, we derive new bounds on the privacy leakage of a single DP-SGD update that incorporates how concentrated these distributions are at small values, i.e., have many mini-batches that produce almost the same gradient. Using this bound we can show that for many datapoints in common benchmark datasets, the individual per-step guarantee for that point can be magnitudes lower than the data-independent guarantee.

Building on our analysis for a single DP-SGD update, we give a per-instance bound on the \emph{overall} privacy leakage of a \emph{full DP-SGD run}. The current analysis considers the model that leaks the most privacy at every step (the worst-case model) and notes that summing the maximum per-step leakages bounds the overall privacy leakage of a DP-SGD run. %
Yet, the sensitivity distributions that we introduce are heavily dependent on the model being updated:  e.g., there is a difference between gradients computed using a partially-trained model and a randomly-initialized model. 
Towards not relying on analyzing worst-case models, intuitively it should not matter what the privacy leakage of the worst-case models is if they are  unlikely to be reached. %
More rigorously, we 
develop a new composition theorem
which allows us to upper-bound the overall per-instance privacy leakage of using DP-SGD by the expected privacy leakage at each step during training.

These analytic results give %
a new framework to understand the privacy guarantees of DP-SGD for individual datapoints. However, it remains to verify whether this analysis is tight enough to show that many datapoints have better privacy when training on benchmark datasets.
We thus turn to experimentation.\footnote{The code is at \url{https://github.com/cleverhans-lab/Gradients-Look-Alike-Sensitivity-is-Often-Overestimated-in-DP-SGD}.} 
The crux of implementing our results is to repeat training several times to compute the expected per-step privacy leakage. %
Because this one-dimensional statistic is bounded by the existing worst-case privacy analysis, one  achieves non-trivial estimates with few samples.  Doing this:

\begin{enumerate}[leftmargin=*,noitemsep,topsep=0pt]
    \item We show that when training on common benchmark datasets, many data points have better per-instance privacy than what the current data-independent guarantee associated with DP-SGD tells us.
    For some datapoints, we observe more than a magnitude improvement in the privacy guarantee  $\varepsilon$. %
    This explains the prior results we motivated our work with: for many datapoints, attacks that can only observe the outputs from training with or without the datapoint will fail.
    
    \item  In our framework, we observe a disparity where correctly-classified points obtain better privacy guarantees than misclassified points. In other words, %
    training algorithms that lead to high-performing models quantifiably leak less per-instance privacy for many points. This is as they reach states that have similar updates for large clusters of datapoints. We hypothesize that %
    designing model architectures to be more performative may also make them more private.
    \item In classical privacy analysis, training with higher mini-batch sampling rates leaks more privacy. However we find that for certain update rules, training with higher sampling rates can give better per-instance privacy because mini-batch updates concentrate on the dataset mean; this leads to many mini-batches with similar updates. 
    
\end{enumerate}
 
 \looseness=-1
 The consequences of our work are far reaching: having better per-instance DP guarantees has implications for unlearning, generalization, memorization and privacy auditing because of how DP formulates privacy by preventing distinguishability between the models trained with or without a datapoint. %
 For unlearning, a strong per-instance DP guarantee implies that the models coming from training with a datapoint are indistinguishable to the models trained without it. We further discuss in the paper how per-instance DP guarantees can still be used to satisfy a private notion of unlearning, and provide a naive first algorithm which we hope motivates future work. For generalization and memorization, a strong per-instance DP guarantee implies that the models coming from training without a datapoint perform similarly to those that had trained with it. For privacy auditing, our work provides empirical upper-bounds to complement previous work on lower-bounds established by strong privacy attacks, allowing future work to test if such attacks are tight. We note however that privacy auditing can affect the data-independent privacy guarantee, and discuss mitigations and paths for future work in the paper. With our framework, one can now say a specific datapoint does not need to be unlearned, or that a datapoint will not be memorized. However, our work leaves open how to apply our analysis to obtain better data-independent privacy guarantees when possible.

%% file: Sections/Background.tex
\section{Background}

Here we describe the current data-independent privacy analysis of DP-SGD (Section~\ref{ssec:DP-SGD_back}) and the relevance of per-instance DP in explaining empirical privacy attacks in contrast to past approaches (Section~\ref{ssec:per-instance_back}). We also discuss the implication of per-instance DP for unlearning and memorization in Section~\ref{ssec:per-instance_back}. Later, in Section~\ref{ssec:back_full_comp}, we describe past work on generalizing composition theorems and how they are not applicable for a better per-instance analysis of DP-SGD.

\paragraph{Machine Learning Notation}
We consider a learning setup where we have a dataset $X = \{x_1,\cdots x_n\}$ with datapoints from some space $\mathfrak{X}$ (e.g., images and their labels). Given a loss function $\mathcal{L}: \mathbb{R}^d \times \mathfrak{X} \rightarrow \mathbb{R}$, our objective is to minimize the loss $\frac{1}{n}\sum_{x_i \in X} \mathcal{L}(\theta,x_i)$ with respect to the parameters $\theta \in \mathbb{R}^d$ of some model. The canonical approach to do this for deep learning models is to use stochastic gradient descent (SGD). However, we consider having an additional requirement that the models we obtain should not leak the ``privacy" of individual datapoints

\subsection{DP-SGD Analysis}
\label{ssec:DP-SGD_back}

DP~\citep{dwork2006calibrating} is the de-facto definition of privacy used in ML. The typical definition used in machine learning is given below, where one thinks of $M$ as the training algorithm:

\begin{definition}[$(\epsilon,\delta)$-DP]
    An algorithm $M$ is said to be $(\epsilon,\delta)$-DP if for all neighbouring datasets $X, X'$ (i.e. Hamming distance $1$ apart), we have that $$\mathbb{P}(M(X) \in S) \leq e^{\epsilon} \mathbb{P}(M(X') \in S) + \delta$$
\end{definition}

DP-SGD~\citep{abadi2016deep,bassily2014private, song2013stochastic} is an $(\epsilon,\delta)$-DP version of stochastic gradient descent (SGD) which clips the individual gradients and adds Gaussian noise to the mini-batch update. Formally, given a dataset $X$, DP-SGD repeatedly computes the following deterministic update rule $U(X_{B} = \{x: x \sim X~with~probability~\frac{L}{|X|}\}) = \sum_{x \in X_B} \nabla_{\theta}\mathcal{L}(\theta,x)/ \max(1,\frac{||\nabla_{\theta}\mathcal{L}(\theta,x)||_2}{C})$ and then updates $\theta  \rightarrow \theta - \eta \frac{1}{L}(U(X_B) + N(0,\sigma^2 C^2))$.

The current tightest privacy analysis of DP-SGD uses R\'enyi-DP (RDP)~\citep{mironov2017renyi} which implies $(\epsilon,\delta)$-DP; the merits of first working with RDP is that it provides a tighter privacy analysis for releasing the composition of multiple steps in DP-SGD -- where each step is an update computed on a different mini-batch $X_B$. An algorithm $M$ is $(\alpha,\epsilon)$-R\'enyi DP if for all neighbouring datasets $X, X'$ we have
$D_{\alpha}(M(X)|| M(X')) \leq \epsilon$
where for two probability distributions $P,Q$ we define the $\alpha$-R\'enyi divergence as
\vspace{-3mm} $$D_{\alpha}(P||Q) \coloneqq \frac{1}{\alpha -1} \ln \mathbb{E}_{x \sim Q} (\frac{P}{Q})^{\alpha}$$ \vspace{-4mm}

The RDP analysis follows two steps:

\begin{enumerate}[leftmargin=*,noitemsep,topsep=0pt]
    \item Per Step: Analyzing the privacy guarantee of each training step $\eta \frac{1}{L}(U(X_B) + N(0,\sigma^2 C^2))$, which is the same as $U(X_B) + N(0,\sigma^2 C^2)$ by the post-processing property of RDP (the output of a DP algorithm can be post-processed without degrading the DP guarantee provided). 
    \item Composition: Understanding the accumulated RDP guarantee of releasing all the updates. 
\end{enumerate}

The first part was analytically studied in \citet{mironov2019r} and is called the sampled Gaussian mechanism. The accumulation step follows from the composition theorem for RDP~\citep{mironov2017renyi}. In this paper, we provide new per-step and composition privacy analyses for DP-SGD that are \emph{specific} to a pair of neighbouring datasets $X,X'$.

\subsection{Motivation for Studying Per-Instance DP} 
\label{ssec:per-instance_back}

In contrast to the classical analysis of DP-SGD, we will analyze its per-instance R\'enyi DP guarantees~\citep{wang2019per} -- also known as "Individual R\'enyi DP"~\cite{feldman2021individual}. That is, the RDP guarantee specific to a \textit{given} pair of neighbouring datasets.

\begin{definition}[Per-Instance R\'enyi DP]
\label{def:per-instance-DP}
    We say an algorithm $M$ is $(\alpha, \epsilon)$ per-instance R\'enyi DP for a pair of datasets $X,X' = X \cup x^*$ if $$\max \{D_{\alpha}(M(X)||M(X')), D_{\alpha}(M(X')||M(X)) \}  \leq \epsilon$$
\end{definition}

Colloquially, when $X$ is understood from context, we will specify the per-instance guarantee by saying an algorithm is $(\alpha,\epsilon)$-R\'enyi DP for a point $x^*$ (which determines $X'$).

Per-instance DP guarantees provide a privacy upper bound for an adversary trying to distinguish between a specific pair of datasets $X,X'$ given the ouput of $M$ on one of them, %
and not a bound for all neighbouring datasets like classical DP. However, this granularity allows for tighter analysis of each $X,X'$ case. The tighter per-instance DP bounds we derive will allow us to say that for specific pairs of datasets $X,X' = X \cup x^*$, privacy attacks against $x^*$ will fail when the adversary can only observe the outputs from $X$ or $X'$. More generally, if there are strong per-instance guarantees for all the neighbouring datasets the adversary can observe, then they will still fail. Instantiating our analysis, we will show that on common benchmark datasets, an adversary trying to distinguish if a specific point was added or not will fail for many points.

\looseness=-1
Our work on upper bounding per-instance DP guarantees is contrasted with past work on rigorously explaining when privacy attacks against DP-SGD will perform worse than what is implied by the current (tight) data-independent analysis. One line of work has been to upper-bound the performance of specific attacks. Putting aside the limitation in only upper-bounding specific attacks, this line of work either lacks an individualized guarantee to explain the difficulty for individual points~\citep{mahloujifar2022optimal} or relies on a particular threat model to explain better privacy~\citep{thudi2022bounding}. In the case improved individual upper-bounds were achieved~\cite{guo2019certified}, this was with bounds that can fail due to assumptions. In short, this line of work lacks the generality/strength of Definition~\ref{def:per-instance-DP} in explaining why \textit{any} empirical privacy attack will perform worse in some data settings.

\looseness=-1
A more recent line of work has been to attempt to do individual (i.e., per-instance) DP accounting for DP-SGD~\citep{yu2022individual}. However, \citet{yu2022individual} could not analyze the per-instance guarantees of DP-SGD and instead relied on a weaker guarantee that holds if intermediate models were not random (which is not true for DP-SGD)%
. The main technical bottleneck to extend their approach to analyze DP-SGD, as also noted by \citet{yu2022individual}, was how to effectively analyze composition when the intermediate models are random variables. Our work provides a new composition theorem to handle this technical issue, and in doing so provides proper per-instance DP guarantees without the assumptions present in \citet{yu2022individual}.

\looseness=-1
Per-instance DP guarantees are also important beyond privacy. Memorization~\citep{feldman2020does} is a per-instance quantity (only reasoning about a particular pair $X,X'$), and hence is bounded by Definition~\ref{def:per-instance-DP}. Similarly, unlearning is a per-instance quantity, and a growing section of the literature uses per-instance DP guarantees to quantify unlearning~\citep{guo2019certified}. Hence in providing per-instance DP bounds for DP-SGD, we have also quantified a set of points that will not be memorized nor need to be unlearned (as they are already unlearnt). We however remark that care must be taken to user per-instance guarantees without voiding other privacy guarantees, and discuss this for unlearning where the privacy guarantee is to be agnostic to the order of unlearning requests in Section~\ref{ssec:apps}. We refer the reader to \citet{kulynych2022you} for a more general discussion on the utility of DP inequalities in studying properties of deep learning.

%% file: Sections/Analysis_V2.tex
\section{A Per-Instance Analysis of DP-SGD}
\label{sec:analysis}

We now present our new analysis of DP-SGD which removes the data-independent nature of the per-step and composition analyses currently used for DP-SGD. The impact of this new analysis is presented in Section~\ref{sec:main_body_emp_results}, where we show that many datapoints have much better privacy than suggested by the current analysis of DP-SGD, explaining the failure of many privacy attacks in practice.

The technical contributions that led to this are two-fold. At the per-step level, we generalize the notion of sensitivity to what we term \emph{sensitivity distributions}; given two datasets, sensitivity distributions capture how similar the updates between mini-batches from either dataset are. At the composition step, we generalize RDP composition to do accounting by the ``expected" intermediate privacy losses during training as opposed to the largest possible intermediate privacy losses. Together, we can now study the data-dependent behaviour of DP-SGD.

\subsection{Sensitivity Distribution Generalize the $(\epsilon,\delta)$-DP Analysis}
\label{ssec:eps_delta_case}

We first turn to $(\epsilon,\delta)$-DP, which is not used to analyze DP-SGD for composition reasons, but allows for simpler expressions to demonstrate the improvements afforded by particular data-dependent random variables we call \textit{sensitivity distributions}. In particular, in this section we will first consider the classical data-independent $(\epsilon,\delta)$-DP analysis of the sampled Gaussian mechanism $M$ and show how one can generalize this analysis and obtain tighter per-instance $(\epsilon,\delta)$-DP guarantees.

Recall that for an update rule $U$, the Gaussian mechanism is defined as $A(X) = U(X) + N(0,\sigma)$. The sampled Gaussian mechanism is then defined as $M(X) = A(\mathbf{X_B})$ where $\mathbf{X_B}$ is a mini-batch constructed from a dataset $X$ by sampling each datapoint $x \in X$ independently with probability $\mathbb{P}_{x}(1)$ (unless otherwise stated we think of $X_B$, not bold-face, as a specific mini-batch). Note, one assumes the sampling probability $\mathbb{P}_{x}(1)$ is only a function of $x$ and not the full dataset $X$, e.g., some fixed constant. The classical data-independent $(\epsilon,\delta)$-DP analysis of the sampled Gaussian mechanism follows two steps. First, we derive the guarantee for just the Gaussian mechanism. To do so, one first assumes a data-independent sensitivity bound $C_U$ on $U$: for all $X,X' = X \cup \{x^*\}$ we have $||U(X) - U(X')||_{2} \leq C_{U}$. This can be achieved by clipping the output values of $U$ to have a small norm. With this constant $C_U$ one has that the Gaussian mechanism $A$ gives the $(\epsilon,\delta)$-DP guarantee $\epsilon = C_{\delta,\sigma} C_U$ for some constant $C_{\delta,\sigma}$ depending on $\delta$ and $\sigma$ where $\sigma$ is the standard deviation of the added Gaussian noise~\footnote{For example, one can take $C_{\delta, \sigma} = \frac{\sqrt{2 \ln (1.25/\delta)}}{\sigma}$~\citep{dwork2014algorithmic}.}. To then analyze the sampled Gaussian mechanism one would incorporate the privacy gain from not sampling $x^*$ sometimes~\citep{beimel2014bounds}\citep{kasiviswanathan2011can} to get the privacy guarantees of $M$ %
as $(\epsilon',\delta')$-DP where $\epsilon' = \ln( \mathbb{P}_{x^*}(1) e^{C_{\delta,\sigma}~C_U} + \mathbb{P}_{x^*}(0))$ and $\delta' = \mathbb{P}_{x^*}(1) \delta$. Here $\mathbb{P}_{x^*}(0)=1-\mathbb{P}_{x^*}(1)$, and this gain in privacy by sometimes not using the datapoint is called privacy amplification by sampling.

Towards tightening this analysis into a per-instance analysis, let $$\Delta_{U,x^*}(X_B) \coloneqq ||U(X_B) - U(X_B \cup \{x^*\})||_2$$
then $\Delta_{U,x^*}(\mathbf{X_B})$ is a data-dependent random variable which we will call a \emph{sensitivity distribution}: it captures the change in the distribution of mini-batches updates caused by adding a point $x^*$ to the mini-batch. The classical data-independent analysis only (implicitly) uses sensitivity distributions via the data-independent bound $|\Delta_{U,x^*}(X_B)| \leq C_{U}~\forall X_B$. Instead, we will show how to directly use the $L_p$ norms $||\Delta_{U,x^*}(\mathbf{X_B})||_{p} = (\mathbb{E}_{X_B} (\Delta_{U,x^*}(X_B)^p))^{1/p}$ (or generally the $L_p$ norm of some monotonic transformation of $\Delta_{U,x^*}(\mathbf{X_B})$) to obtain tighter per-instance privacy guarantees. Furthermore, when using $p < \infty$, this analysis will be able to translate the phenomenon that many mini-batches produce similar updates into better privacy guarantees (as the sensitivity distribution concentrates at smaller values and hence has smaller $p$-norms). To emphasize this ability, past work that studied sampling relied mainly on the intuition that by sampling a datapoint with low probability, we have any given step often does not leak privacy for that point as it was not used. This translates to better privacy guarantes. By using the $L_p$ norms of sensitivity distributions with $p< \infty$ we make an additional observation, which is that if many of the other mini-batches produce the same update, then effectively we have an even lower probability of an attacker observing a noticeable shift due solely to that point.  %

In particular, recall that to prove per-instance $(\epsilon,\delta)$-DP for a pair of datasets $X,X'= X \cup \{x^*\}$ we need to bound $\mathbb{P}(M(X') \in S) \leq e^{\epsilon} \mathbb{P}(M(X) \in S) + \delta$ and $\mathbb{P}(M(X) \in S) \leq e^{\epsilon} \mathbb{P}(M(X') \in S) + \delta$. As a proof-of-concept on the role of sensitivity distributions, we present an analysis for the first inequality in Corollary~\ref{cor:eps_delta_sens} \footnote{We will later turn to R\'enyi-DP which provides both inequalities.}. Inspecting Corollary~\ref{cor:eps_delta_sens}, we see that it approximately follows the formula given by the classical analysis except the role of $C_U$ is replaced with a dependency on how concentrated $\Delta_{U,x^*}(X_B)$ is at small values (the $L_p$ norm of an exponential applied to $\Delta_{U,x^*}(X_B)$). When enough mini-batches provide updates more similar than the upper-bound $C_U$, the per-instance guarantee of Corollary~\ref{cor:eps_delta_sens} will significantly beat the classical data-independent analysis, as demonstrated for MNIST and CIFAR10 in Appendix~\ref{sec:detailed_emp_res}.

\begin{corollary}
\label{cor:eps_delta_sens}
For $p \in (1,\infty)$, let $a_p = \mathbb{P}_{x^*}(1) (\mathbb{E}_{x_{B}}(e^{C_{\delta,\sigma} \Delta_{U,x^*}(X_B)p}))^{1/p}$, $\epsilon' = \ln(a_p^{\frac{1}{1-1/p}}\delta'^{\frac{-1}{p-1}} + \mathbb{P}_{x^*}(0)) $ and $\delta'' = \mathbb{P}_{x^*}(1)\delta + \delta'$. Then, for $X' = X \cup \{x^*\}$ we have the following per-instance guarantee $$\mathbb{P}(M(X') \in S) \leq e^{\epsilon'} \mathbb{P}(M(X) \in S) + \delta''$$

\end{corollary}

\emph{Proof Sketch:} The proof of Corollary~\ref{cor:eps_delta_sens} follows two stages. First by expanding mini-batch sampling and applying Holder's inequality, we can show

\vspace{-5mm} 
\begin{multline}
    \mathbb{P}(M(X') \in S) \\ \leq \mathbb{P}_{x^*}(1) \mathbb{E}_{X_B}(e^{C_{\delta,\sigma} \Delta_{U,x^*}(X_B)p})^{1/p} \mathbb{P}(M(X) \in S)^{1-1/p} \\ + \mathbb{P}_{x^*}(1)\delta + \mathbb{P}_{x^*}(0) \mathbb{P}(M(X) \in S)
\end{multline} \vspace{-3mm}

This is stated as Lemma~\ref{lem:holder_approach}. One then analyzes the previous inequality in cases (first case is $\delta$ is upper-bounded, and else it is lower-bounded) to obtain an $(\epsilon,\delta)$-DP inequality.
The full proof of Corollary~\ref{cor:eps_delta_sens} is in Appendix~\ref{proof:eps_delta_sens}.

\subsection{Per-Instance R\'enyi-DP Analysis for DP-SGD}

With now an understanding of the power of incorporating $L_p$ norms of sensitivity distributions (upto some transformations) into DP analyses, we turn to analyzing the R\'enyi-DP guarantees of DP-SGD. R\'enyi-DP is more suited to compose the guarantees of each step of DP-SGD to obtain the guarantees for an entire training run. We first present per-step analyses for the sampled Gaussian mechanism, and then a new composition theorem to reason about the entire training run. We then discuss how to analyze DP-SGD for general update rules, i.e., not just the sum of gradients.

Our per-step analyses will focus on integer values of $\alpha$ for R\'enyi-DP. This is for simplicity, as R\'enyi divergences $D_{\alpha}(P||Q) \coloneqq \frac{1}{\alpha -1} \ln \mathbb{E}_{x \sim Q} (\frac{P}{Q})^{\alpha}$ are increasing in their order $\alpha$, hence we can bound the guarantee for any $\alpha$ by the guarantee for $\lceil \alpha \rceil$. In terms of notation, we will use ${X_B}^{\tilde \alpha} = ({X_B}^1,\cdots,{X_B}^{\alpha})$ to denote $\alpha$ mini-batches from $X$ (sampled independently if random). Analogously we use ${X'_B}^{\tilde \alpha}$ and $X'_B$ for $X'$.

\subsubsection{Per-Instance R\'enyi DP for the Sum Update Rule}
\label{ssec:sum_update}

In Section~\ref{ssec:eps_delta_case} we introduced the sensitivity distribution $\Delta_{U,x^*}(\mathbf{X_B}) = ||U(\mathbf{X_B}) - U(\mathbf{X_B \cup \{x^*\}})||_2$ and showed how directly leveraging its $L_p$ norms gives better per-instance DP analysis. In particular, how $p < \infty$ allows one to take advantage of expected sensitivity over mini-batches. However, for update rules of the form $U(X_B) = \sum_{x_i \in X_B} g(x_i)$ (i.e., the sum update rule typically used in DP-SGD) we have $\Delta_{U,x^*}(\mathbf{X_B})$ is always a constant: $\Delta_{U,x^*}(\mathbf{X_B}) = ||g(x^*)||_2$. Hence an analysis of the sampled Gaussian mechanism that used $\Delta_{U,x^*} \coloneqq \sup_{X_B \sim X} \Delta_{U,x^*}(X_B)$ would effectively capture all $L_p$ norms of the sensitivity distribution $\Delta_{U,x^*}(X_B)$ for the sum update rule. We state such a per-instance version of the classical RDP analysis of the sampled Gaussian mechanism below.

\begin{theorem}
\label{thm:easy_renyi_dp}
    For integer $\alpha > 1$, the sampled Gaussian mechanism with noise $\sigma$ and sampling probability $\mathbb{P}_{x^*}(1)$ for $x^*$ is $(\alpha,\epsilon)$ per-instance R\'enyi DP for $X, X' = X \cup x^*$ with:

    \vspace{-3mm}
    $$\epsilon = \frac{1}{\alpha -1} \ln(\sum_{k=0}^{\alpha} {\alpha \choose k} (1 - \mathbb{P}_{x^*}(1))^{\alpha -k} \mathbb{P}_{x^*}(1)^k \\ \exp{\frac{\Delta_{U,x^*}^2(k^2 - k)}{2 \sigma^2}})$$

\end{theorem}

Note that some key variables in Theorem~\ref{thm:easy_renyi_dp} are the sampling rate $\mathbb{P}_{x^*}(1)$ (increasing it typically increases the bound), the standard deviation of noise $\sigma$ (increasing it typically decreases the bound), and the sensitivity upper-bound over minibatches $\Delta_{U,x^*}$ (increasing it typically increases the bound). The proof strategy is analogous to \citet{mironov2019r} and replaces their sensitivity upper-bound with the per-instance bound $\Delta_{U,x^*}$ on the mini-batches. 

\emph{Proof Sketch: } The proof follows by noting the density function for $M(X')$ can be written as a convex combination of $M(X)$ and a translated version of $M(X)$. One then proceeds to apply the quasi-convexity of R\'enyi divergences, and direct calculations with the Gaussian density function and the symmetry between the terms (due to translation). The full proof of Theorem~\ref{thm:easy_renyi_dp} is in Appendix~\ref{proof:easy_renyi_dp}.

\subsubsection{A Generalized R\'enyi-DP Composition}
\label{ssec:comp}

With now an analysis for the per-step guarantees from DP-SGD (which as currently implemented uses the sum update rule), we now resolve how to obtain a per-instance RDP bound for a full training run with DP-SGD without the limitations of past composition theorem (see Section~\ref{ssec:back_full_comp} for a discussion on past composition bounds). In particular, we provide a composition theorem that bounds the overall per-instance privacy leakage by the ``expected" per-instance privacy guarantee at each step when training on a given dataset. This is presented in Theorem~\ref{thm:better_composition}.

\looseness=-1
More technically, we once again generalize the classical analysis to look at arbitrary $L_p$ norms, but now for the composition step. The classical R\'enyi DP composition theorem implicity uses the $L_\infty$ norm of the distribution of per-step guarantees at each step (coming from the distribution of possible models at each step as training is random), and Theorem~\ref{thm:better_composition} generalizes this to arbitrary $L_p$ norms of the exponential of the per-step guarantees (with some constants to scale). By using $L_p$ norms with $p < \infty$ we take advantage of cases where many models have better privacy guarantees than the worst model.

\begin{theorem}
\label{thm:better_composition}

    Let $p \in (1,\infty)$ and consider a sequence of functions $X_1(x_1),$ $X_2(x_1,x_2),$ $X_3(x_2,x_3),\cdots X_n(x_{n-1},x_n)$ where $X_{i}$ is a density function in the second argument for any fixed value of the first argument, except $X_1$ which is a density function in $x_1$. Consider an analogous sequence $Y_1(x_1),\cdots, Y_n(x_{n-1}, x_n)$. Then letting $X = \prod_{j=1}^{n} X_j$ be the density function for a sequence $x_1,\cdots,x_n$ generated according to the Markov chain defined by $X_i$, and similarly $Y$, we have

    \vspace{-5mm} \begin{multline}
        D_{\alpha}(X || Y)  \leq  \\ \frac{1}{\alpha -1} (\sum_{i=0}^{n-2} \frac{(p-1)^i}{p^{i+1}} \\ \ln (\mathbb{E}_{X_1,\cdots X_{n-(i+1)}}  (e^{(g_p^{i}(\alpha) -1)D_{g_p^{i}(\alpha)}(X_{n-i}|| Y_{n-i})p}))) \\ + \frac{1}{\alpha -1} (\frac{p-1}{p})^{n-1} (g_p^{n-1}(\alpha) -1)D_{g_p^{n-1}(\alpha)}(X_{1}|| Y_{1}) 
    \end{multline} \vspace{-3mm}

    where $g_p(\alpha) = \frac{p}{p-1} \alpha - \frac{1}{p}$ and $g_p^{i}$ is $g_p$ composed $i$ times, where we defined $g_p^{0}(\alpha) = \alpha$.
\end{theorem}

Note some key variables in Theorem~\ref{thm:better_composition} are a flexible parameter $p$ (which we'll soon describe leads to blow-up as it gets smaller), and the distribution of per-step guarantees $D_{g_p^{i}(\alpha)}(X_{n-i}|| Y_{n-i})$ (the more concentrated at $0$ they are, the smaller the upper-bound). The proof relies on using an induction argument to continually break up the composition and is presented below.

\begin{proof}

The proof follows by repeating a similar reduction as Theorem~\ref{thm:composition}. First note 
    
\vspace{-5mm} \begin{multline}
    \int (X_1 \cdots X_n)^{\alpha} (Y_1 \cdots Y_n)^{1 - \alpha} dx_1 \cdots dx_n \\  = \int (X_1 \cdots X_{n-1})^{\alpha - 1/p} (Y_1 \cdots Y_{n-1})^{1 - \alpha}  \\ (\int X_n^{\alpha} Y_n^{1- \alpha} dx_n) (X_1 \cdots X_{n-1})^{1/p} dx_1 \cdots dx_{n-1}
    \\ \leq ( \int (X_1 \cdots X_n)^{\frac{p}{p-1}\alpha - \frac{1}{p-1}} (Y_1 \cdots Y_n)^{ \frac{p}{p-1}(1 - \alpha)}  dx_1 \cdots dx_{n-1})^{\frac{p-1}{p}} \\ (\int (\int X_n^{\alpha} Y_n^{1- \alpha} dx_n)^p (X_1 \cdots X_{n-1}) dx_1 \cdots dx_{n-1})^{1/p} 
\end{multline} \vspace{-3mm}

where the first equality was from using the markov property, and the last inequality was from Holder's inequality with Holder constant $p$. Do note that, defining $g_p(\alpha) = \frac{p}{p-1}\alpha - \frac{1}{p-1}$, we have $\frac{p}{p-1}(1 - \alpha) = 1 - g_p(\alpha)$. So now looking at the first term of the upper-bound we got, we are back to the original expression but with $\alpha \rightarrow g_p(\alpha)$ and $n \rightarrow n-1$, and an exponent to $\frac{p-1}{p}$. Note the second term is an expectation over the $n-1$ model state of the Markov chain. Do note $\int X_n^{\alpha} Y_n^{1- \alpha} dx_n$ is $e^{(\alpha -1)D_{\alpha}(X_{n-i}|| Y_{n-i})}$ for a fixed $n-1$ model state (i.e., fixed $x_{n-1}$ ). So repeating this step on the first term until we are left only with an integral over $x_1$ we have

\vspace{-5mm} \begin{multline}
    \int (X_1 \cdots X_n)^{\alpha} (Y_1 \cdots Y_n)^{1 - \alpha} dx_1 \cdots dx_n \\  
    \leq (\prod_{i=0}^{n-2} (\mathbb{E}_{X_1,\cdots X_{n-(i+1)}}  ((e^{(g_p^{i}(\alpha) -1)D_{g_p^{i}(\alpha)}(X_{n-i}|| Y_{n-i})})^p))^{\frac{(p-1)^i}{p^{i+1}}}) \\ ( (e^{(g_p^{n-1}(\alpha) -1)D_{g_p^{n-1}(\alpha)}(X_{1}|| Y_{1})})^p)^{\frac{(p-1)^{n-1}}{p^n}}
\end{multline} \vspace{-3mm}

So now noting $$D_{\alpha}(X || Y) = \frac{1}{\alpha -1} \ln(\int (X_1 \cdots X_n)^{\alpha} (Y_1 \cdots Y_n)^{1 - \alpha} dx_1 \cdots dx_n)$$

we conclude by the previous expression that 

\vspace{-5mm} \begin{multline}
        D_{\alpha}(X || Y) \\ \leq \frac{1}{\alpha -1} (\sum_{i=0}^{n-2} \frac{(p-1)^i}{p^{i+1}} \\  \ln (\mathbb{E}_{X_1,\cdots X_{n-(i+1)}}  ((e^{(g_p^{i}(\alpha) -1)D_{g_p^{i}(\alpha)}(X_{n-i}|| Y_{n-i})p}))) \\ + \frac{1}{\alpha -1} ((\frac{(p-1)^{n-1}}{p^n}) \ln ((e^{(g_p^{n-1}(\alpha) -1)D_{g_p^{n-1}(\alpha)}(X_{1}|| Y_{1})})^p)) 
    \end{multline} \vspace{-3mm}

which completes the proof as the last term simplifies to the term stated in the theorem.
\end{proof}

\paragraph{Applying to DP-SGD.} To interpret Theorem~\ref{thm:better_composition} in the context of DP-SGD, we can let $X_i$ be the distribution of the $i'th$ model update (for a fixed $(i-1)'th$ model) when training on one dataset $D$, and similarly $Y_i$ when training on a neighbouring dataset $D'$. Letting $Train_{DP-SGD}$ denote the Markov chain of the intermediate model updates when using DP-SGD, we have the maximum over the bound given by Theorem~\ref{thm:better_composition} on $D_{\alpha}(Train_{DP-SGD}(D)||Train_{DP-SGD}(D'))$ and $D_{\alpha}(Train_{DP-SGD}(D')||Train_{DP-SGD}(D))$ provides our per-instance RDP guarantee for DP-SGD.

\paragraph{Balancing the value of $p$.}To understand the dependence on $p$ in Theorem~\ref{thm:better_composition}, consider for a moment $p =2$. In this case, we observe that at the $i$'th step, we need to compute a R\'enyi divergence of order $\sim 2^{i} \alpha$. It is known that the R\'enyi divergence $D_{c}(P||Q)$ grows with $c$ \citep{van2014renyi}, and in the case of the Gaussian mechanism, this growth is linear with $c$~\citep{mironov2017renyi}. Hence this exponential growth in the R\'enyi divergence order can prove impractical as a useful tool to analyze DP-SGD. However, as $p \rightarrow \infty$ we see that the growth on the order of the divergence shrinks.

Yet, by taking larger $p$ values we are effectively taking larger $L_{p}$-norms of the per-step guarantees seen in training and so effectively turn to worst-case per-step analysis as $p \rightarrow \infty$. Hence it is desirable to choose $p$ just sufficient for there to not be a significant blow-up in the order of the divergences for a given $n$. This can be done by analyzing how $g_{p}^{i}(\alpha)$ grows.

\begin{fact}\label{fact:p_control}
    If $p = O(n)$ then $g_{p}^i(\alpha) \leq 2 \alpha~\forall i \leq n$. In particular, $p = 3n$ works for sufficiently large $n$.
\end{fact}

The proof follows from direct calculations with the formula for $g_{p}(\alpha)$. 

\begin{proof}

Note that $g_{p}(\alpha) \leq \frac{p}{p-1}\alpha$ hence $g_{p}^{i}(\alpha) \leq (\frac{p}{p-1})^{i}\alpha$. From this we see showing $\frac{p}{p-1}^{n} \leq 2$ for $p = O(n)$ will imply $g_p^{i}(\alpha) \leq 2 \alpha~\forall 1 \leq n$.

Note we can equivalently show $ln(\frac{p}{p-1}) = \ln(p) - \ln(p-1) \leq \frac{\ln (2)}{n}$. But if we take $p = 3n$ note $\ln(3n) - \ln(3n-1) \leq \frac{1}{3n-1}$ by the derivative of $\ln(x) \leq \frac{1}{3n-1}$ for $x \geq 3n-1$. So it suffices to show $\frac{1}{3n-1} \leq \frac{\ln(2)}{n}$, but this is true for sufficiently large $n$.

\end{proof}

\paragraph{Estimating Theorem~\ref{thm:better_composition}}

In cases where one does not know the expectations used in Theorem~\ref{thm:better_composition} analytically, as is the case with DP-SGD when it is applied to deep learning, one can resort to empirically estimating the means. Our goal is to understand how much better our data-dependent guarantees are than the data-independent baseline for DP-SGD on common datasets. Hence, we wish to estimate the expression of Theorem~\ref{thm:better_composition} (or specifically the per-step contributions) with an error
$c \epsilon$ for $c < 1$ where $\epsilon$ is the data-independent guarantee (per-step).

The following fact focuses on estimating the $i'th$ per-step guarantee with an error relative to the worst-case per-step guarantee when $p = 3n$ as is used in our experiments. In particular, letting $f \coloneqq (e^{(g_p^{i}(\alpha) -1)D_{g_p^{i}(\alpha)}(X_{n-i}|| Y_{n-i})})^p$ we have the $i'th$ per-step guarantee is $ \frac{1}{\alpha-1} \frac{(p-1)^i}{p^{i+1}} \ln (\mathbb{E}_{X_1, \cdots, X_{n-(i+1)}} f)$ and %
is less than the data-independent per-step privacy guarantee $\epsilon/n$ if $\mathbb{E}_{X_1,\cdots X_{n-(i+1)}} f \leq e^{(\alpha-1) 3 \epsilon}$ for $p = 3n$. Hence we describe the number of samples needed to estimate $\mathbb{E} f$ with precision relative to $e^{(\alpha-1) 3 \epsilon}$ (with high probability), which can be done in a constant number of samples relative to the data-independent bound.

\begin{fact}\label{fact:estimating}
    Let $\epsilon/n$ be the classical $\alpha$-R\'enyi DP guarantee for the $i'th$ step, and $\epsilon'/n$ be the analogous $2\alpha$-R\'enyi DP guarantee for the $i'th$ step. Then for $l \geq \frac{- \ln(J)}{c^2} e^{6(\alpha-1)\epsilon - 3(2\alpha -1) \epsilon'}$ and $p = 3n$ with $n$ s.t $g_p^{n-1} \leq 2\alpha$, we have $\mathbb{P}(|\mathbb{E}^{l} f - \mathbb{E}f| \geq c e^{(\alpha-1) 3 \epsilon}) \leq J$. Here $\mathbb{E}^l$ denotes the empirical mean over $l$ samples.
\end{fact}

The proof follows from Hoeffding's inequality.

\begin{proof}

For the given choice of $p$ and $\alpha$ we have $g_{p}^{i} \leq 2\alpha$ hence $D_{g_p^{i}(\alpha)}(X_{n-i} || Y_{n-i}) \leq D_{2 \alpha}(X_{n-i} || Y_{n-i}) \leq \epsilon'/n$ where $\epsilon'$ is determined by $\epsilon$ (when accounting for the increase due to the $\alpha$-order). Hence we have that $f \leq e^{3 (2\alpha -1) \epsilon'}$.

By Hoeffding's inequality we can hence conclude $\mathbb{P}(|\mathbb{E}^l f - \mathbb{E}f| \geq c e^{3(\alpha-1)\epsilon}) \leq e^{-\frac{e^{6(\alpha-1)\epsilon}c^2 l}{e^{3(2\alpha - 1) \epsilon'}}}$. Now upper-bounding the right-hand side by $J$ and rearranging to isolate for $l$, we can conclude the stated condition on $l$.

\end{proof}

\subsubsection{Per-Instance R\'enyi DP for General Updates}

The results of Section~\ref{ssec:sum_update} and Section~\ref{ssec:comp} provide a complete per-instance RDP analysis of the current implementation of DP-SGD. In particular, with the per-step update rule being the sum of gradients. In this section we ask, how should we analyze per-step guarantees (and hence DP-SGD given our composition theorem) if the update rule is not the sum? In general, the worst-case sensitivity over mini-batches may be far higher than the expected sensitivity over mini-batches (unlike the sum update rule), meaning the analysis from Theorem~\ref{thm:easy_renyi_dp} may be as bad as a data-independent analysis. For example, the typical update rule used in normal SGD is the mean update rule. However, $\Delta_{U,x^*}(X_B)$ for the mean update rule is the difference between the update for the datapoint $x^*$ and the mean of the updates on $X_B$; this difference is not the same for all minibatches $X_B$ and hence would be overestimated with the analysis of Theorem~\ref{thm:easy_renyi_dp}. One could resolve this issue of overestimating sensitivity by using the $L_p$ norms $||\Delta_{U,x^*}(\mathbf{X_B})||_{p} = (\mathbb{E}_{X_B} (\Delta_{U,x^*}(X_B)^p))^{1/p}$ with $p < \infty$ in the RDP analysis of the sampled Gaussian mechanism, as was done in the $(\epsilon,\delta)$-DP case. %
However, we are not aware of an approach to do this for R\'enyi DP.

\looseness=-1
Instead, we show how a new sensitivity distribution comparing all mini-batches $X_B$ in $X$ to all mini-batches $X'_B$ in $X' = X \cup \{x^*\}$, as opposed to just a single point $x^*$ as done with $\Delta_{U,x^*}(X_B)$, is amenable to a R\'enyi-DP analysis of the sampled Gaussian mechanism that does not look at the maximum privacy leakage over mini-batches. %
If the distribution of all updates given by $X$ is similar to the distribution of all updates given by $X'$, then analysis with this new sensitivity distribution can be expected to beat the current data-independent analysis.

Specifically, given $\alpha$ minibatches sampled from $X$, ${X_{B}}^{\tilde \alpha} \sim X$, and a particular minibatch sampled from $X'$, $X'_B \sim X'$, we define a new sensitivity distribution for $\alpha$-R\'enyi DP as follows: 

\vspace{-5mm}
\begin{multline*}
\Delta_{U,\alpha}({X_{B}}^{\tilde \alpha}, X'_B) \\ \coloneqq \sum_{i} ||U({X_B}^i)||_2^2 - (\alpha-1) ||U(X'_B)||_2^2 - ||\Delta_{\alpha}({X_B}^{\tilde \alpha},X'_B)||_2^2
\end{multline*}

where $\Delta_{\alpha}({X_B}^{\tilde \alpha},X'_B) = (\sum_{i} U({X_B}^i)) - (\alpha - 1) U(X'_B)$. When letting ${X_{B}}^{\tilde \alpha}$ and $X'_B$ be random variables, $\Delta_{U,\alpha}$ effectively compares all the mini-batches in $X'$ to all the $\alpha$-tuples of mini-batches in $X$. The $\alpha$-tuples appear here due to their equivalence with an expectation over mini-batches to the power of $\alpha$ which appears when analyzing $\alpha$-R\'enyi divergences. As described earlier, comparing this to the previous sensitivity distribution $\Delta_{U,x^*}(X_B)$, we see that this new sensitivity will compare all mini-batches in $X$ to all mini-batches in $X'$ (and not just to a point $x^*$) and hence captures more ``global" changes in updates due a datapoint $x^*$.

Theorem~\ref{thm:renyi_dp_sens} states the R\'enyi diveregence of the sampled Gaussian mechanism $M$ between two arbitrary datasets using $\Delta_{U, \alpha}$ through applying a transformation on its fixed $X'_B$ marginal values and taking its expectation over $X'_B$. Taking the maximum of the bounds for $D_{\alpha}(M(X)||M(X'))$ and $D_{\alpha}(M(X')||M(X))$ from Theorem~\ref{thm:renyi_dp_sens} where $X' = X \cup \{x^*\}$ gives a per-instance guarantee of $M$ for $X,X'$.

\begin{theorem}
\label{thm:renyi_dp_sens}
Let $\alpha > 1$ be an integer. Given two arbitrary datasets $X,X'$, the sampled Gaussian mechanism $M$ with noise $\sigma$ satisfies: 

$$D_{\alpha}(M(X')||M(X)) \leq \frac{1}{(\alpha-1)} \mathbb{E}_{X_B} (\ln (\mathbb{E}_{{X'_{B}}^{\tilde \alpha}}(e^{\frac{-1}{2\sigma^2}\Delta_{U,\alpha}({X'_{B}}^{\tilde \alpha}, X_B)})))$$

\end{theorem}

Some key variables in Theorem~\ref{thm:renyi_dp_sens} is the standard deviation of noise $\sigma$ (increasing it decreases the upper-bound) and the sensitivity distribution $\Delta_{U,\alpha}({X_{B}}^{\tilde \alpha}, X'_B)$ (the more concentrated at $0$ it is, the smaller the upper-bound). The proof relies on convexity, which is always true for the second argument of the R\'enyi divergence $D_{\alpha}(A||B)$, and then direct calculations involving Gaussians. %

\begin{proof}

For simpler notation, we use $\mu_X = U(X)$. We proceed by taking $\alpha$ to be an integer (to use an expansion similar to Section 3.3 in \citet{mironov2019r}) and utilizing Theorem 12 in~\citet{van2014renyi}. We will let $N_{X_B} = N(\mu_{X_B},\sigma^2)$ where $\mu_{X_B} = U(X_B)$ as stated earlier.

We proceed to bound $D_{\alpha}(M(X') || M(X))$ for arbitrary $X',X$. Hence a completely analogous argument will allow us to also bound $D_{\alpha}(M(X) || M(X'))$ when $X'$ is specifically $X \cup \{x^*\}$. First note

\vspace{-5mm} \begin{multline}
    D_{\alpha}(M(X') || M(X)) = D_{\alpha}(\sum_{X'_B} \mathbb{P}(X'_B) N_{X'_B} || \sum_{X_B} \mathbb{P}(X_B) N_{X_B}) \\ \leq \sum_{X_B} \mathbb{P}(X_B) D_{\alpha}(\sum_{X'_B} \mathbb{P}(X'_B) N_{X'_B} || N_{X_B})
\end{multline} \vspace{-3mm}

where the last inequality is from the fact the divergence is convex in the second argument (Theorem 12 in~\citet{van2014renyi}). 

Now note 
\vspace{-3mm} \begin{multline}
    e^{(\alpha-1)D_{\alpha}(\sum_{X'_B} \mathbb{P}(X'_B) N_{X'_B} || N_{X_B})} \\ = \int (\sum_{X'_B}\mathbb{P}(X'_B) \frac{1}{(\sigma \sqrt{2\pi})^d} e^{\frac{-1}{2\sigma^2} |x - \mu_{X'_B}|^2})^{\alpha} \\ (\frac{1}{(\sigma \sqrt{2\pi})^d} e^{\frac{-1}{2\sigma^2}|x - \mu_{X_B}|^2})^{1- \alpha} dx \\ = \sum_{{X'_B}^{\tilde \alpha}} \mathbb{P}({X'_B}^{\tilde \alpha}) \frac{1}{(\sigma \sqrt{2\pi})^d} \int e^{\frac{-1}{2\sigma^2} ( (\sum_{{X'_B}^i}|x- \mu_{{X'_B}^i}|^2) - (\alpha - 1)|x - \mu_{X_B}|^2)}
\end{multline} \vspace{-3mm}

where we expanded $(\sum_{X'_B}\mathbb{P}(X'_B) \frac{1}{(\sigma \sqrt{2\pi})^d} e^{\frac{-1}{2\sigma^2} |x - \mu_{X'_B}|^2})^{\alpha}$ by noting each term in the product is just iterating through all $\alpha$ tuples of mini-batches from $X'$.

Note we can for now consider the integral in each dimension, as the overall integral is the product of each dimension. Also recall from the theorem statement that we define $$\Delta_{\alpha}({X'_B}^{\tilde \alpha},X_B) = (\sum_{i} \mu_{{X'_B}^i}) - (\alpha - 1) \mu_{X_B}$$ Hence (letting everything be one dimensional for now) we have

\vspace{-5mm} \begin{multline}
    (\sum_{{X'_B}^i}|x- \mu_{{X'_B}^i}|^2) - (\alpha - 1)|x - \mu_{X_B}|^2 \\ = x^2 - 2 \Delta_{\alpha}({X'_B}^{\tilde \alpha},X_B)x + \sum_{i} \mu_{{X'_B}^i}^2 - (\alpha-1) \mu_{X_B}^2 \\ = (x - \Delta_{\alpha}({X'_B}^{\tilde \alpha},X_B))^2 + \sum_{i} \mu_{{X'_B}^i}^2 - (\alpha-1) \mu_{X_B}^2 - \Delta_{\alpha}({X'_B}^{\tilde \alpha},X_B)^2
\end{multline} \vspace{-3mm}

Hence, we have 

\vspace{-5mm} \begin{multline}
    \int e^{\frac{-1}{2\sigma^2} ( (\sum_{{X'_B}^i}|x- \mu_{{X'_B}^i}|^2) - (\alpha - 1)|x - \mu_{X_B}|^2)} \\ = e^{\frac{-1}{2\sigma^2}(\sum_{i} {\mu_{{X'_B}^i}}^2 - (\alpha-1) \mu_{X_B}^2 - \Delta_{\alpha}({X'_B}^{\tilde \alpha},X_B)^2)} \int e^{\frac{-1}{2\sigma^2} (x - \Delta_{\alpha}({X'_B}^{\tilde \alpha},X_B))^2} \\ = \sigma \sqrt{2 \pi} e^{\frac{-1}{2\sigma^2}(\sum_{i} \mu_{{X'_B}^i}^2 - (\alpha-1) \mu_{X_B}^2 - \Delta_{\alpha}({X'_B}^{\tilde \alpha},X_B)^2)}
\end{multline} \vspace{-3mm}

Note going back to the integral over all dimensions we get $= (\sigma \sqrt{2 \pi})^{d} e^{\frac{-1}{2\sigma^2}(\sum_{i} ||{\mu_{{X'_B}^i}||_2}^2 - (\alpha-1) ||\mu_{X_B}||_2^2 - ||\Delta_{\alpha}({X'_B}^{\tilde \alpha},X_B)||_2^2)}$.

Thus to conclude we get 

\vspace{-5mm} \begin{multline}
    D_{\alpha}(M(X') || M(X))  \leq \sum_{X_B} \mathbb{P}(X_B) D_{\alpha}(\sum_{X'_B} \mathbb{P}(X'_B) N_{X'_B} || N_{X_B}) \\ = \sum_{X_B} \mathbb{P}(X_B) \frac{1}{(\alpha-1)} \\ \ln (\sum_{{X'_B}^{\tilde \alpha}} \mathbb{P}({X'_B}^{\tilde \alpha}) e^{\frac{-1}{2\sigma^2}(\sum_{i} ||\mu_{{X'_B}^i}||_2^2 - (\alpha-1) ||\mu_{X_B}||_2^2 - ||\Delta_{\alpha}({X'_B}^{\tilde \alpha},X_B)||_2^2)})
\end{multline} \vspace{-3mm}

A completely analogous calculation gives the same bound with just $X_B$ replaced with $X_B'$ (and vice-versa) for $D_{\alpha}(M(X)||M(X'))$. Taking the max over both these divergences gives a bound on the per-step per-instance R\'enyi-DP guarantee.

\end{proof}

Hence we now have a per-step RDP analysis for DP-SGD that takes advantage of when expected minibatch sensitivity to $x^*$ is much better than the worst cast minibatch sensitivity. While this phenomenon is not useful for studying the sum update rule (what is currently used for DP-SGD) as every mini-batch has the same sensitivity to $x^*$, in Section~\ref{ssec:exp_hard_renyi} we show this analysis allows us to provide a tighter analysis of the mean update rule. Hence, this opens the possibility of future work deploying DP-SGD with different update rules.

%% file: Sections/Main_Body_Empirical_Results.tex
\section{Empirical Results}
\label{sec:main_body_emp_results}

In Section~\ref{sec:analysis} we provided the first framework to analyze DP-SGD's per-instance privacy guarantees. This followed by providing new per-step analyses (Theorem~\ref{thm:easy_renyi_dp} and~\ref{thm:renyi_dp_sens}), and a new composition theorem that relies on summing ``expected" per-step guarantees (Theorem~\ref{thm:better_composition}). 
We now highlight several conclusions our framework allows us to make about per-instance privacy when using DP-SGD. For conciseness, we defer a subset of the experimental results to Appendix~\ref{sec:detailed_emp_res}. %

\begin{figure*}[!t]

\centering
\subfloat[Training with the datapoint ($X \cup \{x^*\}$) \label{fig:compo_1_more}]
{
\includegraphics[width=0.32\linewidth]{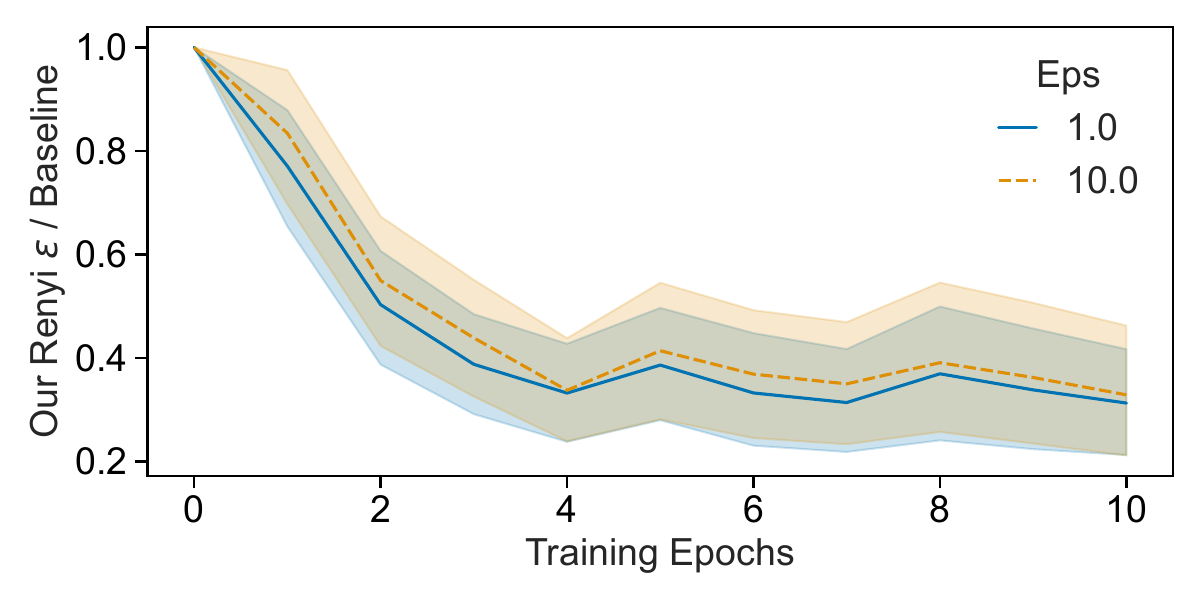}
}
\subfloat[Training with the datapoint ($X \cup \{x^*\}$)\\$\text{~~~~~}$($10^{th}$ percentile) \label{fig:compo_1_more_10per}]
{
\includegraphics[width=0.32\linewidth]{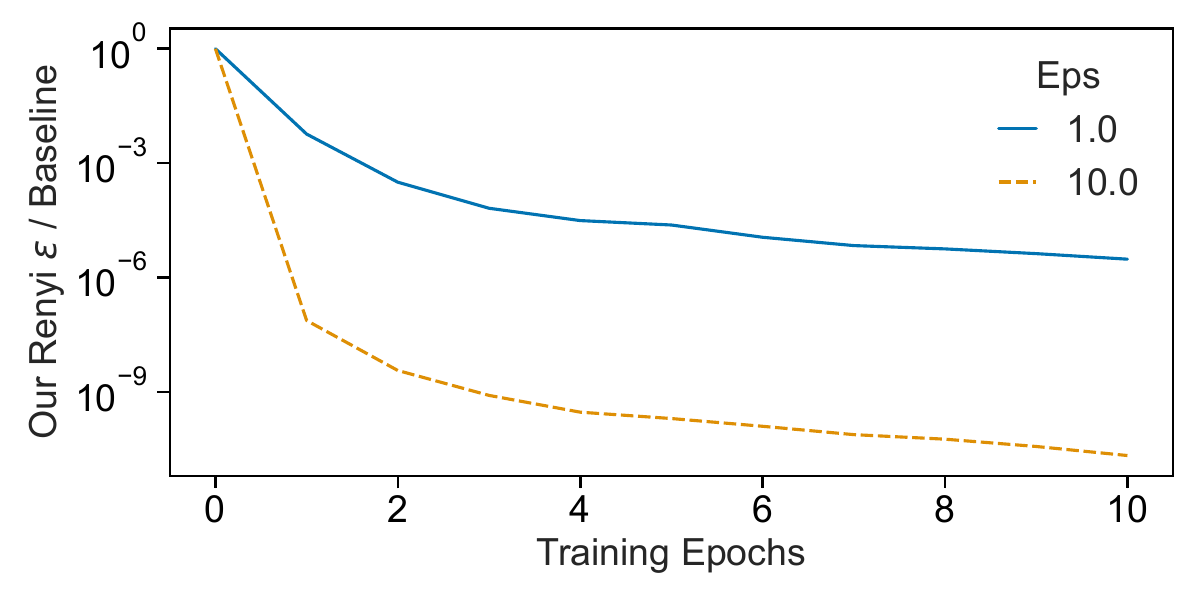}
}
\subfloat[Training with and without the datapont \\ ($X \cup \{x^*\}$ and $X$) \label{fig:compo_1_less}]
{
\includegraphics[width=0.32\linewidth]{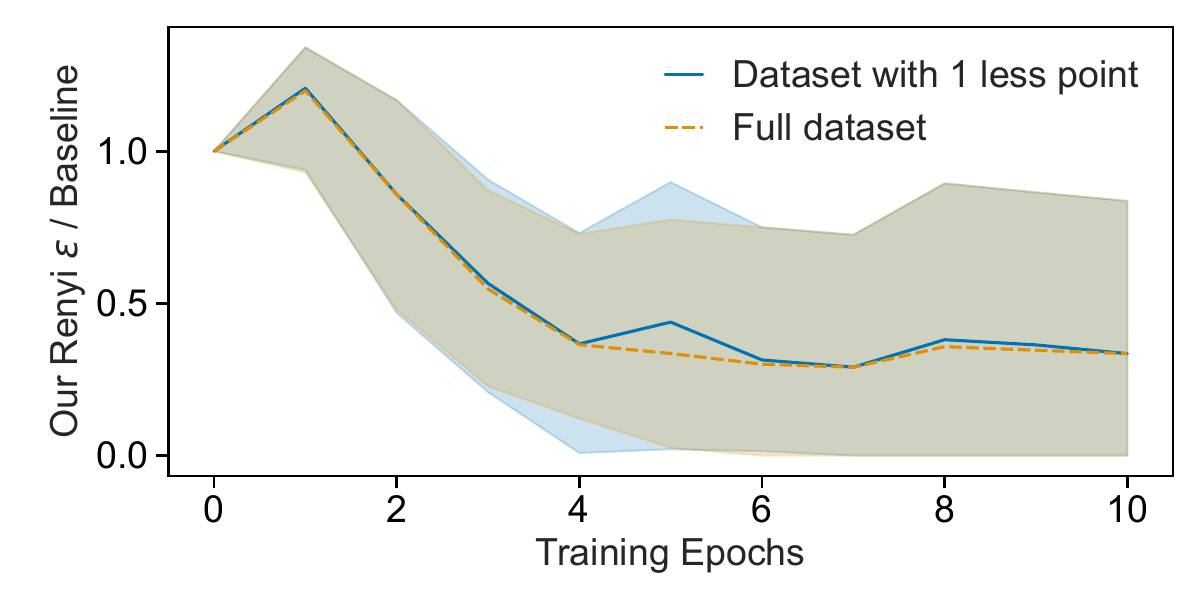}
}
\caption{Per-step privacy contribution from our composition theorem (Theorem~\ref{thm:better_composition}) using the per-step guarantee for the sum update rule (Theorem~\ref{thm:easy_renyi_dp}) as needed for DP-SGD, plotted as a fraction of the baseline data-independent per-step DP-SGD guarantee (Section 3.3 in~\citet{mironov2019r}). The x-axis represents the release of the intermediate models up to a given step in training. The y-axis represents the per-instance privacy leakage for a point given by the release of the model at that training step (relative to the data-independent guarantee); summing all the steps gives the overall privacy leakage of training. The different lines represent changing the Gaussian noise to train with different data-independent $\epsilon,\delta$-DP values. %
The expectations for Theorem~\ref{thm:better_composition} are computed over 10 trials. Figure~\ref{fig:compo_1_more} plots the average relative per-step contribution of 100 random points in MNIST for different strengths of the DP guarantee (i.e., different upper bounds $\varepsilon$) used when training on $X' = X \cup \{x^*\}$. The $10^{th} percentile$ is plotted in Figure~\ref{fig:compo_1_more_10per}. Figure~\ref{fig:compo_1_less} plots expectation over 10 random points in MNIST when training on $X'$ and $X$. We see from both subfigures our per-step contribution more tightly captures the per-instance privacy than the baseline as training progresses: using Theorem~\ref{thm:better_composition} one can conclude that many datapoints have better overall data-dependent privacy guarantees than expected by classical analysis. Note our analysis does worse at the first few steps of training as our composition theorem has a blow up in the order of the R\'enyi divergence for the per-step guarantee for early steps of training; if the sensitivity does not drop quickly enough our composition theorem accounts higher privacy leakage to early steps than the data-independent bounds.
}

\label{fig:composition}
\end{figure*}

\paragraph{Experimental Setup.} In the subsequent experiments, we apply our analysis on MNIST~\citep{lecun1998mnist} and CIFAR-10~\citep{krizhevsky2009learning}. Unless otherwise specified, LeNet-5~\citep{lecun1989backpropagation} and ResNet-20~\citep{he2016deep} were trained on the two datasets for 10 and 200 epochs respectively using DP-SGD, with a mini-batch size equal to 128, $\epsilon=10$, $\delta = 10^{-5}$, $\alpha = 8$ (in cases of R\'enyi DP), and clipping norm $C = 1.0$. All the experiments are repeated 100 times by sampling 100 data points to obtain a distribution/confidence interval if not otherwise stated.
Regarding hardware, we used NVIDIA T4 to accelerate our experiments.

\paragraph{Data Access Assumptions}
We now clarify the data access assumptions needed to run our methods. Theorem~\ref{thm:easy_renyi_dp} for the sum update rule only needs the individual datapoint $x$ and the model, and the composition theorem only additionally needs the checkpoints obtained during training. Hence, as one only needs to compute Theorem~\ref{thm:easy_renyi_dp} and then plug those values in our composition to obtain per-instance guarantees, computing the per-instance DP guarantee for DP-SGD does not require access to the underlying dataset but only the checkpoints and the point $x^*$ in question (applicable for the results in Section~\ref{ssec:exp_better_privacy}). However, Theorem~\ref{thm:renyi_dp_sens} requires sampling minibatches from the datasets, hence our approach to analyze the mean update rule requires further access to the whole dataset (applicable for the results in Section~\ref{ssec:exp_hard_renyi}).

\subsection{Many Datapoints have Better Privacy}
\label{ssec:exp_better_privacy}

Here we describe how our per-instance RDP analysis of DP-SGD, using Theorem~\ref{thm:easy_renyi_dp} for the per-step analysis (with the update rule being the sum of gradients as is typically used) and Theorem~\ref{thm:better_composition} for the composition analysis, allows us to explain why per-instance privacy attacks will fail for many datapoints: many points have better per-instance privacy than the data-independent analysis. We further investigate the distribution of the per-instance privacy guarantees, and which points exhibit better per-instance privacy with our analysis.

\paragraph{Improved Per-Instance Analysis for Most Points} We compare the guarantees given by Theorem~\ref{thm:easy_renyi_dp} for the per-step guarantee in DP-SGD to the guarantee given by the data-independent analysis (see Section 3.3 in \citet{mironov2019r}), and plot per-step contribution coming from our composition theorem. In particular, we take $X$ to be the full MNIST training set, and randomly sample a data point $x^*$ from the test set to create $X' = X \cup x^*$ (as mentioned earlier, we repeat the sampling of $x^*$ 100 times to obtain a confidence interval). We train 10 different models on $X$ with the same initialization and compute the per-step contribution from Theorem~\ref{thm:better_composition} between $X$ and $X'$ (using Theorem~\ref{thm:easy_renyi_dp} to analyze the per-step guarantee from a given model) over the training run, shown in Figure~\ref{fig:compo_1_more}.
We can see that our analysis of the per-step contribution decreases with respect to the baseline as we progress through training. This persists regardless of the expected mini-batch size, the strength of DP used during training, and model architectures; see Figure~\ref{fig:renyi_simple_composition_mnist_sum} in Appendix~\ref{sec:detailed_emp_res}.
By Theorem~\ref{thm:better_composition} we conclude that $D_{\alpha}(Train_{DP-SGD}(X) || Train_{DP-SGD}(X'))$ is significantly less than the baseline for many data points. %

\looseness=-1
To see our improvement over the max of $D_{\alpha}(Train_{DP-SGD}($ $X) || Train_{DP-SGD}(X'))$ and $D_{\alpha}(Train_{DP-SGD}(X') || Train_{DP}$ $_{-SGD}(X))$, i.e., the R\'enyi-DP guarantee, we computed the expectation when training on $X$ and $X'= X \cup \{x^*\}$  for $10$ training points $x^*$ where $X$ is now the training set of MNIST with one point removed and $X'$ is the full training set. Our results are shown in Figure~\ref{fig:compo_1_less} where we see a similar decreasing trend relative to the baseline over training: we conclude by Theorem~\ref{thm:better_composition} that many datapoints have better per-instance R\'enyi DP than the baseline. In other words, we conclude many datapoints have stronger per-instance RDP guarantees than can be demonstrated through the classical data-independent analysis.

\begin{figure}[!t]

\centering
\subfloat[Mini-batch Size = 128 \label{fig:simple_renyi_training_stage}]
{
\includegraphics[width=0.8\linewidth]{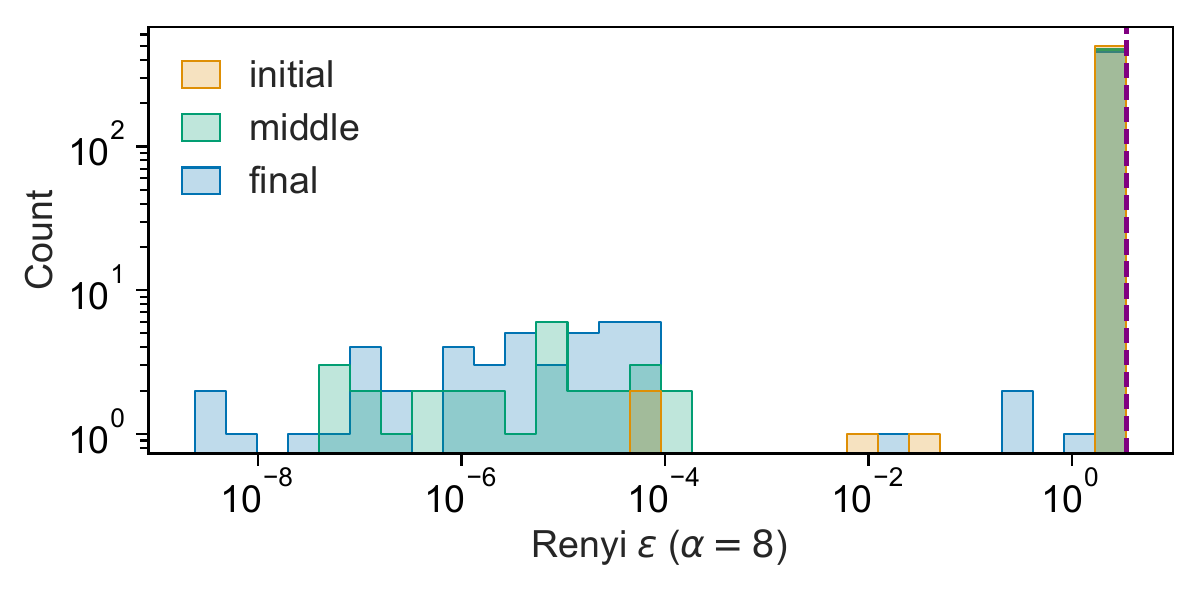}
}

\subfloat[Varying Mini-batch Size \label{fig:simple_renyi_vary_bs}]
{
\includegraphics[width=0.8\linewidth]{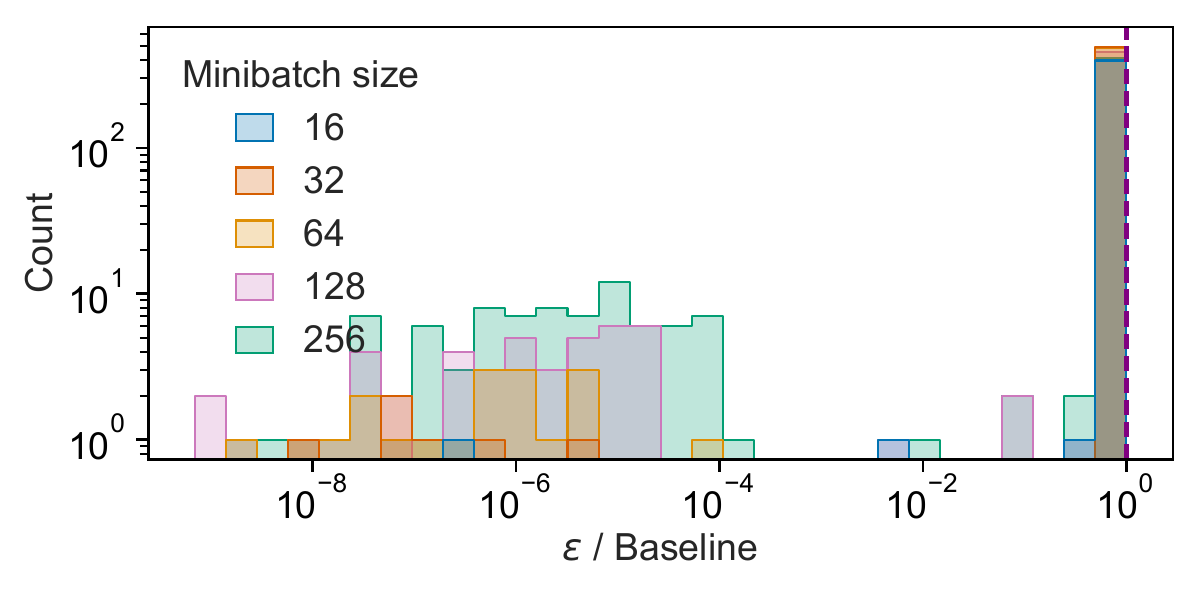}
}

\caption{Distribution plots of our per-step guarantees for the sum update rule given by  Theorem~\ref{thm:easy_renyi_dp} for $500$ datapoints in CIFAR10 with respect to: (a) different stages of training, and (b) varying mini-batch size. The x-axis represents the per-instance guarantee relative to the data-independent guarantee: i.e., the further the mass is to the left, the more our data-dependent guarantees improves upon the data-independent baseline. The purple dashed line 
represents the data-independent baseline. We observe a "long tail" of datapoints with magnitudes better privacy than expected in both plots, illustrated by the log scale on the x-axis.
\vspace{-2mm}
}
\label{fig:simple_renyi}
\end{figure}

\begin{figure}[!t]
\centering
\includegraphics[width=0.8\linewidth]{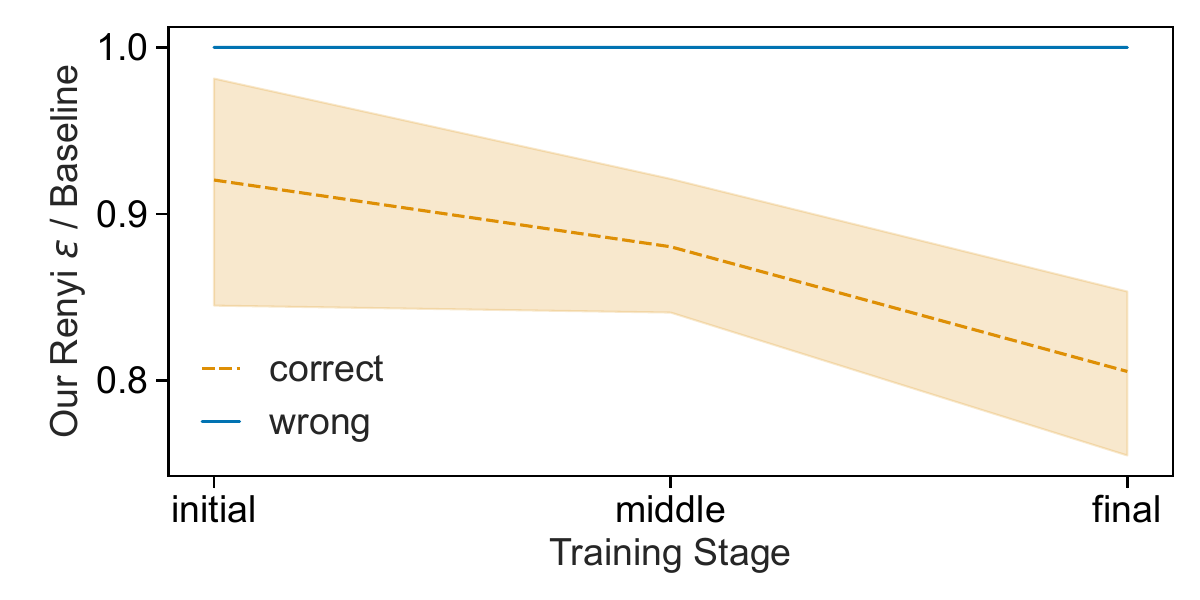}
\caption{Per-step guarantees given by Theorem~\ref{thm:easy_renyi_dp} for $500$ datapoints in CIFAR10 across training stages with respect to correct or incorrect classifications. It can be seen that correctly classified datapoints are on average more private than incorrectly classified ones.\vspace{-5mm}}
\label{fig:correct_incorrect}
\end{figure}

\paragraph{Long-Tail of Better Per-Instance Privacy.} However, the previous figures only show the average effect over datapoints. In Figures~\ref{fig:simple_renyi_training_stage} and~\ref{fig:simple_renyi_vary_bs} we plot the distribution of per-step guarantees over $500$ data points in CIFAR10. The key observations are (1) there exists a long tail of data points with significantly better per-instance privacy than the baseline illustrated by the log-scale in Figures~\ref{fig:simple_renyi_training_stage} and~\ref{fig:simple_renyi_vary_bs},  (2) such improvements mostly exist in the 
later half
of the training process, and (3) such improvements are mostly independent of mini-batch size.

\paragraph{Correct Points are More Per-Instance Private.} Next, we turn to understanding what datapoints are experiencing better privacy when using DP-SGD. In Figure~\ref{fig:correct_incorrect}, we plot the per-step guarantees given by Theorem~\ref{thm:easy_renyi_dp} for correctly and incorrectly classified data points at the beginning, middle, and end of training
on CIFAR10 (and for MNIST in Figure~\ref{fig:renyi_simple_fraction_curve_vary_arch_mnist} in Appendix~\ref{sec:detailed_emp_res}). We see that, on average, correctly classified data points have better per-step privacy guarantees than incorrectly classified data points across training. This disparity holds most strongly towards the end of training.

\subsection{Higher Sampling Rates can give Better Privacy}
\label{ssec:exp_hard_renyi}

We now highlight how our analysis, if it uses Theorem~\ref{thm:renyi_dp_sens} for the per-step analysis, allows us to better analyze DP-SGD with other update rules (not the sum of gradients which is what the current implementation of DP-SGD uses and Section~\ref{ssec:exp_better_privacy} analyzed). In particular, we will analyze the mean update rule and show how it has a privacy trade-off with sampling rate that is opposite to the trade-off for the sum update rule.

\looseness=-1
In normal SGD (with gradient clipping), one computes a mean for the per-step update 
$U(X_B) = \frac{1}{|X_B|}$ $\sum_{x \in X_B} \nabla_{\theta}\mathcal{L}(\theta,x)/ \max(1,\frac{||\nabla_{\theta}\mathcal{L}(\theta,x)||_2}{C})$. 
However, DP-SGD computes a weighted sum $U(X_B) = \frac{1}{L} \sum_{x \in X_B} \nabla_{\theta}\mathcal{L}(\theta,x)/ $ $\max(1,\frac{||\nabla_{\theta}\mathcal{L}(\theta,x)||_2}{C})$. Note the subtle difference between dividing by a fixed constant $L$ (typically the expected mini-batch size when Poisson sampling datapoints) and by the mini-batch size $|X_B|$. This means for the sum the upper-bound on sensitivity is $\frac{C}{L}$, while for the mean the upper-bound on sensitivity is only $C$ (consider neighbouring mini-batches of size 1 and 2). Hence using the mean update rule requires far more noise and so is not practical to use. We highlight how our per-instance analysis by sensitivity distributions provides better guarantees for the mean update rule.

\begin{figure}[!t]
\centering
\subfloat[$D_{\alpha}(M(X')||M(X))$ $\text{~~~~~}$Mini-batch Size = $128$ \label{fig:hard_renyi_training_stage}]
{
\includegraphics[width=0.8\linewidth]{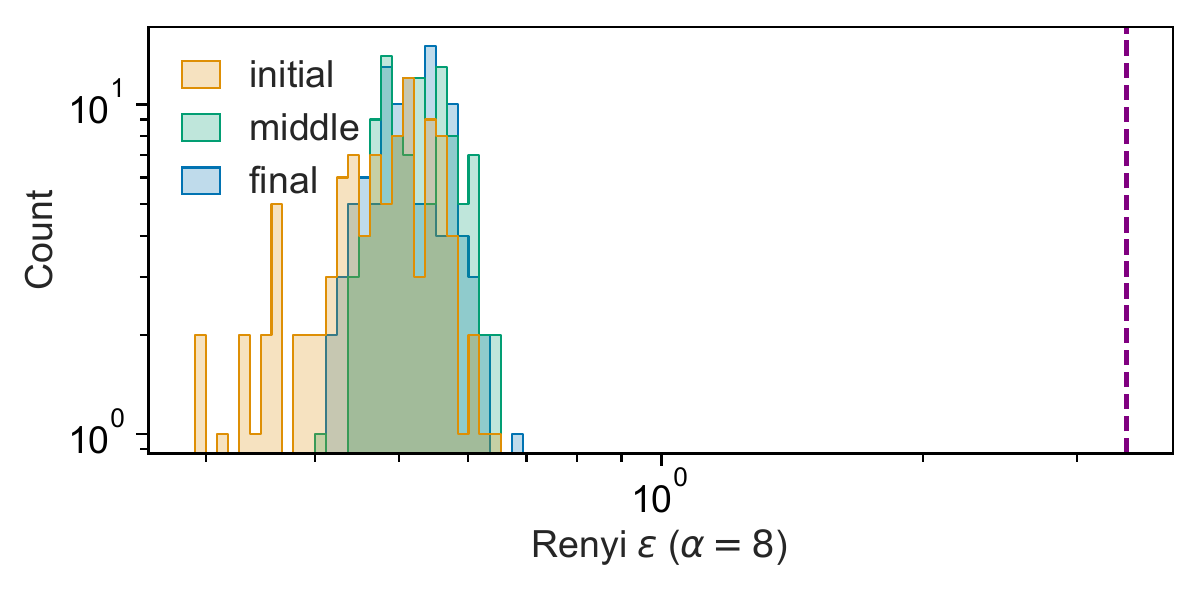}
}
\\\subfloat[$D_{\alpha}(M(X)||M(X'))$ $\text{~~~~~}$Mini-batch Size = $128$
\label{fig:hard_renyi_reverse}]
{
\includegraphics[width=0.8\linewidth]{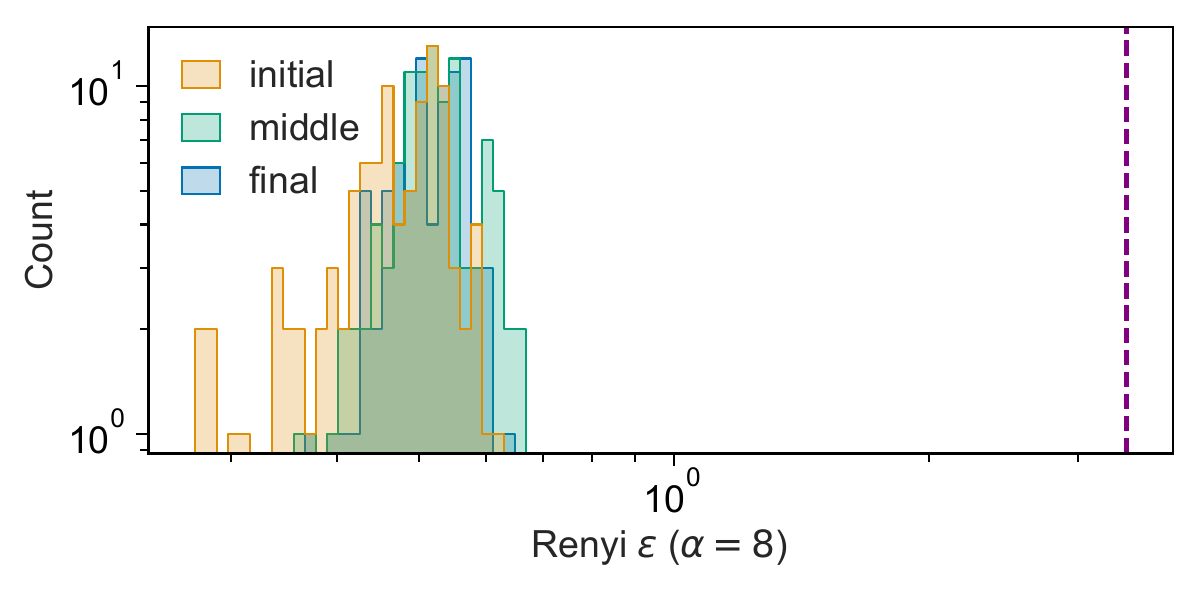}
}
\\\subfloat[$D_{\alpha}(M(X')||M(X))$ $\text{~~~~~}$Varying Mini-batch Size \label{fig:hard_renyi_vary_bs}]
{
\includegraphics[width=0.8\linewidth]{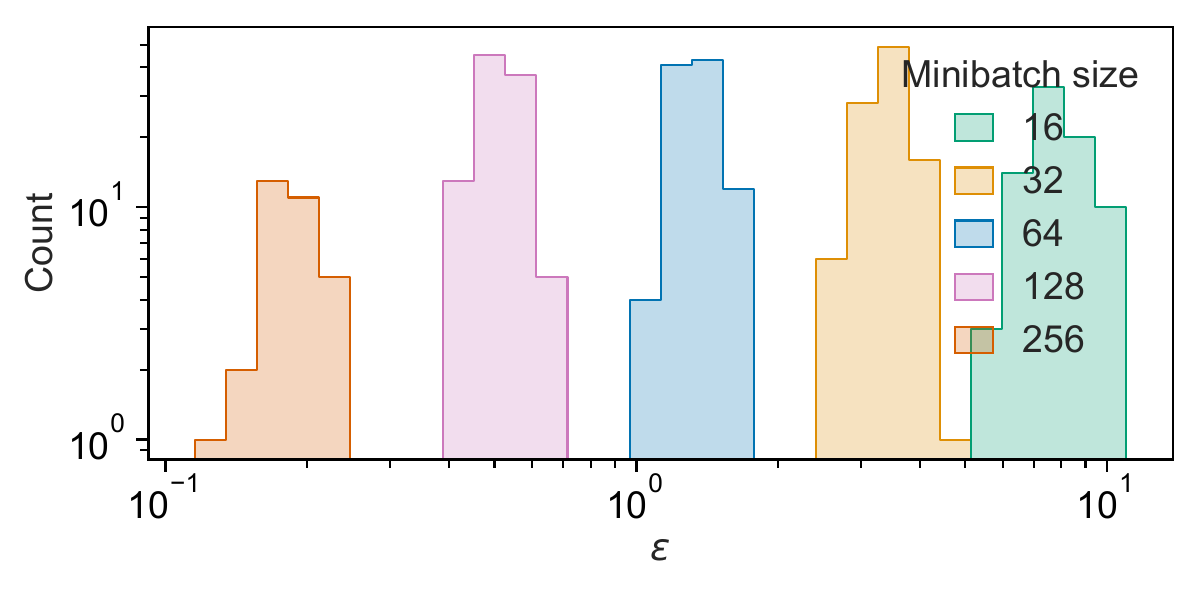}
}

\caption{ Distribution plots (log scale) of our per-step guarantees for the mean update rule (Theorem~\ref{thm:renyi_dp_sens}) for $500$ datapoints in CIFAR10 with respect to different training stages and mini-batch sizes. Bounds on both $D_{\alpha}(M(X)||M(X'))$ and $D_{\alpha}(M(X')||M(X))$ are shown (the maximum of both is the per-instance R\'enyi-DP guarantee) for an expected mini-batch size of 128.  From Figures~\ref{fig:hard_renyi_training_stage},\ref{fig:hard_renyi_reverse}, we conclude our per-step guarantees for the mean update rule (Theorem~\ref{thm:renyi_dp_sens}) gives better data-dependent guarantees for the mean update rule than classical analysis, and from Figure~\ref{fig:hard_renyi_vary_bs} that increasing the expected mini-batch size decreases our bound for this update rule (counter-intuitive to privacy amplification by subsampling).
\vspace{-5mm}
}
\label{fig:hard_renyi}
\end{figure}

\paragraph{Better Analysis of the Mean Update Rule. } Letting $M$ now be the sampled Gaussian mechanism with the mean update rule, we compute the bound on $D_{\alpha}(M(X')||M(X))$ and $D_{\alpha}(M(X)||M(X'))$ given by Theorem~\ref{thm:renyi_dp_sens}, where we estimated the inner and outer expectation using $20$ samples, i.e., $20$ random $X_B'^{\alpha}$ (or $X_B^{\alpha}$) for each of the $20$ random $X_B$ (or $X_B'$). We obtain Figure~\ref{fig:hard_renyi_training_stage} and~\ref{fig:hard_renyi_vary_bs} by repeating this for $500$ data points in CIFAR10 while varying the training stage. We observe that for both divergences, we beat the baseline analysis by more than a magnitude at the middle and end of training. We conclude Theorem~\ref{thm:renyi_dp_sens} gives us better per-instance R\'enyi DP guarantees for the mean update rule.

\paragraph{Per-Instance Privacy Improves with Higher Sampling Rate.} Furthermore, counter-intuitively to typical subsample privacy amplification, in Figure~\ref{fig:hard_renyi_vary_bs} we see that our bound decreases with increasing expected mini-batch size: 
we attribute this to the law of large numbers, whereby increasing the expected mini-batch size leads to sampled mini-batches having similar updates more often and hence the sensitivity distribution concentrates at smaller values. An analogous result is shown for MNIST in Figure~\ref{fig:renyi_hard_eps_distrib_bs_mnist_mean} (in Appendix~\ref{sec:detailed_emp_res}).

%% file: Sections/Discussion.tex
\section{Discussion}

Here we first discuss past work on composition theorems (Section~\ref{ssec:back_full_comp}) and the current computational trade-offs of our analysis which future work could improve (Section~\ref{ssec:comp_lims}). We then discuss some theoretical questions based on observations from our analysis (Section~\ref{ssec:theoretical_questions}). Lastly, we describe several applications of our analysis (Section~\ref{ssec:apps}).

\subsection{Fully Adaptive Composition Theorems}
\label{ssec:back_full_comp}

One of the main technical contributions of this paper is generalizing the normal R\'enyi DP composition theorem (Proposition 1 in \citet{mironov2017renyi}), which sums worst-case per-step guarantees, to allow for better per-instance analysis. Other work have also generalized the composition theorem to have better per-instance analysis~\citep{feldman2021individual, koskela2022individual}, and called these new theorems Fully Adaptive Composition. For R\'enyi DP, \citet{feldman2021individual} showed that composition can be done by considering the worst-case sum of the per-step guarantees from a DP-SGD training run (Theorem~3.1 in \citet{feldman2021individual}), as opposed to summing the worst-case guarantee at each step. \citet{koskela2022individual} state an analogous composition theorem for Gaussian DP. However, for DP-SGD, the degree of improvement provided by the worse-case sum compared to the normal composition is not clear. It could be that the worst-case sum is equal to the sum of the worst-case per-step guarantees, which is the case if the training run goes to worst-case states with non-zero probability at each step. Furthermore, it is hard to measure the worst-case sum to show this is not the case. Specifically, it requires measuring the $L_{\infty}$ norm of a distribution, which without further assumptions beyond boundedness is much harder than the p-norms needed for our method. In short, we are the first to provide a per-instance composition theorem that can be used to determine better per-instance guarantees for DP-SGD.

\subsection{Computational Limitations and Future Improvements}
\label{ssec:comp_lims}

As explained in Section~\ref{ssec:comp}, there is a tension between preventing blow-up in our composition theorem (Theorem~\ref{thm:better_composition}) and estimating the per-step contributions with few samples: both require manipulating a parameter $p$, with the former requiring large $p$ and the latter requiring small $p$. We showed in Section~\ref{ssec:comp} how the value of $p$ we chose for our experiments strikes a balance where we can limit blow-up while still estimating the per-step contribution to the composition with few samples. This balance is further backed up by the confidence intervals for our estimates of the per-step contributions (see Figure~\ref{fig:confid_10points}). Nevertheless, we still require several training runs to compute the per-instance guarantee for a specific point.

\begin{figure}
    \centering
    \includegraphics[width=0.8\linewidth]{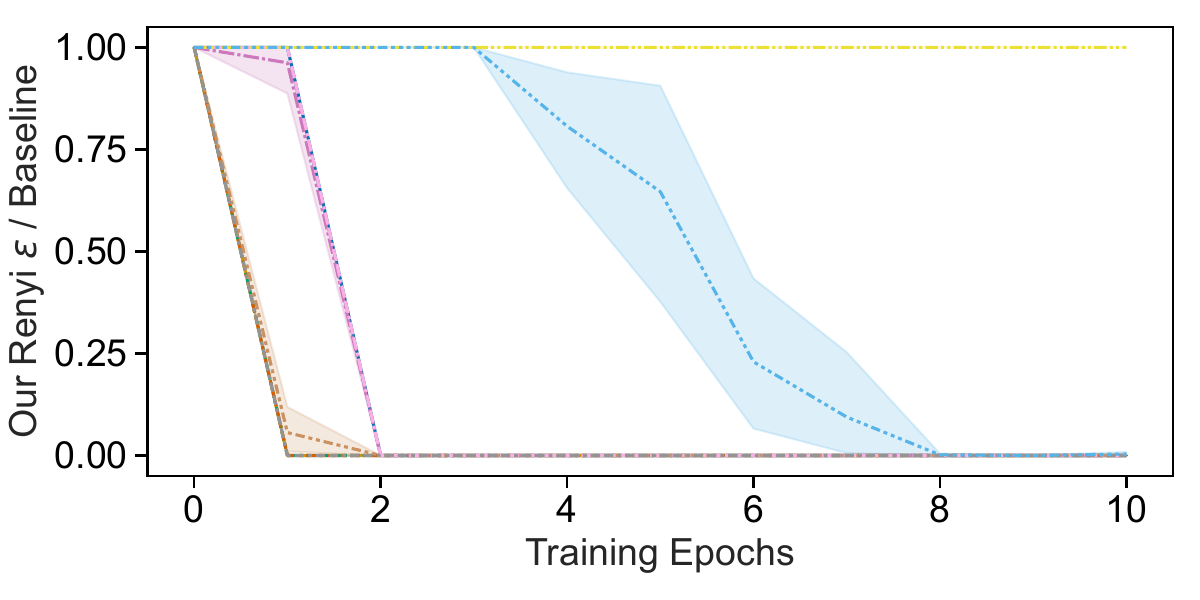}
    \caption{The expected reweighted per-step contributions which are summed for our composition theorem (Theorem~\ref{thm:better_composition}) using Theorem~\ref{thm:easy_renyi_dp} for the unweighted per-step guarantee for 10 different points in MNIST. The guarantees are computed once each epoch when training with the datapoint (i.e., on $X' = X \cup \{x^*\}$). The shaded region is the $95\%$ confidence interval over 10 trials. As seen by the confidence intervals having a width a small fraction of the baseline value, with just 10 trials we are very confident in the estimates of the per-step contributions for most points.
    \vspace{-5mm}
    }
    \label{fig:confid_10points}
\end{figure}

Future work may be able to improve (analytically) the trade-off between blow-up and sample complexity for Theorem~\ref{thm:better_composition}, and hence make it cheaper to compute the per-instance DP guarantees. Future work may also be able to derive composition theorems analogous to Theorem~\ref{thm:better_composition} which are easier to estimate. Similarly, Theorem~\ref{thm:renyi_dp_sens} is computationally expensive to compute; we require computing several mini-batch updates at every step to estimate the expectations. Future work may be able to derive alternative per-step guarantees for general update rules that are easier to empirically estimate.

\subsection{Theoretical Questions}
\label{ssec:theoretical_questions}

In this paper we applied our analysis of DP-SGD to deep learning. In particular, our analysis led to data-dependent guarantees but stopped short of data-independent guarantees; the technical problem for deriving data-independent guarantees is bounding the "expected per-step guarantee" for all datasets. However, one can still apply our analysis of DP-SGD to classical machine learning. For example, theoretical work has shown improved data-independent privacy guarantees for convex losses \cite{altschuler2022privacy}; the proof relied on the update rule being a contraction for convex losses. In contrast to this approach, our analysis can directly translate smaller gradient norms (in expectation) during training to better privacy guarantees. Hence, we believe an interesting future direction is applying our analysis to learning settings that permit direct calculations of sensitivity distributions and hence potentially derive data-independent guarantees. 

\looseness=-1
In this direction, a particular phenomenon we wish to highlight is that by training with less noise we sometimes see a disproportional decrease in per-instance privacy consumption. This is shown in Figure~\ref{fig:compo_1_more_10per} where training with the noise level for $\epsilon=10$ (data-independent bound) resulted in disproportionately smaller per-step privacy consumption than training with the higher noise level for $\epsilon=1$ (data-independent bound) for the 10th percentile of points with the smallest sensitivity. Stated another way, our analysis shows that for some datapoints more noise is not always better for privacy. In settings where sensitivity distributions can be explicitly calculated, one may be able to compute what this noise vs. per-instance privacy trade-off is.

\subsection{Applications}
\label{ssec:apps}

In this paper, we focused on explaining how privacy attacks will fail for many datapoints if the adversary only observes typical datasets common to deep learning. This was done by providing a new per-instance DP analysis for DP-SGD. We now highlight other applications of our analysis.

\paragraph{Estimating Privacy} A growing theme in private deep learning is empirically ``estimating" privacy in different data settings, and is broadly encapsulated by privacy auditing~\cite{nasr2021adversary,Nasr:2023, Zanella-Beguelin:2022, Jagielski:2020}. However, here estimating means obtaining lower bounds on privacy leakage (e.g., the parameter $\epsilon$ used in differential privacy). Our work presents a shift in how we can go about estimating privacy. Our analysis provides per-instance \emph{upper-bounds} that complement past lower-bounds, and when direct calculations of expected sensitivity distributions are not possible, one can still estimate our upper-bounds by repeating training. Future work can hence use our analysis to provide potentially matching empirical upper-bounds to previous empirical lower-bounds obtained with specific privacy attacks, and hence be able to conclude that these privacy attacks are optimal in more settings than just the worst-case.

However, it is possible that the result from auditing can be used in a way that increases the data-independent privacy leakage. Hence, to be more specific, we emphasize two use-cases of per-instance guarantees for auditing and when they retain data-independent differential privacy.

\noindent \emph{Internal Audit: }The first is an internal audit, formally:

\begin{enumerate}[leftmargin=*,noitemsep]
    \item Input dataset D, auditing dataset D*, training algorithm $T$
    \item  Compute auditing statistics $S(D^*, T)$ 
    \item Release $T(D)$
\end{enumerate}

If $T$ is a DP algorithm, we have the outputs of this protocol is private in the training dataset $D$. In our case, T is DP-SGD and our method provides tools to estimate $S$ when $S$ are per-instance guarantees for $x \in D^*$ (that are better than the data-independent guarantee).

\noindent \emph{Audit to Modify Training: } However, if the audit affects the public release, then we can leak privacy. This leads to the second use-case, stated formally as:

\begin{enumerate}[leftmargin=*,noitemsep]
    \item Input D, reference algorithm $T$, final algorithm $T'$
    \item Compute $S(D, T)$
    \item Release $T’(D,S)$ 
\end{enumerate}

An example of a possible $T'$ is computing our per-instance guarantees as $S$, and then dropping all the least private points in the training set (in the hope of having better privacy guarantees and hence better utility vs. privacy); this is broadly captured by privacy filtering where one defines $T’(D,S) = T(D(S))$ where $D(S) \subset D$. However, note $T'$ will now depend on the training dataset $D$ through $S$. In this case, to preserve data-independent privacy, one needs to bound the sensitivity of $T’$ to $S$ and the sensitivity of $S$ to $D$. How $T’$ differs from $T$ is not specified, giving a somewhat ill-posed problem. However, a specific privacy filtering algorithm was studied by \citet{feldman2021individual} which relied on using per-step per-instance guarantees (not composition of steps). While in this paper we do not provide a filtering algorithm, we believe our composition theorem provides a new tool that future work may use in designing privacy filtering algorithms. However, we remark that at least for the filtering algorithm of \citet{feldman2021individual}, our composition theorem can be worse than their composition theorem; their algorithm enforces the almost every condition for their composition theorem which is then tighter than our composition theorem as it does not have divergence order blow-up for early steps (which ours does). This is in contrast to DP-SGD (see Section~\ref{ssec:back_full_comp}), highlighting the interdependence between algorithm and composition theorem choices.

\paragraph{Estimating Memorization} Related to estimating privacy, a broad literature is concerned with measuring memorization~\cite{carlini2019secret,zhang2019identity,feldman2020neural,carlini2022quantifying,tirumala2022memorization}. The methodology for estimating memorization varies, but includes privacy attacks~\cite{carlini2022quantifying}, or approximations of influence~\cite{carlini2019secret,feldman2020neural}. Our work provides, to the best of our knowledge, the first approach to estimating memorization via upper bounds. Hence, our work may provide a complementary tool to past work on memorization.

\paragraph{Estimating Unlearning} Unlearning a datapoint $x^*$ is to obtain the model (distribution) coming from training on the dataset $D\setminus x^*$ given a model trained on the dataset $D$ \cite{cao2015towards}. The only known methods to do this exactly for deep learning are variations of naively retraining on $D \setminus x^*$ \cite{bourtoule2021machine}. Given the general intractability of exact unlearning, significant work has looked into \emph{approximate unlearning}; approximate unlearning is to obtain the same model (distribution) as training with $D \setminus x^*$ up to some error in a predefined metric. A popular measure of approximate unlearning has been using per-instance DP guarantees \cite{guo2019certified}, or only one of the per-instance DP inequalities~\cite{gupta2021adaptive}\cite{ginart2019making} (which is implied by the former). However, the only known methods (to the best of our knowledge) to achieve this kind of guarantee for deep learning is to train with DP-SGD and use the data-independent DP bound as the unlearning guarantee. Our analysis allows for unlearning guarantees that are specific to individual points. While DP-SGD does not explicitly target specific points to have better unlearning guarantees, future work may be able to use our analysis to derive a modified version of DP-SGD that explicitly unlearns a subpopulation of the training set (hence future deletion requests for that subpopulation are already handled).

Another influential notion of unlearning is adaptive machine unlearning~\cite{gupta2021adaptive}, which requires unlearning to also be private to the sequence of update requests; that is, regardless of the order in which people request for unlearning, one finally returns a model close to retraining without any of their data. Per-instance guarantees naturally leak information about the dataset, and so one might wonder if it is still possible to satisfy adaptive machine unlearning when using per-instance DP guarantees to unlearn. We now illustrate a specific unlearning algorithm using per-instance guarantees to not retrain when possible, which satisfies adaptive machine unlearning. However, this does not immediately give a more efficient unlearning algorithm than retraining. This is because checking the per-instance guarantees for each unlearning request is expensive with our current method. However, we hope this serves as motivation for future work on efficiently computing per-instance guarantees.

We now state a R\'enyi divergence version of adaptive machine unlearning; R\'enyi divergence implies the $(\epsilon,\delta)$-DP 
inequalities used in the original definition~\cite{gupta2021adaptive}, while also being consistent with our paper. Note, we use $u^i$ to denote the $i^{th}$ unlearning request in a sequence (only deletion), $D^{i} = D \setminus \{\cup_{j =1}^{i} u^j\}$, and $s$ for other hyperparameters.

\begin{definition}[R\'enyi $\alpha, \beta, \gamma$- adaptive unlearning]

We say $R_{A}$ is a R\'enyi $\alpha, \beta, \gamma$- adaptive unlearning algorithm for $A$ if  for all update request function $UpdReq$, initial datasets $D$, $t \geq 1$, and with probability $1- \gamma$ over the draw of unlearning requests $u^{1},\cdots,u^{t}$ from $UpdReq$ we have $D_{\alpha}(R_{A}(D^{t-1}, u^{t}, s^{t-1}) || A(D^{t})) \leq \beta$
\end{definition}

\looseness=-1
Our adaptive unlearning algorithm using per-instance guarantees, which we will call Naive Per-Instance Unlearning (NPIU), is defined recursively over the sequence of unlearning requests. Intuitively, it checks if the current distribution of models is far away from the retraining distribution (via a triangle inequality), and if so it retrains from scratch. Formally: 
\begin{enumerate}[leftmargin=*,noitemsep]
    \item If  $\frac{\alpha-1}{\alpha-2} D_{2\alpha}(R_{A}(D^{t-2},u^{t-1},s^{t-2}) || A(D^{t-1})) + D_{2\alpha-1}(A(D^{t-1}) || A(D^{t})) \leq \beta$ then $R_{A}(D^{t-1},u^{t},s^{t-1}) = R_{A}(D^{t-2},u^{t-1},s^{t-2})$, i.e., keep same output as before
    \item Else $R_{A}(D^{t-1},u^{t},s^{t-1}) = A(D^{t})$, i.e., retrain from scratch
\end{enumerate}

\begin{fact}
\label{fact:naive_unl_guarantee}
    NPIU satisfies $(\alpha,\beta, 0)$-adaptive unlearning.
\end{fact}

\begin{proof}

Consider any $t$, $D$, and update requests  $u_1,\cdots,u_t$ (which defines the sequence of datasets). If the first condition is satisfied, then by Corollary 4 in [33] (the weak triangle inequality for R\'enyi Divergences) we have $D_{\alpha}(R_{A}(D^{t-2},u^{t-1},s^{t-2}) || A(D^{t})) \leq \frac{\alpha-1/2}{\alpha-1} D_{2\alpha}(R_{A}(D^{t-2},u^{t-1},s^{t-2}) || A(D^{t-1})) + D_{2\alpha-1}(A(D^{t-1}) || A(D^{t})) \leq \beta$ hence $D_{\alpha}(R_{A}(D^{t-1},u^{t},s^{t-1}) || A(D^{t}) \leq \beta$ and so the unlearning criteria is satisfied at step $t$.

If not we have $R_{A}(D^{t-1},u^{t},s^{t-1}) = A(D^{t})$ and so $D_{\alpha}(R_{A}(D^{t-1},u^{t},s^{t-1}) || A(D^t) ) = 0 \leq \beta$ meaning the unlearning criteria is once again satisfied at step $t$. Hence as we proved the condition holds for arbitrary $t, D$ and sequence of update requests, the algorithm is $(\alpha,\beta,0)$ adaptive machine unlearning. This completes the proof.
    
\end{proof}

Interpreting the algorithm in more detail, we have it checks if the per-instance guarantee $D_{2\alpha-1}(A(D^{t-1}) || A(D^{t}))$ is small enough while recursively using information on what the per-instance guarantee was already with respect to $D^{t-1}$; if the sum of the guarantees is over the budget, it retrains from scratch. In particular, on accounting for the budget already used, note for each unlearning request one needs a divergence of two times the original order to implement the algorithm (see the first condition), which could be done by recursively applying Corollary 4 in [33] (as done in the proof) until one hits only divergences between $A(D^{i-1})$ and $A(D^{i})$ for some $i$, where $i$ is at least as large as the last time one hit the else condition and had to rerun $A$. Our results do not give efficient methods to measure these divergences (though computing different orders can reuse cached data such as checkpoints and gradients), but we hope unlearning motivates further work on efficiently computing per-instance guarantees.

\paragraph{An Alternative Framework for Forgeability} However, underlying machine unlearning (as a legal requirement, e.g., the EU GDPR~\cite{mantelero2013eu}) is the problem of whether an auditor can ever claim an entity did not unlearn a point. That is, can a model trainer claim to have trained without a point even if they in fact did? Forgeability~\cite{thudi2022necessity} is a framework under which a model trainer can claim to have obtained their model by training on a dataset they did not in fact train on. To make a claim of training on a given dataset, forgeability relies on providing a valid Proof-of-Learning (PoL)~\cite{jia2021proof} that uses the claimed dataset (different from what the model trainer originally used). Recall PoL is a sequence of checkpoints and minibatches for which the update from $i'th$ minibatch given the $i'th$ checkpoint leads to the $i+1'th$ checkpoint upto some error $\delta$ in a metric $d$. However, it is currently not known how to properly pick the threshold $\delta$ and metric $d$ to define a ``valid" update for a PoL (due to a lack of models for the backend noise during training), or how to make PoL efficient to verify without introducing additional security risks~\cite{fang2023proof}. Hence the current framework for forgeability may not be robust until PoL is better understood.

As an alternative to using PoL, a per-instance DP guarantee tells us that it is very likely we would have obtained the same sequence of checkpoints with either of the two datasets. Hence, when an auditor claims a model trainer trained on a point (and the point has strong per-instance DP guarantees), the trainer can refute by submitting a dataset without the point and their original sequence of checkpoints and the details of their DP training implementation. Given this, an 
auditor that only has the information provided in the submitted proof can no longer distinguish between whether a trainer had or had not trained on the point.

%% file: Sections/Conclusion.tex
\section{Conclusion}
\looseness=-1
Our work can be viewed as the first existential result showing better per-instance DP guarantees for deep learning when using DP-SGD. In doing so, we provided one resolution to an open problem in the field of privacy attacks against deep learning: why many privacy attacks fail in practice. However, further work is needed to convert our analysis into a fast algorithm to do privacy accounting. Our composition theorem requires computing several training runs, and the per-step analysis of Theorem~\ref{thm:renyi_dp_sens} (which allows better analysis of the mean update rule) requires computing updates for many mini-batches at each step. Future work may be able to significantly reduce the cost associated with using these theorems, or propose alternative theorems that are more efficient to implement. Future work can also likely design algorithms that explicitly take advantage of sensitivity distributions, which we showed are implicit in DP-SGD in explaining its better per-instance privacy guarantees.

%% file: Sections/acknowledgements.tex
\section*{Acknowledgements}

We would like to acknowledge our sponsors, who support our research with financial and in-kind contributions: Amazon, Apple, CIFAR through the Canada CIFAR AI Chair, DARPA through the GARD project, Intel, Meta, NSERC through the COHESA Strategic Alliance and a Discovery Grant, Ontario through an Early Researcher Award, and the Sloan Foundation. Anvith Thudi is supported by a Vanier Fellowship from the Natural Sciences and Engineering Research Council of Canada. Resources used in preparing this research were provided, in part, by the Province of Ontario, the Government of Canada through CIFAR, and companies sponsoring the Vector Institute. We would further like to thank Relu Patrascu at the University of Toronto for providing the compute infrastructure needed to perform the experimentation outlined in this work. We would also like to thank members of the CleverHans lab, Mahdi Haghifam, and our shepherd for their feedback on drafts of the manuscript.

%% file: Sections/Appendix.tex
\section{Proofs}

\subsection{Proof of Lemma~\ref{lem:holder_approach}}
\label{proof:holder_approach}

\begin{lemma}
\label{lem:holder_approach}
With the above notation, and $X' = X \cup \{x^*\}$, we have the per-instance inequality

\vspace{-7mm} \begin{multline}
    \mathbb{P}(M(X') \in S) \\ \leq \mathbb{P}_{x^*}(1) \mathbb{E}_{X_B}(e^{C_{\delta,\sigma} \Delta_{U,x^*}(X_B)p})^{1/p} \mathbb{P}(M(X) \in S)^{1-1/p} \\ + \mathbb{P}_{x^*}(1)\delta + \mathbb{P}_{x^*}(0) \mathbb{P}(M(X) \in S)
\end{multline} \vspace{-5mm}

\end{lemma}

\begin{proof}
    
First note sampling mini-batches from $X'$ is equivalent to sampling a mini-batch $X_B$ from $X$, then sampling $x^*$ with probability $\mathbb{P}_{x^*}(1)$. Hence we have  

 \vspace{-7mm} \begin{multline}
    \mathbb{P}(M(X') \in S) = \sum_{x_B} (\mathbb{P}_{x^*}(1) \mathbb{P}(A(X_B \cup x^*) \in S) \\ + \mathbb{P}_{x^*}(0) \mathbb{P}(A(X_B) \in S))\mathbb{P}(X_B)
\end{multline} \vspace{-5mm}

Now note we have $\mathbb{P}(A(X_B \cup x^*) \in S) \leq e^{C_{\delta,\sigma}\Delta_{U,x^*}(X_B)} \mathbb{P}(A(X_B) \in S) + \delta$ by the $(\epsilon,\delta)$-DP guarantee of the Gaussian mechanism. So considering summing that over $X_B$ we have $\sum_{X_B} \mathbb{P}(A(X_B \cup x^*) \in S) \mathbb{P}(X_B) \leq \sum_{X_B} e^{C_{\delta,\sigma} \Delta_{U,x^*}(X_B)} \mathbb{P}(A(X_B) \in S) \mathbb{P}(X_B) + \delta$. Now we apply Holder's inequality to get $\sum_{X_B} e^{C_{\delta,\sigma} \Delta_{U,x^*}(X_B)} \mathbb{P}(A(X_B) \in S) \mathbb{P}(X_B) \leq \mathbb{E}_{X_B}((e^{C_{\delta,\sigma} \Delta_{U,x^*}(X_B)})^{p})^{1/p} \mathbb{E}_{x_B}(\mathbb{P}(A(x_B) \in S)^{q})^{1/q}$. Note that $\mathbb{P}(A(x_B) \in S)^{q} \leq \mathbb{P}(A(x_B) \in S)$ for $q \geq 1$ as $\mathbb{P}(A(x_B) \in S) \leq 1$.

So we have 

\vspace{-7mm} \begin{multline}
\sum_{X_B} \mathbb{P}(A(X_B \cup x^*) \in S) \\ \leq \mathbb{E}_{x_B}((e^{C_{\delta,\sigma} \Delta_{U,x^*}(X_B)})^{p})^{1/p} \mathbb{E}_{x_B}(\mathbb{P}(A(x_B) \in S))^{1/q} + \delta
\end{multline} \vspace{-5mm}

So to conclude we have

\vspace{-7mm} \begin{multline}
    \mathbb{P}(M(X') \in S) \mathbb{P}(X_B) \\ 
    \leq \mathbb{P}_{x^*}(1) \mathbb{E}_{x_B}((e^{C_{\delta,\sigma} \Delta_{U,x^*}(X_B)})^{p})^{1/p} \mathbb{E}_{x_B}(\mathbb{P}(A(x_B) \in S))^{1-1/p} \\ + \delta + \mathbb{P}_{x^*}(0) \mathbb{E}_{X_B}(\mathbb{P}(A(X_B) \in S))
\end{multline} \vspace{-5mm}

Note $\mathbb{E}_{X_B}(\mathbb{P}(A(X_B) \in S)) = \mathbb{P}(M(X) \in S)$ which completes the proof.

\end{proof}
\vspace{-7mm}

\subsection{Proof of Corollary~\ref{cor:eps_delta_sens}}
\label{proof:eps_delta_sens}

\begin{proof}

The following proof relies on independently analyzing two cases for what the value of $\delta$ could be to conclude the corollary statement (and is inspired by the proof of Proposition 3 in \cite{mironov2017renyi}).

Let $Q = \mathbb{P}(M(X) \in S)$ and note the first term in Lemma~\ref{lem:holder_approach} is then $(a_p^{\frac{1}{1-1/p}}Q)^{1- 1/p}$. Now consider two cases: if $a_p^{\frac{1}{1-1/p}}Q > \delta'^{\frac{p}{p-1}}$ we have $(a_p^{\frac{1}{1-1/p}}Q)^{1- 1/p} \leq a_p^{\frac{1}{1-1/p}}Q \cdot \delta'^{\frac{-1}{p-1}}$, and so the overall expression in Lemma~\ref{lem:holder_approach} is $\leq (a_p^{\frac{1}{1-1/p}} \delta'^{\frac{-1}{p-1}} + \mathbb{P}(0))Q + \mathbb{P}_{x^*}(1)\delta$.

Now else we have the first term is $\leq \delta'$ and the overall expression is $\leq \mathbb{P}_{\mathbf{x}^*}(0) Q + \mathbb{P}_{x^*}(1)\delta + \delta'$. Combining the two scenarios we see we always have $\mathbb{P}(M(X') \in S) \leq (a_p^{\frac{1}{1-1/p}}\delta'^{\frac{-1}{p-1}} + \mathbb{P}_{x^*}(0)) \mathbb{P}(M(X) \in S) + \mathbb{P}_{x^*}(1)\delta + \delta'$, giving the stated condition.

\end{proof}
\vspace{-7mm}

\subsection{Proof of Theorem~\ref{thm:easy_renyi_dp}}
\label{proof:easy_renyi_dp}

\begin{proof}

    The proof is analogous to the results of \citet{mironov2019r} with slight modifications to make it per-instance.
    
    Recall that the density function for the sampled Gaussian mechanism $M(X)$ is $$\sum_{X_B \subset D} \mathbb{P}(X_B) N(U(X_B),\sigma^2 \mathbb{I}^d)$$ and for $X' = X \cup x^*$ it is 
    
    \vspace{-7mm} \begin{multline}
        \sum_{X'_B \subset X'} \mathbb{P}(X'_B) N(U(X'_B),\sigma^2 \mathbb{I}^d) \\ = \sum_{X_B \subset D} \mathbb{P}(X_B) (\mathbb{P}_{x^*}(0)N(U(X_B),\sigma^2 \mathbb{I}^d) \\ + \mathbb{P}_{x^*}(1)N(U(X_B \cup x^*),\sigma^2 \mathbb{I}^d)) 
    \end{multline} \vspace{-5mm}

    By quasi concavity of the R\'enyi divergence we then have 

    \vspace{-7mm} \begin{multline}
        D_{\alpha}(M(X)|| M(X')) \\ \leq sup_{X_B \subset X} D_{\alpha}(N(U(X_B),\sigma^2 \mathbb{I}^d)|| \mathbb{P}_{x^*}(0) N(U(X_B),\sigma^2 \mathbb{I}^d) \\ + \mathbb{P}_{x^*}(1) N(U(X_B \cup x^*),\sigma^2 \mathbb{I}^d)) \\ \leq sup_{X_B \subset X} D_{\alpha}(N(0,\sigma^2 \mathbb{I}^d)|| \mathbb{P}_{x^*}(0) N(0,\sigma^2 \mathbb{I}^d) \\ + \mathbb{P}_{x^*}(1) N(U(X_B \cup x^*)- U(X_B),\sigma^2 \mathbb{I}^d))
    \end{multline} \vspace{-5mm}

    where we also used that R\'enyi divergences are translationally invariant. Now by noting the covariances are symmetric, we can apply a change of variables such that $U(X_B \cup x^*)- U(X_B) \rightarrow ||U(X_B \cup x^*)- U(X_B)||_2 e_1$ without changing the divergence. As now the change is only along one dimension of the product distribution, by additivity of R\'enyi divergences we conclude 
    
    \vspace{-7mm} \begin{multline}
        D_{\alpha}(M(X)|| M(X')) \leq sup_{X_B \subset D} D_{\alpha}(N(0,\sigma^2)|| \mathbb{P}_{x^*}(0) N(0,\sigma^2) \\ + \mathbb{P}_{x^*}(1) N(||U(X_B \cup x^*)- U(X_B)||_2,\sigma^2))
    \end{multline} \vspace{-5mm}
    
    However as $||U(X_B \cup x^*)- U(X_B)||_2 \leq \Delta_{U,x^*}$ we can conclude (by change of variables and post-processing) 
    
    \vspace{-7mm} \begin{multline}
        D_{\alpha}(M(X)|| M(X')) \leq D_{\alpha}(N(0,\sigma^2)|| \mathbb{P}_{x^*}(0) N(0,\sigma^2) \\ + \mathbb{P}_{x^*}(1) N(\Delta_{U,x^*},\sigma^2))
    \end{multline} \vspace{-5mm}
    
    Analogous calculations show 

    \vspace{-7mm} \begin{multline}
        D_{\alpha}(M(X')|| M(X)) \\ \leq D_{\alpha}(\mathbb{P}_{x^*}(0) N(0,\sigma^2) + \mathbb{P}_{x^*}(1) N(\Delta_{U,x^*},\sigma^2)|| N(0,\sigma^2))
    \end{multline} \vspace{-5mm}

    Now Theorem 5 in \citet{mironov2019r} states that if $P,Q$ are two differentiable distributions s.t $P(x) = Q(v(x))$ where $v(v(x)) = x$ and $v$ is also differentiable, then for all $\alpha \geq 1$ and $q \in [0,1]$ $D_{\alpha}((1-q)P + qQ|| Q) \geq D_{\alpha}(Q|| (1-q)P + qQ)$. Now defining $Q = N(0,\sigma^2)$ and $P = N(\Delta_{U,x^*},\sigma^2)$ we have $v(x) = \Delta_{U,x^*} - x$ which is differentiable and $v(v(x)) = x$, and so conclude 

    \vspace{-7mm} \begin{multline}
        D(M(X) || M(X')) \leq D_{\alpha}(N(0,\sigma^2)|| \mathbb{P}_{x^*}(0) N(0,\sigma^2) + \\ \mathbb{P}_{x^*}(1) N(\Delta_{U,x^*},\sigma^2)) 
 \\ \leq D_{\alpha}(\mathbb{P}_{x^*}(0) N(0,\sigma^2) + \mathbb{P}_{x^*}(1) N(\Delta_{U,x^*},\sigma^2)|| N(0,\sigma^2))
    \end{multline} \vspace{-5mm}
    
    As the last expression already bounds $D(M(X')||M(X))$, we proceed to bound it to get our desired guarantee.

    Following the calculations of Section 3.3 in \citet{mironov2019r} we use $\mu = \mathbb{P}_{x^*}(0) N(0,\sigma^2) + \mathbb{P}_{x^*}(1) N(\Delta_{U,x^*},\sigma^2)$, $\mu_0  = N(0,\sigma^2) $, and $\mu_{\Delta_{U,x^*}} = N(\Delta_{U,x^*},\sigma^2)$. We are thus interested in $D_{\alpha}(\mu || \mu_0) = \frac{1}{\alpha- 1} \ln{\int (\mu(w) / \mu_0(w))^{\alpha} \mu_0(w) }$. Note $(\mu / \mu_0)^{\alpha} = \sum_{k = 0}^{\alpha} {\alpha \choose k} (1-\mathbb{P}_{x^*}(1))^{\alpha -k} (\frac{\mu_{\Delta_{U,x^*}}}{\mu_0})^{k}$ and so plugging that into the previous integral, and then completing the square in the exponent for the Gaussian density function, we get 
    
    \begin{multline}
        D_{\alpha}(\mu || \mu_0) = \frac{1}{\alpha- 1} \ln  \sum_{k=0}^{\alpha} {\alpha \choose k} (1-\mathbb{P}_{x^*}(1))^{\alpha -k}  \\ \frac{1}{\sigma \sqrt{2 \pi}} \int \exp{\frac{-(x-k)^2 + (k^2 \Delta_{U,x^*}^2 - k \Delta_{U,x^*}^2 ) }{2\sigma^2}}  \\ = \frac{1}{\alpha- 1} \ln{ \sum_{k=0}^{\alpha} {\alpha \choose k} (1-\mathbb{P}_{x^*}(1))^{\alpha -k} \exp{\frac{\Delta_{U,x^*}^2 (k^2 -k) }{2\sigma^2}} }
    \end{multline}

    This concludes the proof.

\end{proof} \vspace{-7mm}

\subsection{Proof of Theorem~\ref{thm:composition}}
\label{proof:composition}

\looseness=-1
We first present a version of Theorem~\ref{thm:better_composition} that uses Cauchy-Schwarz, i.e., Holder's inequality with Holder constant $p=2$. This we believe is easier to follow, and makes clearer the specific role of the Holder's constants in the proof of Theorem~\ref{thm:better_composition}

\begin{theorem}
\label{thm:composition}

    Consider a sequence of functions $X_1(x_1), X_2(x_1,x_2), X_3(x_2,x_3),\cdots X_n(x_{n-1},x_n)$ where $X_{i}$ is a density function in the second arugment for any fixed value of the first argument, except $X_1$ which is a densitiy function in $x_1$. Consider an analogous sequence $Y_1(x_1),\cdots, Y_n(x_{n-1})$. Then letting $X = \prod_{j=1}^{n} X_j$ be the density function for a sequence $x_1,\cdots,x_n$ generated according to the Markov chain defined by $X_i$, and similarly $Y$, we have

    \vspace{-7mm} \begin{multline}
        D_{\alpha}(X || Y) \\ \leq \frac{1}{\alpha -1} (\sum_{i=0}^{n-2} \frac{1}{2^{i+1}} \ln (\mathbb{E}_{X_1,\cdots X_{n-(i+1)}}  ((e^{(g^{i}(\alpha) -1)D_{g^{i}(\alpha)}(X_{n-i}|| Y_{n-i})})^2))) \\ + \frac{1}{\alpha -1} ((\frac{1}{2})^{n} \ln ((e^{(g^{n-1}(\alpha) -1)D_{g^{n-1}(\alpha)}(X_{1}|| Y_{1})})^2)) 
    \end{multline} \vspace{-5mm}

    where $g(\alpha) = 2\alpha -1$ and $g^i$ means $g$ composed $i$ times, where $g^{0}(\alpha) = \alpha$
\end{theorem}

\begin{proof}

The proof relies on repeating the same reduction on the number of steps being considered. First note 
    
\vspace{-7mm} \begin{multline}
    \int (X_1 \cdots X_n)^{\alpha} (Y_1 \cdots Y_n)^{1 - \alpha} dx_1 \cdots dx_n \\  = \int (X_1 \cdots X_{n-1})^{\alpha - 1/2} (Y_1 \cdots Y_{n-1})^{1 - \alpha}  \\ (\int X_n^{\alpha} Y_n^{1- \alpha} dx_n) (X_1 \cdots X_{n-1})^{1/2} dx_1 \cdots dx_{n-1}
    \\ \leq ( \int (X_1 \cdots X_n)^{2\alpha -1} (Y_1 \cdots Y_n)^{1 - (2\alpha-1)}  dx_1 \cdots dx_{n-1})^{1/2} \\ (\int (\int X_n^{\alpha} Y_n^{1- \alpha} dx_n)^2 (X_1 \cdots X_{n-1}) dx_1 \cdots dx_{n-1})^{1/2}
\end{multline} \vspace{-5mm}

where the first equality was from using the markov property, and the last inequality was from Cauchy-Schwarz. So now looking at the first term, we are back to the original expression but with $\alpha \rightarrow g(a) = 2\alpha -1$ and $n \rightarrow n-1$, and an exponent to $1/2$. Note the second term is an expectation over the $n-1$ model state of the Markov chain. Do note $\int X_n^{\alpha} Y_n^{1- \alpha} dx_n$ is $e^{(\alpha -1)D_{\alpha}(X_{n-i}|| Y_{n-i})}$ for a fixed $n-1$ model state (i.e., fixed $x_{n-1}$ ). So repeating this step on the first term until we are left only with an integral over $x_1$ we have

\vspace{-7mm} \begin{multline}
    \int (X_1 \cdots X_n)^{\alpha} (Y_1 \cdots Y_n)^{1 - \alpha} dx_1 \cdots dx_n \\  
    \leq (\prod_{i=0}^{n-2} (\mathbb{E}_{X_1,\cdots X_{n-(i+1)}}  ((e^{(g^{i}(\alpha) -1)D_{g^{i}(\alpha)}(X_{n-i}|| Y_{n-i})})^2))^{(\frac{1}{2})^{i+1}}) \\ ( (e^{(g^{n-1}(\alpha) -1)D_{g^{n-1}(\alpha)}(X_{1}|| Y_{1})})^2)^{(\frac{1}{2})^{n}}
\end{multline} \vspace{-5mm}

So now noting \vspace{-3mm} $$D_{\alpha}(X || Y) = \frac{1}{\alpha -1} \ln(\int (X_1 \cdots X_n)^{\alpha} (Y_1 \cdots Y_n)^{1 - \alpha} dx_1 \cdots dx_n)$$ \vspace{-3mm}

we conclude by the previous expression that 

\vspace{-7mm} \begin{multline}
    D_{\alpha}(X || Y) \\ \leq \frac{1}{\alpha -1} (\sum_{i=0}^{n-2} \frac{1}{2^{i+1}} \ln (\mathbb{E}_{X_1,\cdots X_{n-(i+1)}}  ((e^{(g^{i}(\alpha) -1)D_{g^{i}(\alpha)}(X_{n-i}|| Y_{n-i})})^2))) \\ + \frac{1}{\alpha -1} ((\frac{1}{2})^{n} \ln ((e^{(g^{n-1}(\alpha) -1)D_{g^{n-1}(\alpha)}(X_{1}|| Y_{1})})^2)) 
\end{multline} \vspace{-5mm}

which completes the proof.

\end{proof} \vspace{-7mm}

\ifarxiv
\input{Sections/Empirical_Results}

\else
\input{Sections/Concise_Empirical_Results}
\fi

%% file: Sections/Empirical_Results.tex
\section{Detailed Empirical Results}
\label{sec:detailed_emp_res}

In this section we present evaluation for the three different per-step guarantees presented in Section~\ref{sec:analysis}. In evaluating the guarantees they give for DP-SGD, we also consider two update rules $U$: the "sum" $U(X_B) = \frac{1}{L} \sum_{x \in X_B} \nabla_{\theta}L(\theta,x)/ \max(1,\frac{||\nabla_{\theta}L(\theta,x)||_2}{C})$, and the "mean" $U(X_B) = \frac{1}{|X_B|}\sum_{x \in X_B} \nabla_{\theta}L(\theta,x)/ \max(1,\frac{||\nabla_{\theta}L(\theta,x)||_2}{C})$. Note the subtle difference between dividing by a fixed constant $L$ (typically the expected mini-batch size when Poisson sampling datapoints) and by the mini-batch size $|X_B|$. This means for the sum the upper-bound on sensitivity is $\frac{C}{L}$, while for the mean the upper-bound on sensitivity is only $C$ (consider neighbouring mini-batches of size 1 and 2). In practice, this means using the mean update rule requires far more noise and hence is not practical to use. We will highlight how our analysis by sensitivity distributions allows better analysis of the mean update rule.

\paragraph{Experimental Setup.} In the subsequent experiments, we empirically verify our claimed guarantees on MNIST~\citep{lecun1998mnist} and CIFAR-10~\citep{krizhevsky2009learning}. Unless otherwise specified, LeNet-5~\citep{lecun1989backpropagation} and ResNet-20~\citep{he2016deep} were trained on the two datasets for 10 and 200 epochs respectively, with a mini-batch size equal to 128, clipping norm equal to 1, $\epsilon=10$, and $\delta = 10^{-5}$ (or $\alpha = 8$ in cases of R\'enyi-DP) 
. 
Regarding hardware, we used NVIDIA T4 to accelerate our experiments.

\subsection{Studying the $(\epsilon,\delta)$-DP Case}
\label{ssec:eps_delta_experiments}

Here we compute one of the inequalities needed for $(\epsilon,\delta)$-DP given by Corollary~\ref{cor:eps_delta_sens}.
We set $\delta' = 0.5 \times 10^{-5}$ and $\delta'' = 10^{-5}$, and $p = 10^{4}$. For CIFAR-10, we set $\sigma =  0.488, 0.508, 0.566, 0.625, 0.703$ for mini-batch size $16, 32, 64, 128, 256$ respectively; and for MNIST, we set $\sigma =  0.410, 0.430, 0.449, 0.469, 0.508, 0.566$ for mini-batch size $16, 32, 64, 128, 256, 512$ respectively.
In the proceeding figures we plot results with respect to models at three stages of training: before training starts where the model is just randomly initialized, in the middle of the training where the model has been trained for half the number of epochs, and after the training. We further plot distributions of the guarantee over 500 training points, and indicate whether the data points are correctly or incorrectly classified. Note that we conduct the experiments on both mean and sum update rules. For the case of mean update, which requires computing an expectation of the sensitivity distribution $\Delta_{U,x^*}(X_B)$ to some power, we use 100 samples.

To the best of our knowledge, the standard solely $(\epsilon,\delta)$-DP analysis of the sample Gaussian mechanism would give the guarantee of $\epsilon' = \ln( \mathbb{P}_{x^*}(1) e^{C_{\delta,\sigma} \Delta_{U}} + \mathbb{P}_{x^*}(0))$ and $\delta'' = \mathbb{P}_{x^*}(1) \delta$. We use this as our baseline in this subsection.

\paragraph{Comparison to Baseline.} We see in Figures~\ref{fig:ed_eps_distrib_bs_mnist_mean},\ref{fig:ed_eps_distrib_bs_mnist_sum},\ref{fig:ed_eps_distrib_bs_cifar_mean},\ref{fig:ed_eps_distrib_bs_cifar_sum} that our analysis applied to both the mean and sum beat the baseline analysis of the $(\epsilon,\delta)$-DP guarantee. In particular, we see that for checkpoints in the middle and end of the training, we give a privacy guarantee more than a magnitude better than the baseline for the sum or mean. Furthermore, both our and the baseline's guarantee increases for the sum update rule as we increase the expected mini-batch size, however, for the mean update rule our guarantee decreases (concentrating at smaller values for more datapoints) where as the baseline still increases.

\paragraph{Disparity between Correct and Incorrect.} Shown in Figures~\ref{fig:ed_eps_curve_vary_arch_mnist},\ref{fig:ed_eps_curve_vary_arch_cifar} we see that on average correctly classified data points have better per-step privacy guarantees than incorrectly classified data points across training with different architectures, and this holds most strongly towards the end of training.

\paragraph{Changing DP strength used to get checkpoints.} Now we re-conduct the previous experiments but with varying epsilon, $\epsilon = 1, 3, 10, 30, 100$ shown in Figures~\ref{fig:ed_eps_curve_vary_eps_mnist},\ref{fig:ed_eps_curve_vary_eps_cifar}. We see that changing the strength of DP used to get the checkpoints increases our guarantees, but less so for correctly classified datapoints.

\begin{figure}[t]
\centering
\subfloat[Mini-batch size = 16]
{
\includegraphics[width=0.48\linewidth]{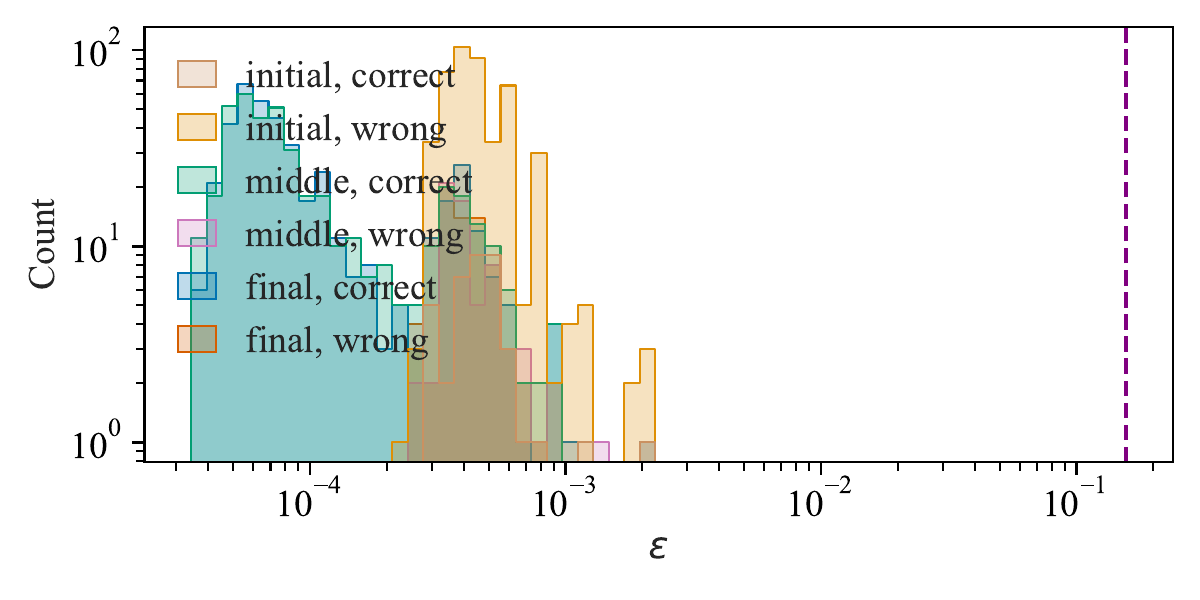}
}
\subfloat[Mini-batch size = 32]
{
\includegraphics[width=0.48\linewidth]{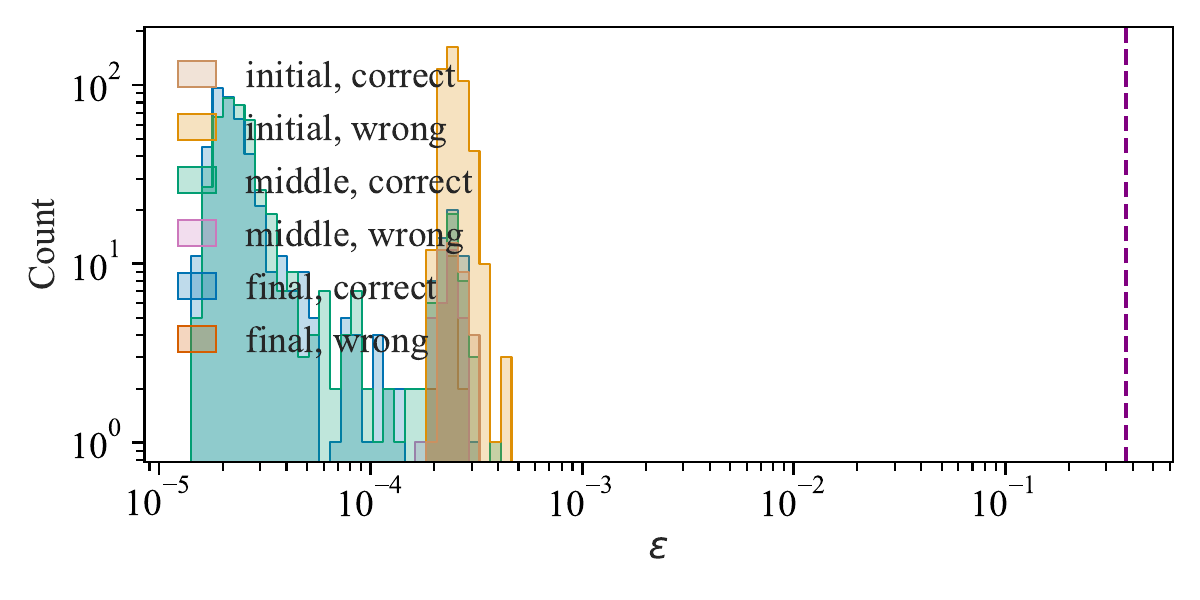}
}
\\\subfloat[Mini-batch size = 64]
{
\includegraphics[width=0.48\linewidth]{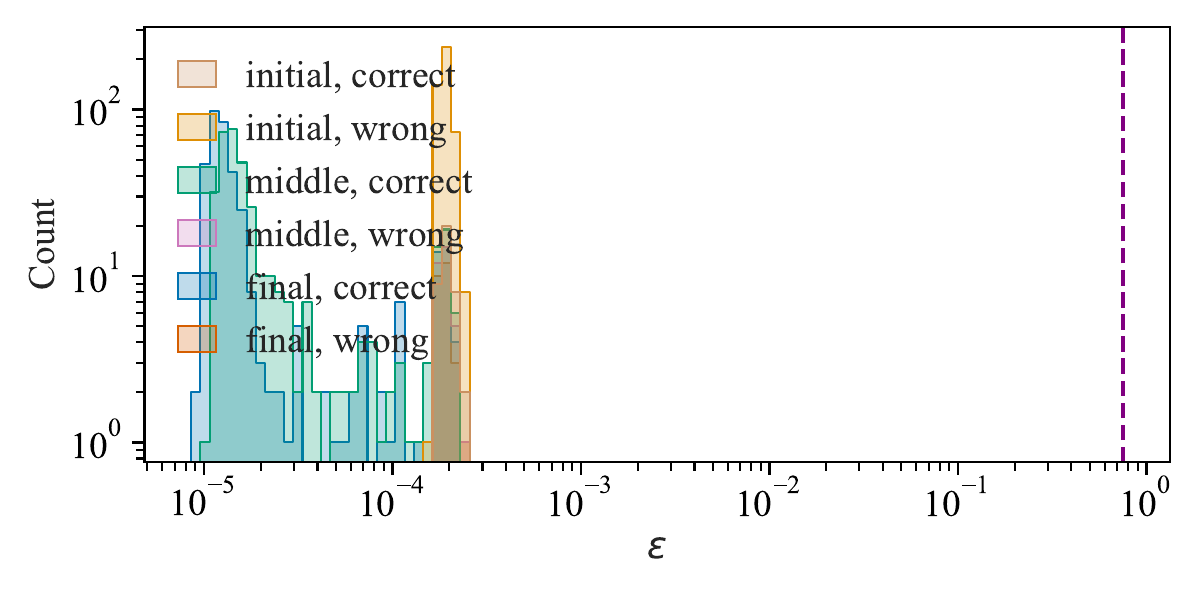}
}
\subfloat[Mini-batch size = 128]
{
\includegraphics[width=0.48\linewidth]{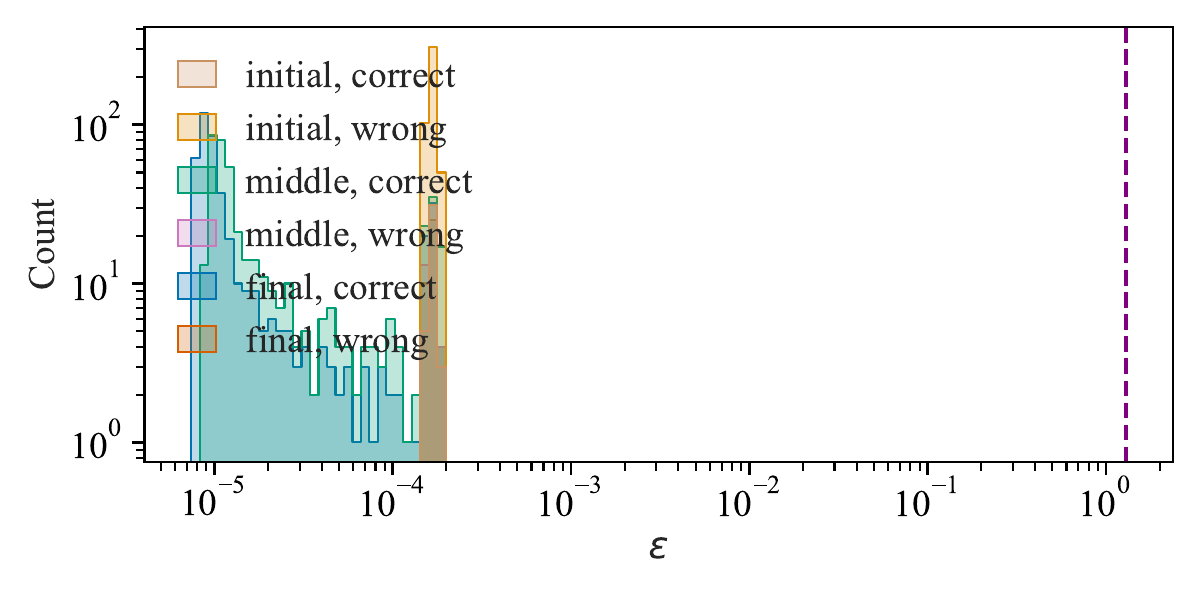}
}
\\\subfloat[Mini-batch size = 256]
{
\includegraphics[width=0.48\linewidth]{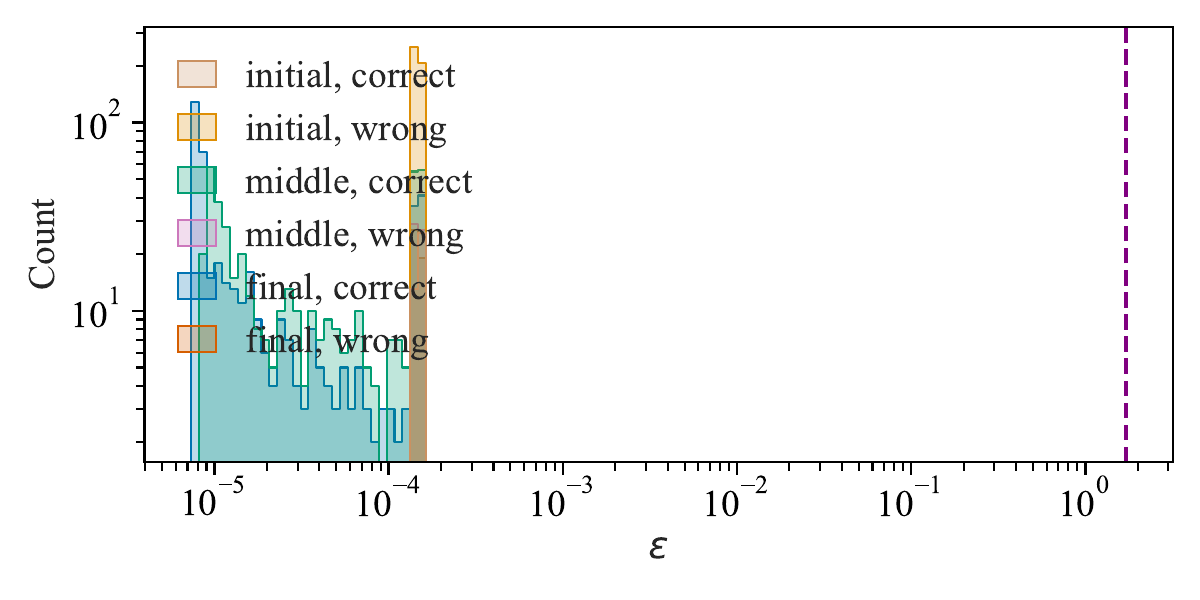}
}
\subfloat[Mini-batch size = 512]
{
\includegraphics[width=0.48\linewidth]{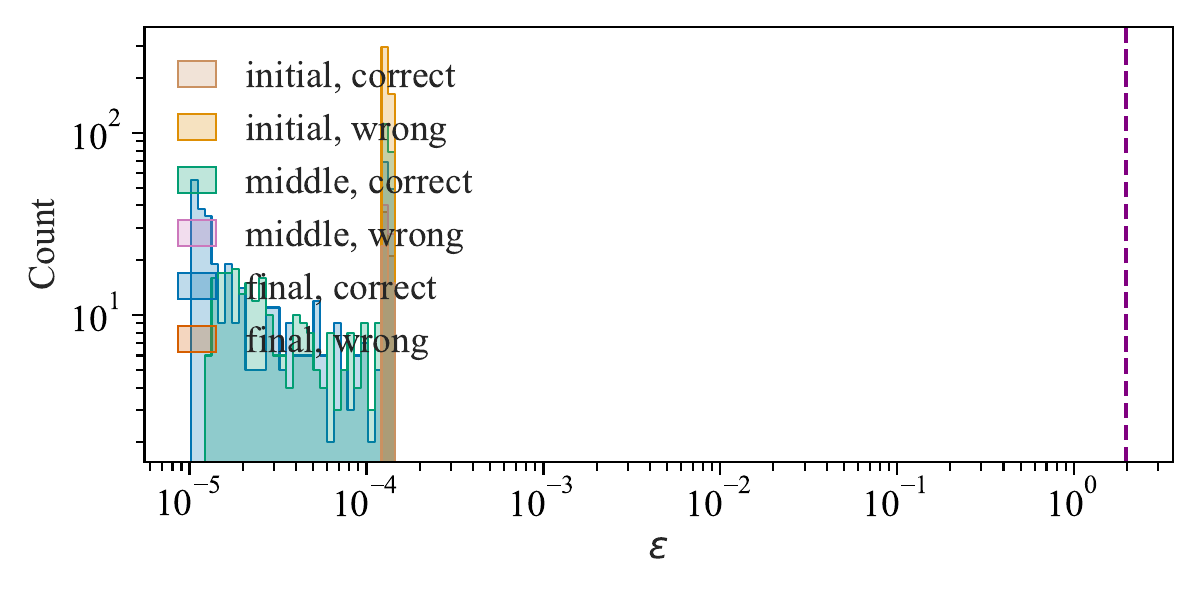}
}
\caption{ Distribution plots of per-step guarantee given by Corollary~\ref{cor:eps_delta_sens} computed on LeNet-5 trained on MNIST with mean update rule and varying mini-batch sizes of $16, 32, 64, 128, 256, 512$. As specified by the legend labels, we group the plotted guarantees by (a) whether the model is at the initial, middle, or final stage of the training, and (b) whether the point on which the data-dependent guarantee is computed is classified correctly or not by the model. It can be seen that in all settings our guarantee is better than the baseline, which is represented by the dashed purple line, by orders of magnitudes. However, the guarantee distributions of incorrectly classified points and points at the initial stage of training are closer to the baseline compared to the other settings.
}
\label{fig:ed_eps_distrib_bs_mnist_mean}
\end{figure}

\begin{figure}[t]
\centering
\subfloat[Mini-batch size = 16]
{
\includegraphics[width=0.48\linewidth]{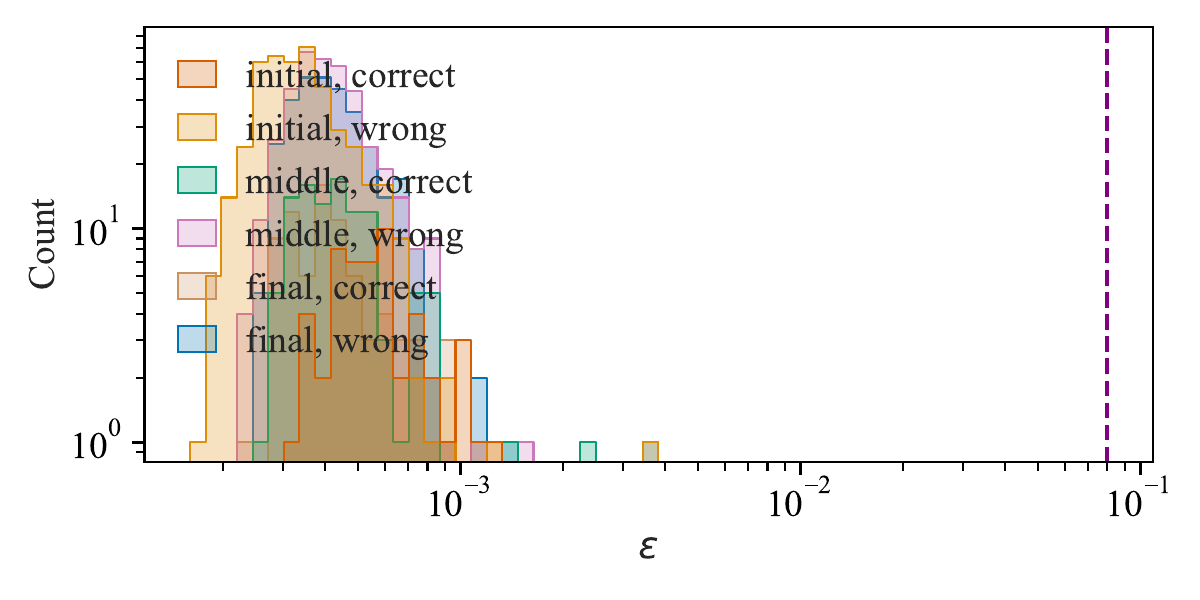}
}
\subfloat[Mini-batch size = 32]
{
\includegraphics[width=0.48\linewidth]{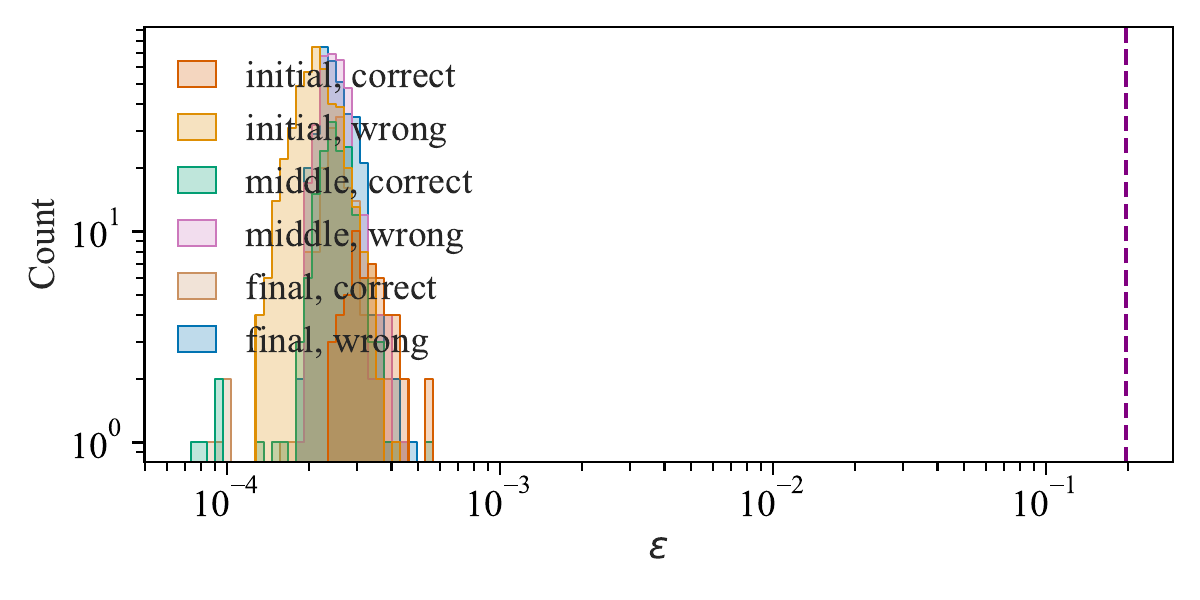}
}
\\\subfloat[Mini-batch size = 64]
{
\includegraphics[width=0.48\linewidth]{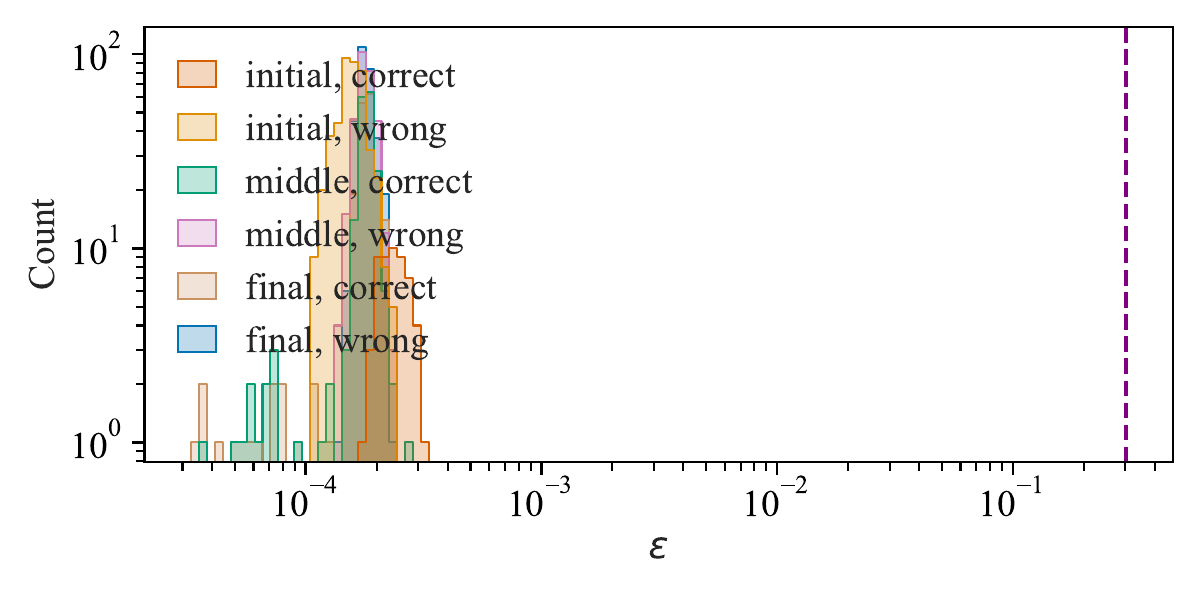}
}
\subfloat[Mini-batch size = 128]
{
\includegraphics[width=0.48\linewidth]{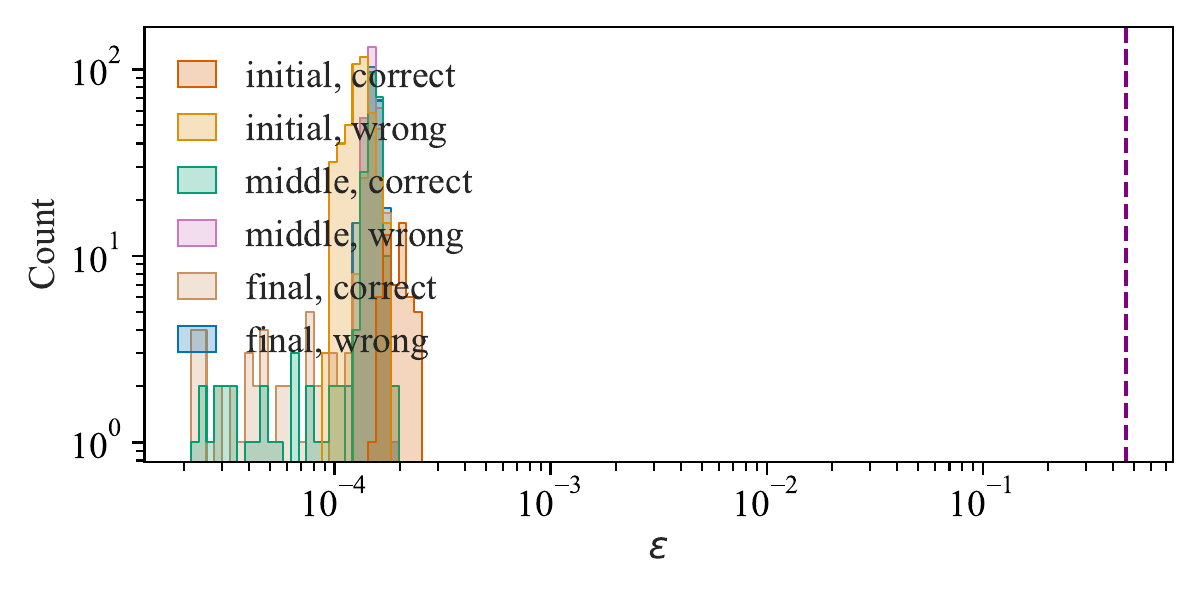}
}
\\\subfloat[Mini-batch size = 256]
{
\includegraphics[width=0.48\linewidth]{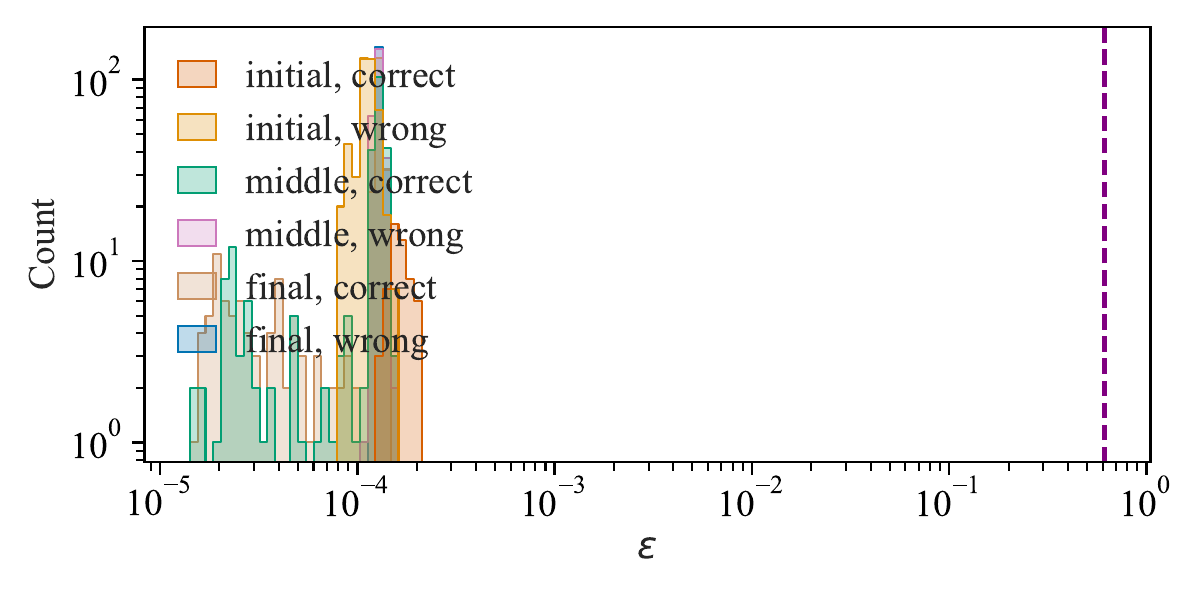}
}
\caption{ This is the reproduction of Figure~\ref{fig:ed_eps_distrib_bs_mnist_mean} except we now use ResNet-20 models trained on CIFAR-10. Similar results are observed so we conclude that our guarantee is effective across datasets.
}
\label{fig:ed_eps_distrib_bs_cifar_mean}
\end{figure}

\begin{figure}[t]
\centering
\subfloat[Mini-batch size = 16]
{
\includegraphics[width=0.48\linewidth]{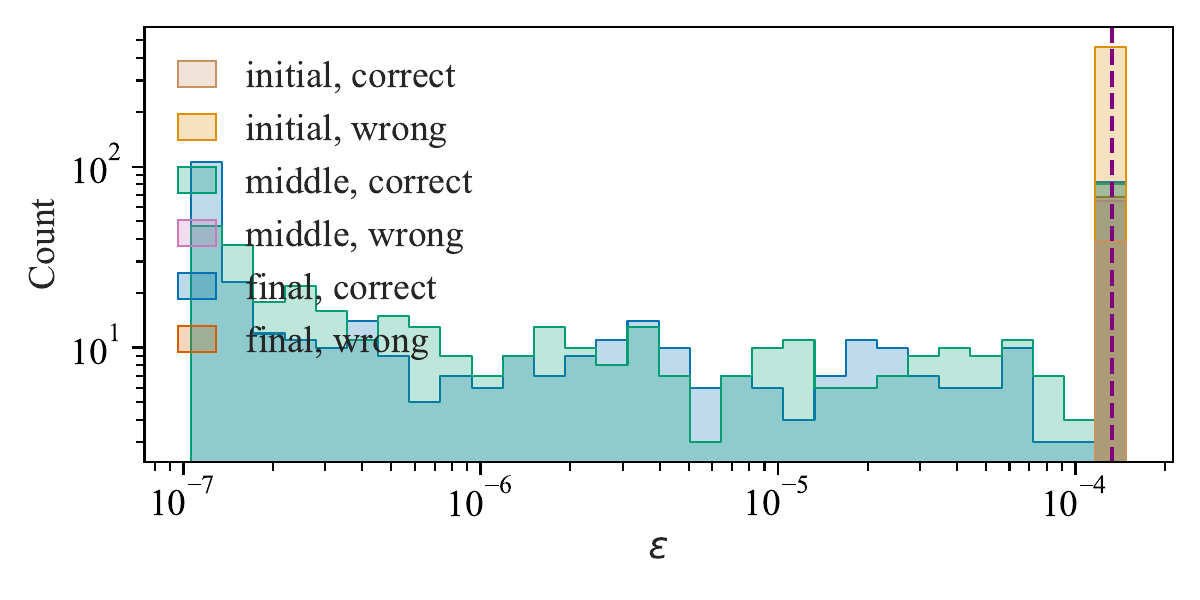}
}
\subfloat[Mini-batch size = 32]
{
\includegraphics[width=0.48\linewidth]{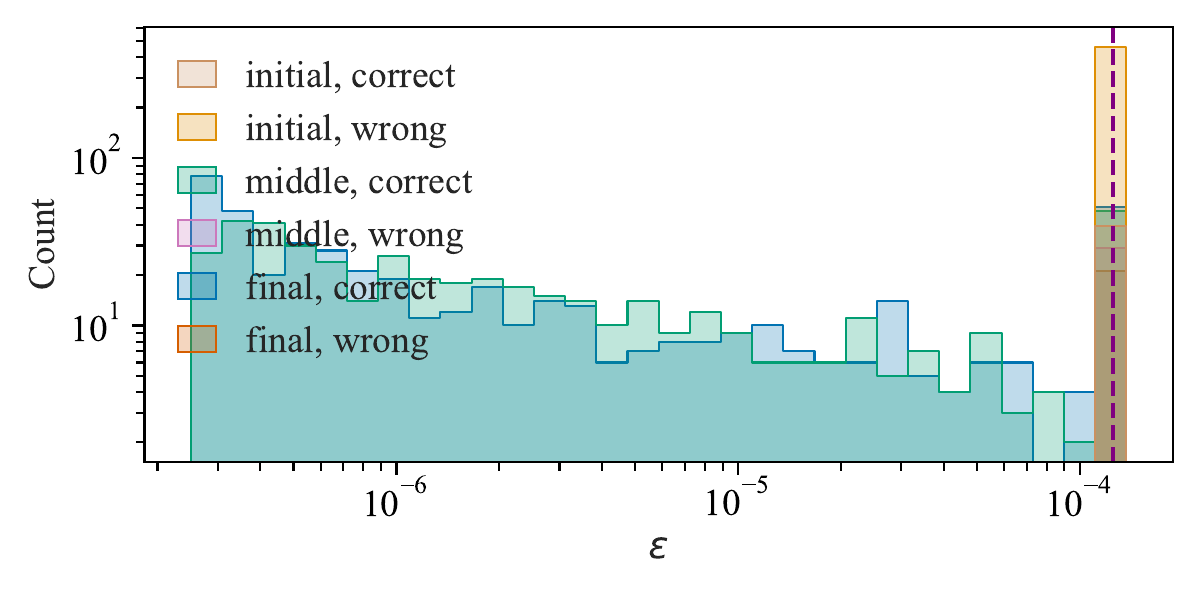}
}
\\\subfloat[Mini-batch size = 64]
{
\includegraphics[width=0.48\linewidth]{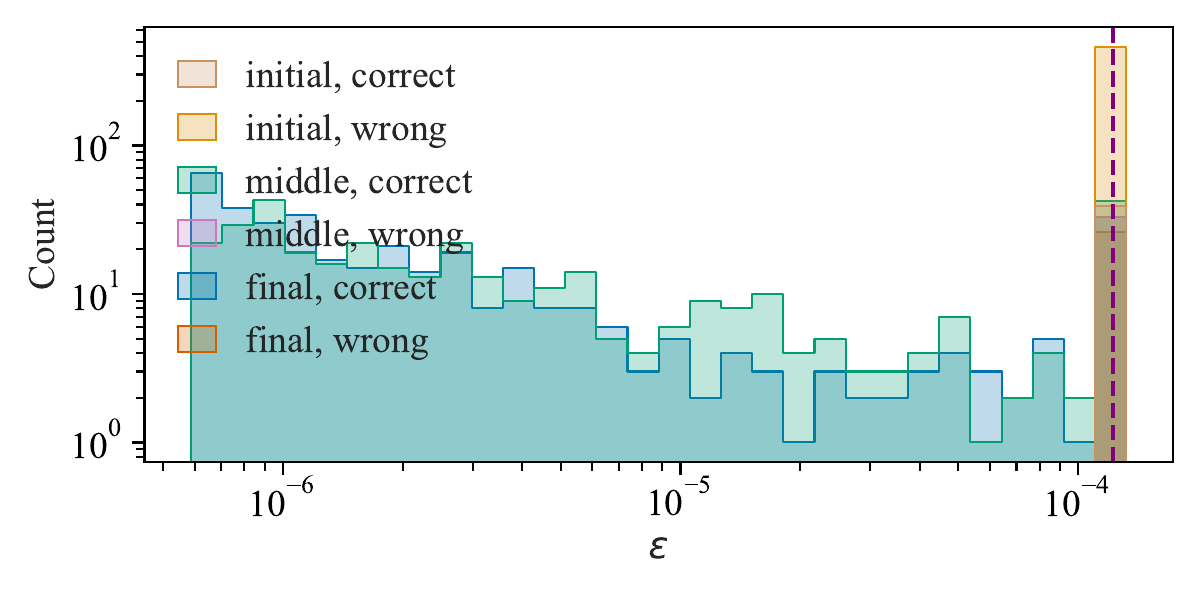}
}
\subfloat[Mini-batch size = 128]
{
\includegraphics[width=0.48\linewidth]{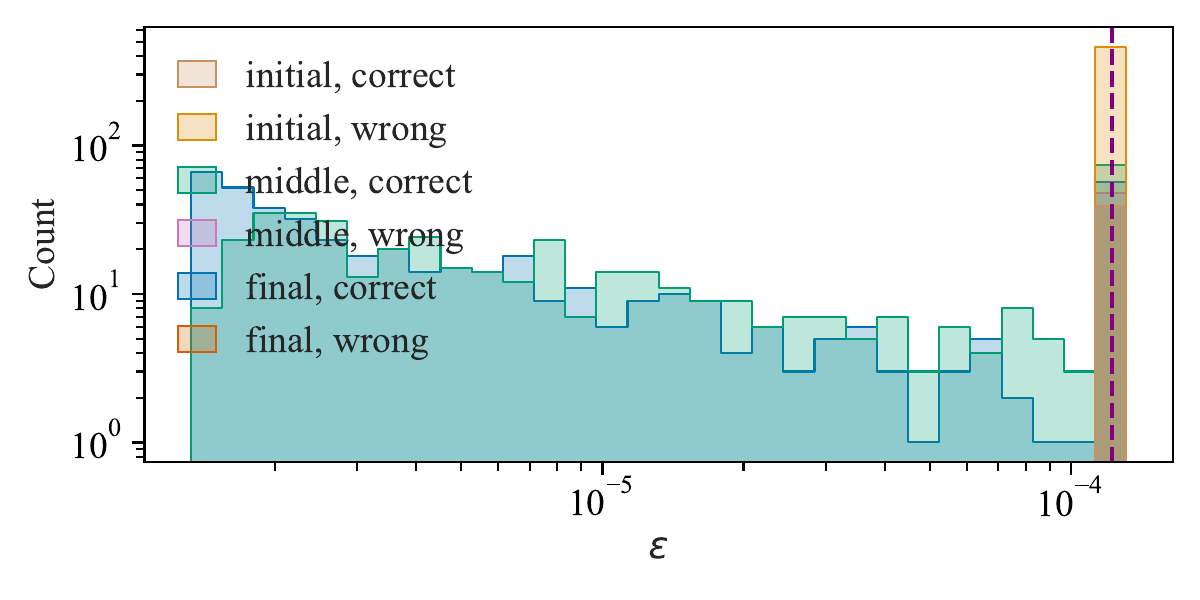}
}
\\\subfloat[Mini-batch size = 256]
{
\includegraphics[width=0.48\linewidth]{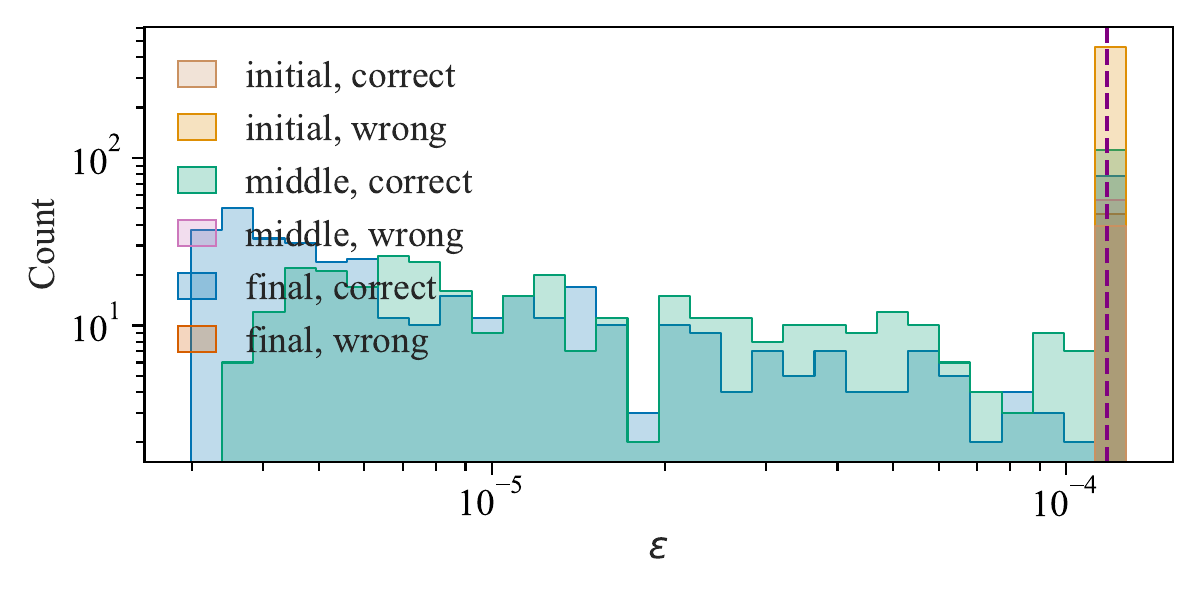}
}
\subfloat[Mini-batch size = 512]
{
\includegraphics[width=0.48\linewidth]{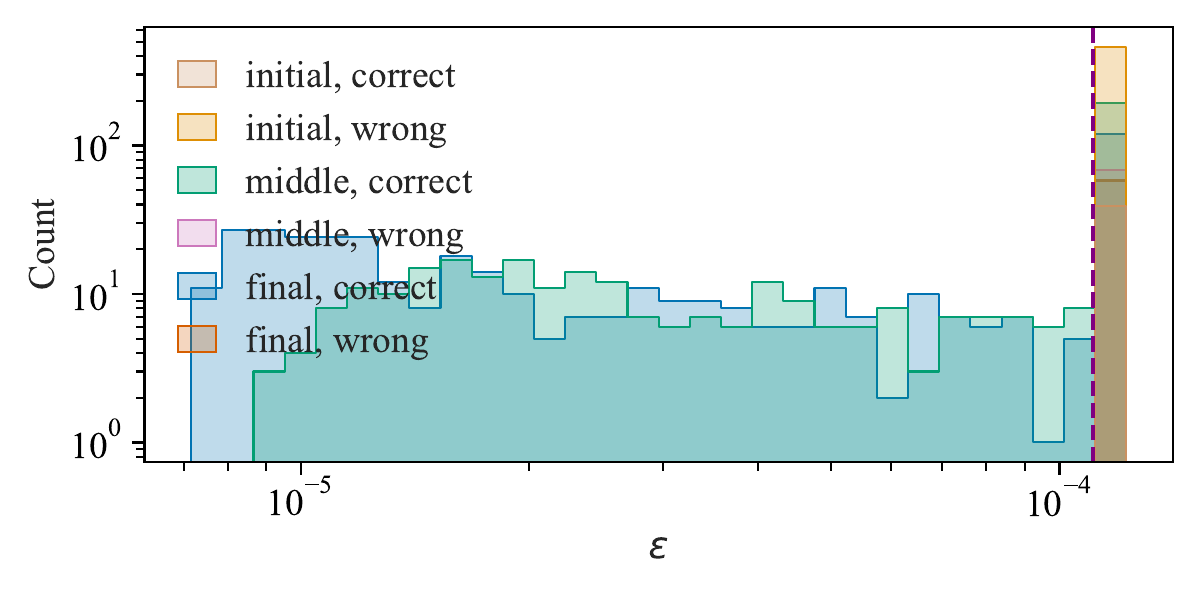}
}
\caption{ This is the reproduction of Figure~\ref{fig:ed_eps_distrib_bs_mnist_mean} except we now use a sum update rule. Unlike the case of using the mean update rule, we now see that our guarantees are more similar to the baseline. Note that the magnitude of our guarantees here does not significantly differ from the case of the mean update rule. The reason for this observation is that the baseline guarantee being much tighter/smaller as explained in Section~\ref{ssec:exp_hard_renyi}.
}
\label{fig:ed_eps_distrib_bs_mnist_sum}
\end{figure}

\begin{figure}[t]
\centering
\subfloat[Mini-batch size = 16]
{
\includegraphics[width=0.48\linewidth]{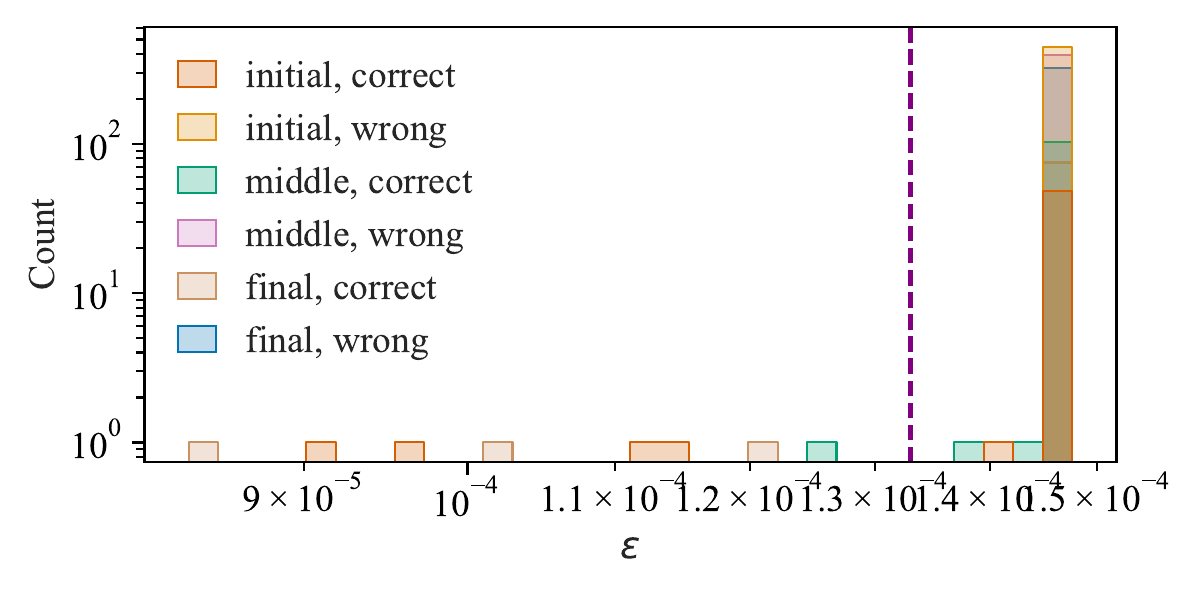}
}
\subfloat[Mini-batch size = 32]
{
\includegraphics[width=0.48\linewidth]{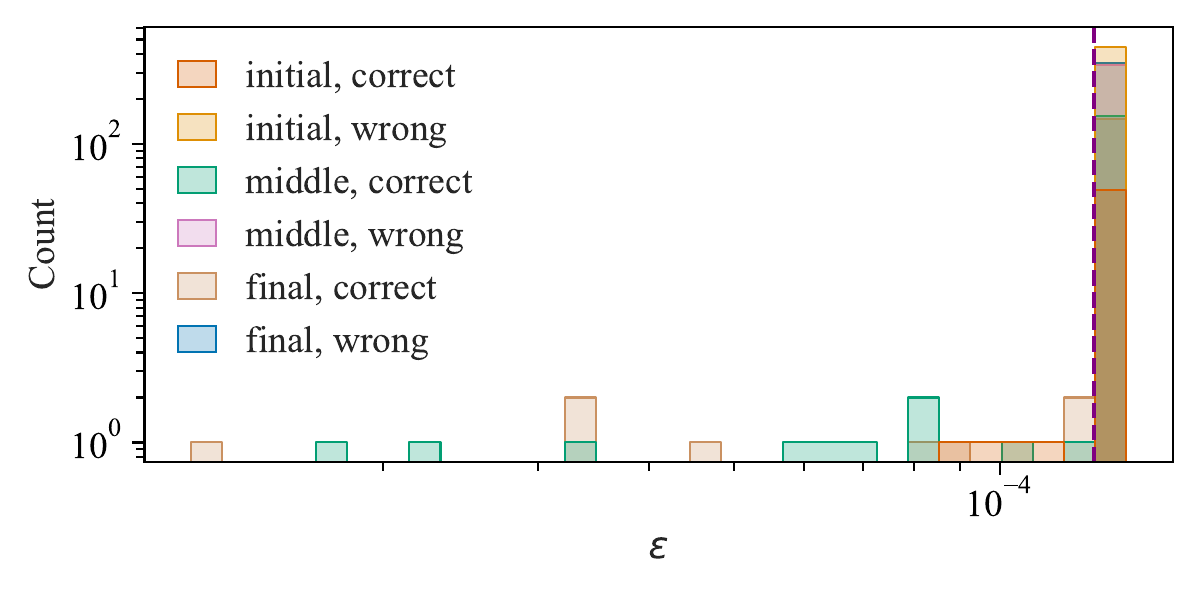}
}
\\\subfloat[Mini-batch size = 64]
{
\includegraphics[width=0.48\linewidth]{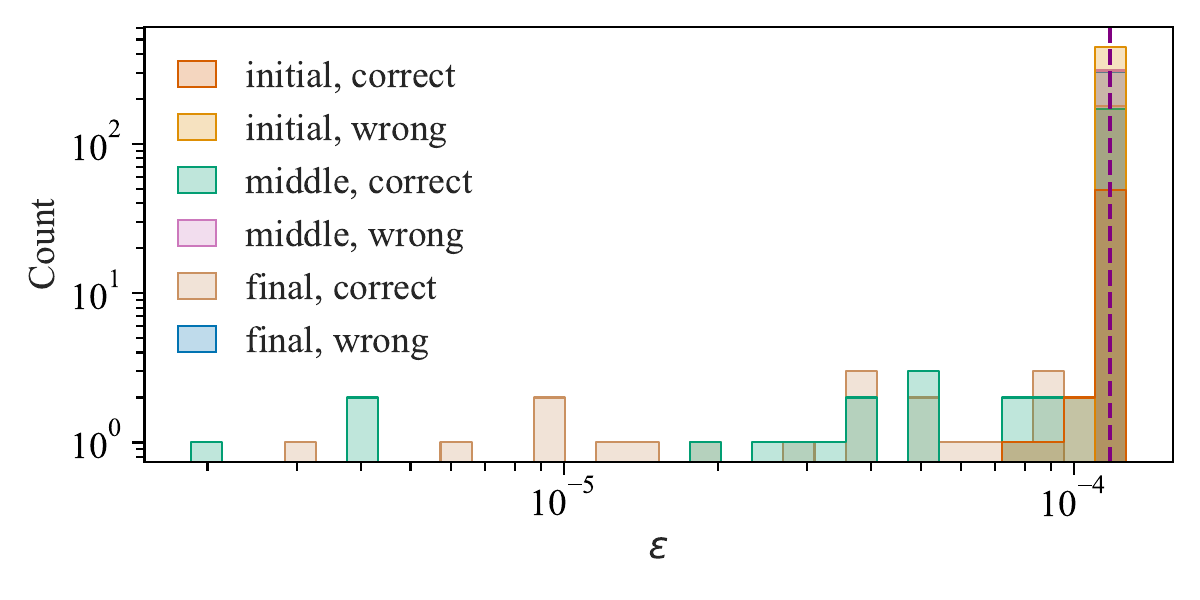}
}
\subfloat[Mini-batch size = 128]
{
\includegraphics[width=0.48\linewidth]{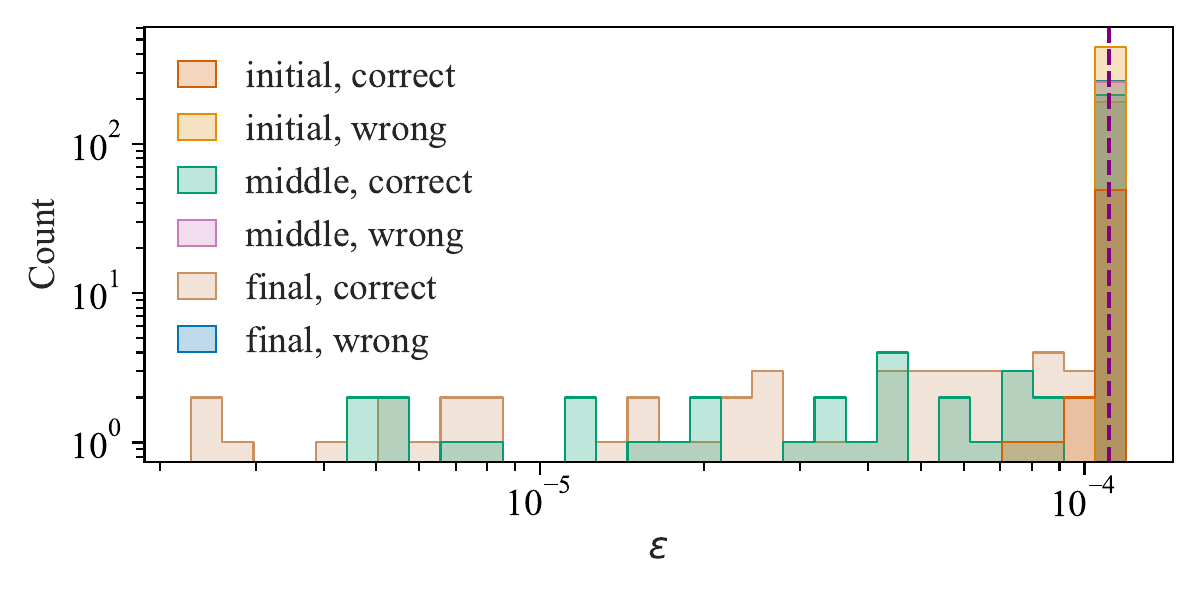}
}
\\\subfloat[Mini-batch size = 256]
{
\includegraphics[width=0.48\linewidth]{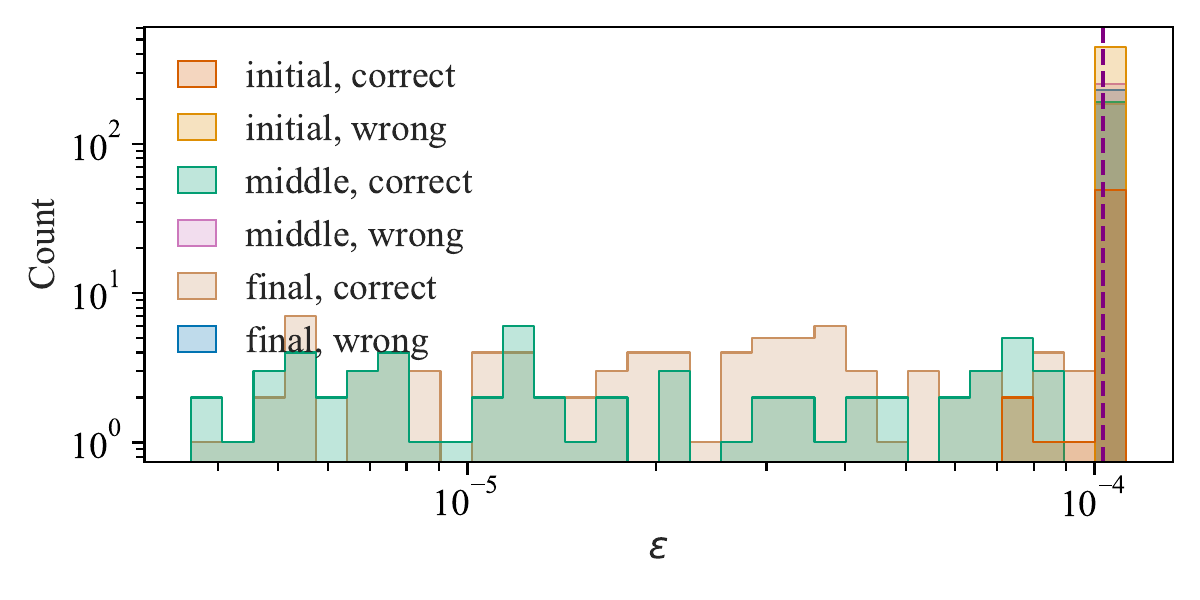}
}
\caption{ This is the reproduction of Figure~\ref{fig:ed_eps_distrib_bs_mnist_sum} except we now use ResNet-20 models trained on CIFAR-10. More of our guarantees concentrate near the baseline guarantee, which may be due to the fact that the CIFAR-10 model has worse accuracy than the MNIST model and we have shown in Section~\ref{ssec:exp_better_privacy} that more accurate points tend to have better data-dependency guarantee.
}
\label{fig:ed_eps_distrib_bs_cifar_sum}
\end{figure}

\begin{figure}[t]
\centering
\subfloat[Update-rule: mean]
{
\includegraphics[width=0.48\linewidth]{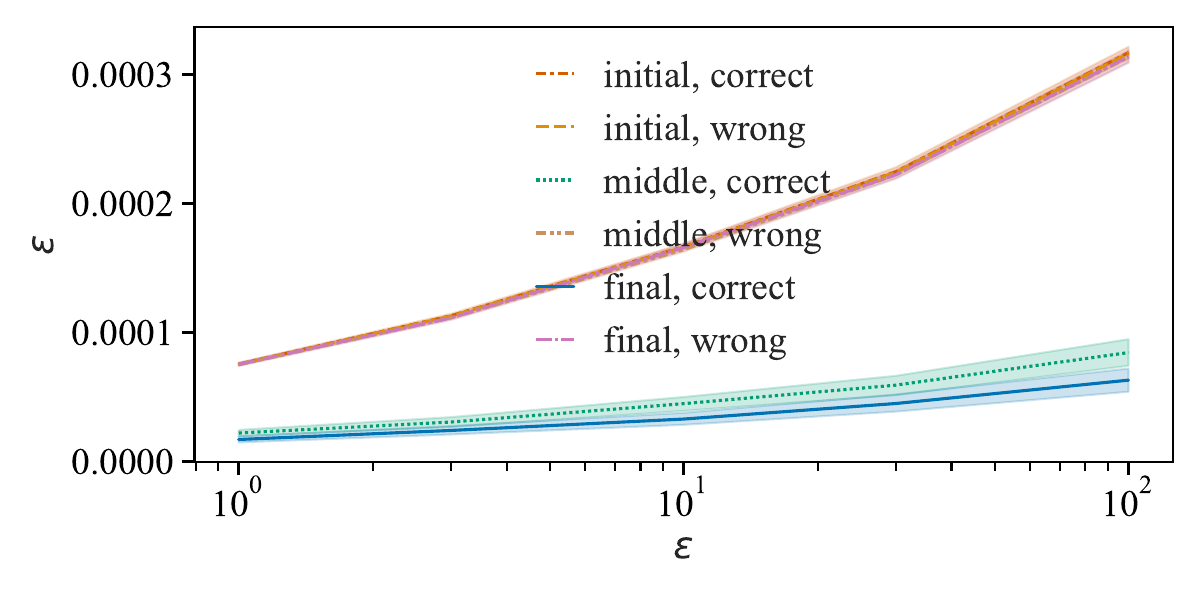}
}
\subfloat[Update-rule: sum]
{
\includegraphics[width=0.48\linewidth]{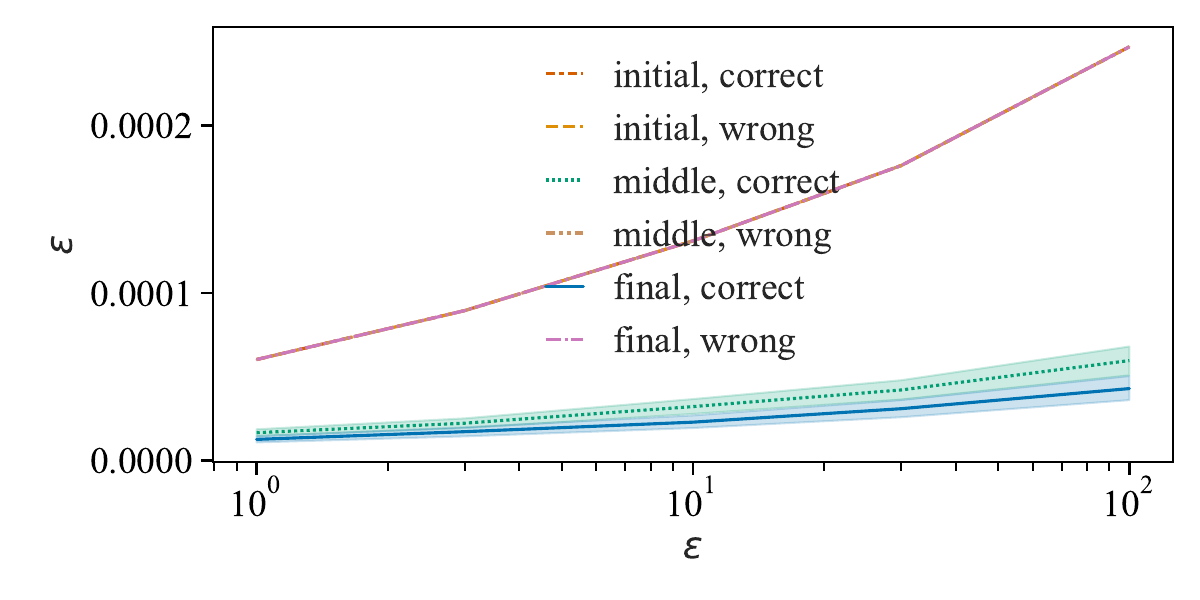}
}
\caption{ 
Per-step guarantee given by Corollary~\ref{cor:eps_delta_sens} with respect to the overall DPSGD privacy guarantee used to obtain the models on MNIST. Each curve corresponds to a certain training stage and whether the data points are correctly classified or not. We include both mean and sum update rules and do not observe a significant difference between them.
}
\label{fig:ed_eps_curve_vary_eps_mnist}
\end{figure}

\begin{figure}[t]
\centering
\subfloat[Update-rule: mean]
{
\includegraphics[width=0.48\linewidth]{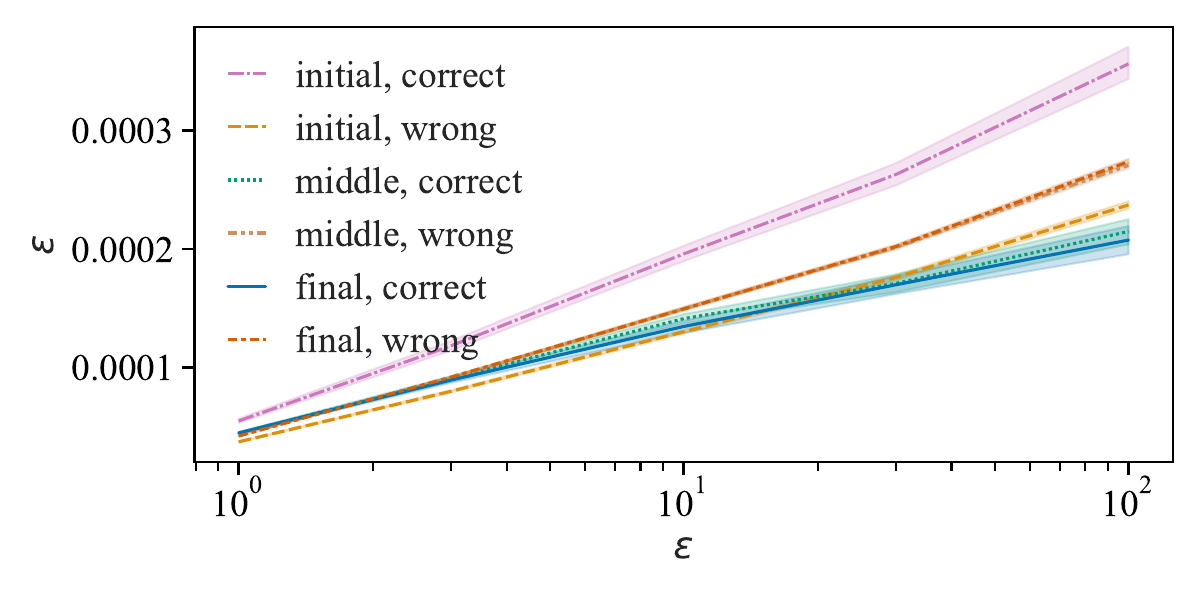}
}
\subfloat[Update-rule: sum]
{
\includegraphics[width=0.48\linewidth]{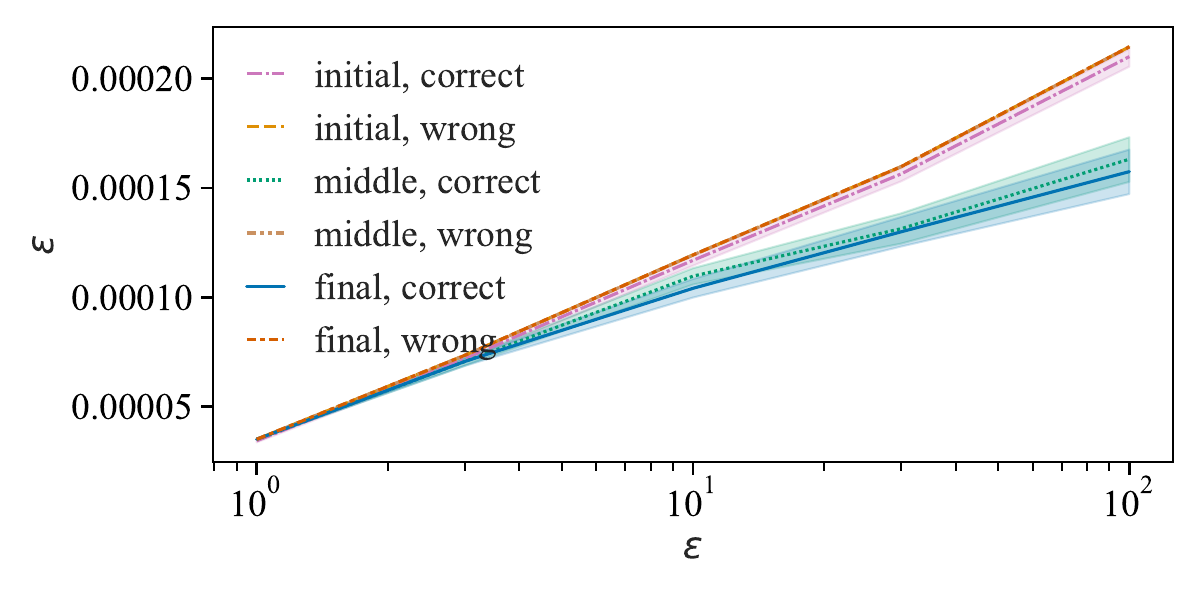}
}
\caption{ This is the reproduction of Figure~\ref{fig:ed_eps_curve_vary_eps_mnist} except we now use ResNet-20 models trained on CIFAR-10. Similar results are observed.
}
\label{fig:ed_eps_curve_vary_eps_cifar}
\end{figure}

\begin{figure}[t]
\centering
\subfloat[Update-rule: mean]
{
\includegraphics[width=0.48\linewidth]{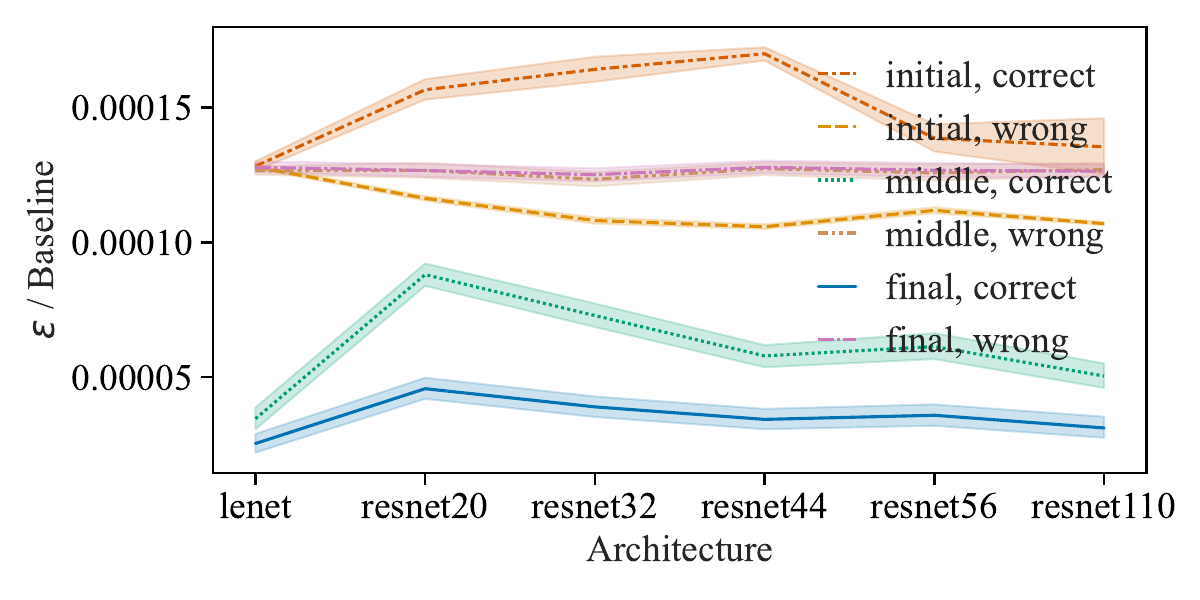}
}
\subfloat[Update-rule: sum]
{
\includegraphics[width=0.48\linewidth]{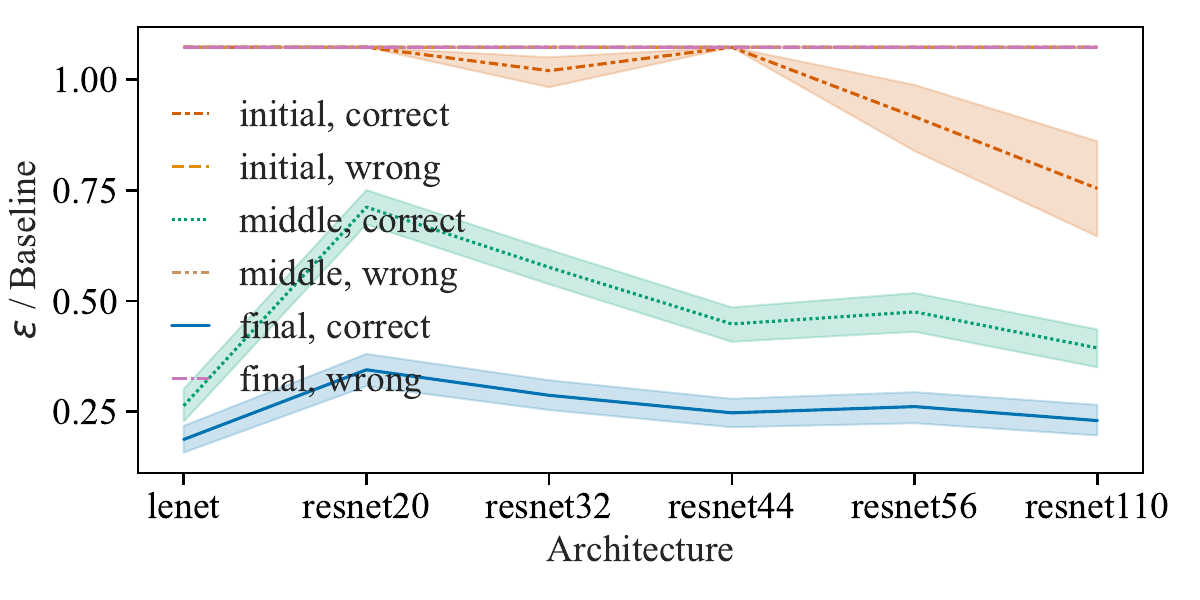}
}
\caption{ Per-step guarantee given by Corollary~\ref{cor:eps_delta_sens} computed by using models trained on MNIST with different architecture, divided by the per-step DPSGD guarantee (baseline). Each curve corresponds to a certain training stage and whether the data points are correctly classified or not. We include both mean and sum update rules and do not observe a significant difference between them. Within each figure, we can see that our guarantee does not vary significantly across different architectures so we may claim to be independent of datasets. Besides, the disparity of these curves also supports our claim that more accurate points have better privacy guarantees.
}
\label{fig:ed_eps_curve_vary_arch_mnist}
\end{figure}

\begin{figure}[t]
\centering
\subfloat[Update-rule: mean]
{
\includegraphics[width=0.48\linewidth]{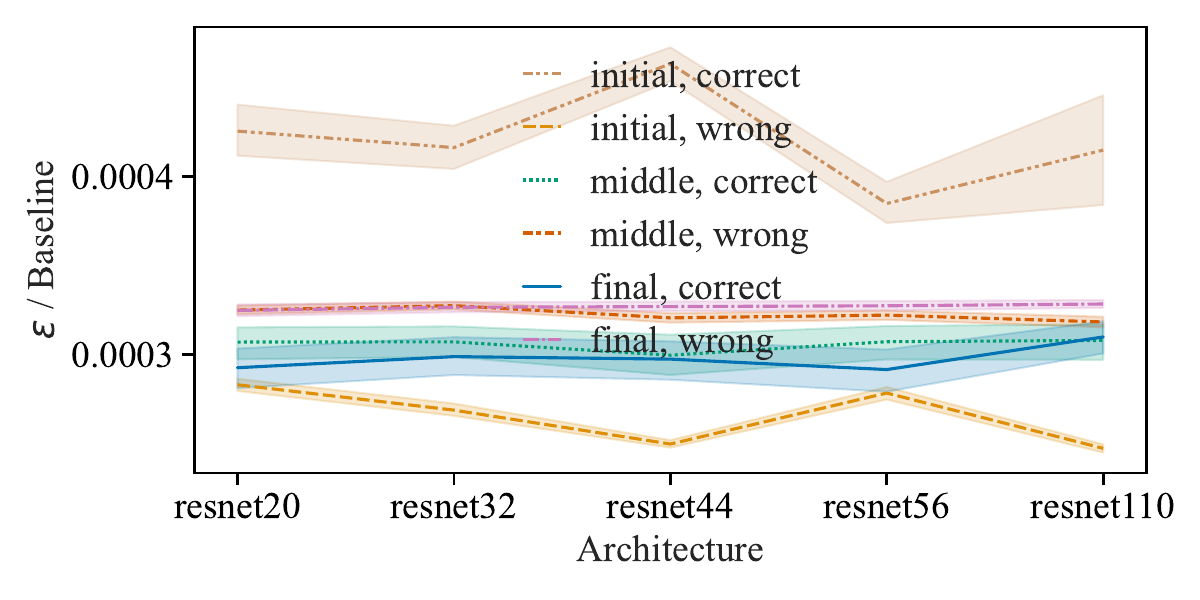}
}
\subfloat[Update-rule: sum]
{
\includegraphics[width=0.48\linewidth]{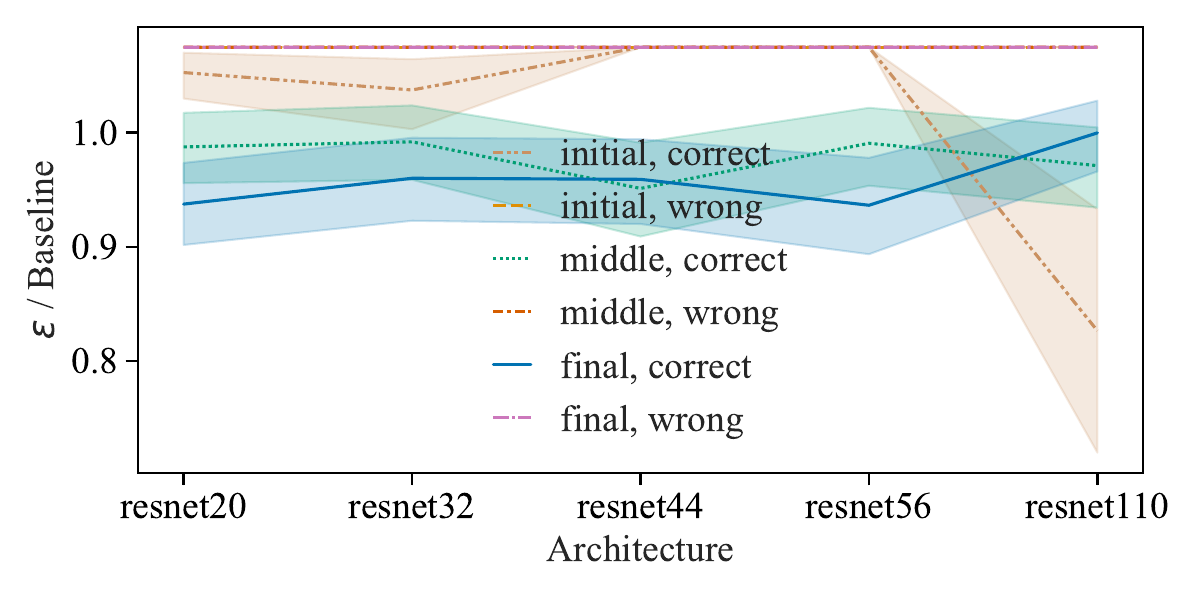}
}
\caption{ This is the reproduction of Figure~\ref{fig:ed_eps_curve_vary_arch_mnist} except we now use ResNet-20 models trained on CIFAR-10. Similar results can be seen so we believe our guarantee given by Corollary~\ref{cor:eps_delta_sens} is independent of the model architecture and the dataset. 
}
\label{fig:ed_eps_curve_vary_arch_cifar}
\end{figure}

\subsection{Studying the Theorem~\ref{thm:easy_renyi_dp} Results}
\label{ssec:eval_easy_renyi}

Here we compute the $(\alpha,\epsilon)$-R\'enyi-DP guarantee given by Theorem~\ref{thm:easy_renyi_dp} and plot results analogous to Section~\ref{ssec:eps_delta_experiments}. We only evaluate the sum using the guarantee of Theorem~\ref{thm:easy_renyi_dp}, as in this case $\Delta_{U,x^*} = \Delta_{U,x^*}(X_B)$ is a constant. For the mean update rule, it is not clear how to get reliable estimates of $\Delta_{U,x^*}$, which is a supremum over mini-batches.

\paragraph{Comparison to Baseline.} We see in Figures~\ref{fig:renyi_simple_eps_distrib_bs_mnist_sum}, \ref{fig:renyi_simple_eps_distrib_bs_cifar_sum} that our analysis applied to sum beats the baseline analysis of the R\'enyi-DP guarantee. In particular, we see that for checkpoints in the middle and end of training, we give a privacy guarante several magnitudes better than the baseline for most datapoints. Considering the effect of the expected mini-batch size, we see both our and the baseline's guarantee increases. However, we see that the datapoints with the most privacy lose several magnitudes of privacy as we increase the expected batch size; that is, all the points get concentrated at more similar privacy guarantees (still magnitudes below the baseline). Considering how our guarantee scales with $\alpha$,  in Figure~\ref{fig:renyi_simple_eps_curve_alpha_mnist_sum},\ref{fig:renyi_simple_eps_curve_alpha_cifar_sum} we see both our and the baseline's guarantees increase. Comparing the relative change, in Figure~\ref{fig:renyi_simple_fraction_eps_curve_alpha_mnist_sum},\ref{fig:renyi_simple_fraction_eps_curve_alpha_cifar_sum} we see we scale proportionally to the baseline with varying $\alpha$.

\paragraph{Disparity between Correct and Incorrect.}  Shown in Figures~\ref{fig:renyi_simple_fraction_curve_vary_arch_mnist} we see that correctly classified data points on average have better per-step privacy guarantees across different architectures.

\paragraph{Expected Guarantees for Composition.} In Figure~\ref{fig:renyi_simple_composition_mnist_sum} we plot the expected guarantee over 10 trials at different steps of training (starting from the same checkpoint for each trial) where training was done with the full training dataset. We are evaluating the guarantee for 100 test points, hence computing the expectations needed to bound $D_{\alpha}(X||X\cup \{x_{\text{test}}\})$ according to Theorem~\ref{thm:better_composition}. We find that our expected guarantee decreases with respect to the baseline guarantee as we progress through training, and this persists regardless of the epsilons used during training (also shown in Figure~\ref{fig:renyi_simple_fraction_eps_curve_vary_eps_mnist_sum}). 

In Figure~\ref{fig:remove_10points} we plot the expected guarantees when generating checkpoints without training on the given datapoint: the max of this quantity and the quantity in Figure~\ref{fig:renyi_simple_composition_mnist_sum} bounds the data-dependent R\'enyi-DP guarantee by Theorem~\ref{thm:better_composition}.

\begin{figure}[t]
\centering
\subfloat[Mini-batch size = 16]
{
\includegraphics[width=0.48\linewidth]{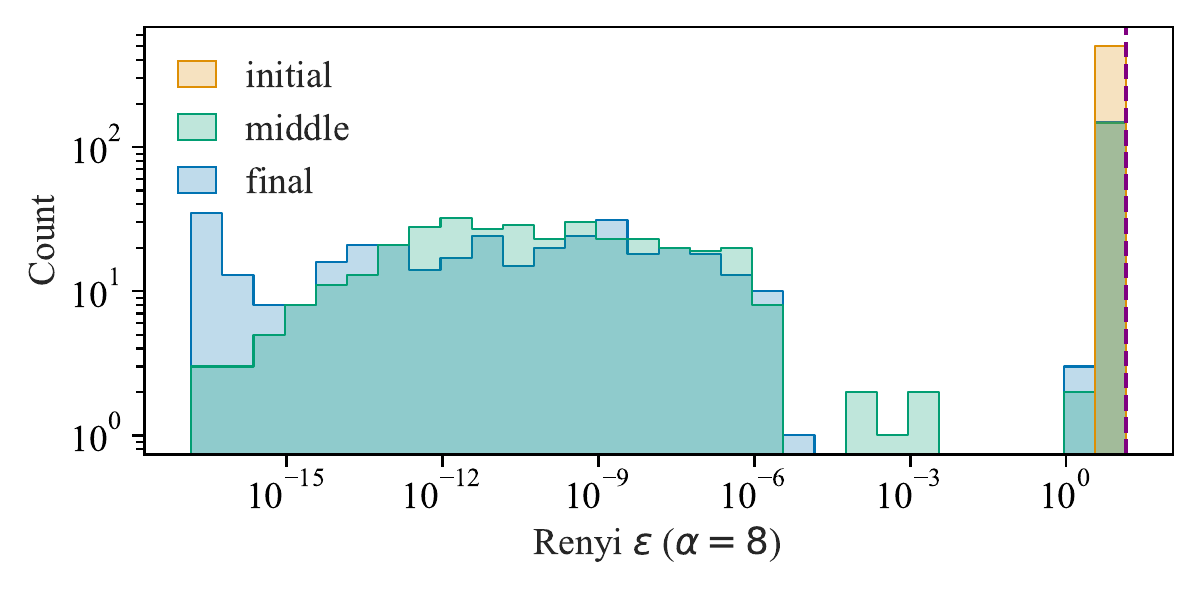}
}
\subfloat[Mini-batch size = 32]
{
\includegraphics[width=0.48\linewidth]{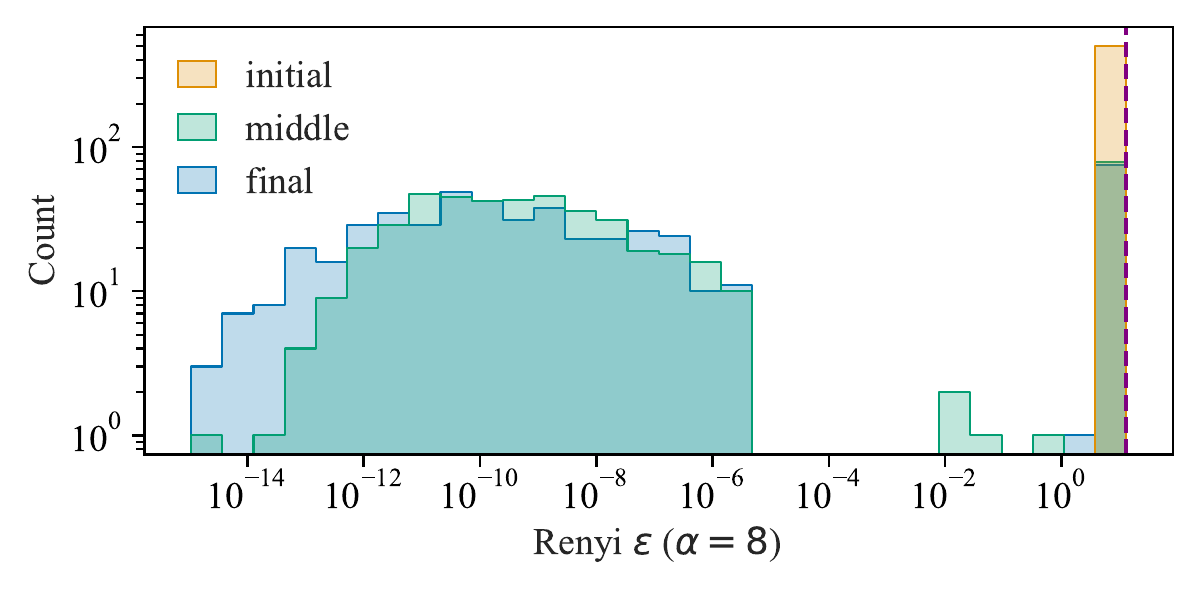}
}
\\\subfloat[Mini-batch size = 64]
{
\includegraphics[width=0.48\linewidth]{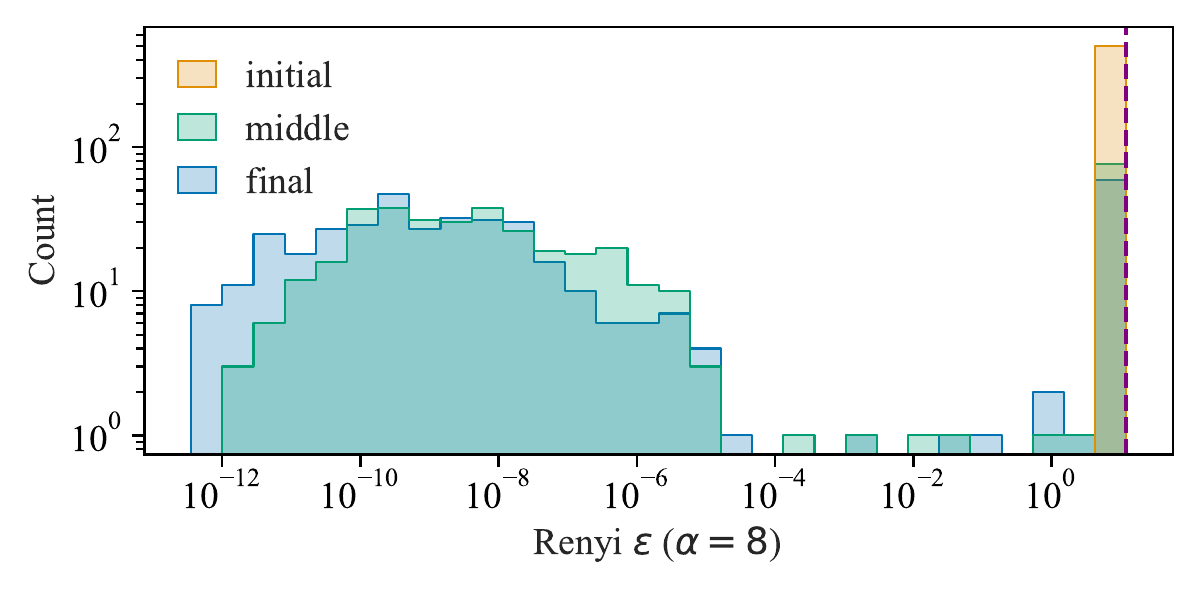}
}
\subfloat[Mini-batch size = 128]
{
\includegraphics[width=0.48\linewidth]{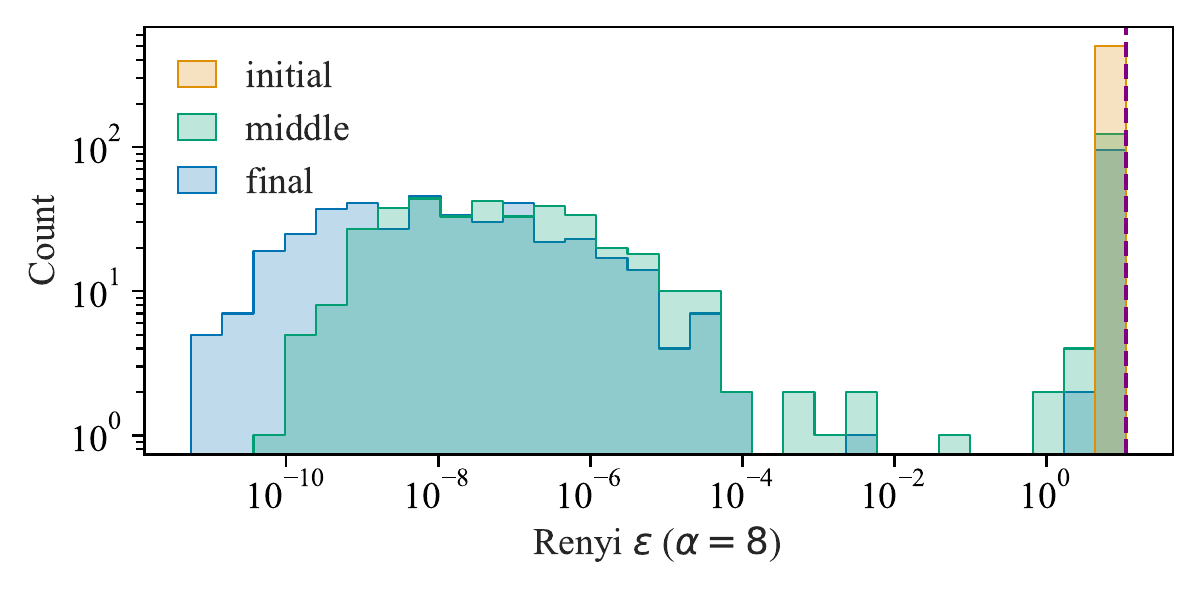}
}
\\\subfloat[Mini-batch size = 256]
{
\includegraphics[width=0.48\linewidth]{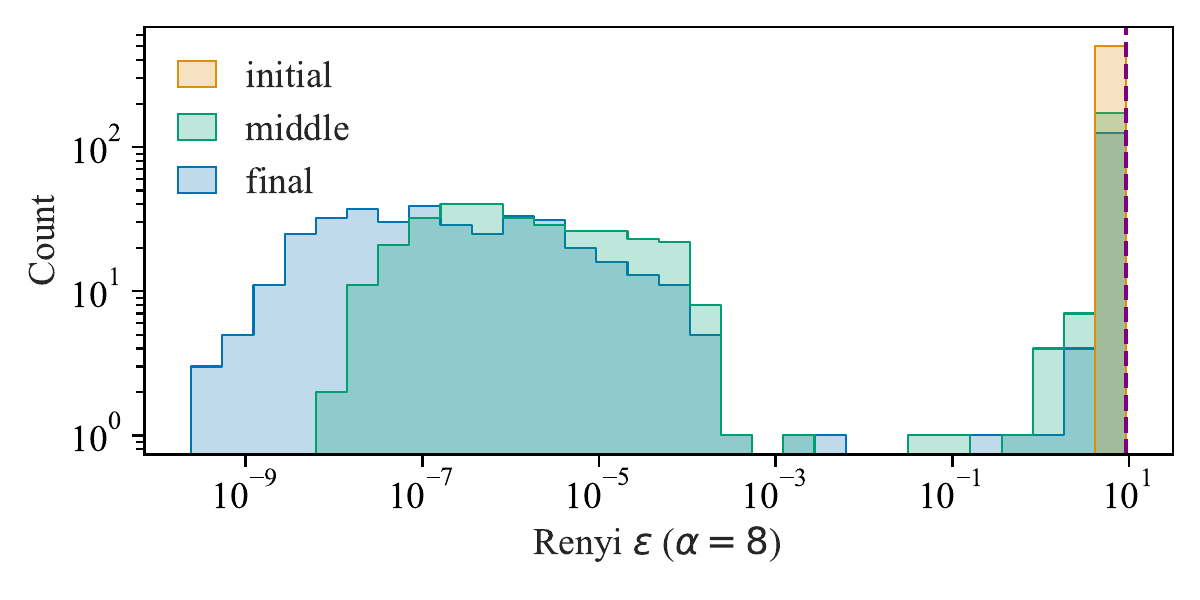}
}
\subfloat[Mini-batch size = 512]
{
\includegraphics[width=0.48\linewidth]{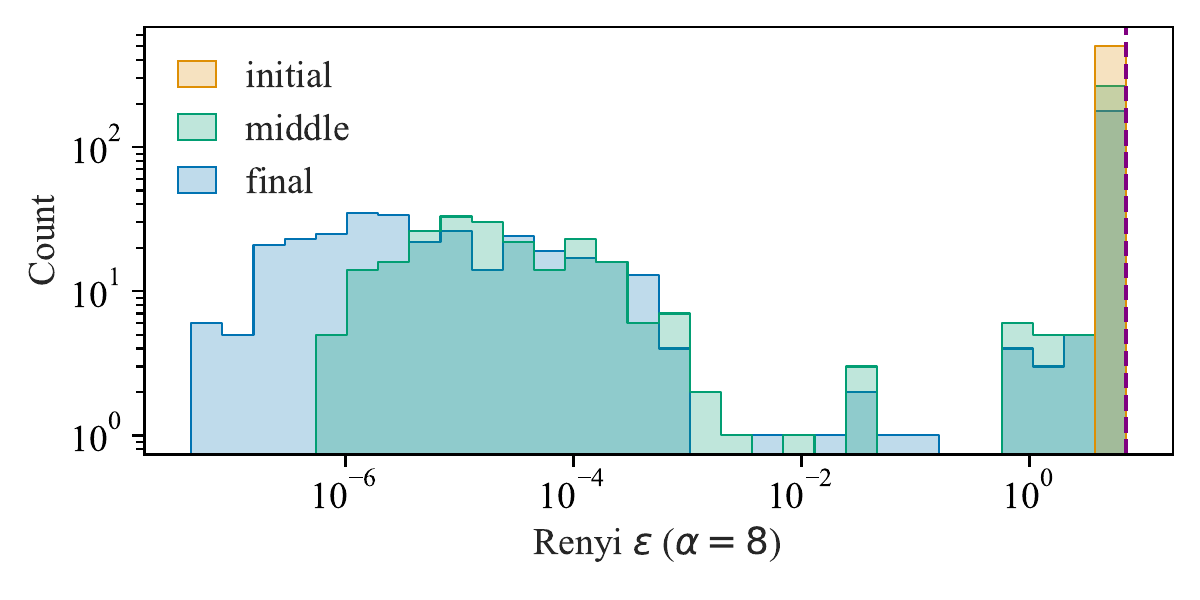}
}
\caption{ Distribution plots of per-step R\'enyi-DP guarantee given by Theorem~\ref{thm:easy_renyi_dp} computed on LeNet-5 trained on MNIST with sum update rule and varying mini-batch sizes of $16, 32, 64, 128, 256, 512$ at three different training stages: at initialization,  in the middle of training, and after training is finalized. The purple dashed line represents the baseline per-step DP-SGD guarantee (Section 3.3 in~\citet{mironov2019r}). We can see our guarantees computed at just initialized models are mostly at the baseline, whereas as training proceeds, more data points obtain guarantees that are better than the baseline by orders of magnitudes. This impact of training stage holds across different mini-batch sizes, while at the same time, the magnitudes of the guarantees increase as the mini-batch size increases--this may be caused by the increasing sampling rate.
}
\label{fig:renyi_simple_eps_distrib_bs_mnist_sum}
\end{figure}

\begin{figure}[t]
\centering
\subfloat[Mini-batch size = 16]
{
\includegraphics[width=0.48\linewidth]{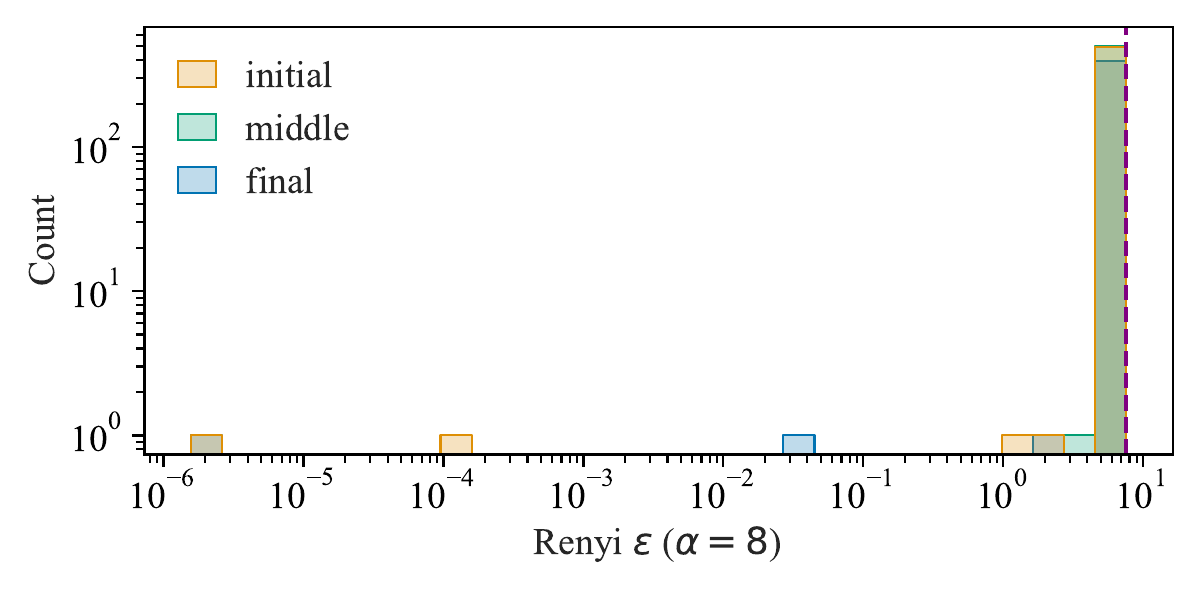}
}
\subfloat[Mini-batch size = 32]
{
\includegraphics[width=0.48\linewidth]{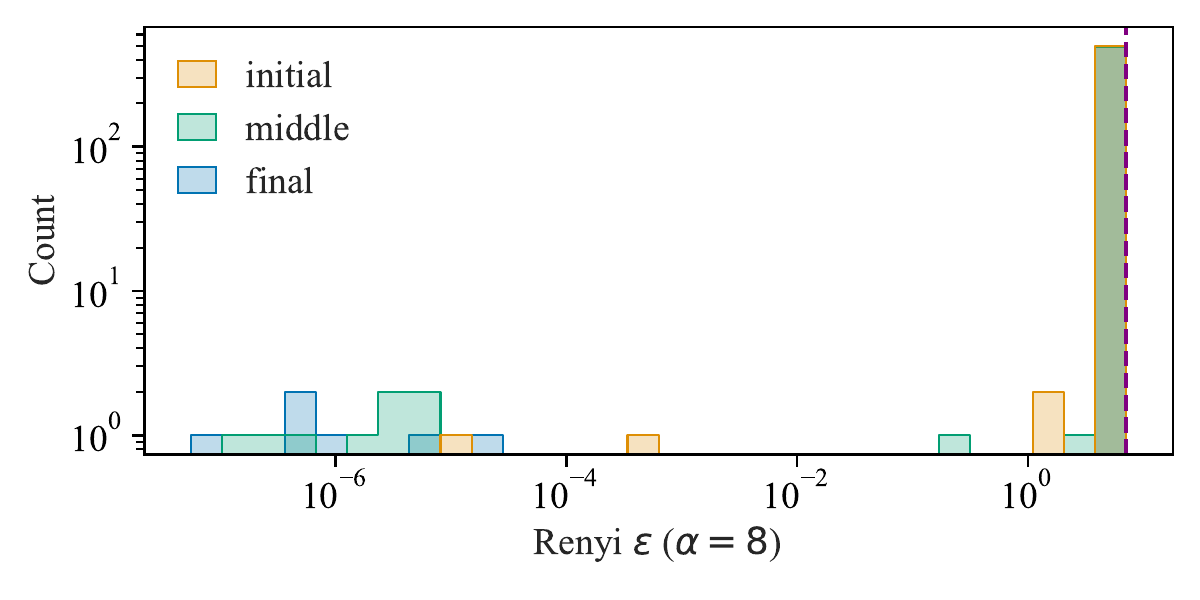}
}
\\\subfloat[Mini-batch size = 64]
{
\includegraphics[width=0.48\linewidth]{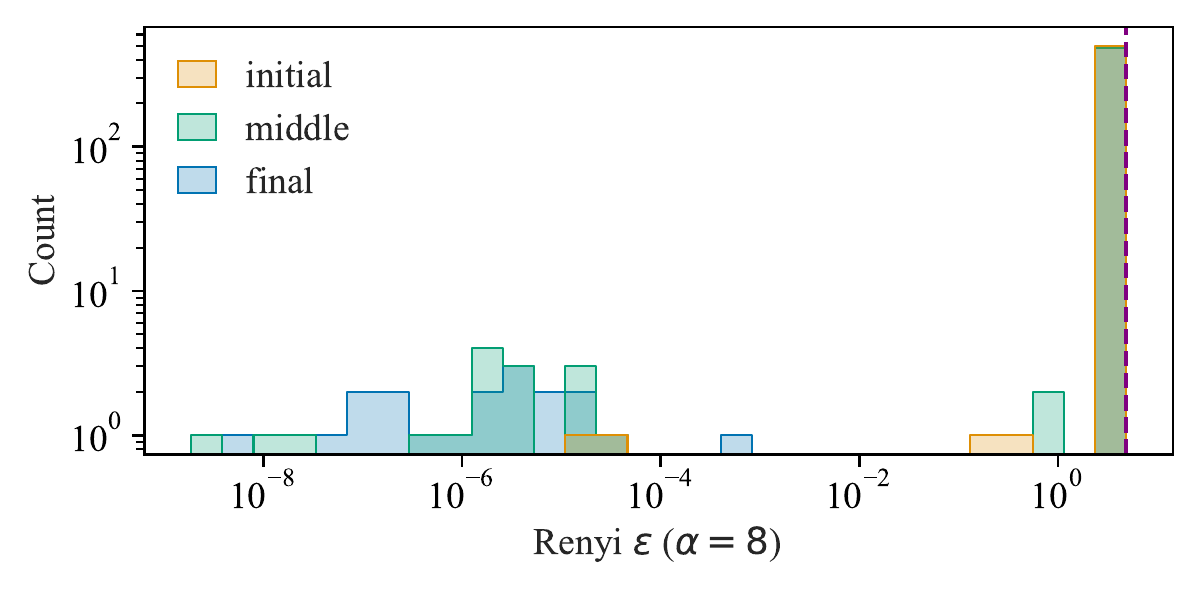}
}
\subfloat[Mini-batch size = 128]
{
\includegraphics[width=0.48\linewidth]{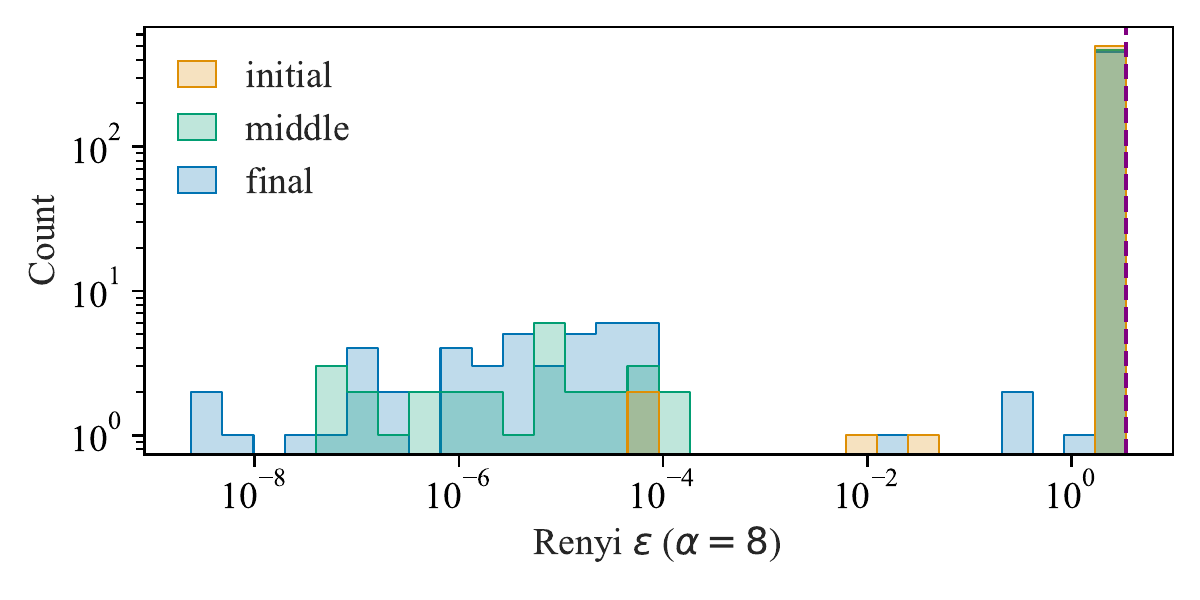}
}
\\\subfloat[Mini-batch size = 256]
{
\includegraphics[width=0.48\linewidth]{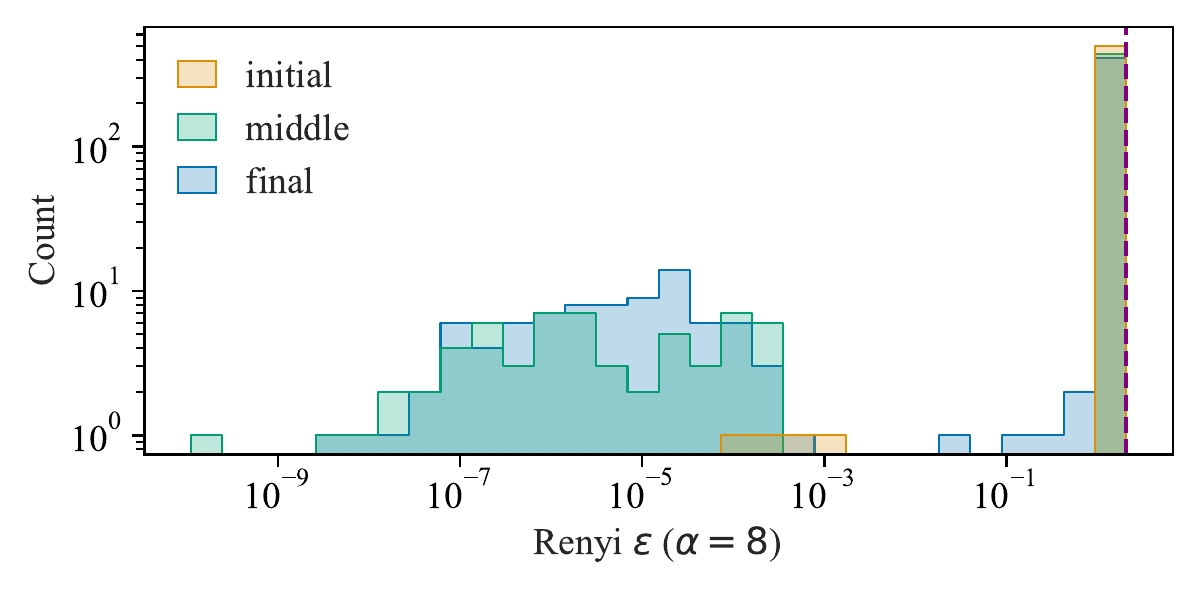}
}
\caption{ This is the reproduction of Figure~\ref{fig:renyi_simple_eps_distrib_bs_mnist_sum} except we now use ResNet-20 models trained on CIFAR-10. We continue to see that some data points are able to obtain significantly better guarantees than the baseline at later stages of training. However, it seems varying mini-batch size does not lead to worse guarantees (and even improves the guarantees for some points). We hypothesize that this may be due to the fact that increasing the mini-batch size for training on CIFAR-10 has a greater impact on model accuracy, which cancels out or even dominates the impact of increasing sampling rate.
}
\label{fig:renyi_simple_eps_distrib_bs_cifar_sum}
\end{figure}

\begin{figure}[t]
\centering
\subfloat[Mini-batch size = 16]
{
\includegraphics[width=0.48\linewidth]{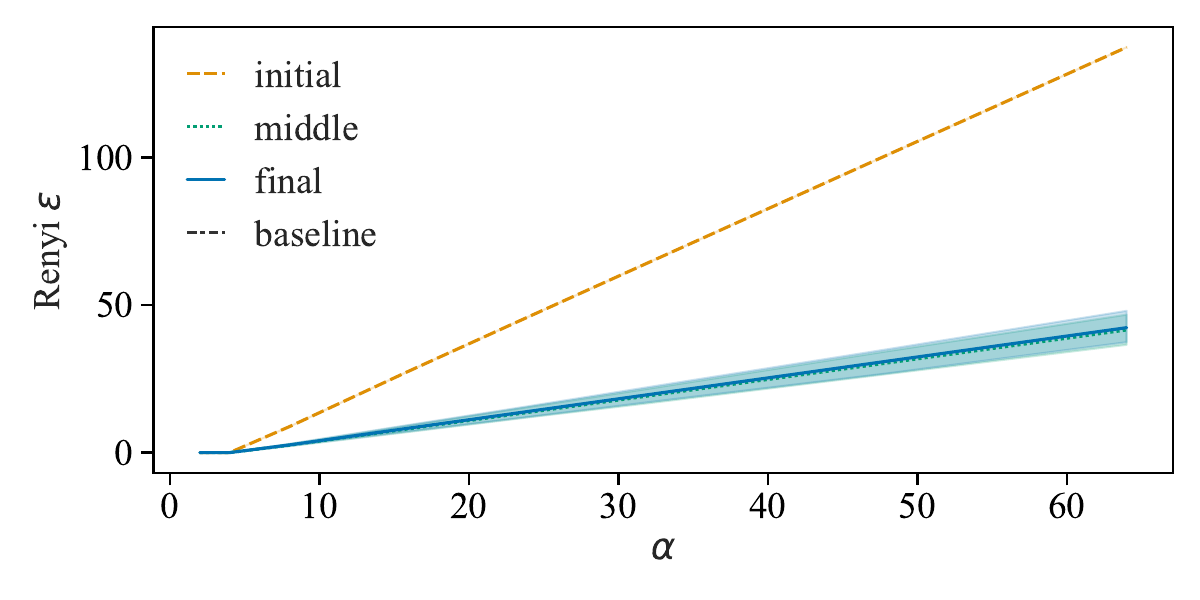}
}
\subfloat[Mini-batch size = 32]
{
\includegraphics[width=0.48\linewidth]{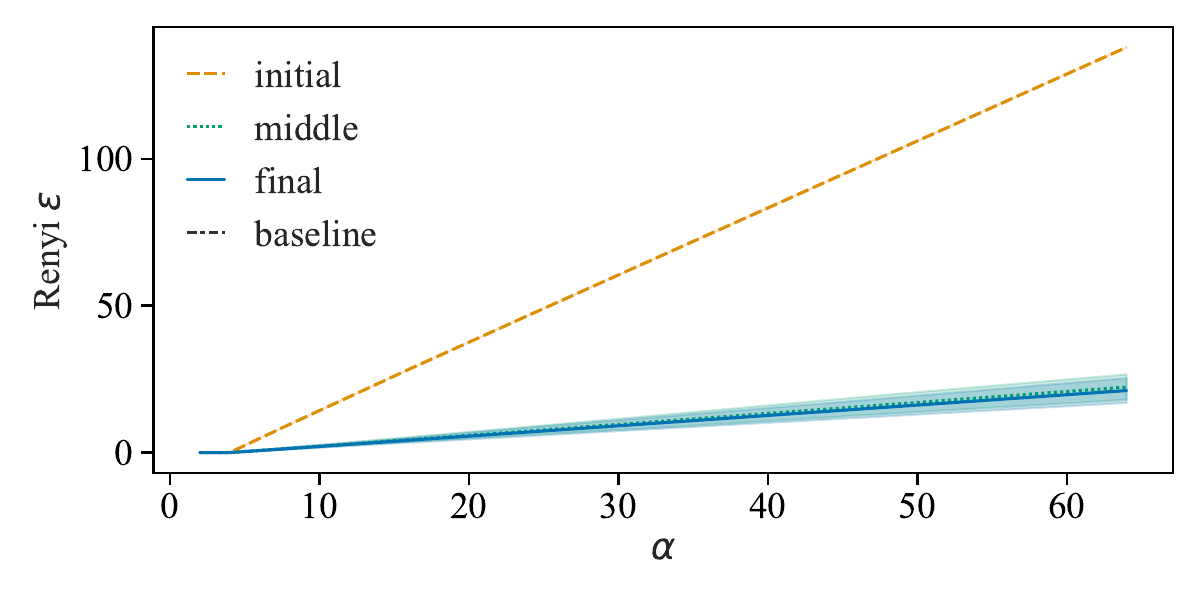}
}
\\\subfloat[Mini-batch size = 64]
{
\includegraphics[width=0.48\linewidth]{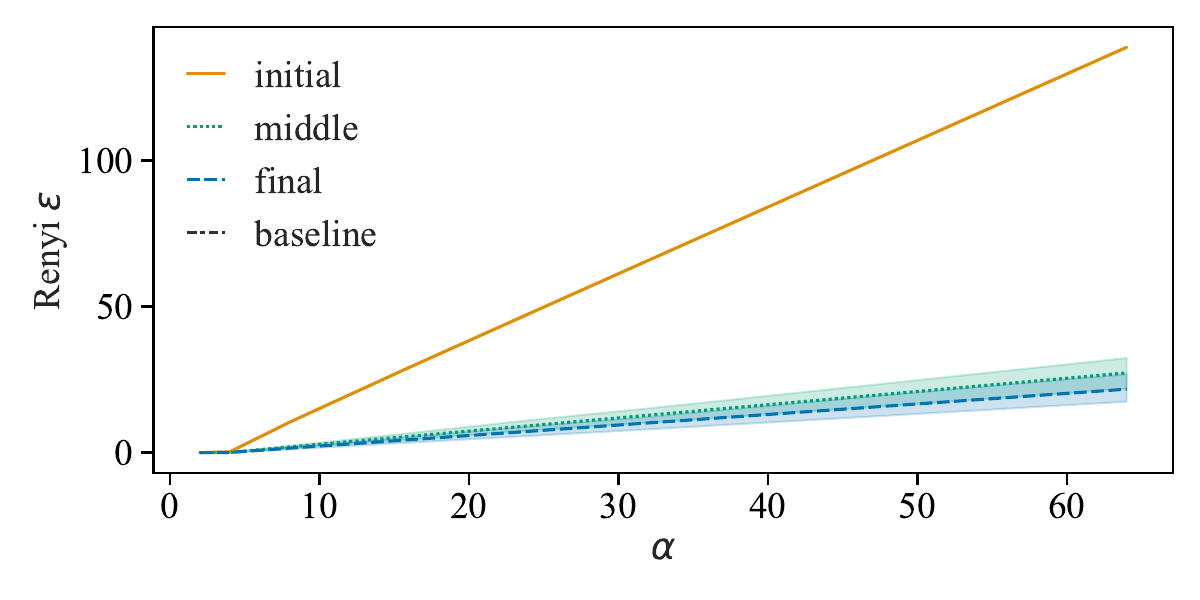}
}
\subfloat[Mini-batch size = 128]
{
\includegraphics[width=0.48\linewidth]{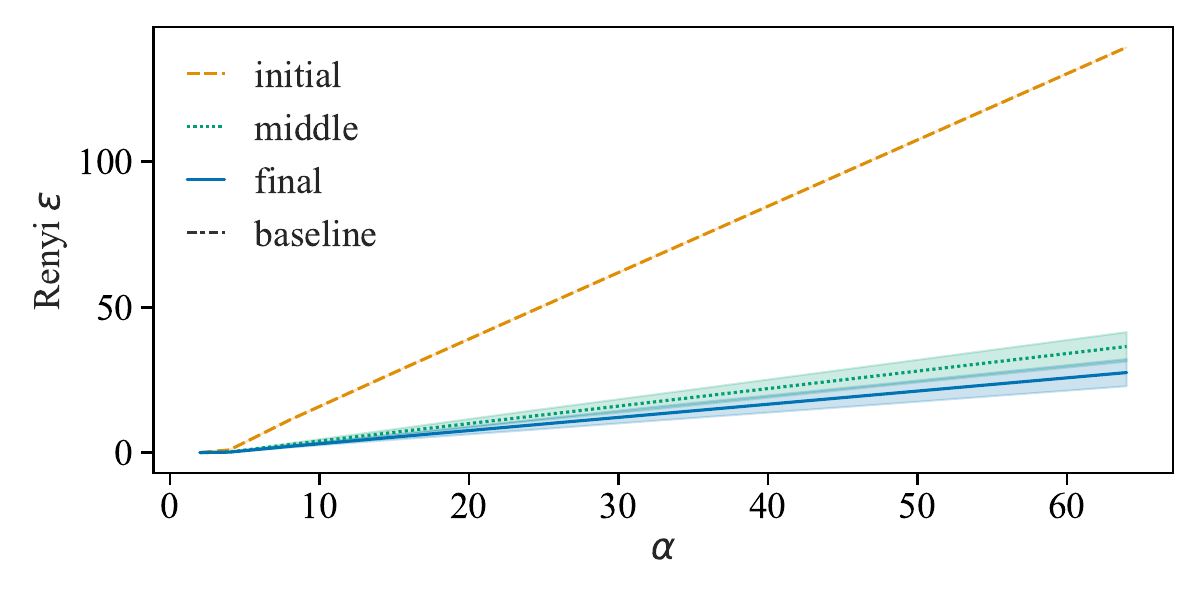}
}
\\\subfloat[Mini-batch size = 256]
{
\includegraphics[width=0.48\linewidth]{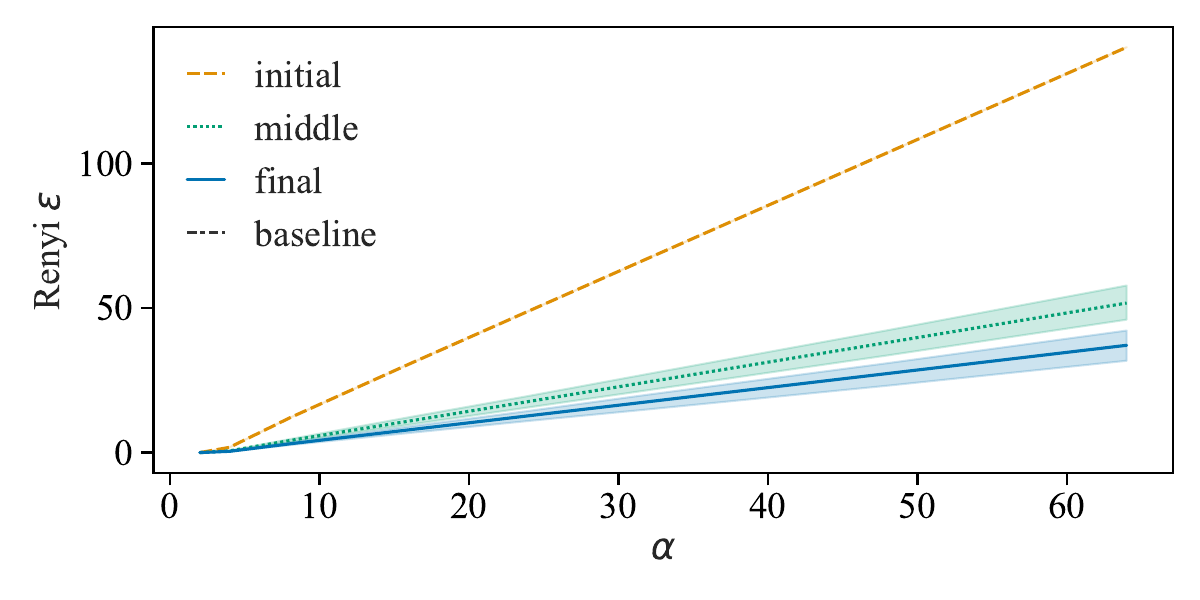}
}
\subfloat[Mini-batch size = 512]
{
\includegraphics[width=0.48\linewidth]{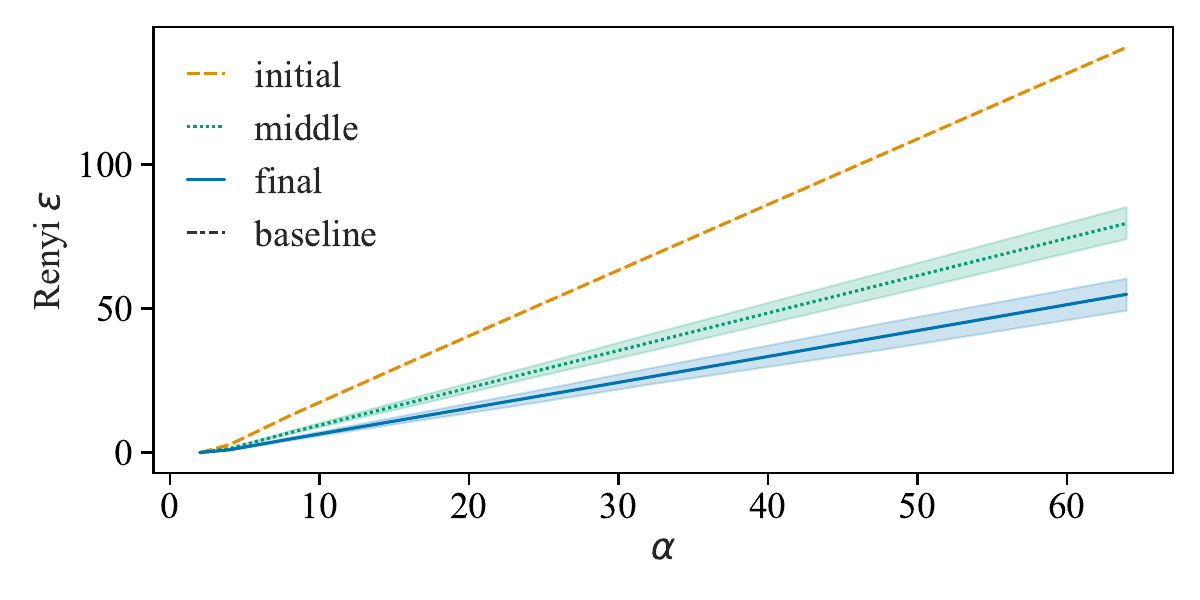}
}
\caption{ Per-step R\'enyi-DP guarantee given by Theorem~\ref{thm:easy_renyi_dp} as a function of $\alpha$, plotted at 3 stages of training and varying mini-batch sizes. It is shown that for all mini-batch sizes, the guarantees at the initial stage overlap with the baseline, whereas the guarantees at the middle and final stages increase slower as $\alpha$ increases. Also see Figure~\ref{fig:renyi_simple_fraction_eps_curve_alpha_mnist_sum} which plots the ratio between our guarantee and the baseline.
}
\label{fig:renyi_simple_eps_curve_alpha_mnist_sum}
\end{figure}

\begin{figure}[t]
\centering
\subfloat[Mini-batch size = 16]
{
\includegraphics[width=0.48\linewidth]{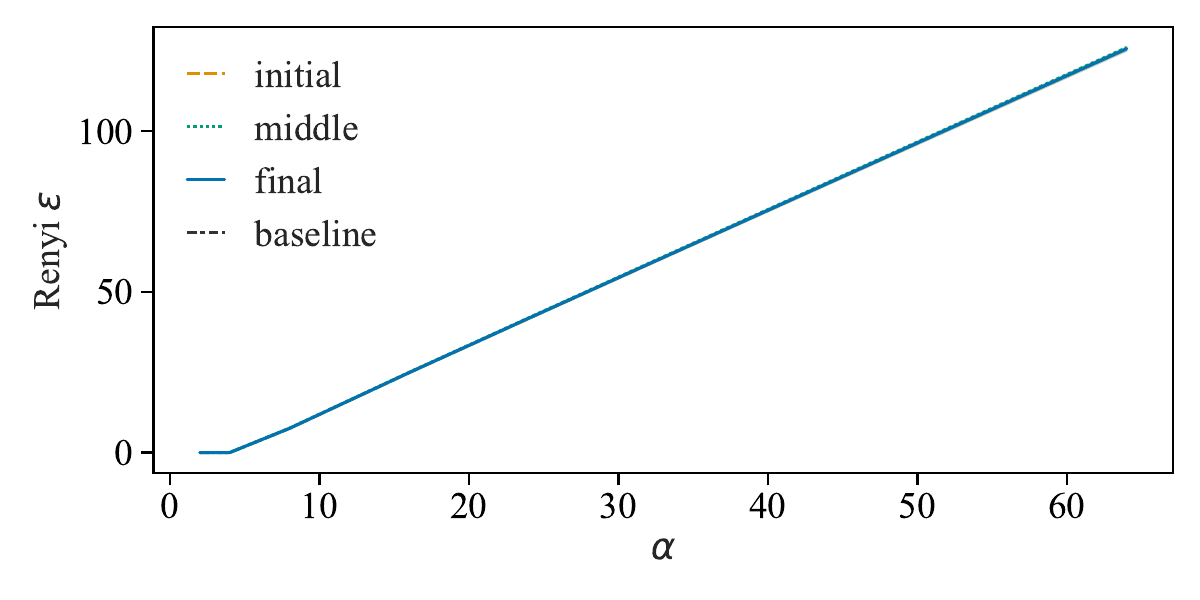}
}
\subfloat[Mini-batch size = 32]
{
\includegraphics[width=0.48\linewidth]{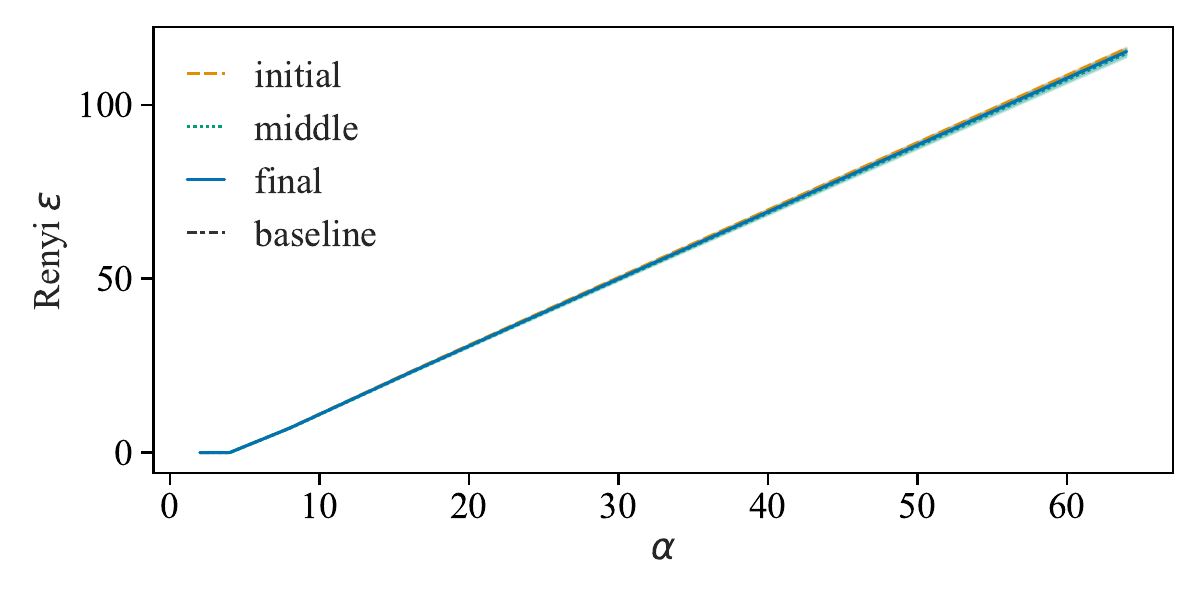}
}
\\\subfloat[Mini-batch size = 64]
{
\includegraphics[width=0.48\linewidth]{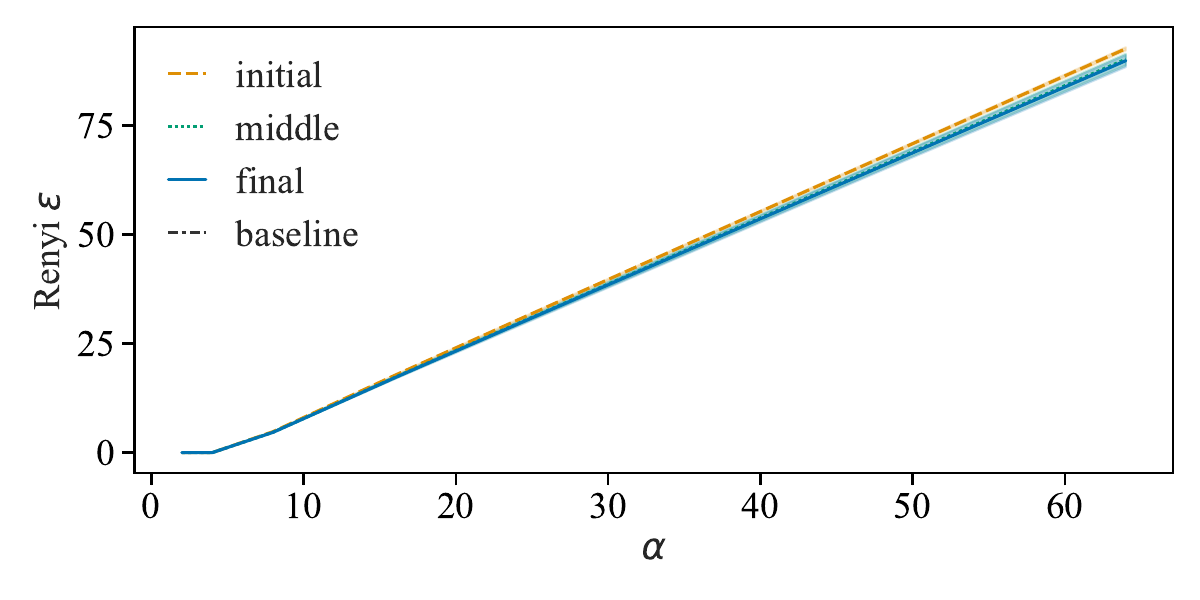}
}
\subfloat[Mini-batch size = 128]
{
\includegraphics[width=0.48\linewidth]{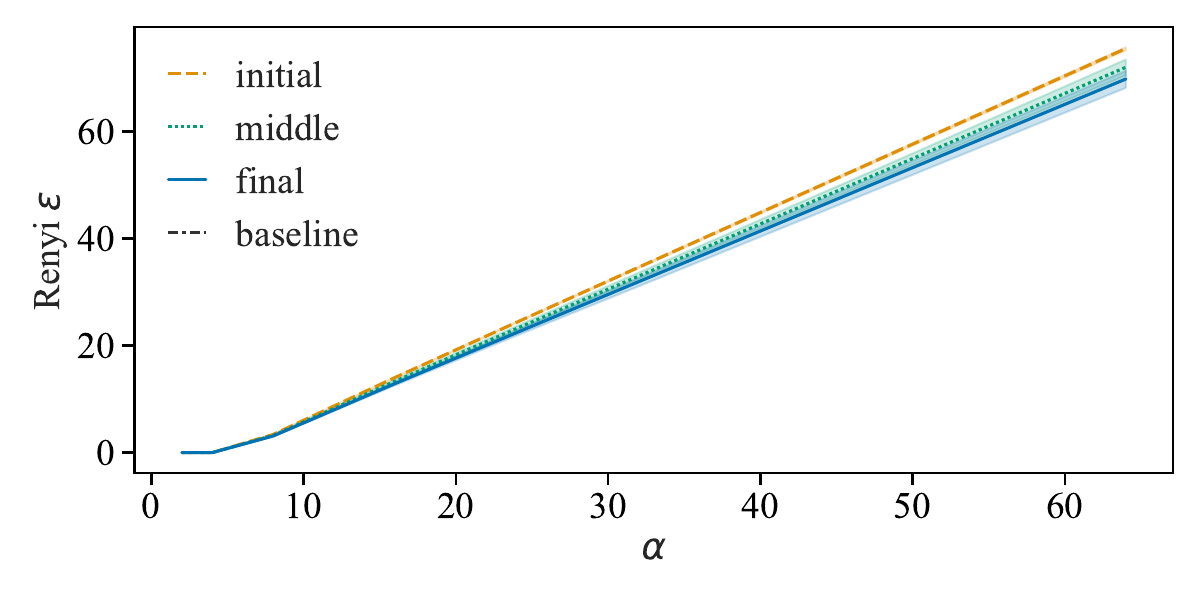}
}
\\\subfloat[Mini-batch size = 256]
{
\includegraphics[width=0.48\linewidth]{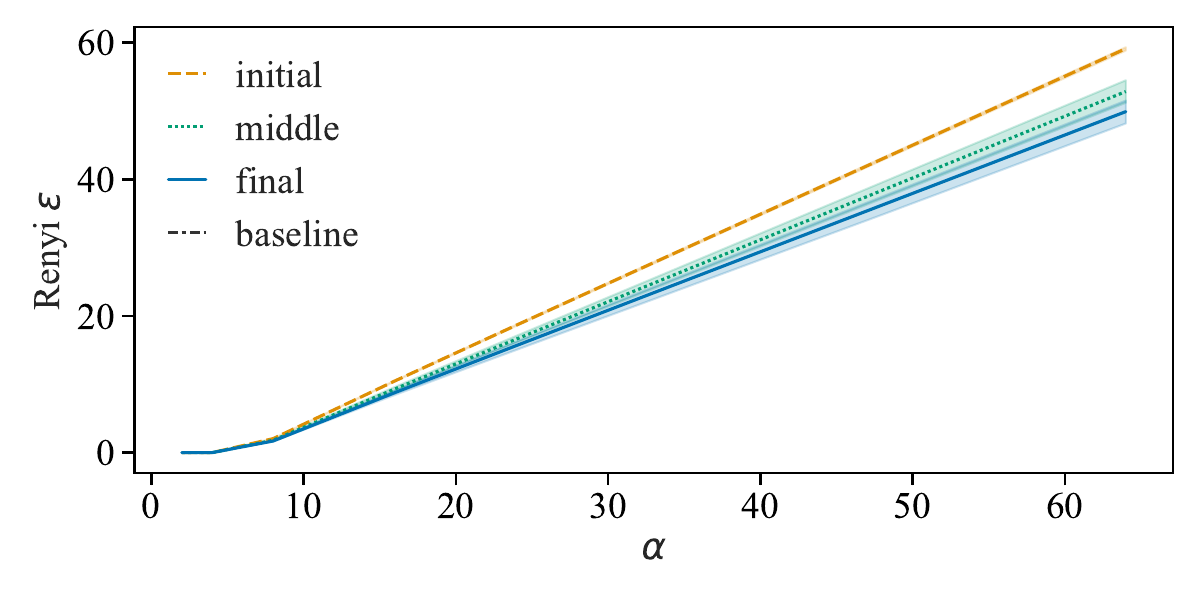}
}
\caption{ This is the reproduction of Figure~\ref{fig:renyi_simple_eps_curve_alpha_mnist_sum} except we now use ResNet-20 models trained on CIFAR-10, and similar results are observed. However, the curves are too close to each other, so for ease of visualization, we plot the ratio between our guarantee and the baseline in Figure~\ref{fig:renyi_simple_fraction_eps_curve_alpha_cifar_sum}.
}
\label{fig:renyi_simple_eps_curve_alpha_cifar_sum}
\end{figure}

\begin{figure}[t]
\centering
\subfloat[Mini-batch size = 16]
{
\includegraphics[width=0.48\linewidth]{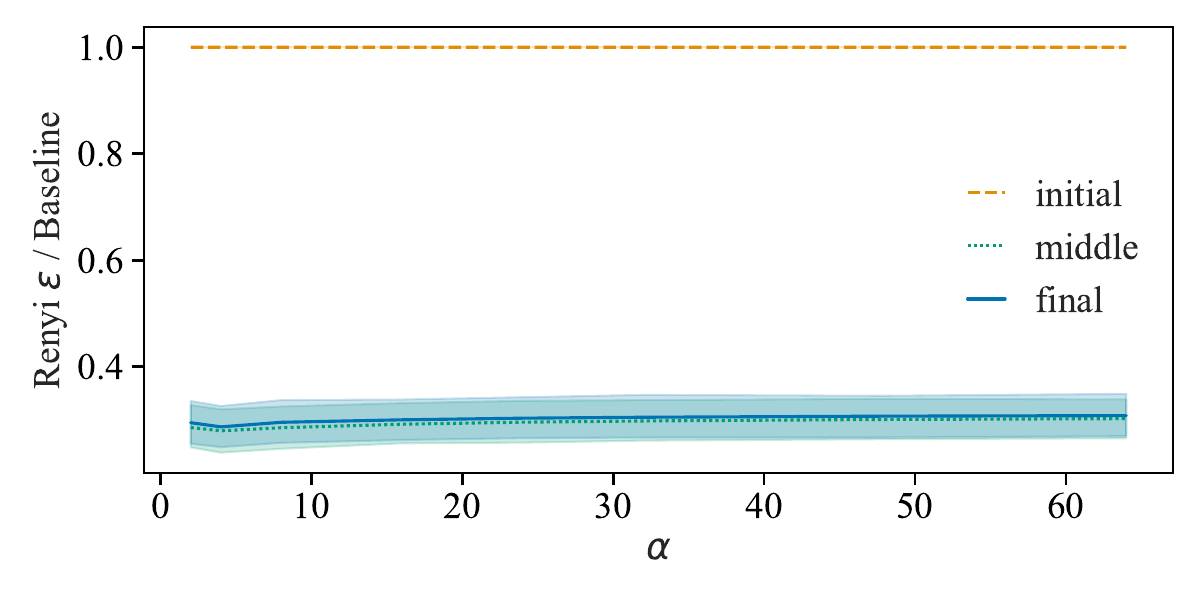}
}
\subfloat[Mini-batch size = 32]
{
\includegraphics[width=0.48\linewidth]{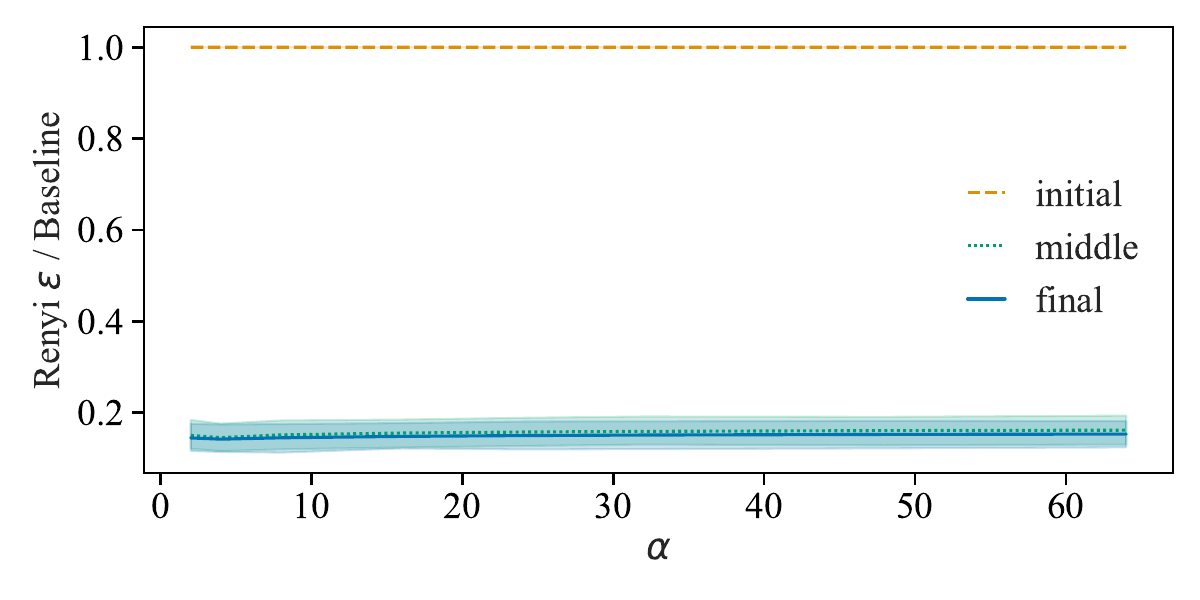}
}
\\\subfloat[Mini-batch size = 64]
{
\includegraphics[width=0.48\linewidth]{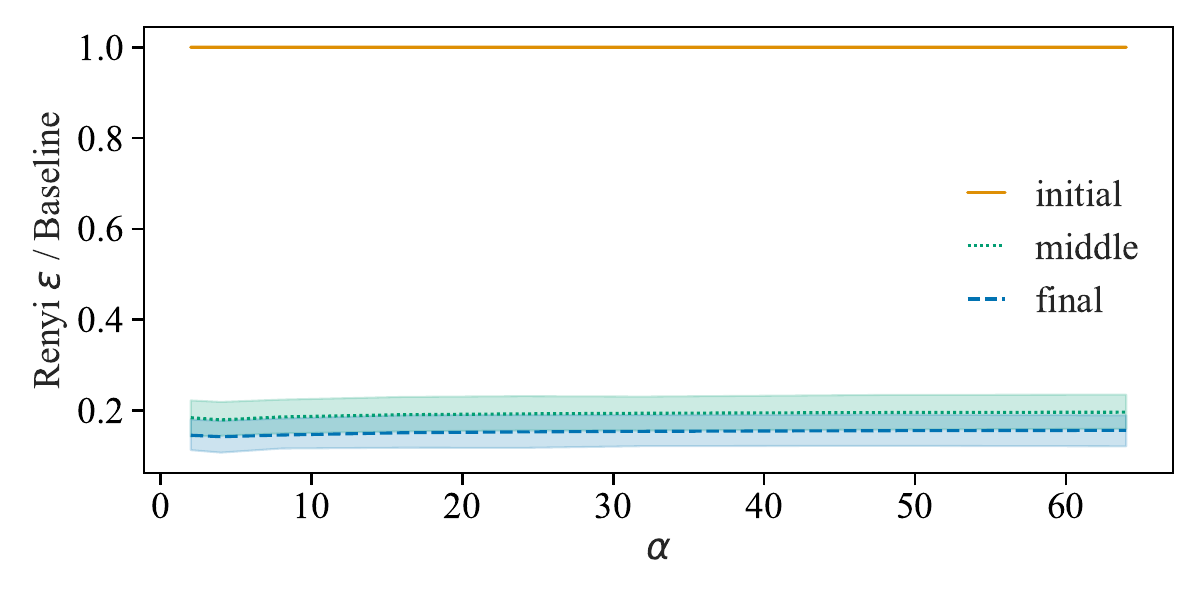}
}
\subfloat[Mini-batch size = 128]
{
\includegraphics[width=0.48\linewidth]{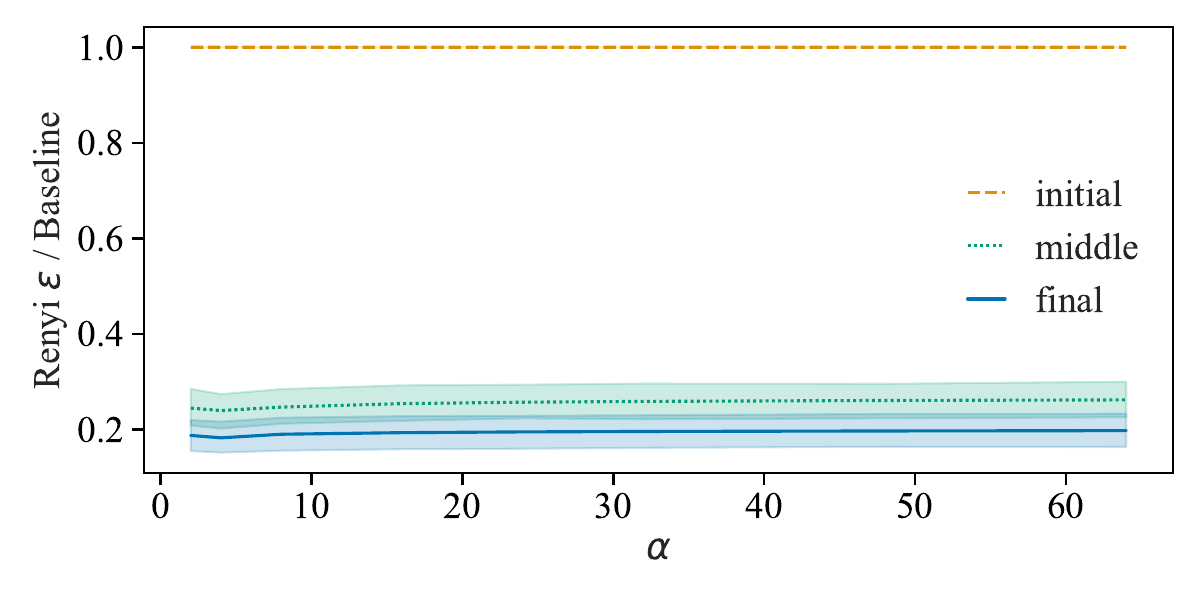}
}
\\\subfloat[Mini-batch size = 256]
{
\includegraphics[width=0.48\linewidth]{figures/renyi_simple_eps_alpha_fraction_curve_MNIST_lenet_128_sum.pdf}
}
\subfloat[Mini-batch size = 512]
{
\includegraphics[width=0.48\linewidth]{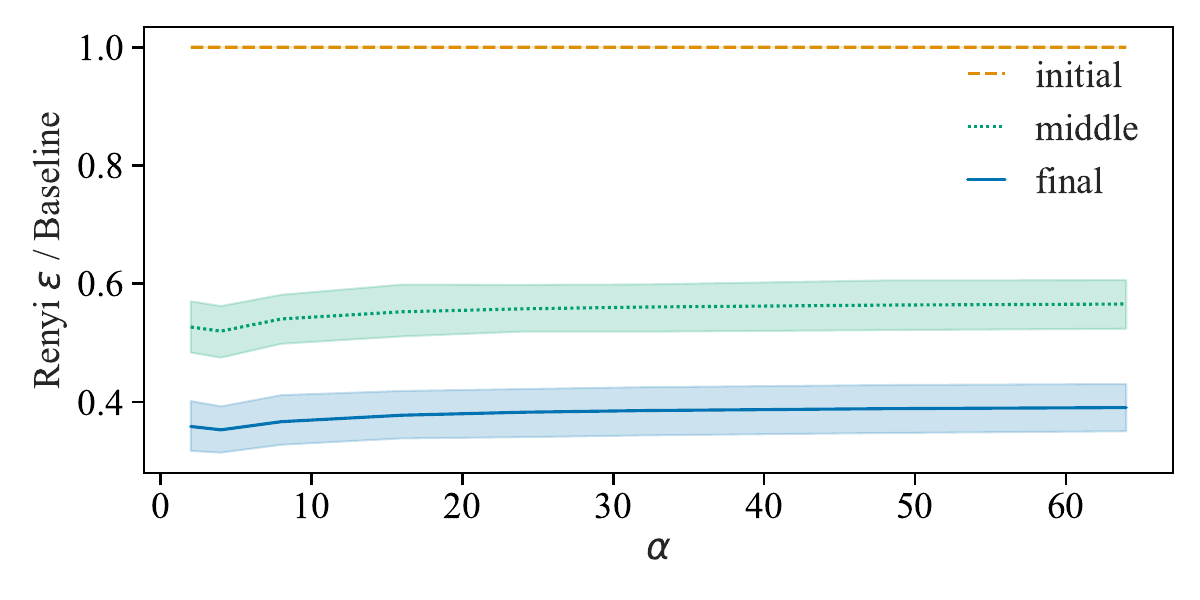}
}
\caption{ This is the exact reproduction of Figure~\ref{fig:renyi_simple_eps_curve_alpha_mnist_sum} except we normalize the y-axis by the baseline guarantee. Besides the takeaway mentioned in Figure~\ref{fig:renyi_simple_eps_curve_alpha_mnist_sum}, we can also see that mini-batch size does not have a significant impact  on how our guarantee changes with respect to $\alpha$ except for mini-batch size $=512$. This exception may be because the model with mini-batch size $512$ was not trained to convergence since the number of training epochs was set to $10$ for all MNIST models.
}
\label{fig:renyi_simple_fraction_eps_curve_alpha_mnist_sum}
\end{figure}

\begin{figure}[t]
\centering
\subfloat[Mini-batch size = 16]
{
\includegraphics[width=0.48\linewidth]{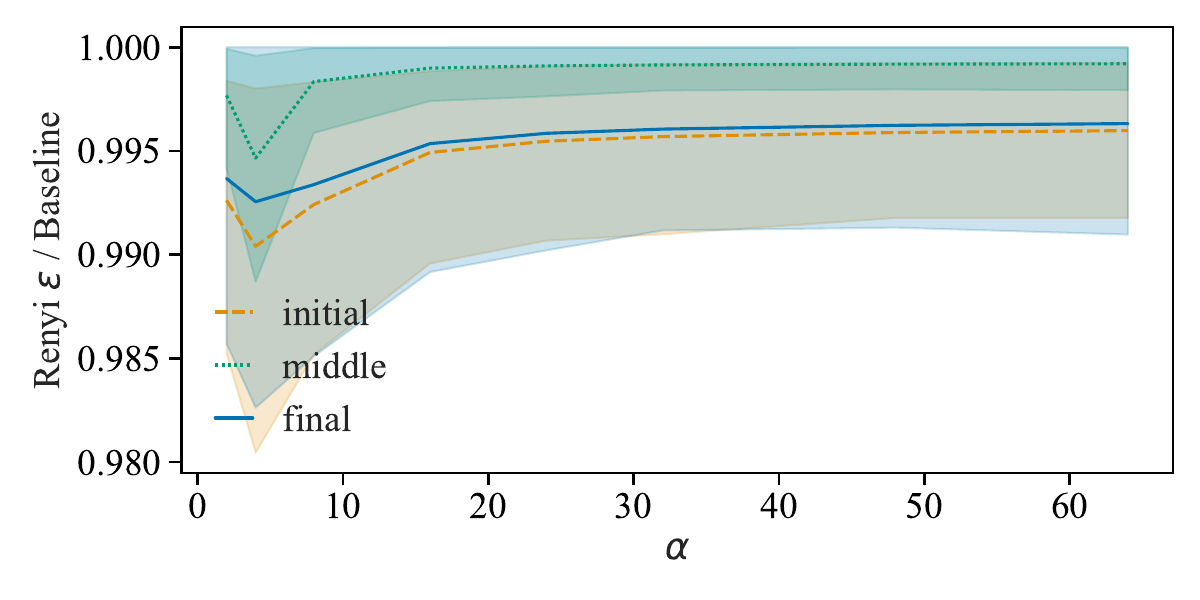}
}
\subfloat[Mini-batch size = 32]
{
\includegraphics[width=0.48\linewidth]{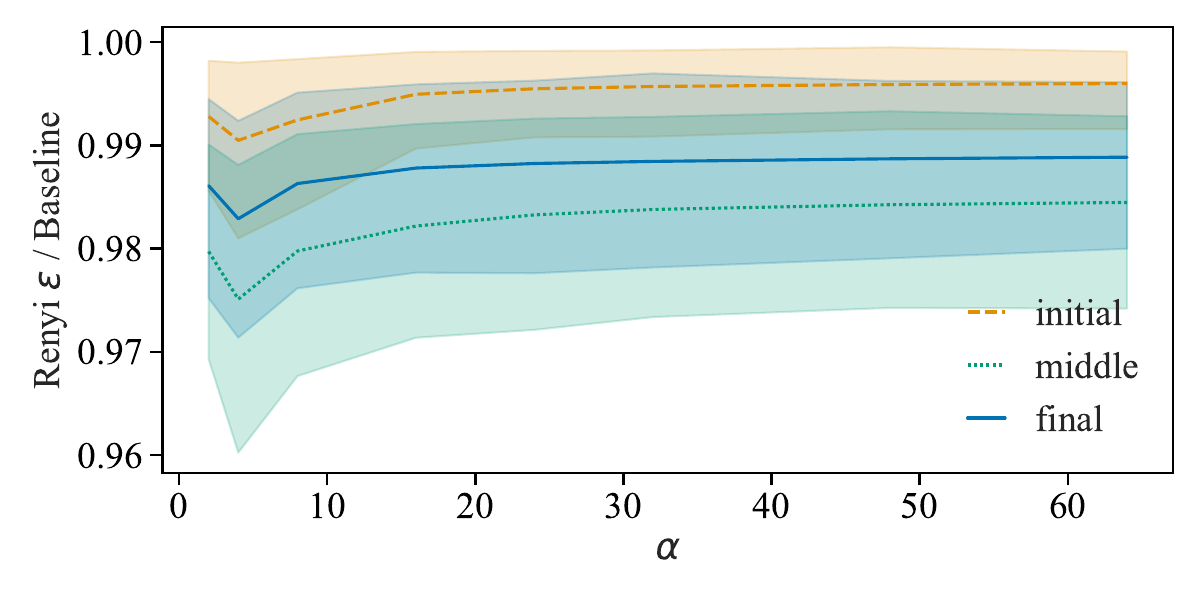}
}
\\\subfloat[Mini-batch size = 64]
{
\includegraphics[width=0.48\linewidth]{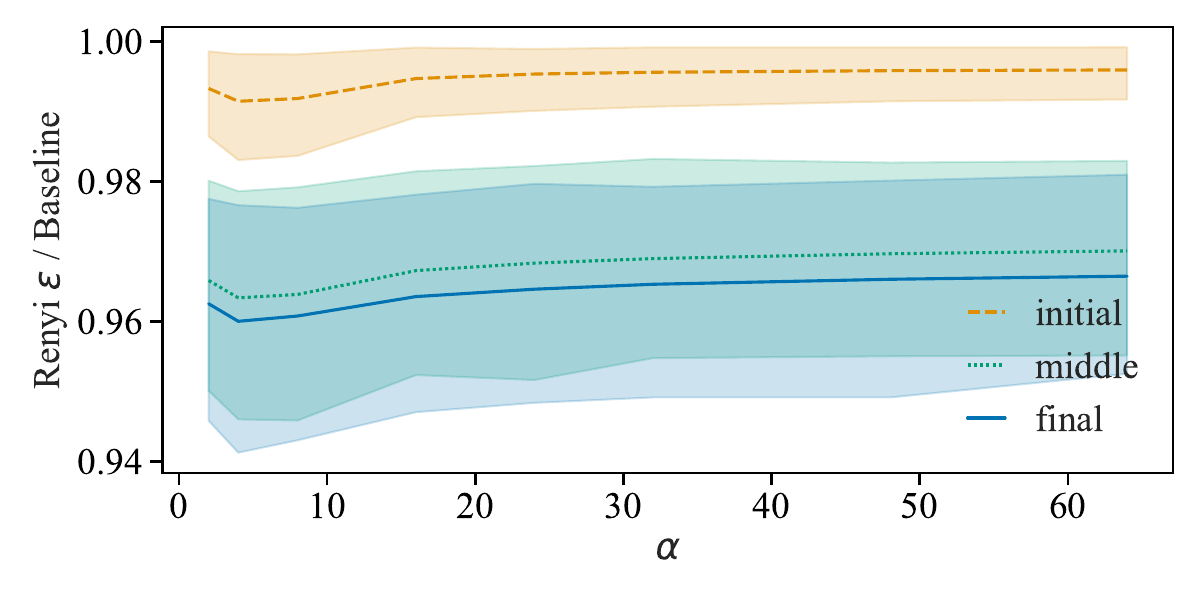}
}
\subfloat[Mini-batch size = 128]
{
\includegraphics[width=0.48\linewidth]{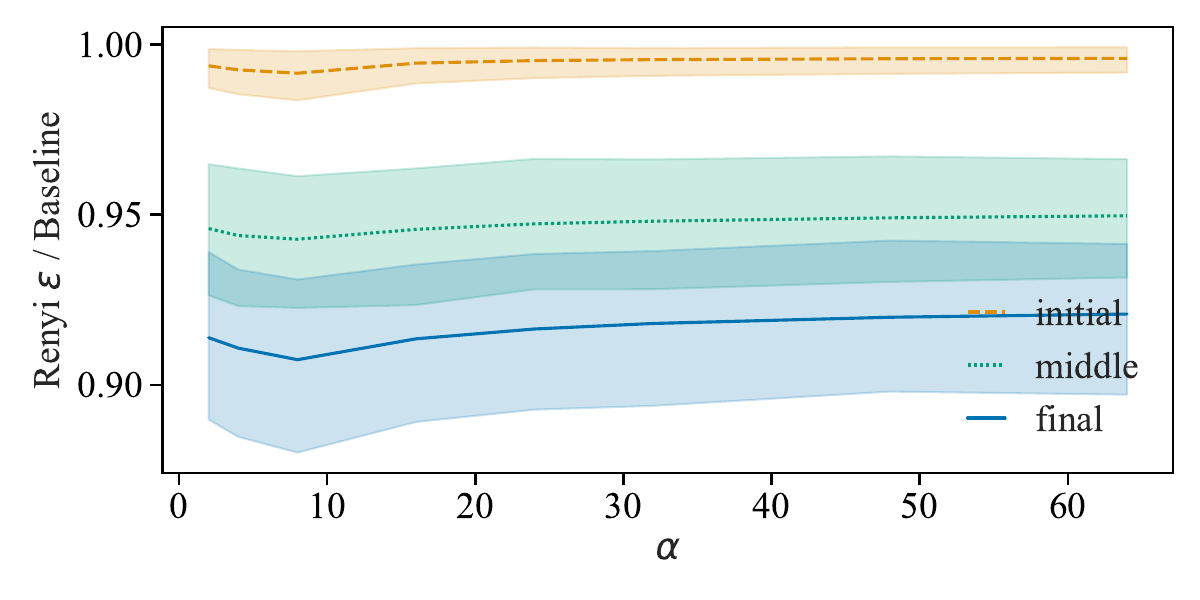}
}
\\\subfloat[Mini-batch size = 256]
{
\includegraphics[width=0.48\linewidth]{figures/renyi_simple_eps_alpha_fraction_curve_CIFAR10_resnet20_128_sum.pdf}
}
\caption{ This is the exact reproduction of Figure~\ref{fig:renyi_simple_eps_curve_alpha_cifar_sum} except we normalize the y-axis by the baseline guarantee. It can be seen now that at later stages of training, our guarantees increases slower comparing to the baseline guarantee.
}
\label{fig:renyi_simple_fraction_eps_curve_alpha_cifar_sum}
\end{figure}

\begin{figure}[t]
\centering
\subfloat[our guarantee]
{
\includegraphics[width=0.48\linewidth]{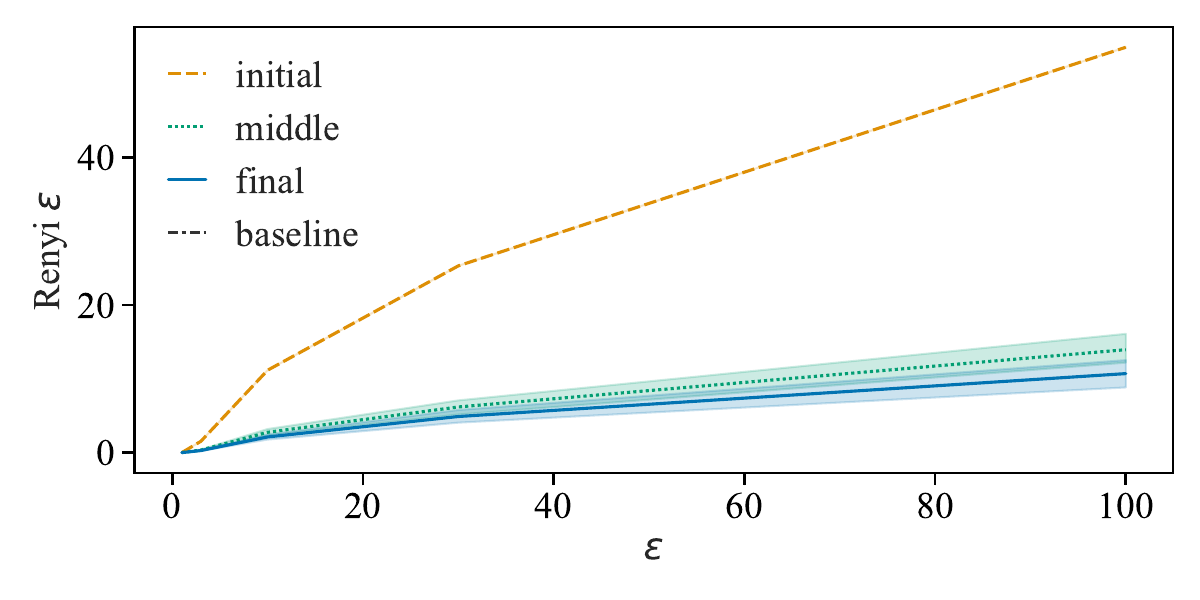}
}
\subfloat[our guarantee divided by baseline]
{
\includegraphics[width=0.48\linewidth]{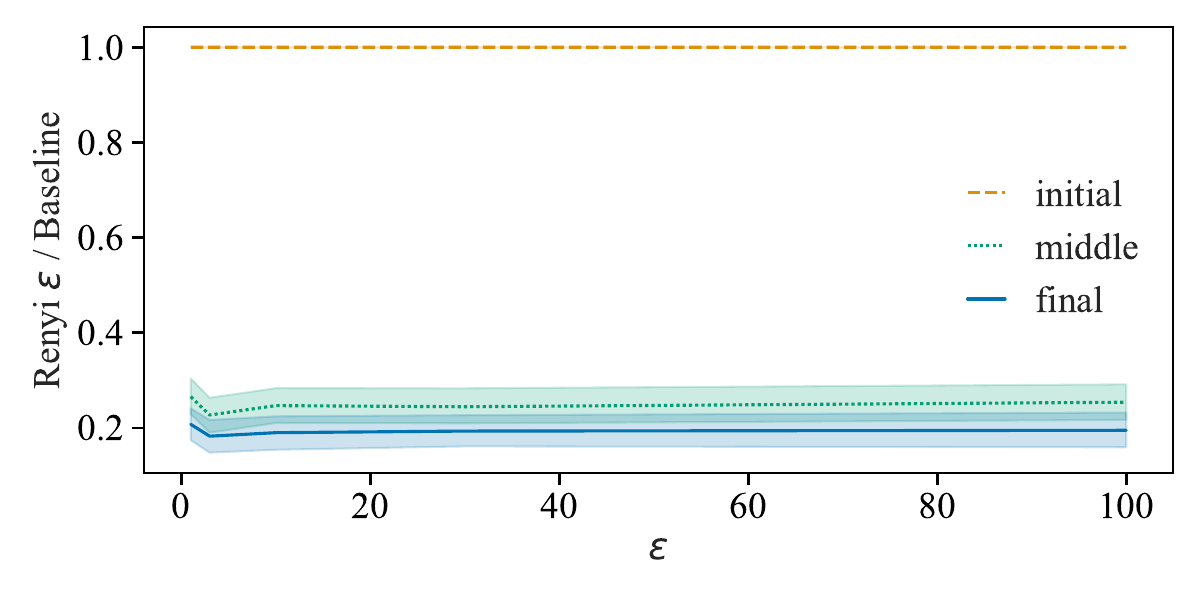}
}

\caption{ Per-step R\'enyi-DP guarantee given by Theorem~\ref{thm:easy_renyi_dp} as a function of $\epsilon$, plotted at 3 stages of training and varying mini-batch sizes. The baseline is either (a) plotted in the figure, or (b) used to normalize the plotted guarantees. We can see that our guarantees increase significantly slower than the baseline as  $\epsilon$ increases.
}
\label{fig:renyi_simple_fraction_eps_curve_vary_eps_mnist_sum}
\end{figure}

\begin{figure}[t]
\centering
\subfloat[our guarantee]
{
\includegraphics[width=0.48\linewidth]{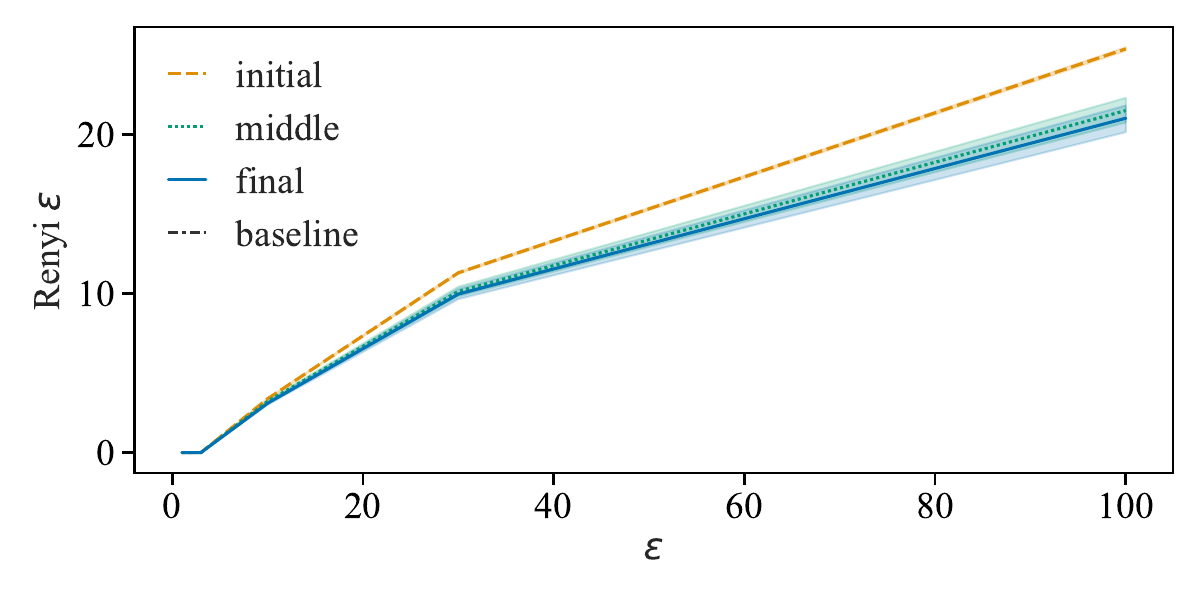}
}
\subfloat[our guarantee divided by baseline]
{
\includegraphics[width=0.48\linewidth]{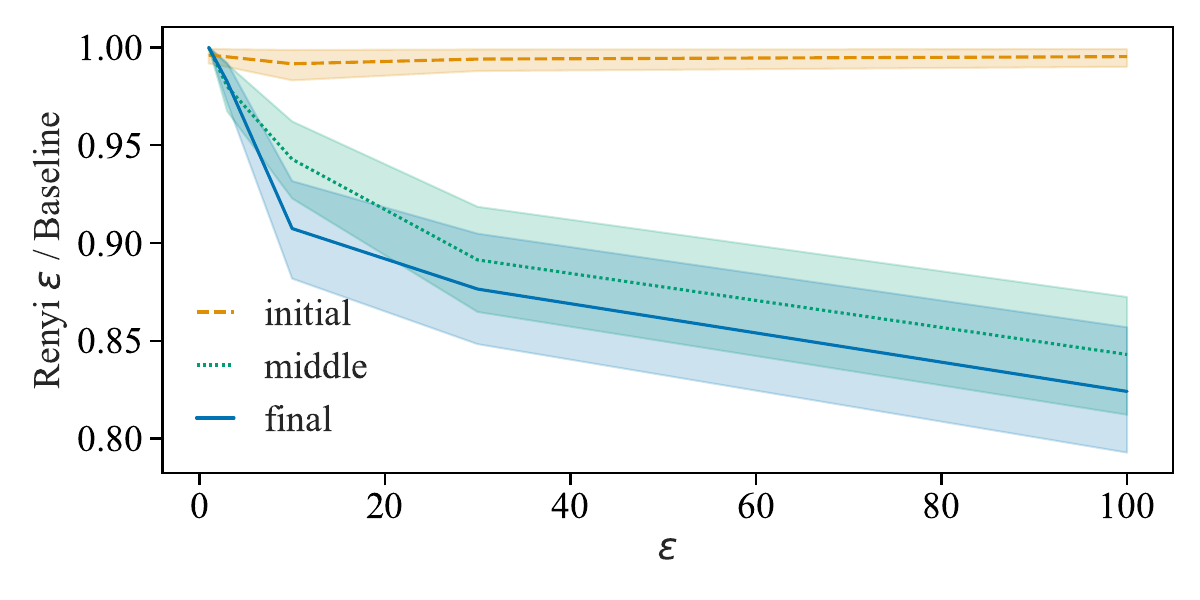}
}

\caption{ This is the reproduction of Figure~\ref{fig:renyi_simple_fraction_eps_curve_vary_eps_mnist_sum} except we now use ResNet-20 models trained on CIFAR-10. Unlike the results of MNIST, we now observe that the ratio between our guarantee and the baseline decreases as $\epsilon$ increases. We hypothesize that this may be because increasing  $\epsilon$ has a much larger impact on accuracy of CIFAR-10 models.
}
\label{fig:renyi_simple_fraction_eps_curve_vary_eps_cifar_sum}
\end{figure}

\begin{figure}[t]
\centering
\subfloat[MNIST]
{
\includegraphics[width=0.48\linewidth]{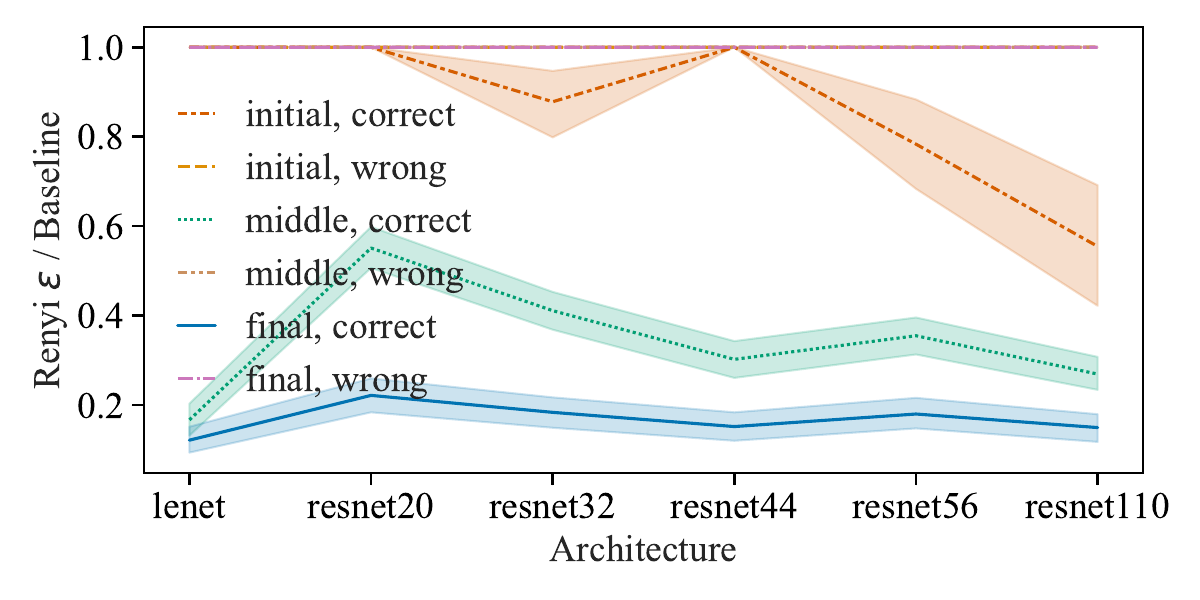}
}
\subfloat[CIFAR10]
{
\includegraphics[width=0.48\linewidth]{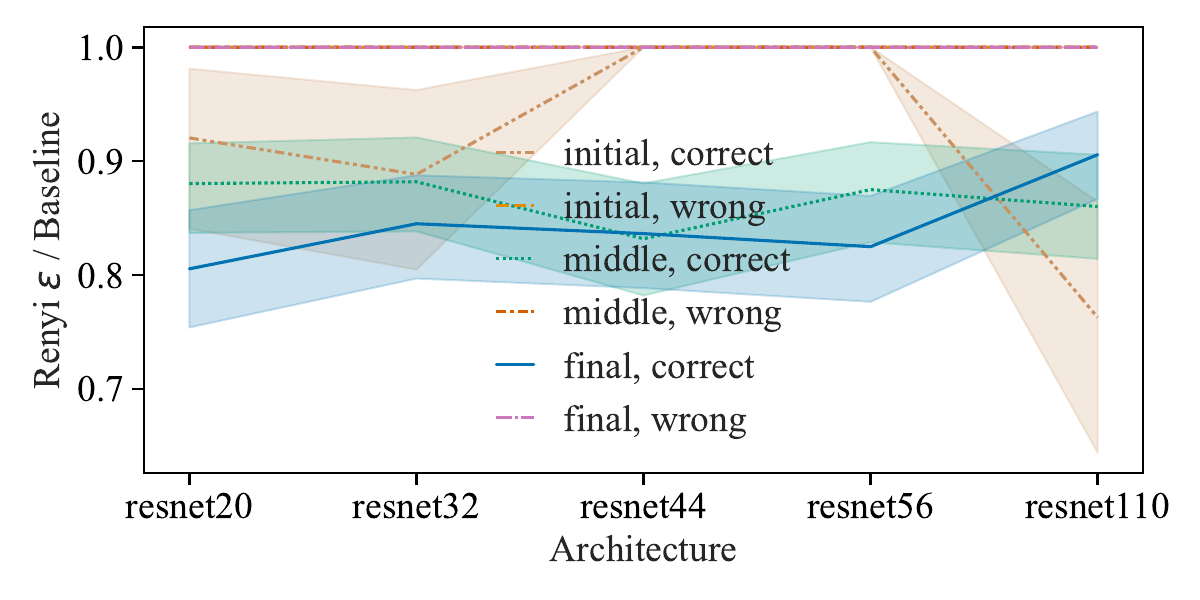}
}
\caption{ Per-step R\'enyi-DP guarantee given by Theorem~\ref{thm:easy_renyi_dp} (divided by the baseline guarantee) plotted with respect to different model architectures trained on MNIST and CIFAR-10, at 3 stages of training. Consistently across different architectures and datasets, data points at later stages of training that are correctly classified have significantly better privacy guarantees than the baseline.
}
\label{fig:renyi_simple_fraction_curve_vary_arch_mnist}
\end{figure}

\begin{figure}[t]
\centering
\subfloat[varying batch size]
{
\includegraphics[width=0.48\linewidth]{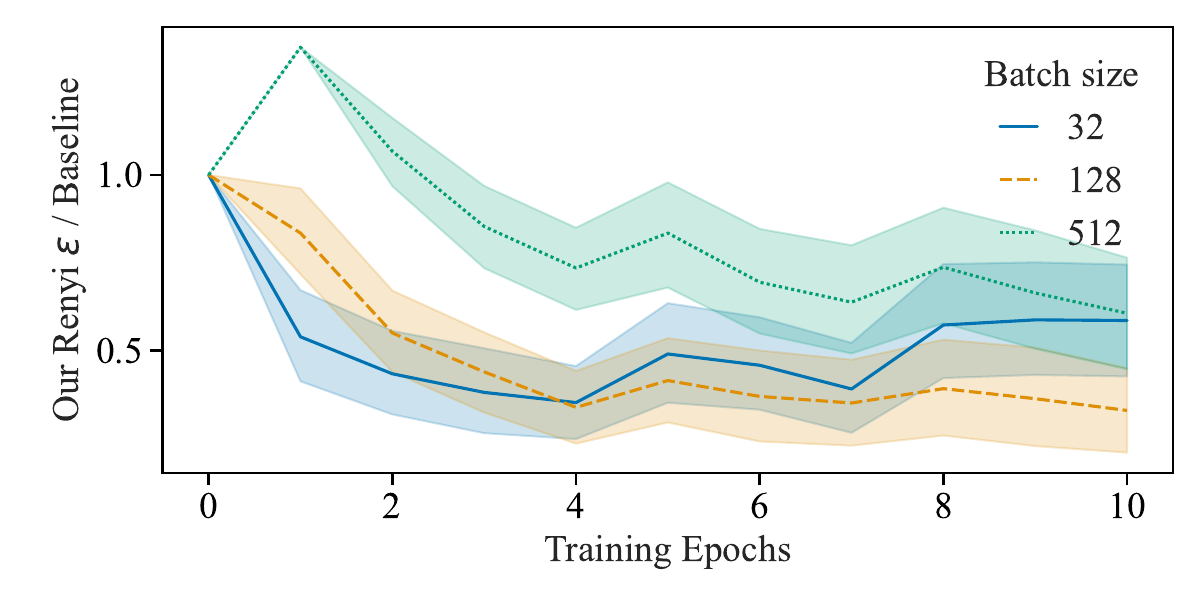}
}
\subfloat[varying architecture]
{
\includegraphics[width=0.48\linewidth]{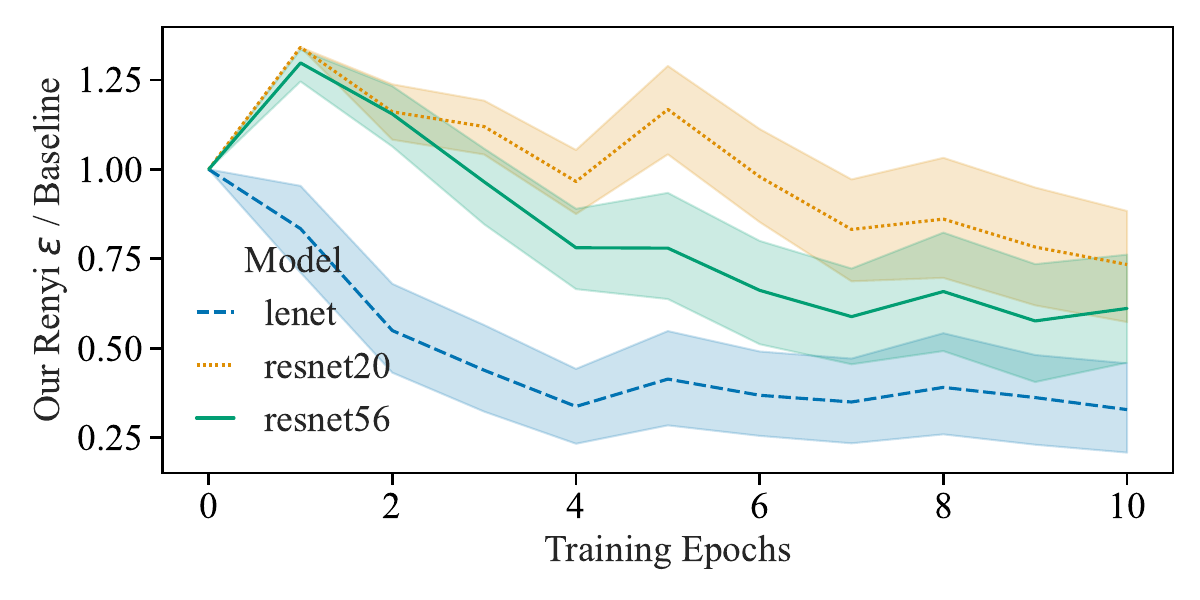}
}
\\\subfloat[varying epsilon]
{
\includegraphics[width=0.48\linewidth]{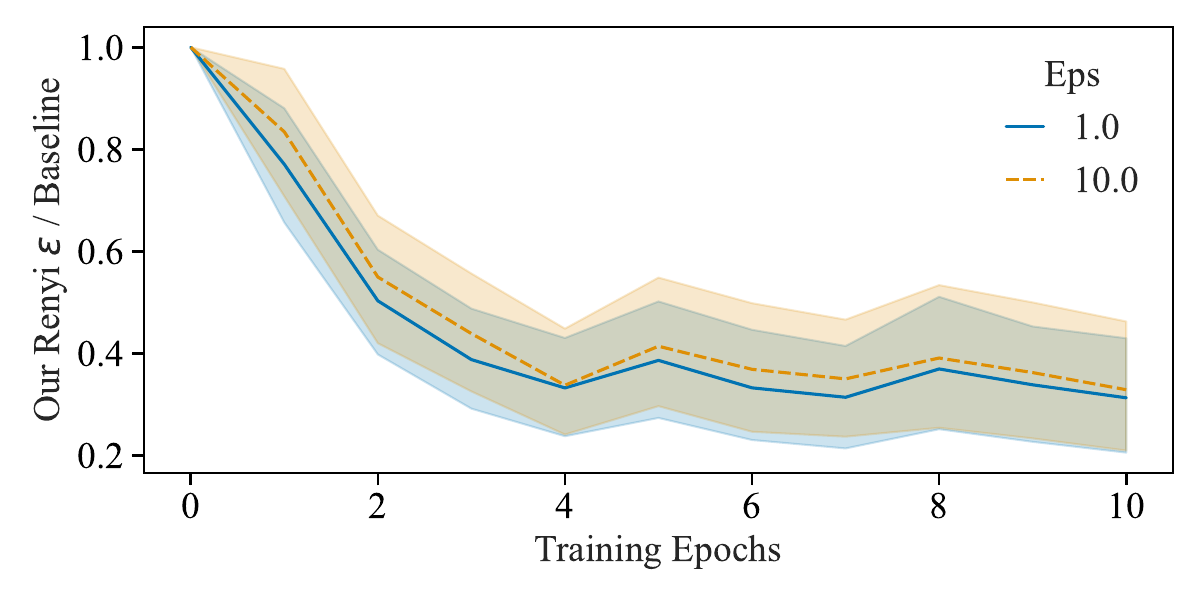}
}
\subfloat[varying batch size\\$\text{~~~~~}$($10^{th}$ percentile)]
{
\includegraphics[width=0.48\linewidth]{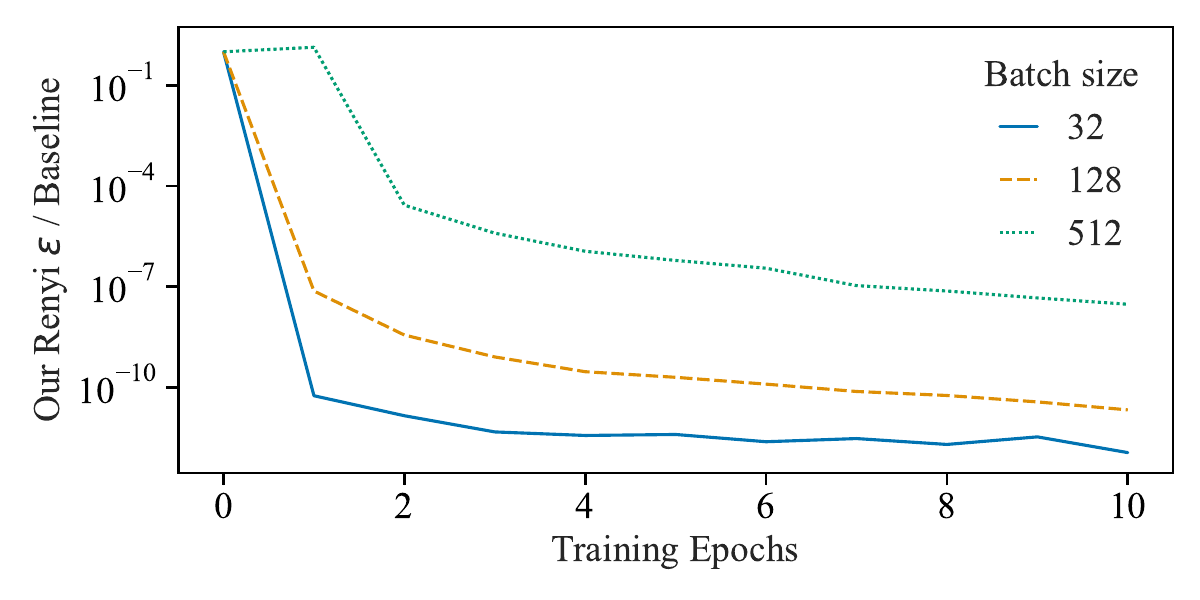}
}
\\\subfloat[varying architecture\\$\text{~~~~~}$($10^{th}$ percentile)]
{
\includegraphics[width=0.48\linewidth]{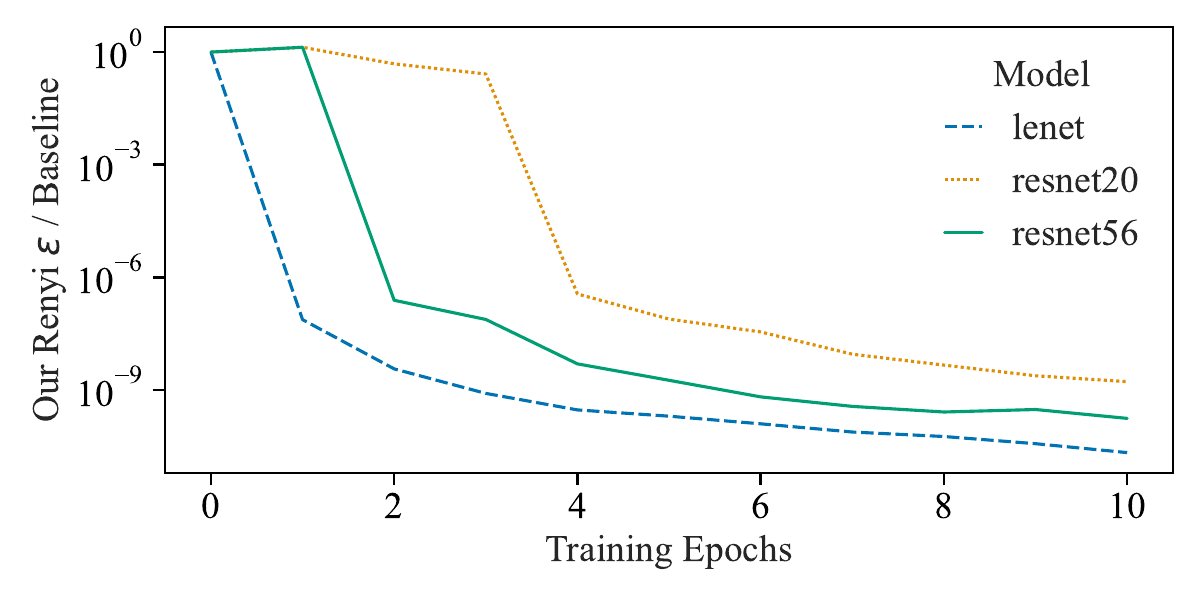}
}
\subfloat[varying epsilon\\$\text{~~~~~}$($10^{th}$ percentile)]
{
\includegraphics[width=0.48\linewidth]{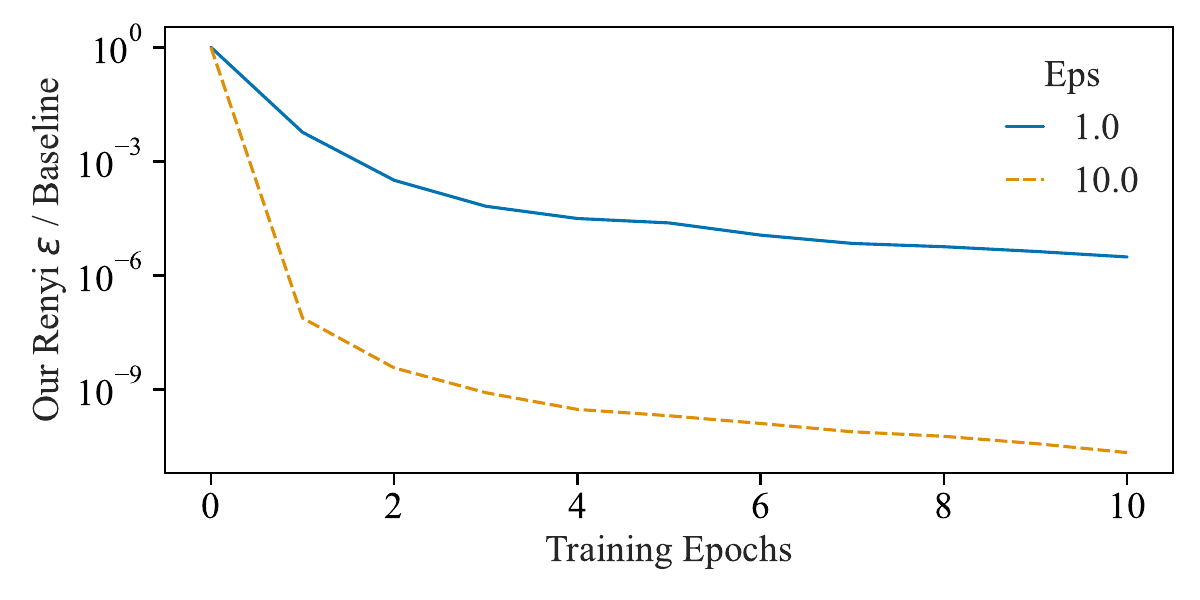}
}
\caption{ Expected privacy guarantees from Theorem~\ref{thm:easy_renyi_dp} plotted as a fraction of the per-step DP-SGD guarantee over training. One can see that the ratio between our guarantee and the per-step DP-SGD guarantee (the baseline) decreases as training approaches the end, and this is consistent across different strengths of DP (i.e., $\epsilon$ set for the entire training), varying mini-batch size, and different model architectures.
}
\label{fig:renyi_simple_composition_mnist_sum}
\end{figure}

\begin{figure}[t]
\centering
\subfloat[average (with confidence interval)]
{
\includegraphics[width=0.48\linewidth]{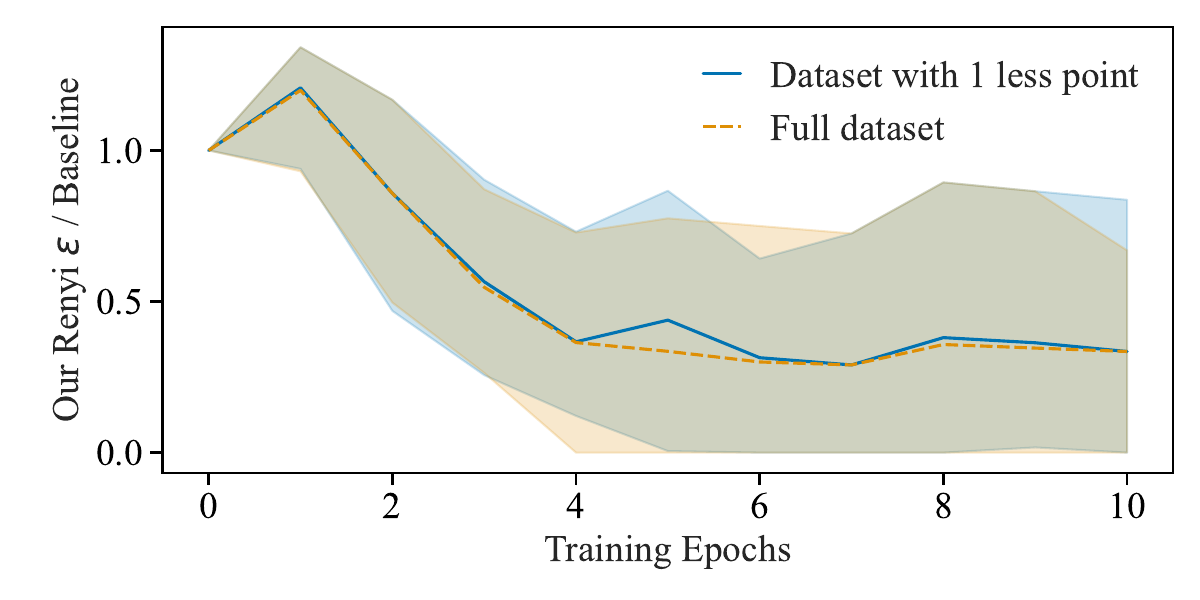}
}
\subfloat[$10^{th}$ percentile]
{
\includegraphics[width=0.48\linewidth]{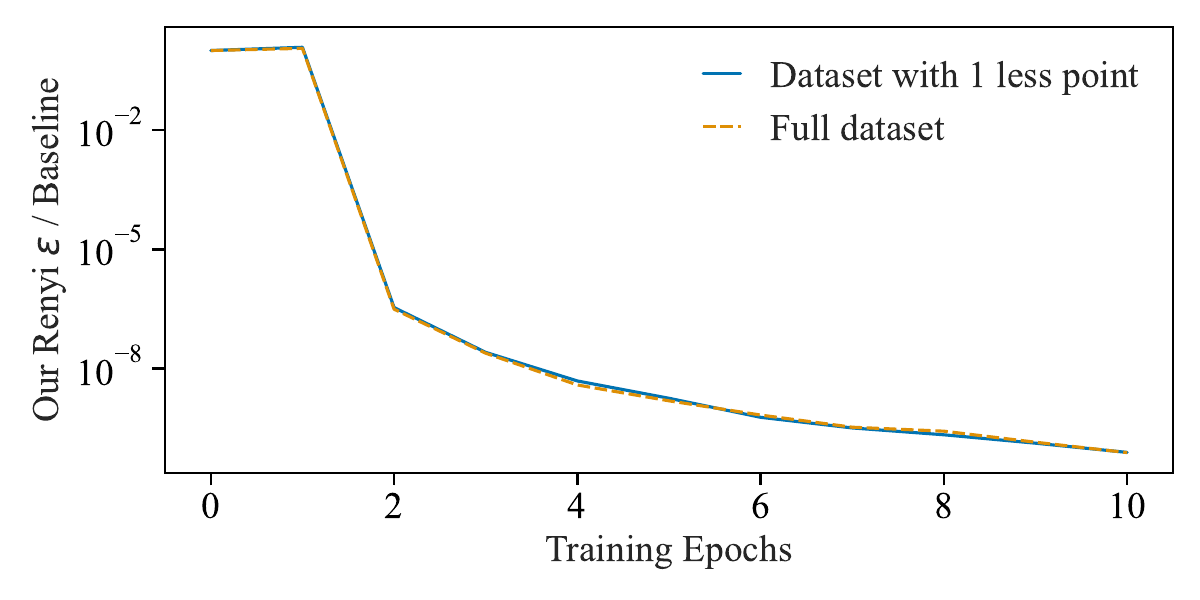}
}

\caption{ This is a reproduction of Figure~\ref{fig:renyi_simple_composition_mnist_sum} except the expected guarantee is computed for $D_{\alpha}(M(X')||M(X))$ instead of $D_{\alpha}(M(X)||M(X'))$. However, a similar trend can be observed.
}
\label{fig:remove_10points}
\end{figure}

\subsection{Studying the Theorem~\ref{thm:renyi_dp_sens} Results}

Here we compute the $(\alpha,\epsilon)$-R\'enyi-DP guarantee given by Theorem~\ref{thm:renyi_dp_sens} (in particular just the $D_{\alpha}(M(X')||M(X))$ upper-bound unless otherwise stated) and plot results analogous to the previous subsections. In computing Theorem~\ref{thm:renyi_dp_sens} guarantees we use 20 samples for both the inner and outer expectations.

We re-use the standard R\'enyi-DP analysis as our baseline, as also used in Section~\ref{ssec:eval_easy_renyi}. In the case of the mean update rule, this means taking our sensitivity norm to be the clipping norm $C = 1.0$ as used in training.

\paragraph{Comparison to Baseline.} We see in Figures~\ref{fig:renyi_hard_eps_distrib_bs_mnist_mean},\ref{fig:renyi_hard_eps_distrib_bs_cifar_mean} that our analysis for the mean beats the baseline analysis of the R\'enyi-DP guarantee. In particular, we see that for checkpoints in the middle and end of training, we give a privacy guarantee roughly a magnitude better than the baseline. In the case of the sum we see we do significantly worse than the baseline with this analysis as shown in Figures~\ref{fig:renyi_hard_eps_distrib_bs_mnist_sum},\ref{fig:renyi_hard_eps_distrib_bs_cifar_sum}.

\paragraph{Varying Expected Batch Size.} We see in Figures~\ref{fig:renyi_hard_eps_distrib_bs_mnist_sum},\ref{fig:renyi_hard_eps_distrib_bs_cifar_sum} for the sum update rule, both our and the baseline's guarantee increases (with our guarantee continuing to be worse than the baseline). However, for the mean update rule our guarantee decreases where as the baseline still increases, as show in Figures~\ref{fig:renyi_hard_eps_distrib_bs_mnist_mean},\ref{fig:renyi_hard_eps_distrib_bs_cifar_mean}. 

\paragraph{Varying Alphas.} In Figure~\ref{fig:renyi_hard_eps_distrib_alpha_mnist_mean},\ref{fig:renyi_hard_eps_distrib_alpha_cifar_mean} we see that our guarantee for the mean update rule does better than the baseline for sufficiently large alpha, but for small alpha (e.g., $\alpha = 2,4$) does worse. Varying $\alpha$ does not make the guarantee given by Theorem~\ref{thm:renyi_dp_sens} for the sum better than the baseline (Figures~\ref{fig:renyi_hard_eps_distrib_alpha_mnist_sum},\ref{fig:renyi_hard_eps_distrib_alpha_cifar_sum}).

\paragraph{Changing DP strength used to get checkpoints.} In Figure~\ref{fig:ed_eps_curve_vary_eps_mnist} we see that changing the strength of DP used to get the checkpoints also increases our guarantees.

\paragraph{The Reverse Divergence.}
In Figures~\ref{fig:reverse_renyi} and~\ref{fig:reverse_renyi_cifar} we plot the upper-bound for $D_{\alpha}(M(X)||M(X'))$ from Theorem~\ref{thm:renyi_dp_sens} for models trained on MNIST and CIFAR-10 respectively, and conclude we once again do better than the baseline for the mean update rule. Hence, as we do better than the baseline for both divergences, we conclude the analysis of Theorem~\ref{thm:renyi_dp_sens} gives tighter guarantees for the mean update rule.

\begin{figure}[t]
\centering
\subfloat[Mini-batch size = 16]
{
\includegraphics[width=0.48\linewidth]{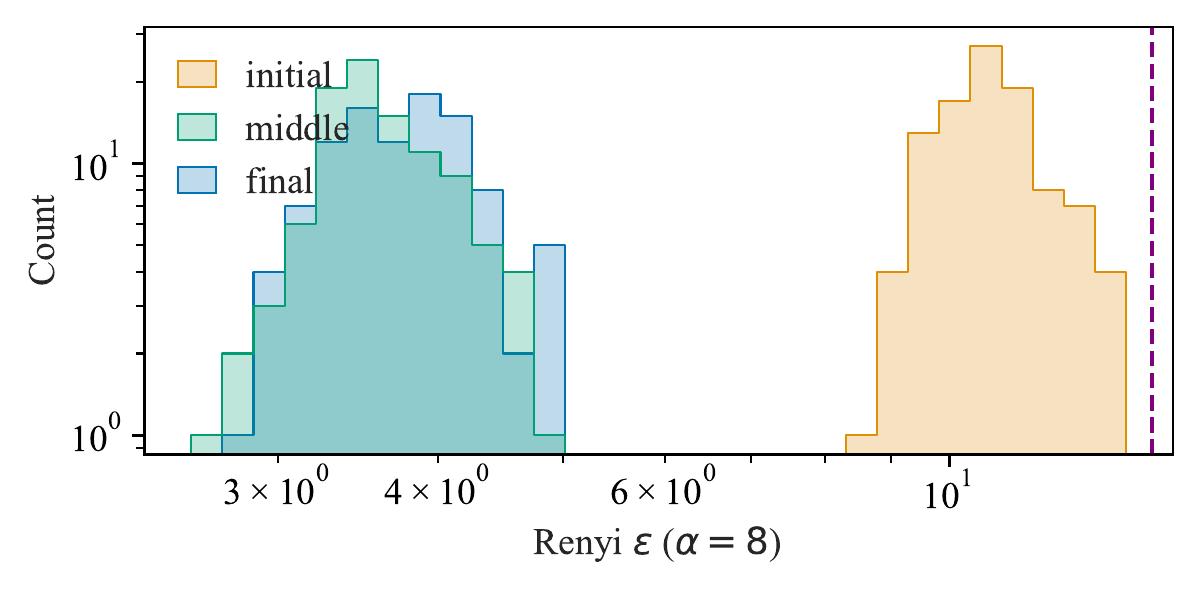}
}
\subfloat[Mini-batch size = 32]
{
\includegraphics[width=0.48\linewidth]{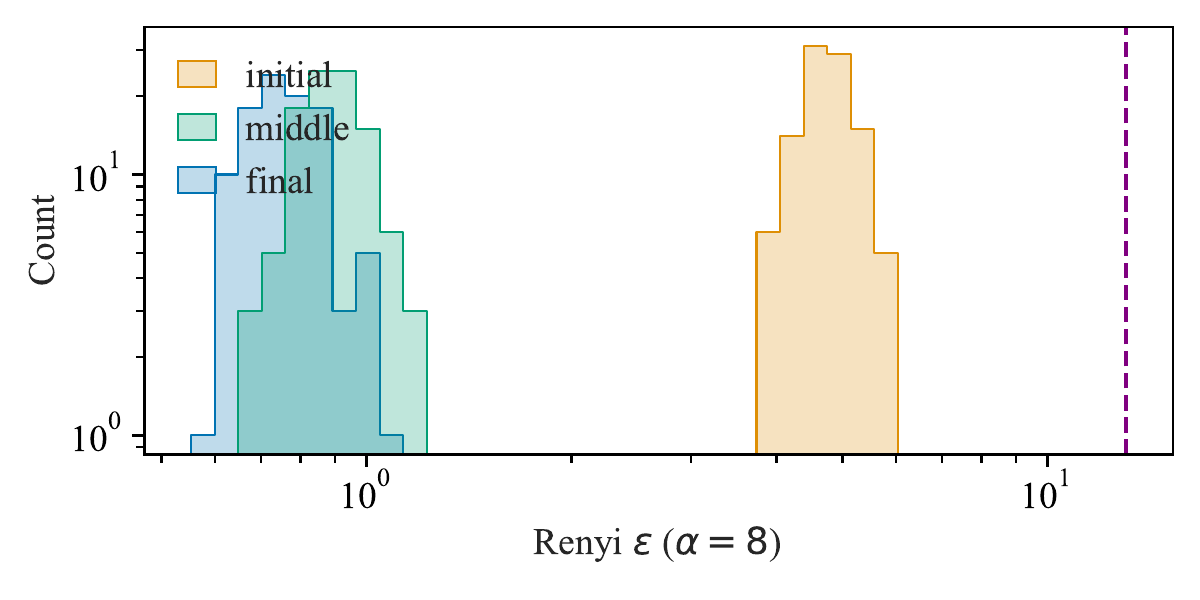}
}
\\\subfloat[Mini-batch size = 64]
{
\includegraphics[width=0.48\linewidth]{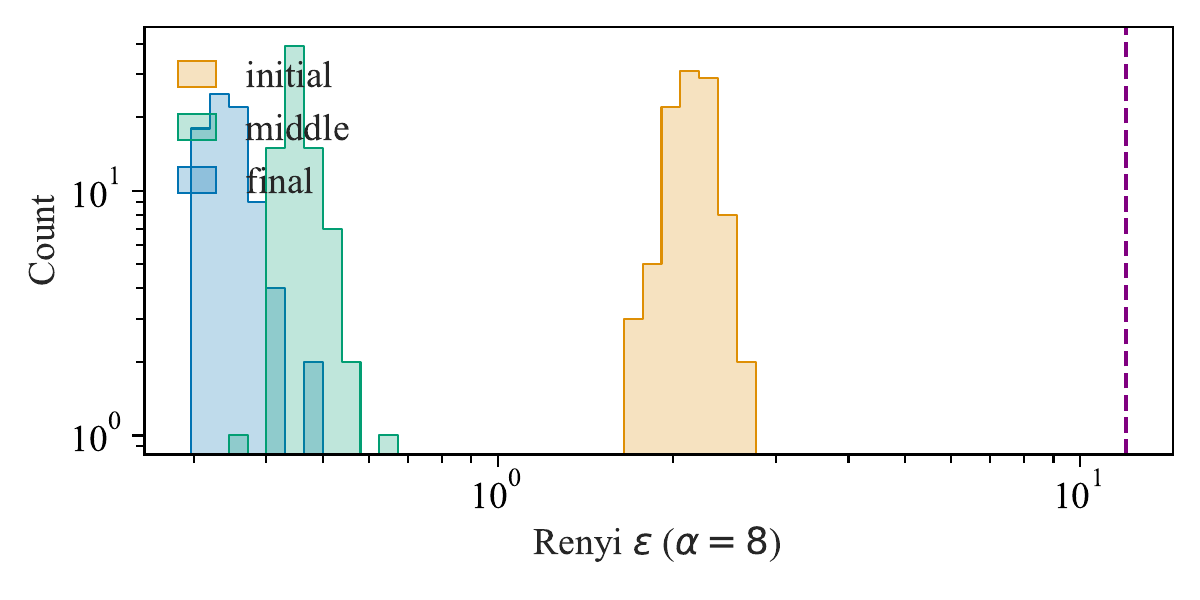}
}
\subfloat[Mini-batch size = 128]
{
\includegraphics[width=0.48\linewidth]{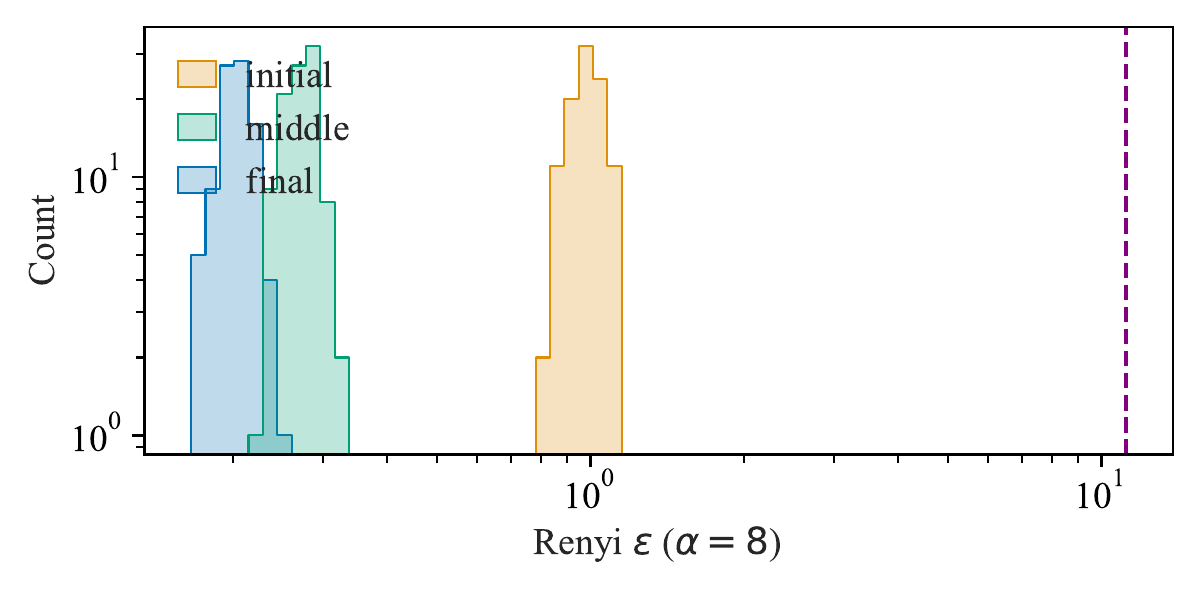}
}
\\\subfloat[Mini-batch size = 256]
{
\includegraphics[width=0.48\linewidth]{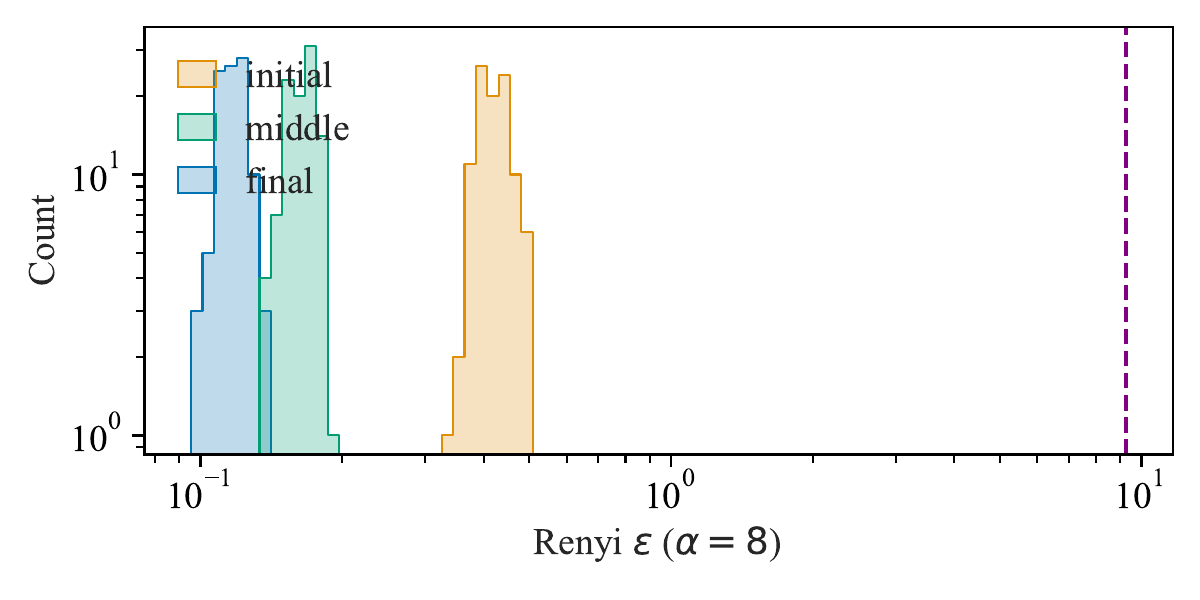}
}
\subfloat[Mini-batch size = 512]
{
\includegraphics[width=0.48\linewidth]{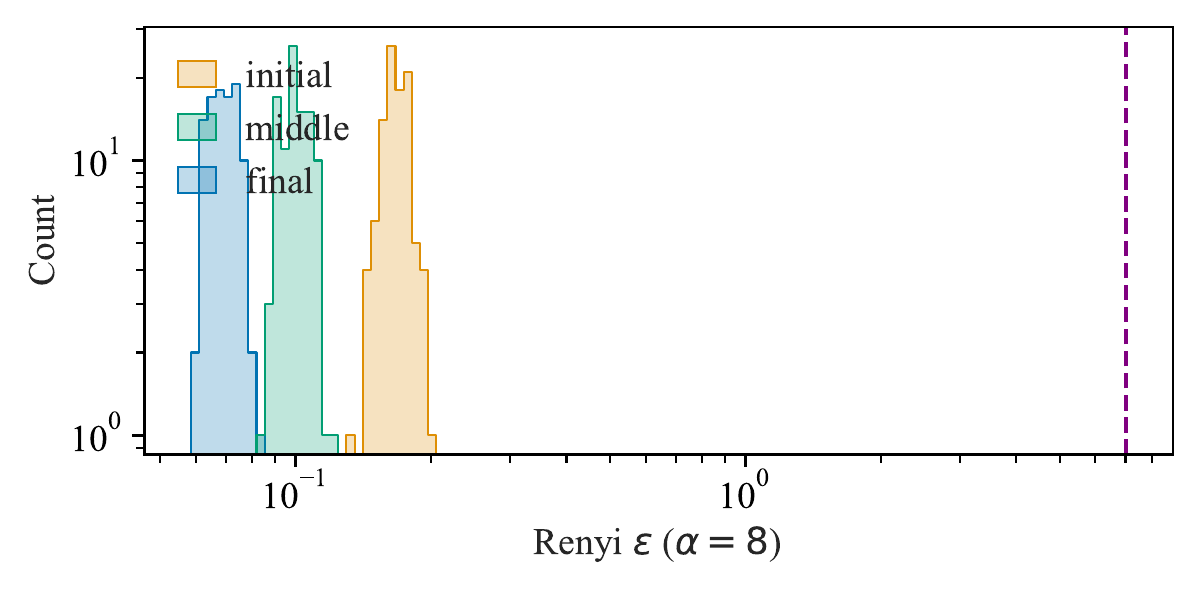}
}

\caption{ Distribution plots (log scale) of per-step guarantees given by Theorem~\ref{thm:renyi_dp_sens} computed on LeNet-5 trained on MNIST with mean update rule and varying mini-batch sizes of $16, 32, 64, 128, 256, 512$. As specified by the legend labels, we group the plotted guarantees by whether the model is at the initial, middle, or final stage of the training. It can be seen that in all settings our guarantee is better than the baseline, which is represented by the dashed purple line. However, the guarantee distributions of points at the initial stage of training are closer to the baseline compared to the other distributions.
Additionally, since a mean update rule is used, the bounds depend on the mini-batch size, and better bounds are achieved when the mini-batch size is large. 
}
\label{fig:renyi_hard_eps_distrib_bs_mnist_mean}
\end{figure}

\begin{figure}[t]
\centering
\subfloat[Mini-batch size = 16]
{
\includegraphics[width=0.48\linewidth]{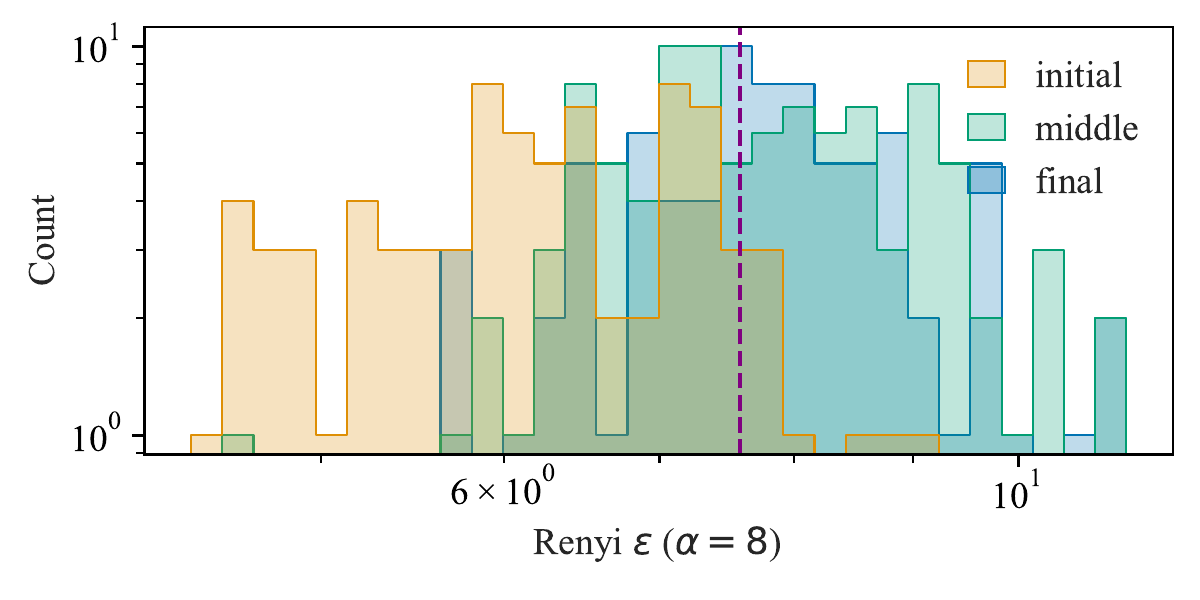}
}
\subfloat[Mini-batch size = 32]
{
\includegraphics[width=0.48\linewidth]{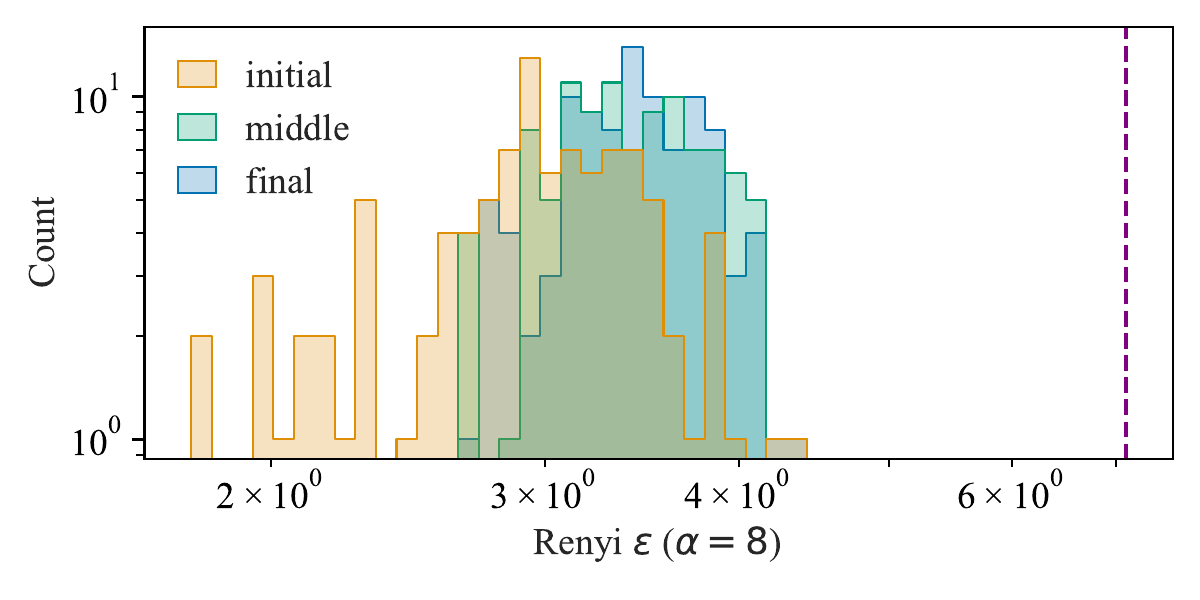}
}
\\\subfloat[Mini-batch size = 64]
{
\includegraphics[width=0.48\linewidth]{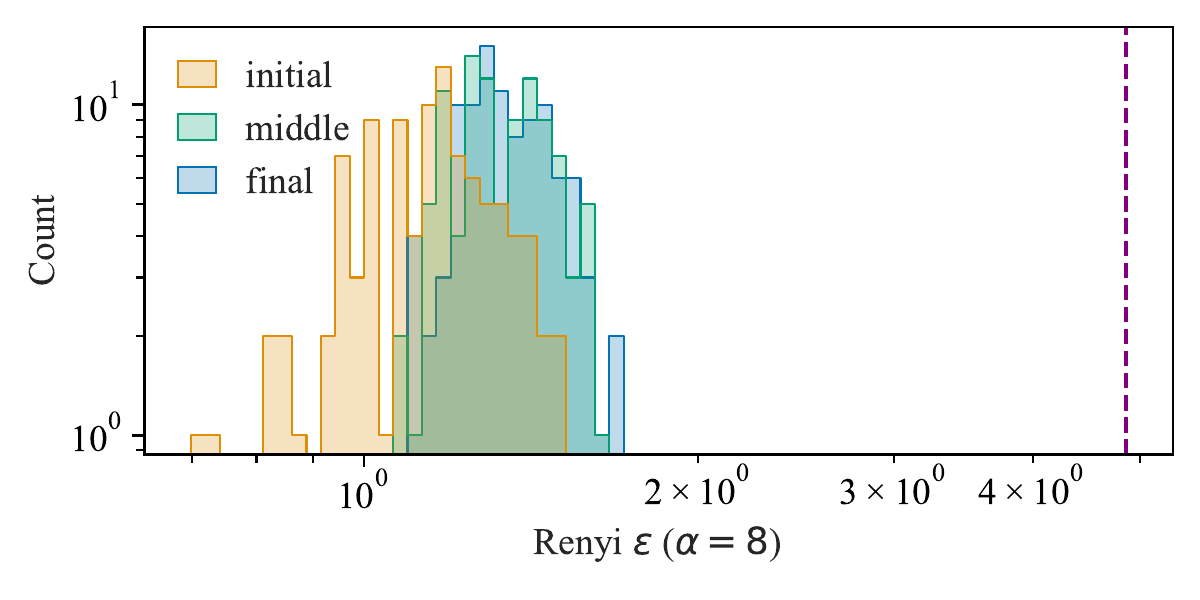}
}
\subfloat[Mini-batch size = 128]
{
\includegraphics[width=0.48\linewidth]{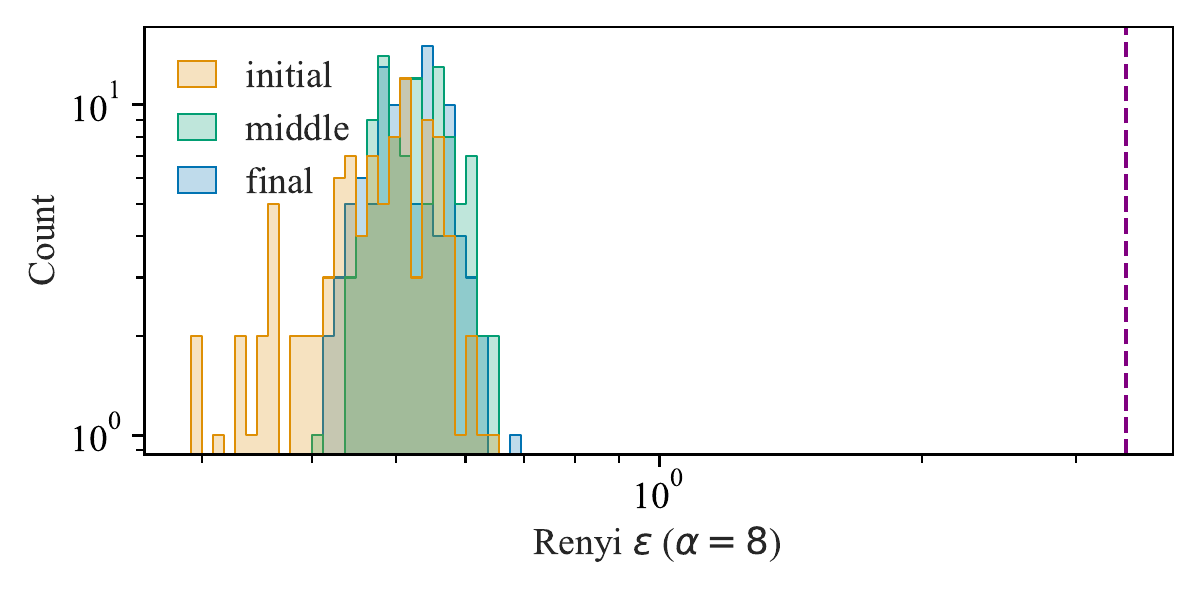}
}
\\\subfloat[Mini-batch size = 256]
{
\includegraphics[width=0.48\linewidth]{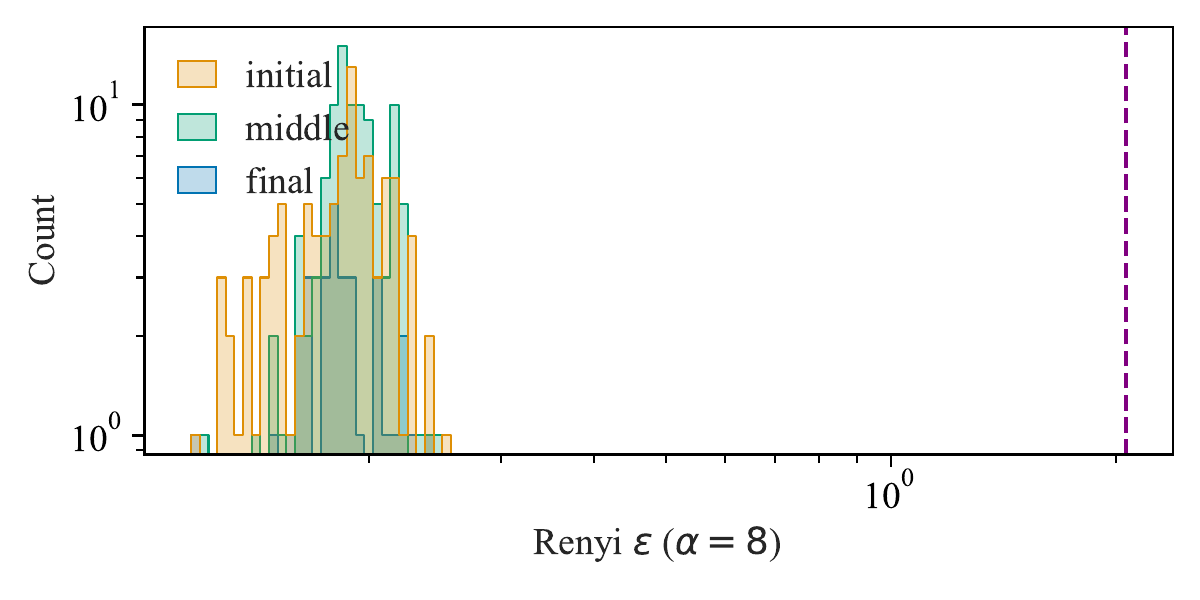}
}
\caption{ This is the reproduction of Figure~\ref{fig:renyi_hard_eps_distrib_bs_mnist_mean} except we now use ResNet-20 models trained on CIFAR-10. Similar results are observed so we conclude that the guarantee given by Theorem~\ref{thm:renyi_dp_sens} is effective across datasets.
}
\label{fig:renyi_hard_eps_distrib_bs_cifar_mean}
\end{figure}

\begin{figure}[t]
\centering
\subfloat[Alpha = 2]
{
\includegraphics[width=0.48\linewidth]{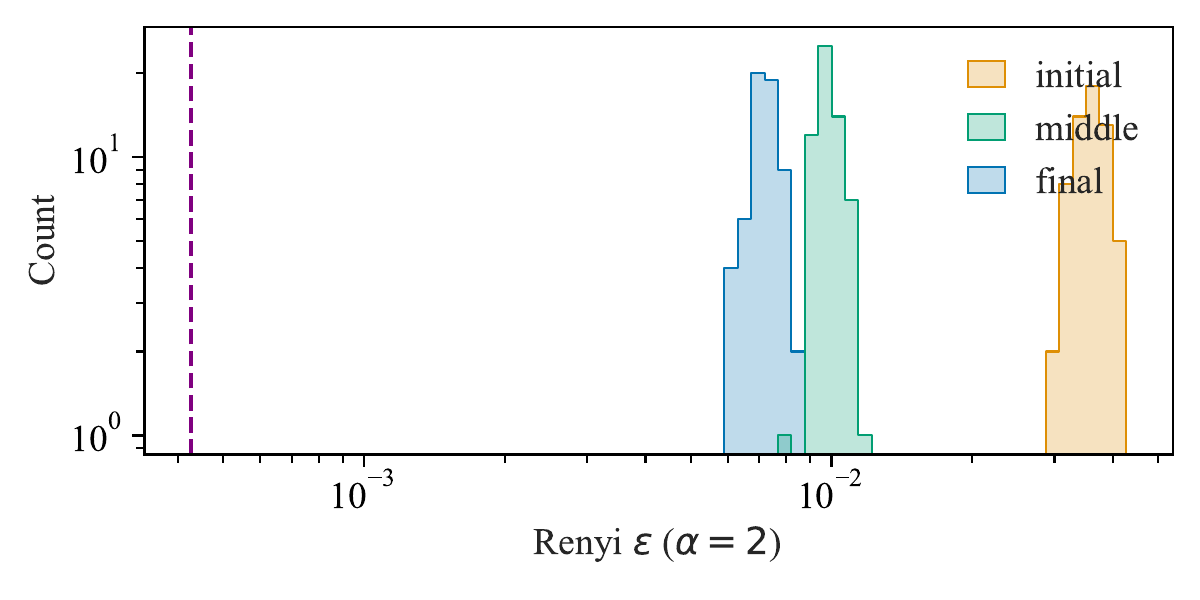}
}
\subfloat[Alpha = 4]
{
\includegraphics[width=0.48\linewidth]{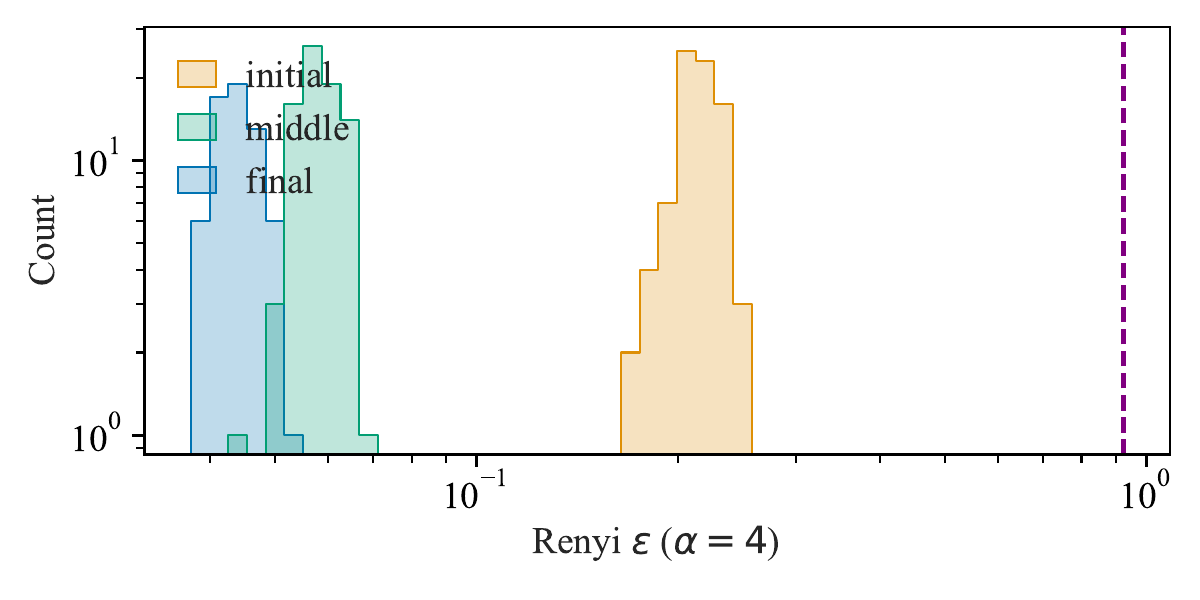}
}
\\\subfloat[Alpha = 8]
{
\includegraphics[width=0.48\linewidth]{figures/renyi_hard_eps_hist_MNIST_lenet_128_8_mean.pdf}
}
\subfloat[Alpha = 16]
{
\includegraphics[width=0.48\linewidth]{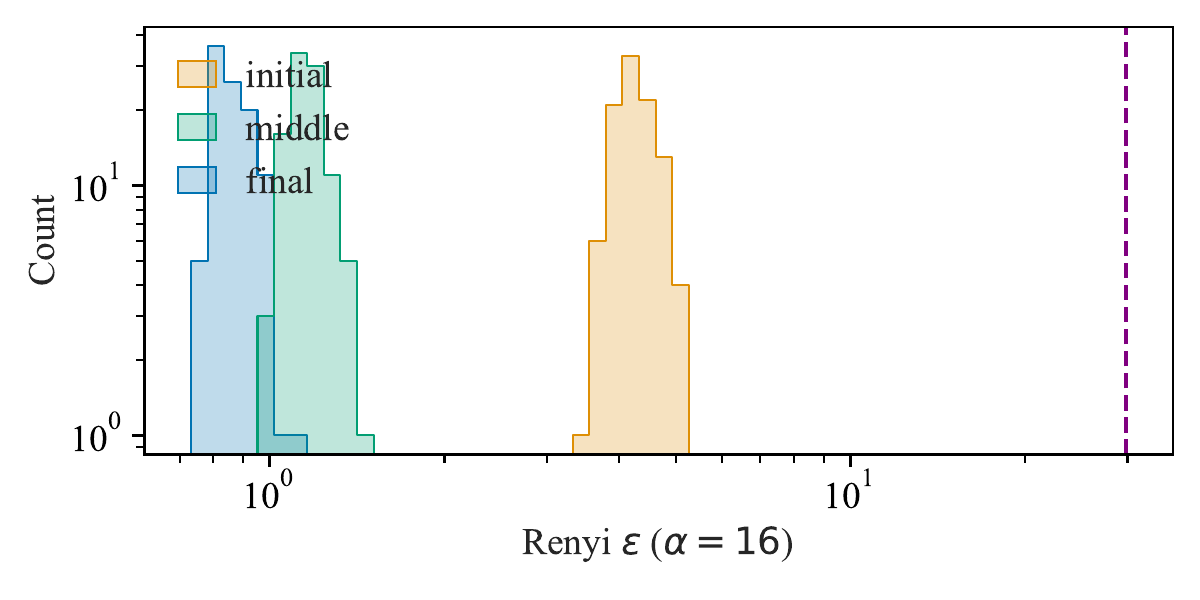}
}
\\\subfloat[Alpha = 32]
{
\includegraphics[width=0.48\linewidth]{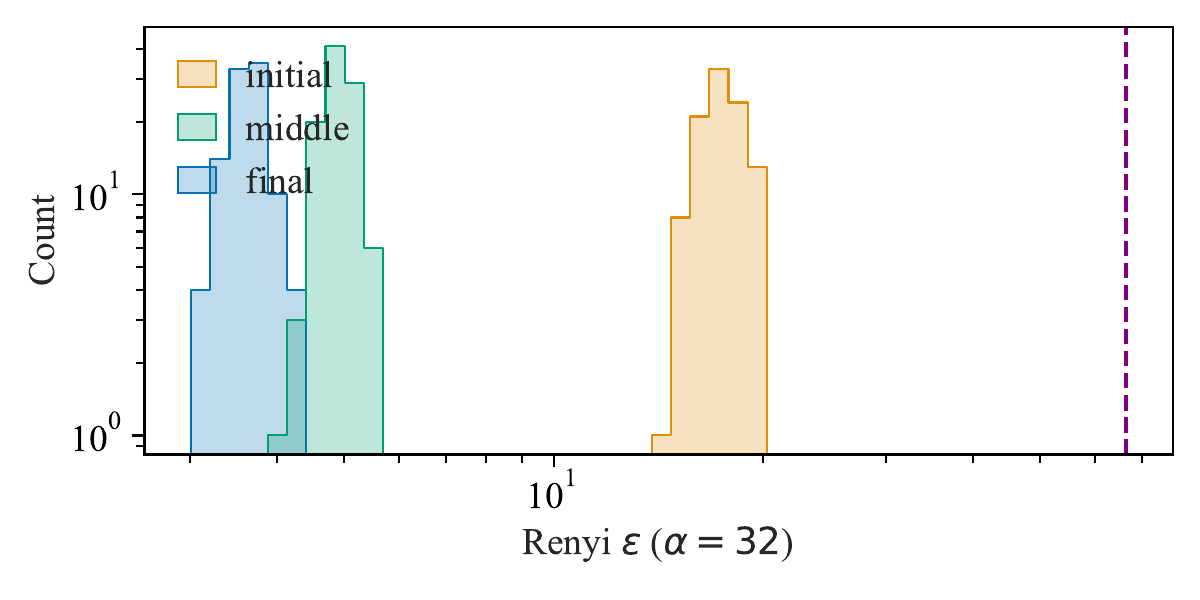}
}
\subfloat[Alpha = 64]
{
\includegraphics[width=0.48\linewidth]{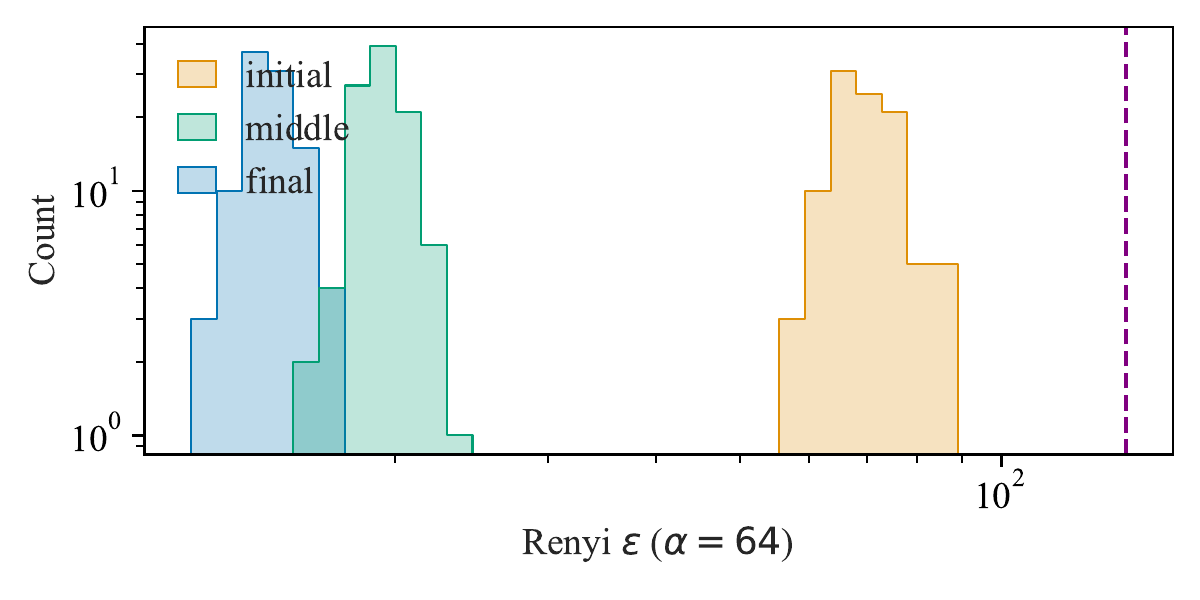}
}
\caption{ This is the reproduction of Figure~\ref{fig:renyi_hard_eps_distrib_bs_mnist_mean} except we now fix the batch size to be 128 and vary the value of $\alpha$. It can be seen that our guarantee is worse than the baseline when $\alpha=2$ and outperforms the latter in all other settings of $\alpha$. This suggests our guarantee is favored when $\alpha$ is large.
}
\label{fig:renyi_hard_eps_distrib_alpha_mnist_mean}
\end{figure}

\begin{figure}[t]
\centering
\subfloat[Alpha = 2]
{
\includegraphics[width=0.48\linewidth]{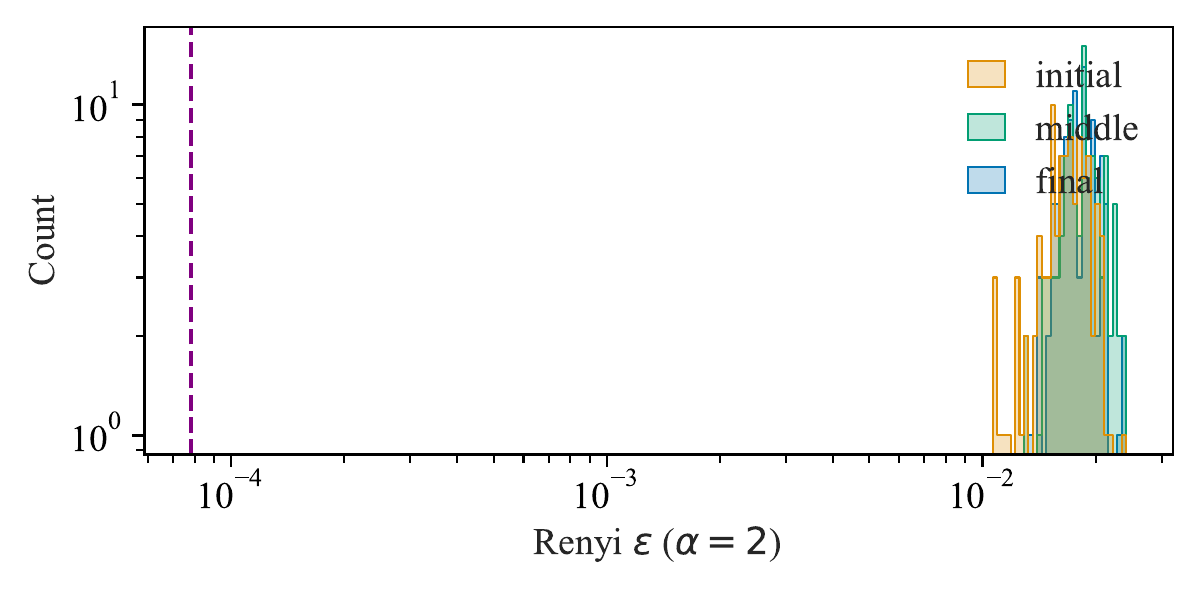}
}
\subfloat[Alpha = 4]
{
\includegraphics[width=0.48\linewidth]{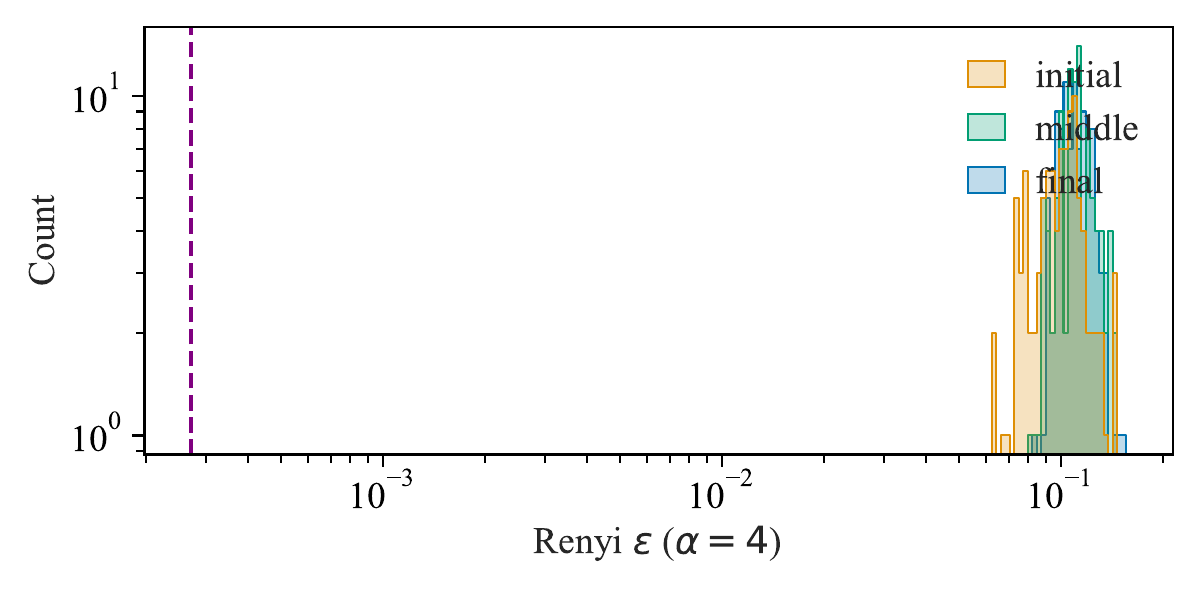}
}
\\\subfloat[Alpha = 8]
{
\includegraphics[width=0.48\linewidth]{figures/renyi_hard_eps_hist_CIFAR10_resnet20_128.0_8_mean.pdf}
}
\subfloat[Alpha = 16]
{
\includegraphics[width=0.48\linewidth]{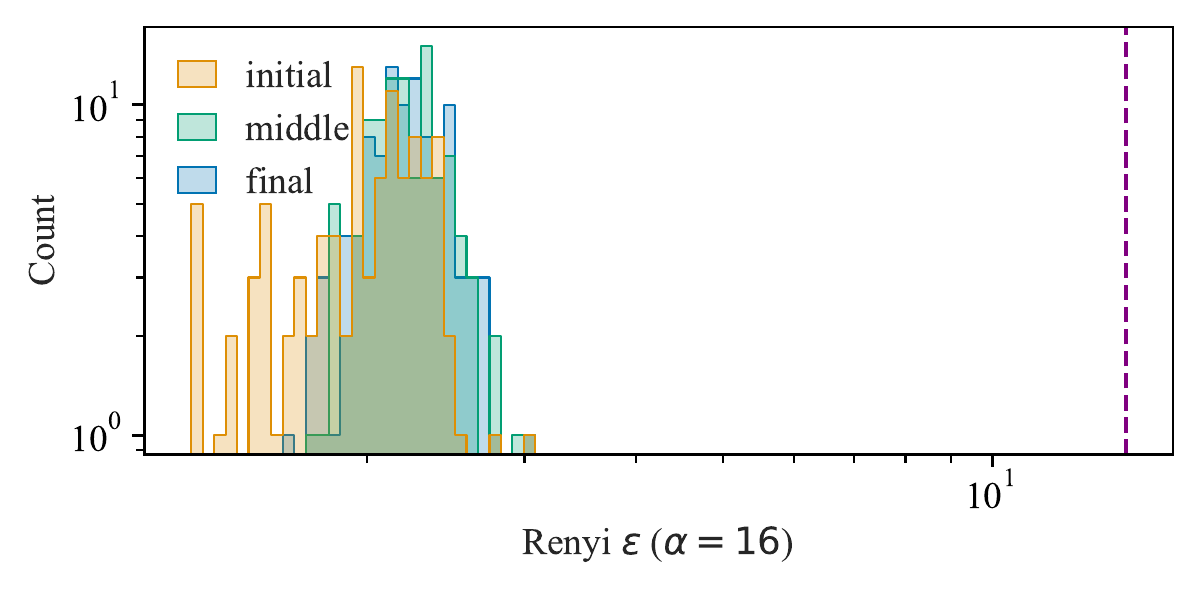}
}
\\\subfloat[Alpha = 32]
{
\includegraphics[width=0.48\linewidth]{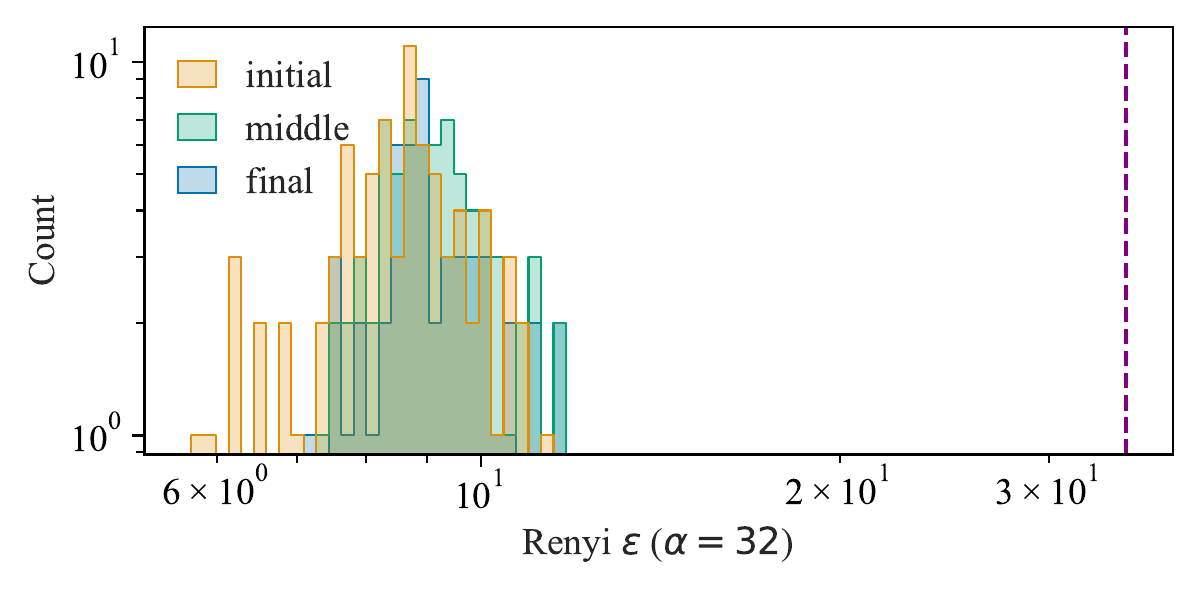}
}
\subfloat[Alpha = 64]
{
\includegraphics[width=0.48\linewidth]{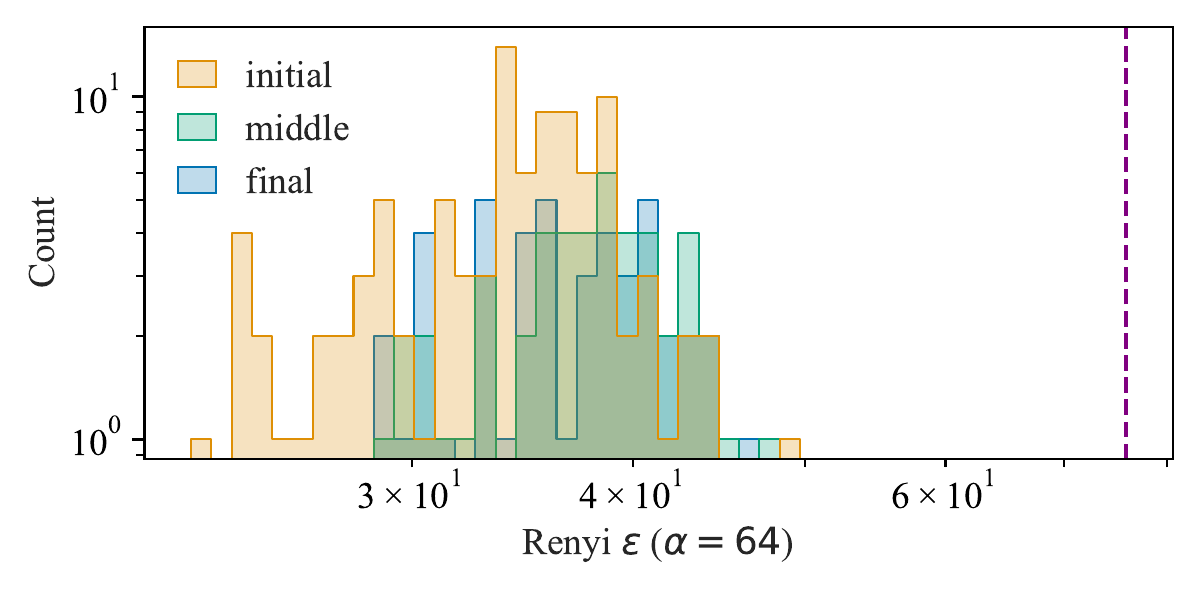}
}
\caption{ This is the reproduction of Figure~\ref{fig:renyi_hard_eps_distrib_alpha_mnist_mean} except we now use ResNet-20 models trained on CIFAR-10. Whereas our guarantees are worse than the baseline for $\alpha=2,4$, the conclusion that our guarantee is favored when $\alpha$ is large still holds.
}
\label{fig:renyi_hard_eps_distrib_alpha_cifar_mean}
\end{figure}

\begin{figure}[t]
\centering
\subfloat[Mini-batch size = 16]
{
\includegraphics[width=0.48\linewidth]{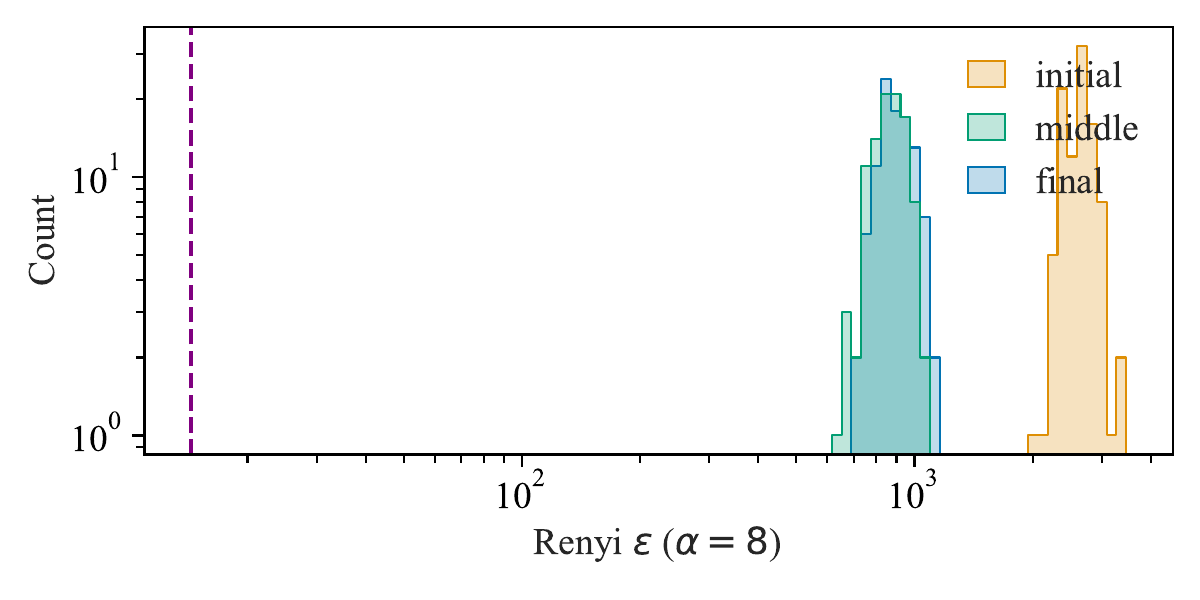}
}
\subfloat[Mini-batch size = 32]
{
\includegraphics[width=0.48\linewidth]{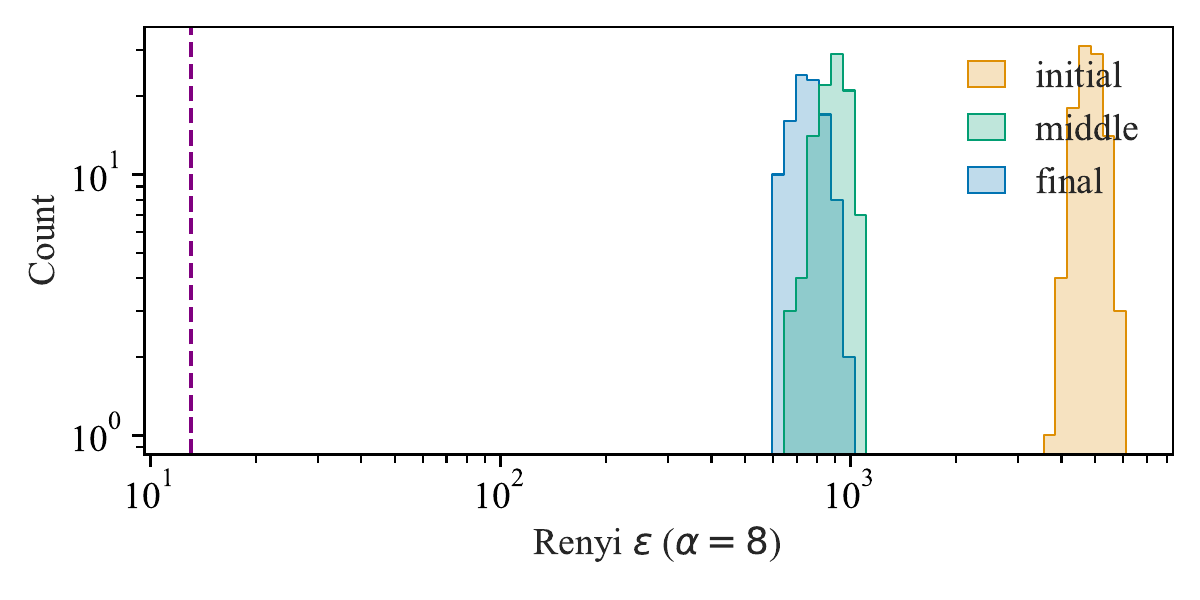}
}
\\\subfloat[Mini-batch size = 64]
{
\includegraphics[width=0.48\linewidth]{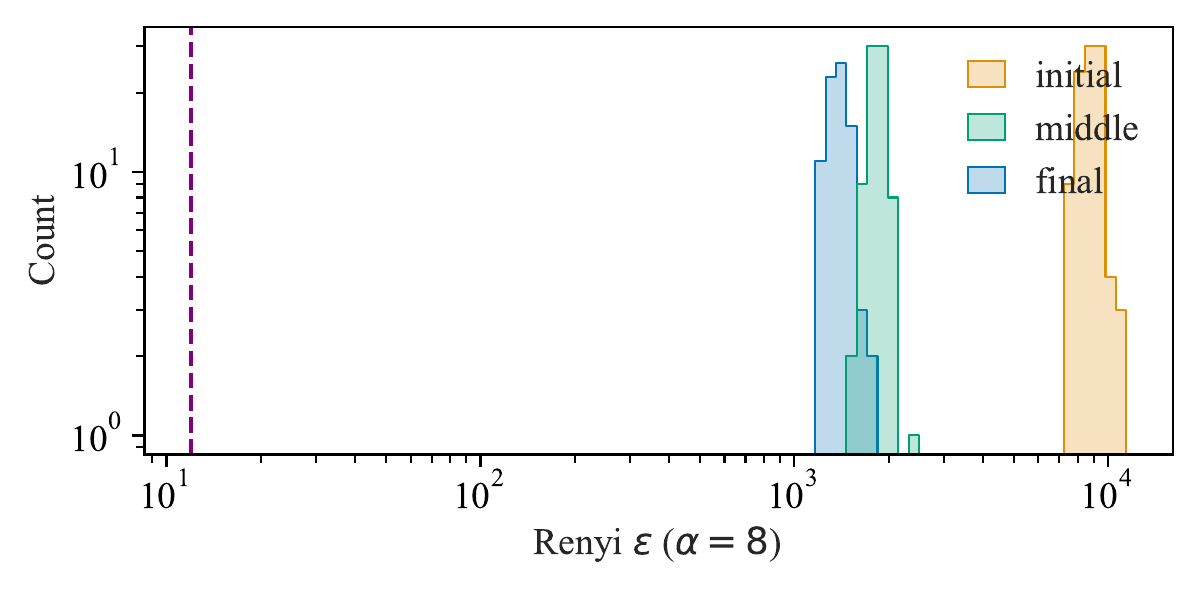}
}
\subfloat[Mini-batch size = 128]
{
\includegraphics[width=0.48\linewidth]{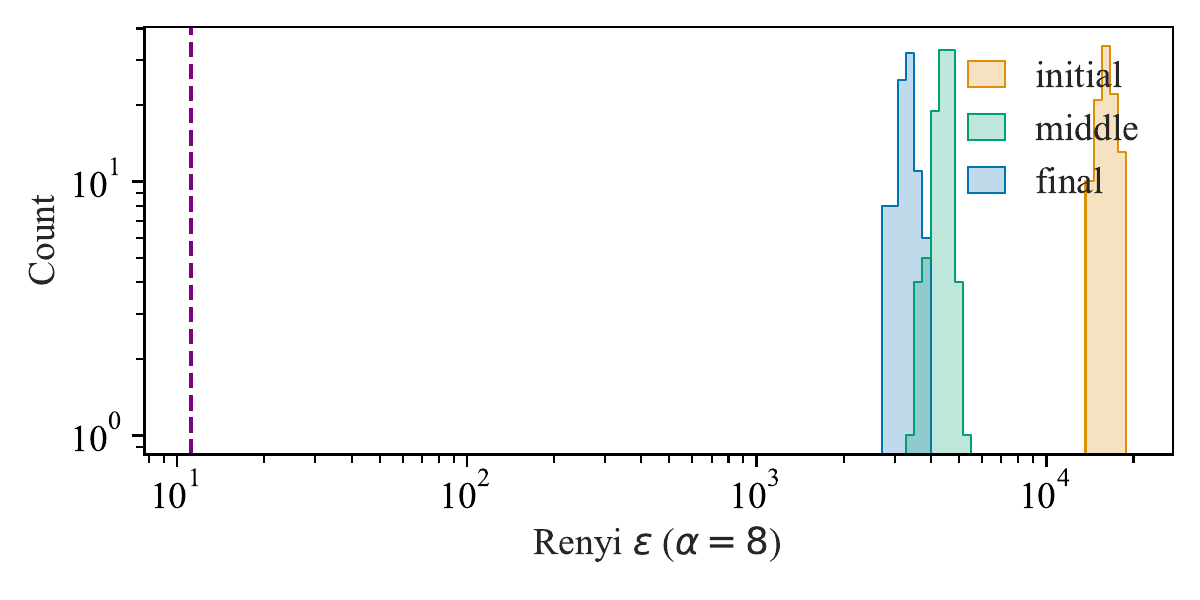}
}
\\\subfloat[Mini-batch size = 256]
{
\includegraphics[width=0.48\linewidth]{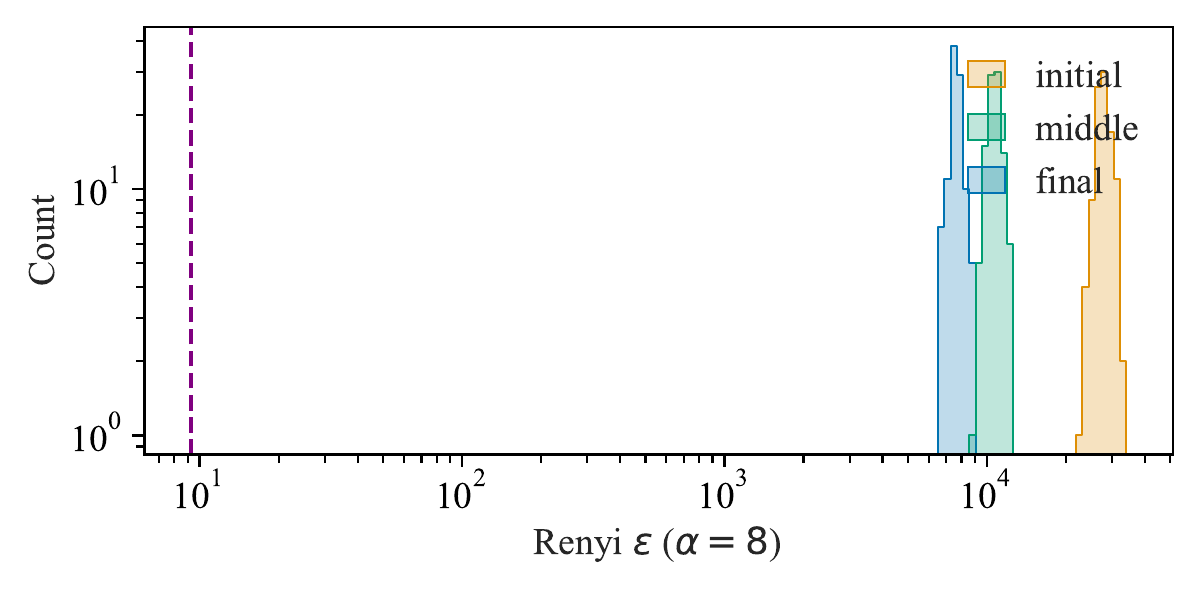}
}
\subfloat[Mini-batch size = 512]
{
\includegraphics[width=0.48\linewidth]{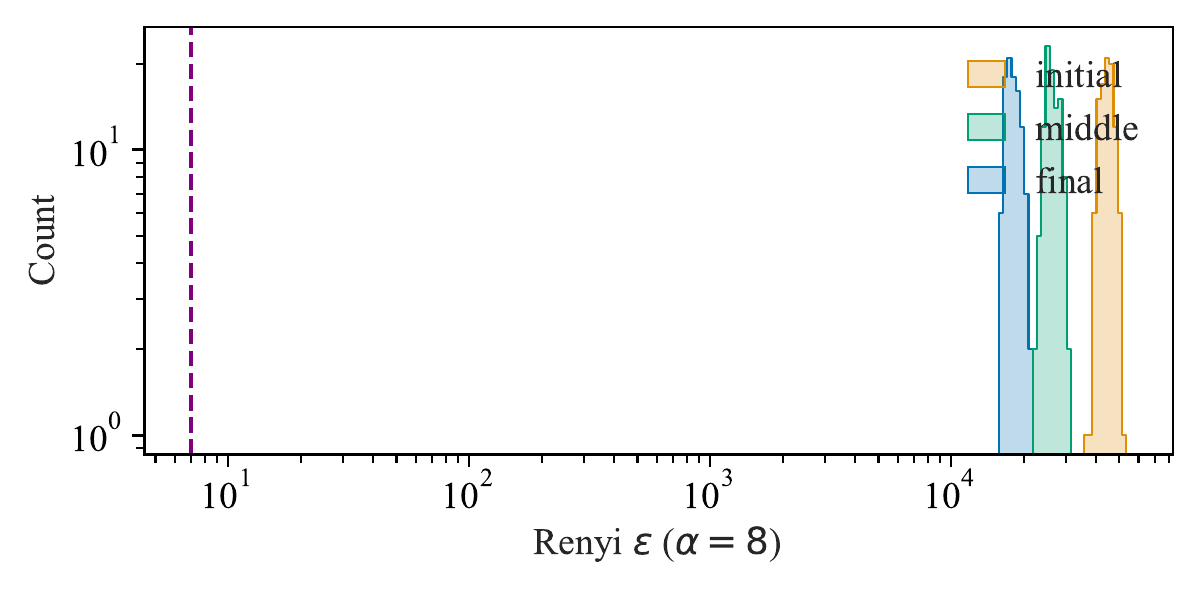}
}
\caption{ This is the reproduction of Figure~\ref{fig:renyi_hard_eps_distrib_bs_mnist_mean} except we now use a sum update rule. Unlike the mean update rule, we observe that the guarantees given by Theorem~\ref{thm:renyi_dp_sens} are significantly larger than the baseline guarantees. We suspect that this is because Theorem~\ref{thm:renyi_dp_sens} essentially computes how similar the gradient of any batch from $X$ is to any other batch from $X'$ in Euclidean space. In the absence of the factor that scale down the gradients by the mini-batch size, we do not expect these gradients concentrate at similar values.
}
\label{fig:renyi_hard_eps_distrib_bs_mnist_sum}
\end{figure}

\begin{figure}[t]
\centering
\subfloat[Mini-batch size = 16]
{
\includegraphics[width=0.48\linewidth]{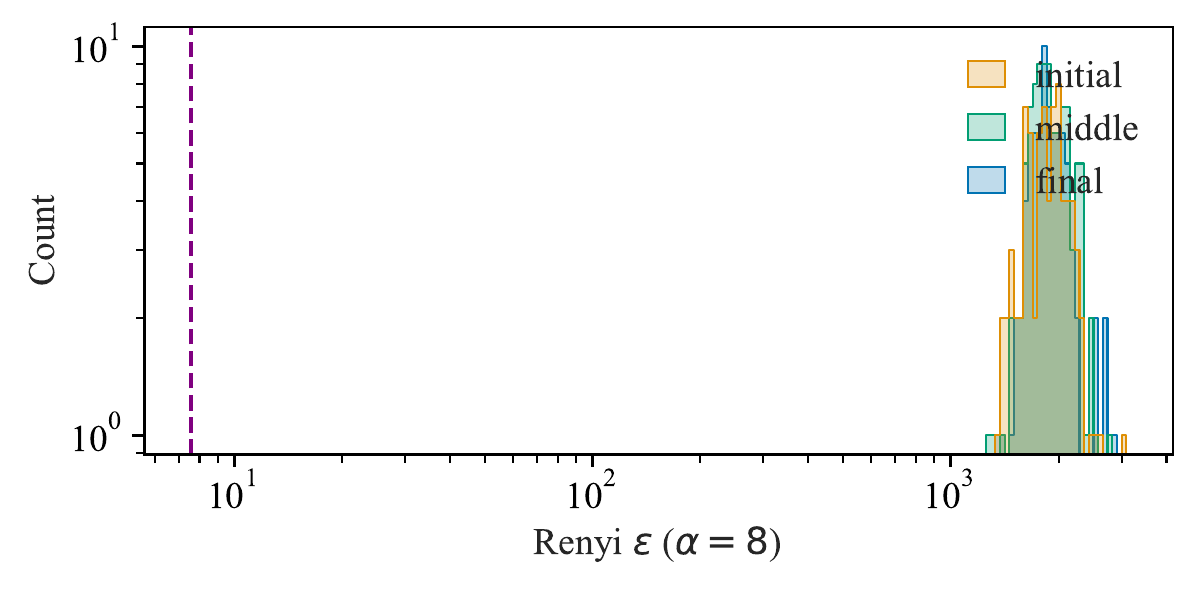}
}
\subfloat[Mini-batch size = 32]
{
\includegraphics[width=0.48\linewidth]{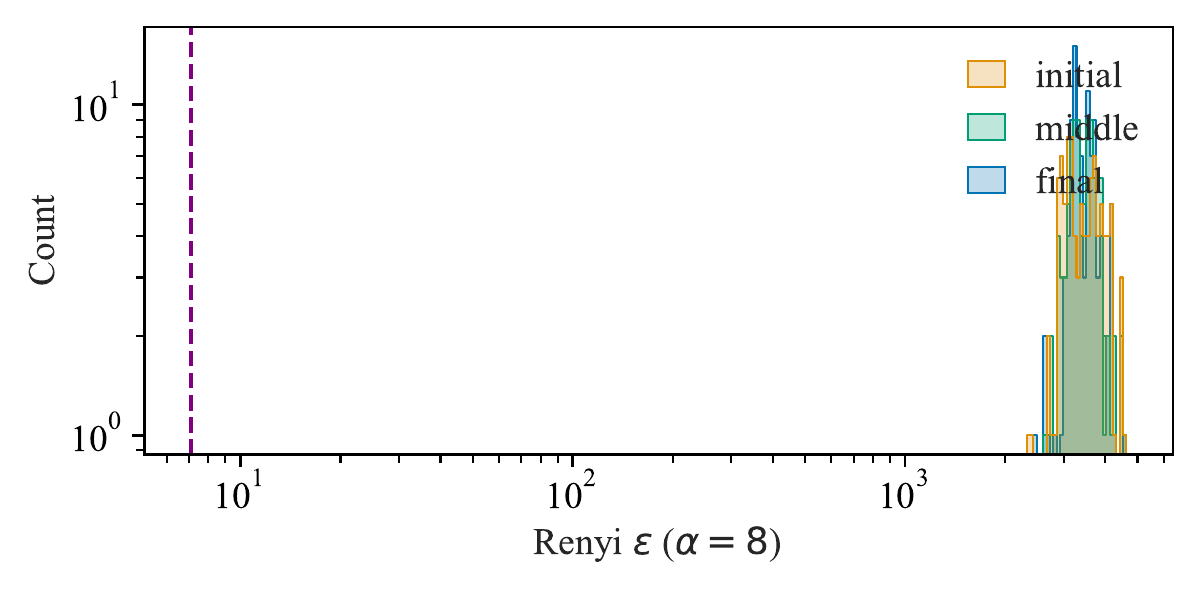}
}
\\\subfloat[Mini-batch size = 64]
{
\includegraphics[width=0.48\linewidth]{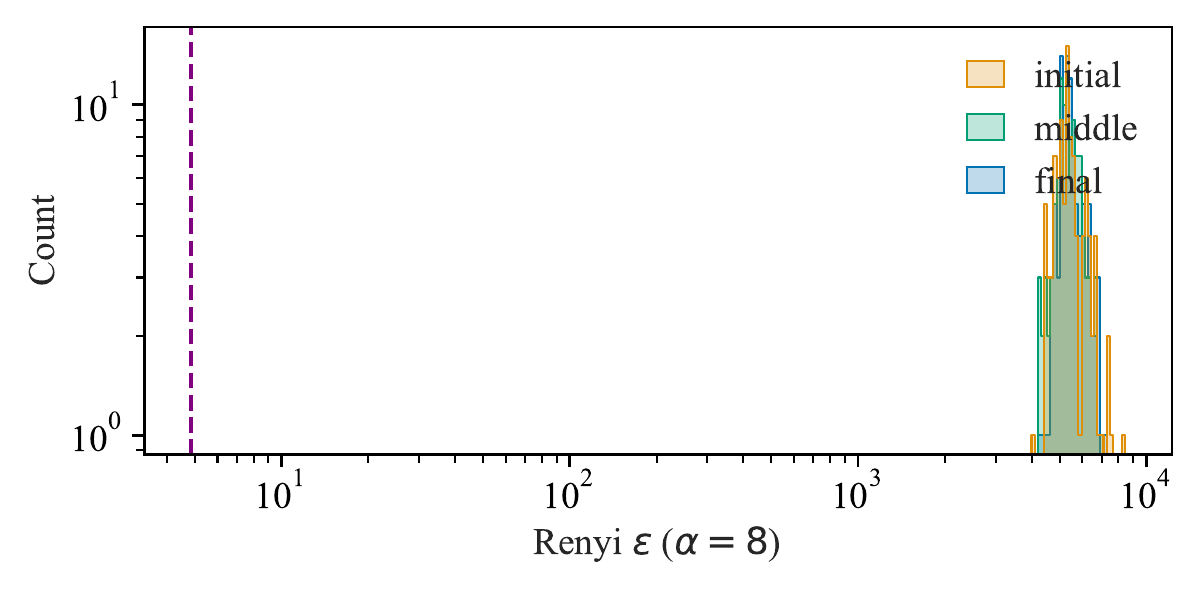}
}
\subfloat[Mini-batch size = 128]
{
\includegraphics[width=0.48\linewidth]{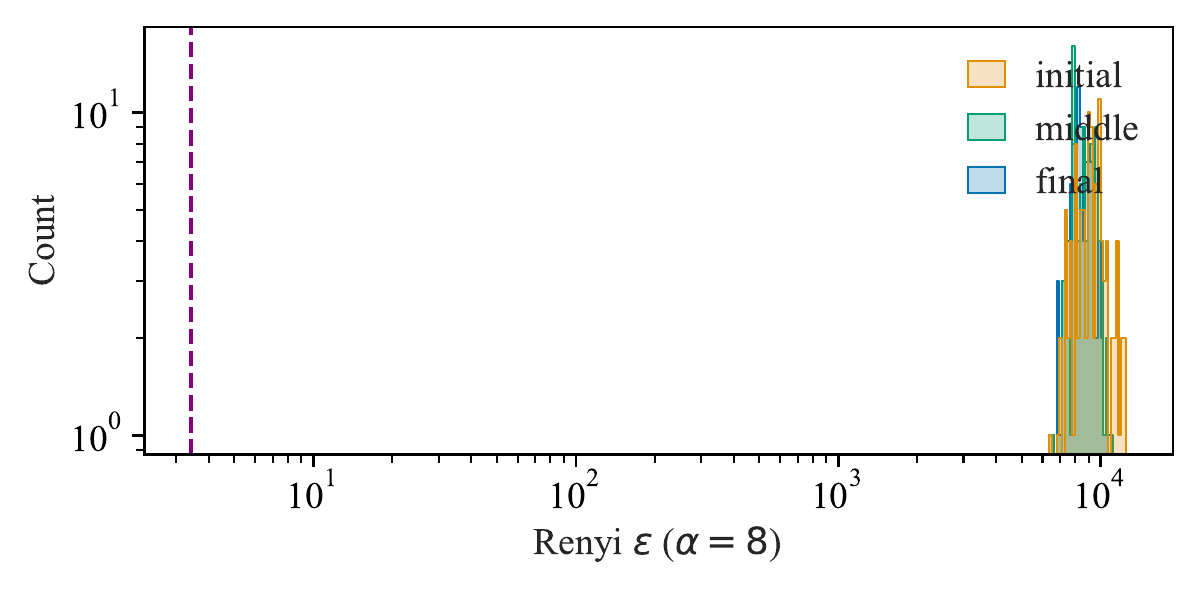}
}
\\\subfloat[Mini-batch size = 256]
{
\includegraphics[width=0.48\linewidth]{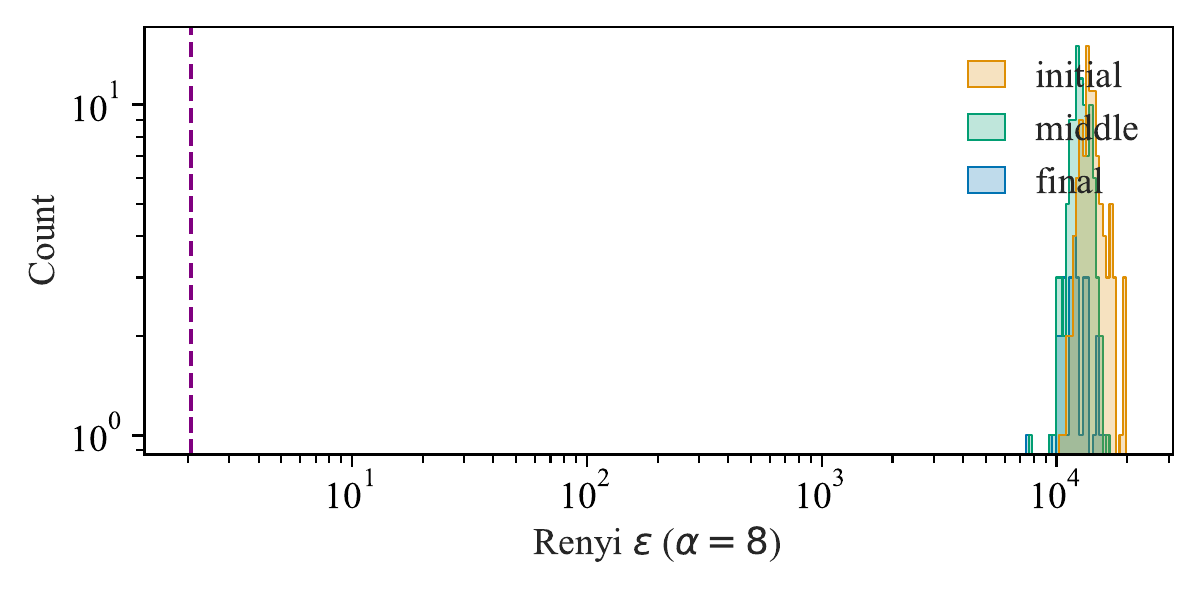}
}
\caption{ This is the reproduction of Figure~\ref{fig:renyi_hard_eps_distrib_bs_mnist_sum} except we now use ResNet-20 models trained on CIFAR-10. The results are similar.
}
\label{fig:renyi_hard_eps_distrib_bs_cifar_sum}
\end{figure}

\begin{figure}[t]
\centering
\subfloat[Alpha = 2]
{
\includegraphics[width=0.48\linewidth]{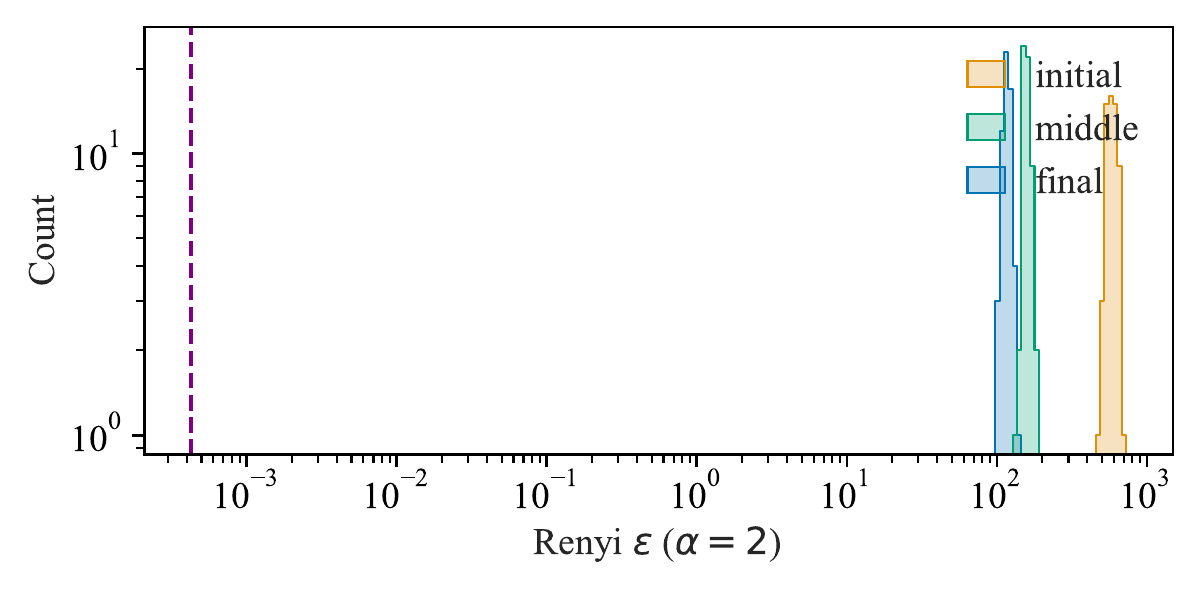}
}
\subfloat[Alpha = 4]
{
\includegraphics[width=0.48\linewidth]{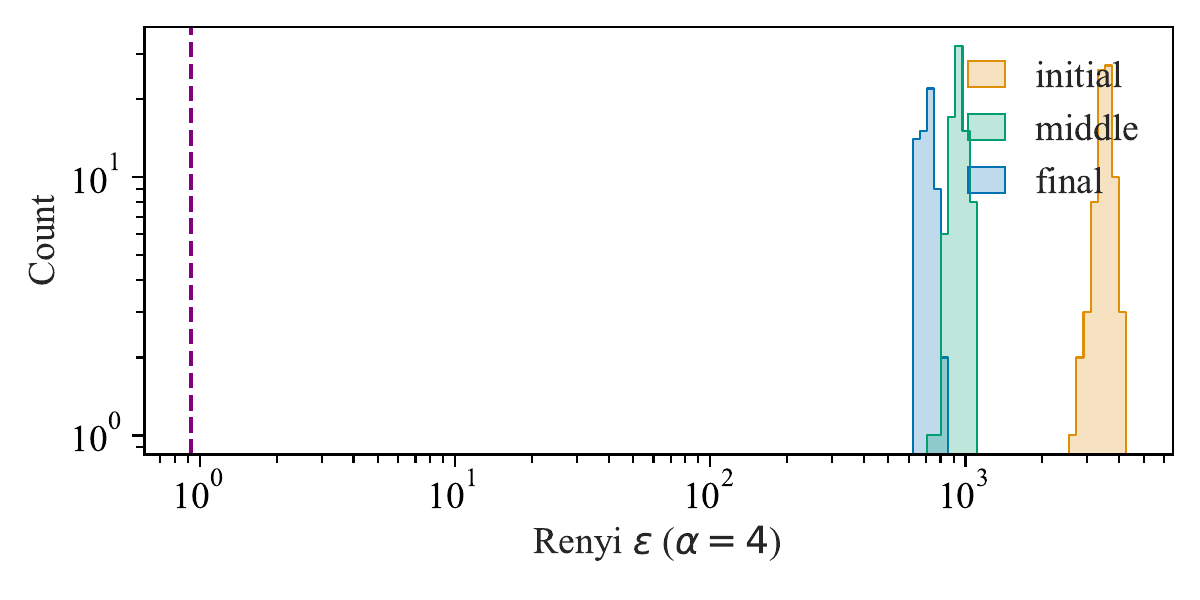}
}
\\\subfloat[Alpha = 8]
{
\includegraphics[width=0.48\linewidth]{figures/renyi_hard_eps_hist_MNIST_lenet_128_8_sum.pdf}
}
\subfloat[Alpha = 16]
{
\includegraphics[width=0.48\linewidth]{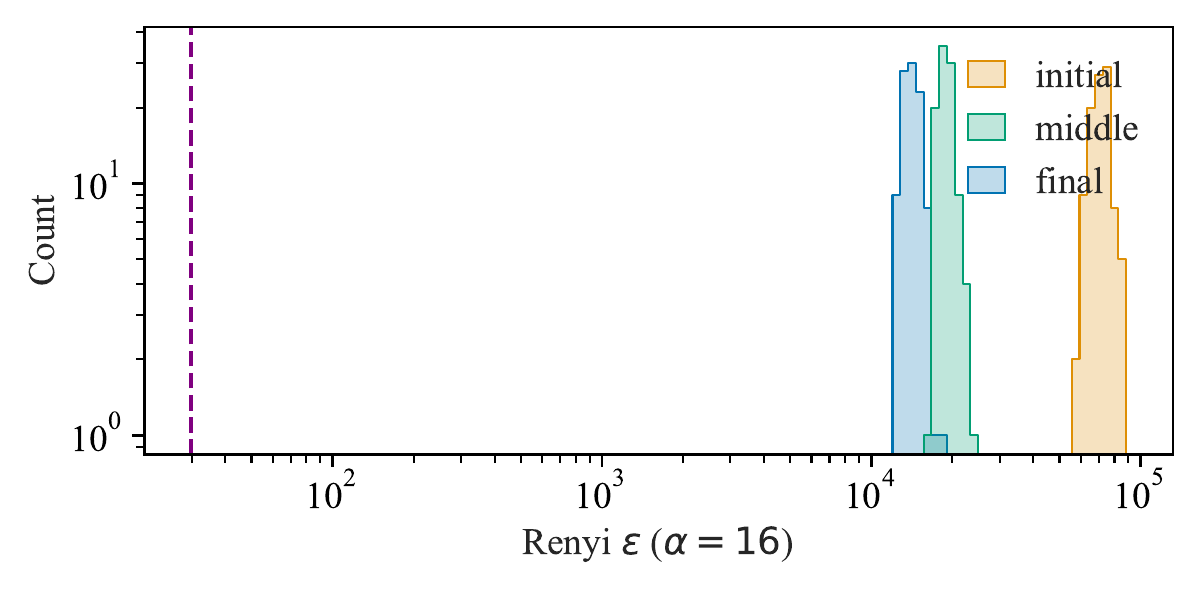}
}
\\\subfloat[Alpha = 32]
{
\includegraphics[width=0.48\linewidth]{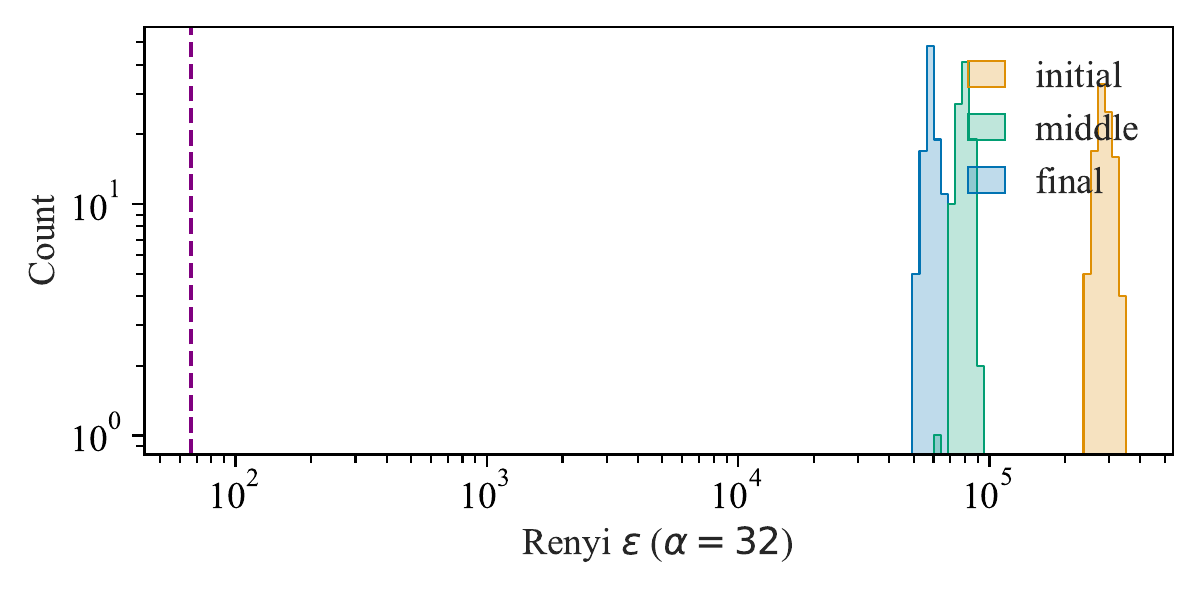}
}
\subfloat[Alpha = 64]
{
\includegraphics[width=0.48\linewidth]{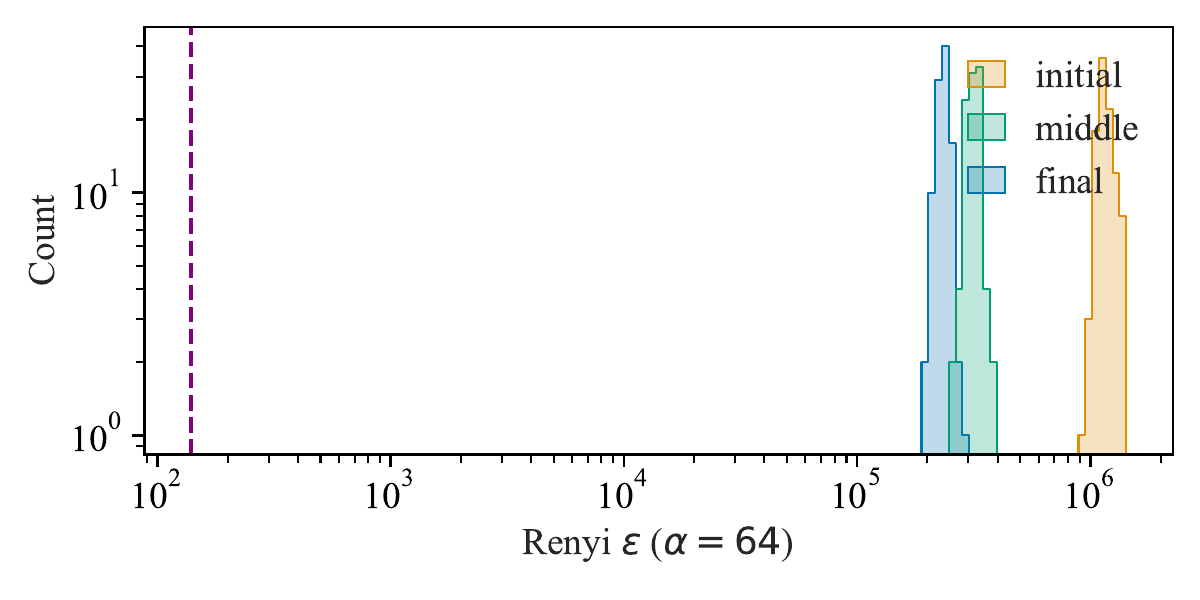}
}
\caption{ This is the reproduction of Figure~\ref{fig:renyi_hard_eps_distrib_bs_mnist_sum} except we now fix the batch size to be 128 and vary the value of $\alpha$. It can be seen that increasing alpha does not bring our guarantee and the baseline guarantee closer.
}
\label{fig:renyi_hard_eps_distrib_alpha_mnist_sum}
\end{figure}

\begin{figure}[t]
\centering
\subfloat[Alpha = 2]
{
\includegraphics[width=0.48\linewidth]{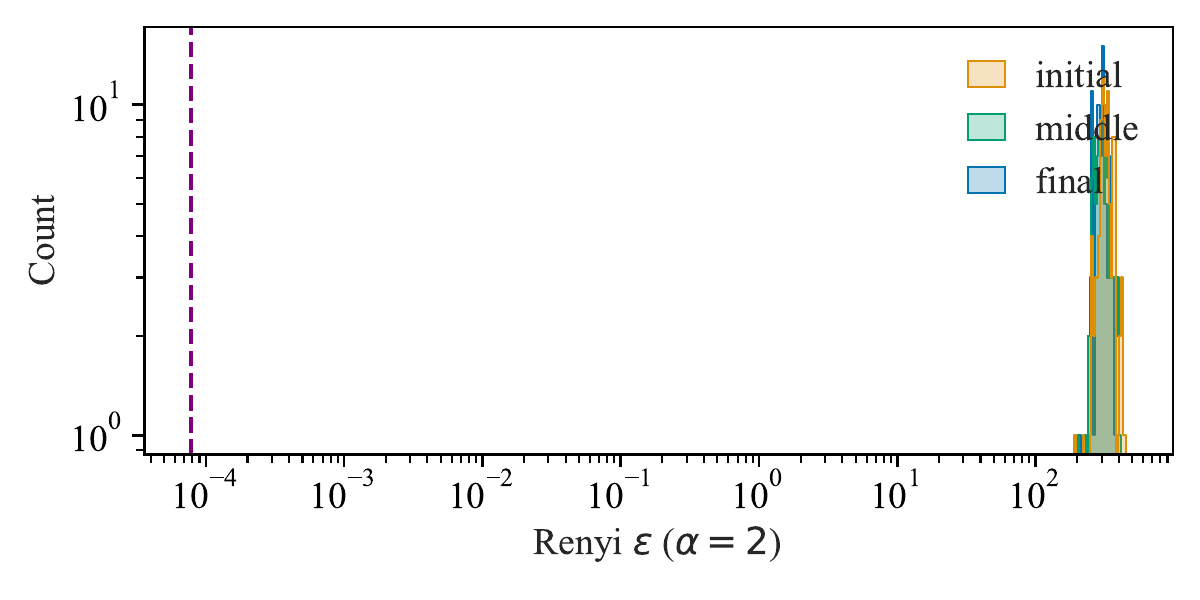}
}
\subfloat[Alpha = 4]
{
\includegraphics[width=0.48\linewidth]{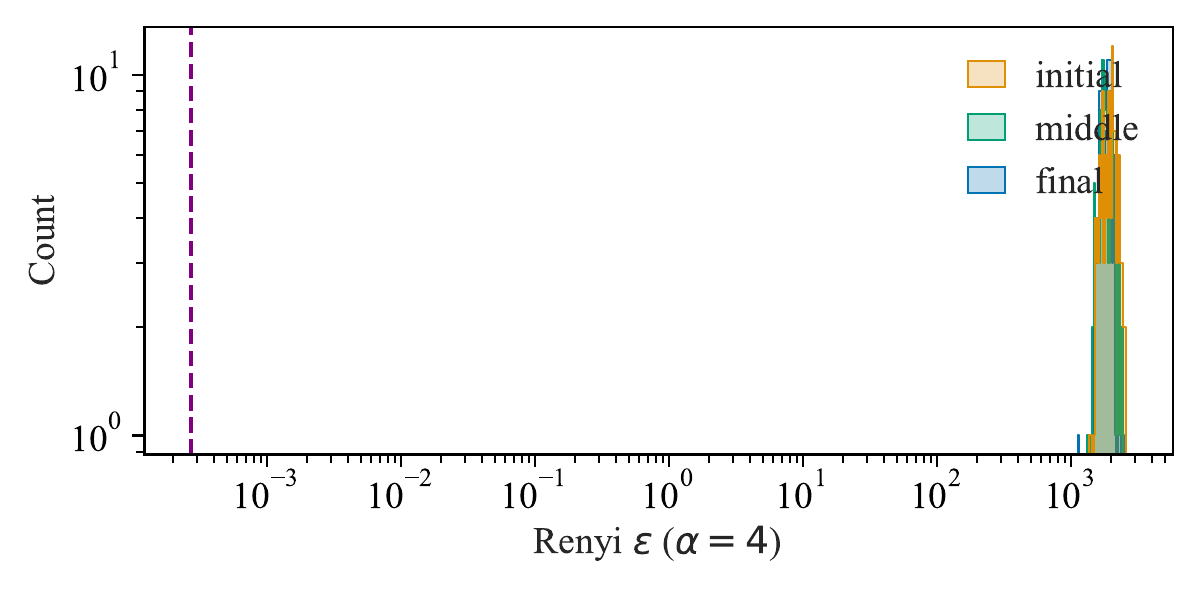}
}
\\\subfloat[Alpha = 8]
{
\includegraphics[width=0.48\linewidth]{figures/renyi_hard_eps_hist_CIFAR10_resnet20_128.0_8_sum.pdf}
}
\subfloat[Alpha = 16]
{
\includegraphics[width=0.48\linewidth]{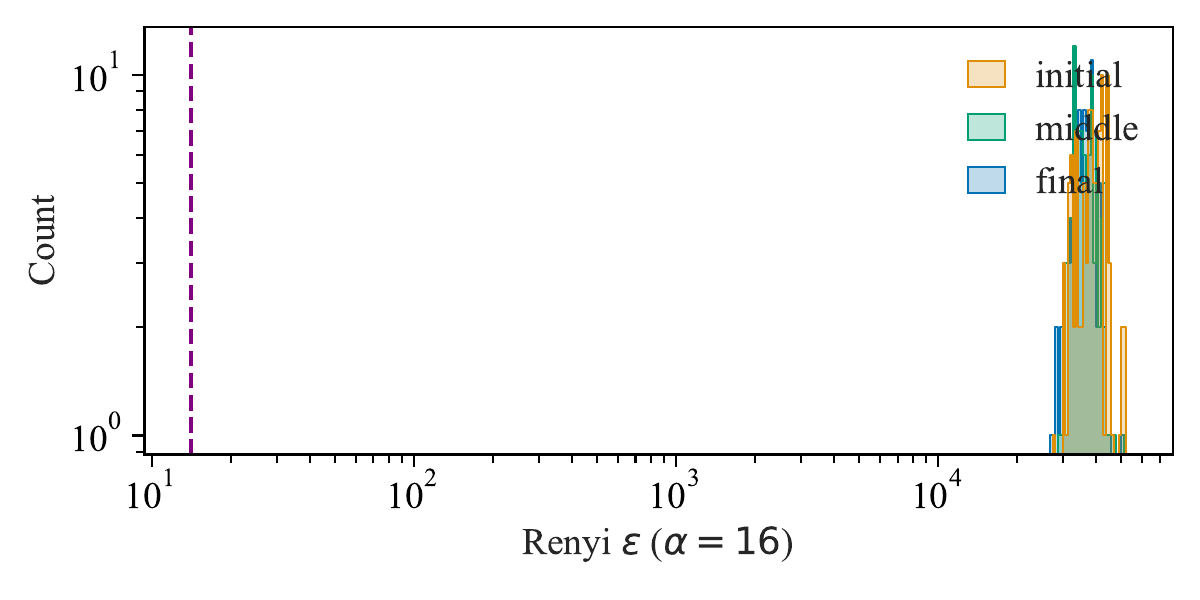}
}
\\\subfloat[Alpha = 32]
{
\includegraphics[width=0.48\linewidth]{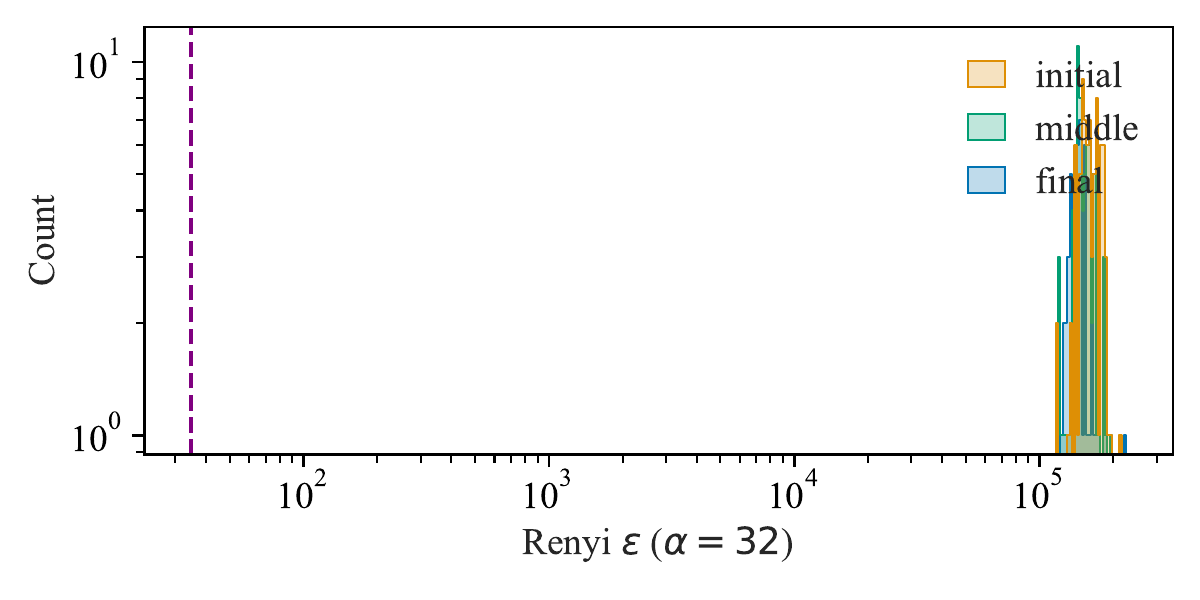}
}
\subfloat[Alpha = 64]
{
\includegraphics[width=0.48\linewidth]{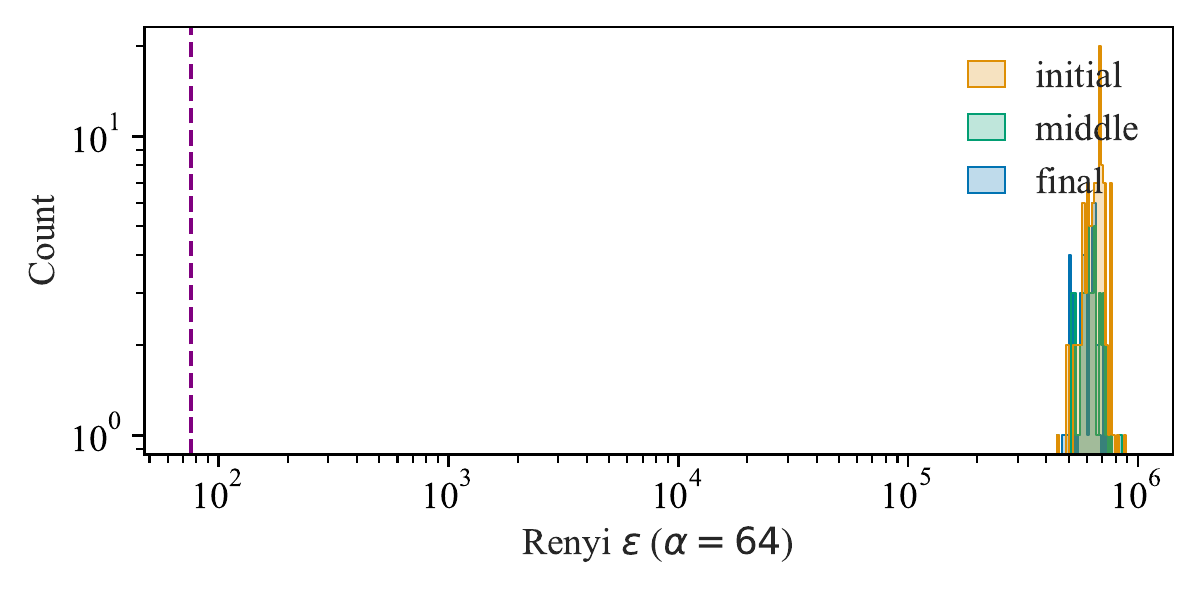}
}
\caption{ This is the reproduction of Figure~\ref{fig:renyi_hard_eps_distrib_alpha_mnist_sum} except we now use ResNet-20 models trained on CIFAR-10. The results are similar.
}
\label{fig:renyi_hard_eps_distrib_alpha_cifar_sum}
\end{figure}

\begin{figure}[t]
\centering
\subfloat[Update-rule: mean]
{
\includegraphics[width=0.48\linewidth]{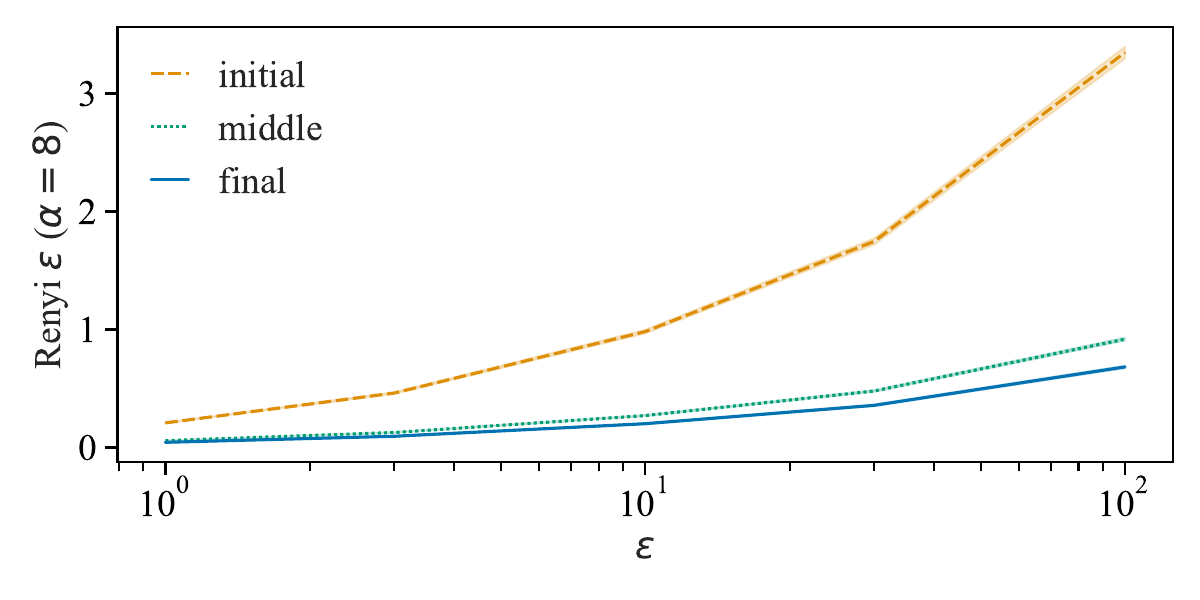}
}
\subfloat[Update-rule: sum]
{
\includegraphics[width=0.48\linewidth]{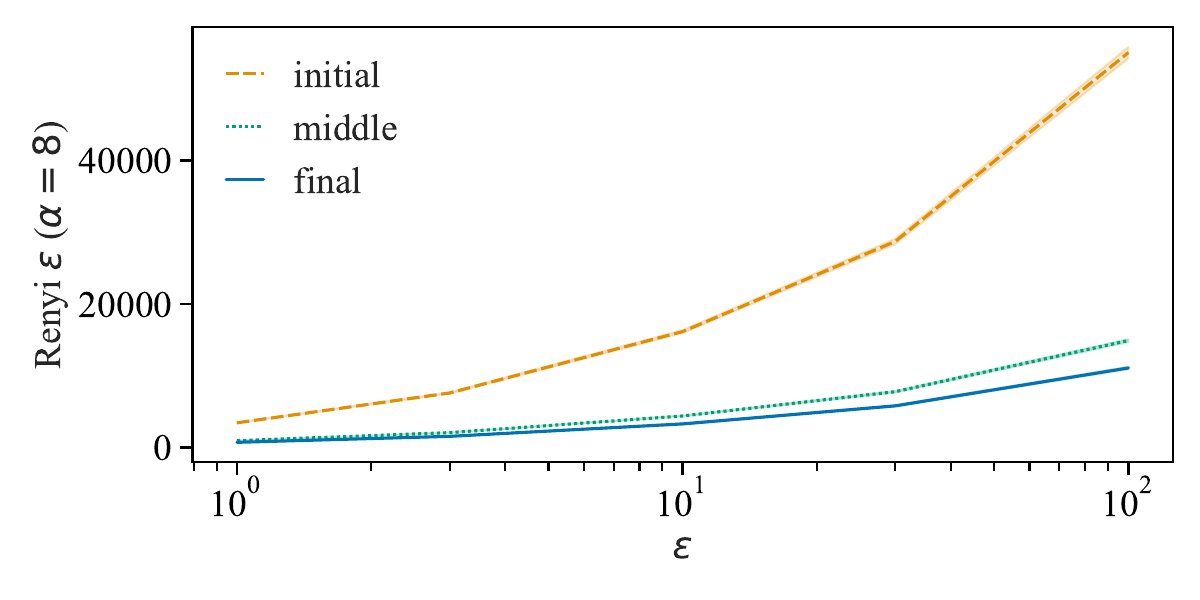}
}
\caption{ Per-step R\'enyi-DP guarantee given by Theorem~\ref{thm:renyi_dp_sens} computed on LeNet-5 trained on MNIST as a function of $\epsilon$, plotted at 3 stages of training and varying mini-batch sizes with 2 different update rules. We can see that as training proceeds, our guarantees increase slower while $\epsilon$ increases.
}
\label{fig:renyi_hard_eps_curve_vary_eps_mnist}
\end{figure}

\begin{figure}[t]
\centering
\subfloat[Update-rule: mean]
{
\includegraphics[width=0.48\linewidth]{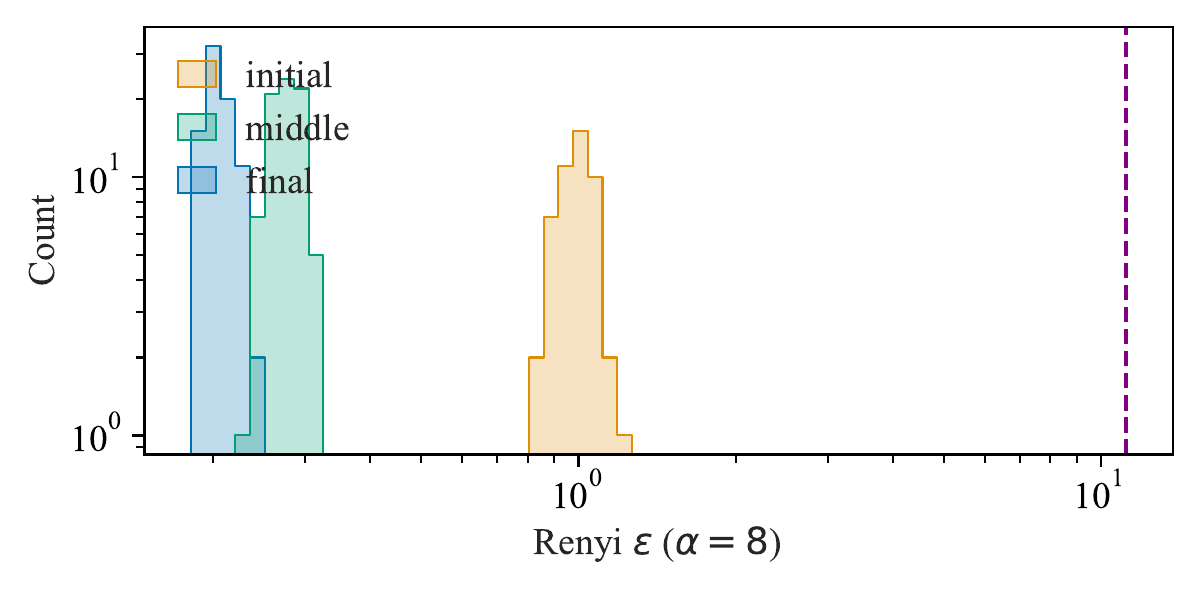}
}
\subfloat[Update-rule: sum]
{
\includegraphics[width=0.48\linewidth]{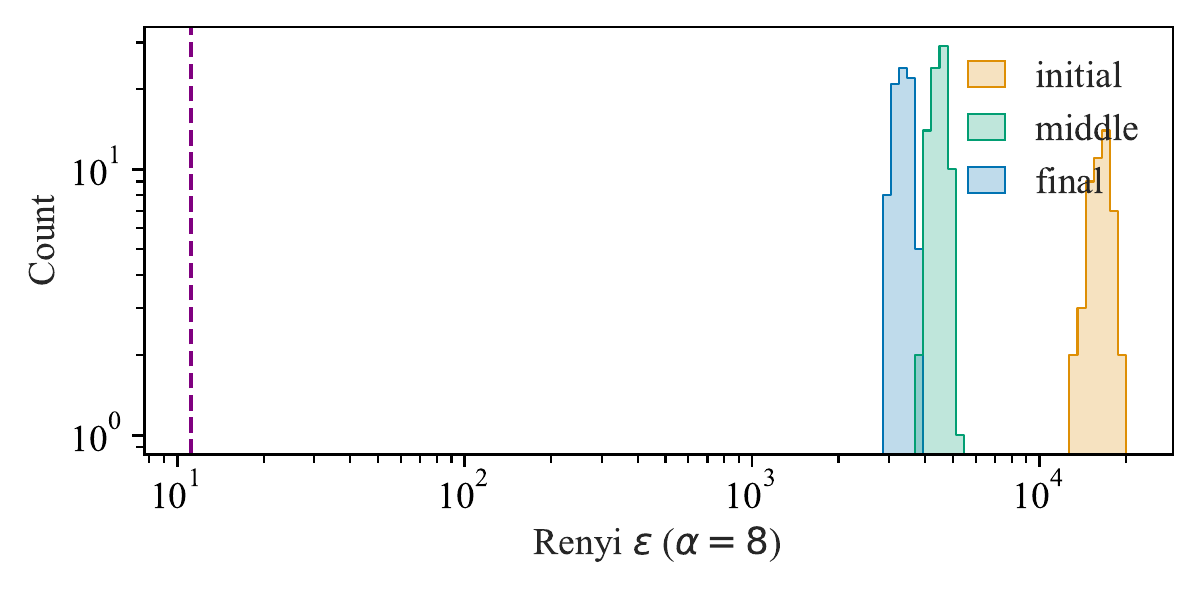}
}
\caption{ This is the reproduction of the mini-batch size 128 case of Figures~\ref{fig:renyi_hard_eps_distrib_bs_mnist_mean} and~\ref{fig:renyi_hard_eps_distrib_bs_mnist_sum} where $X$ and $X'$ are swapped to show that our bounds for both divergences are tighter than the baseline for mean update rule. Thus our guarantees are better than the baseline.
}
\label{fig:reverse_renyi}
\end{figure}

\begin{figure}[t]
\centering
\subfloat[Update-rule: mean]
{
\includegraphics[width=0.48\linewidth]{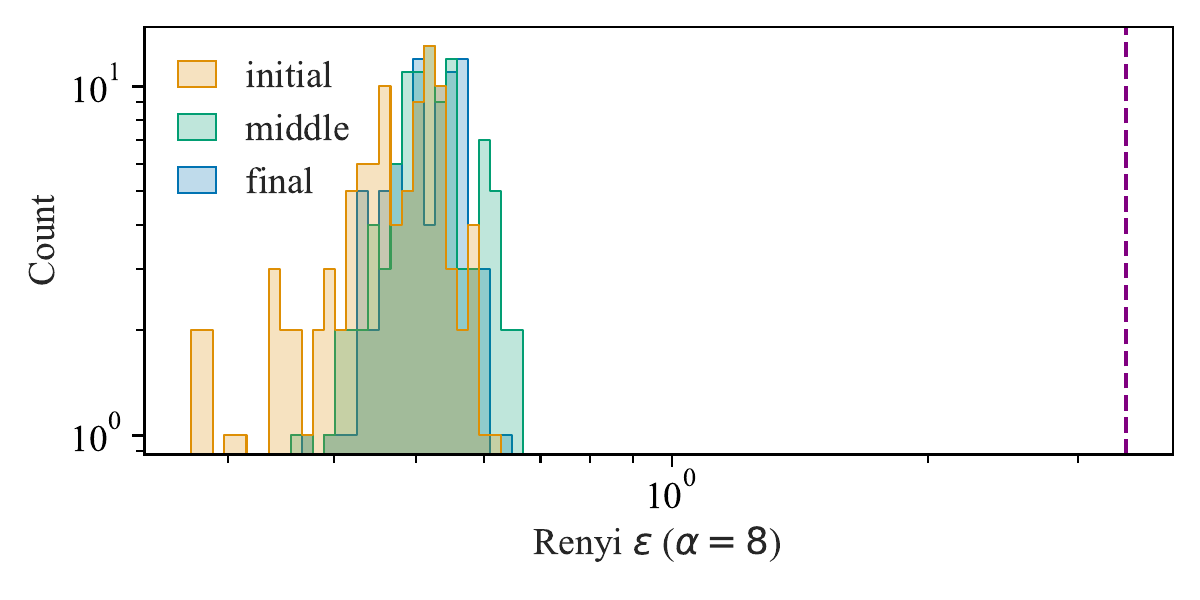}
}
\subfloat[Update-rule: sum]
{
\includegraphics[width=0.48\linewidth]{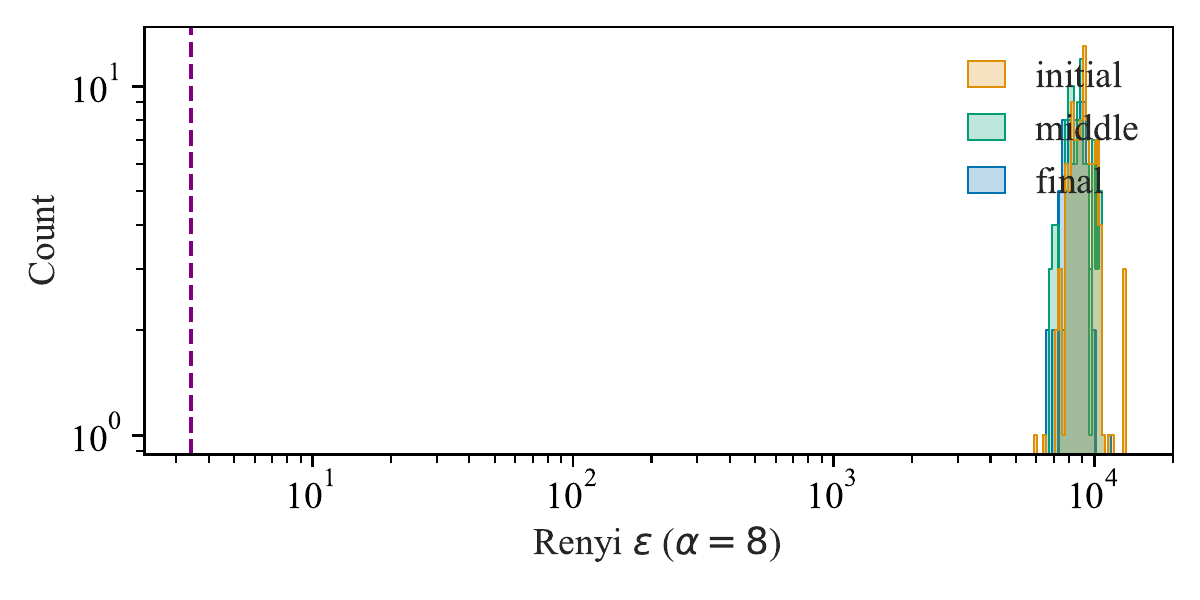}
}
\caption{ This is the reproduction of Figures~\ref{fig:reverse_renyi} except 
we now use ResNet-20 models trained on CIFAR-10. The results are consistent with MNIST.}
\label{fig:reverse_renyi_cifar}
\end{figure}

%% file: Sections/Concise_Empirical_Results.tex
\section{Additional Empirical Results}
\label{sec:detailed_emp_res}

We present additional experiments here.

\begin{figure}[h!]
\centering
\subfloat[varying batch size]
{
\includegraphics[width=0.48\linewidth]{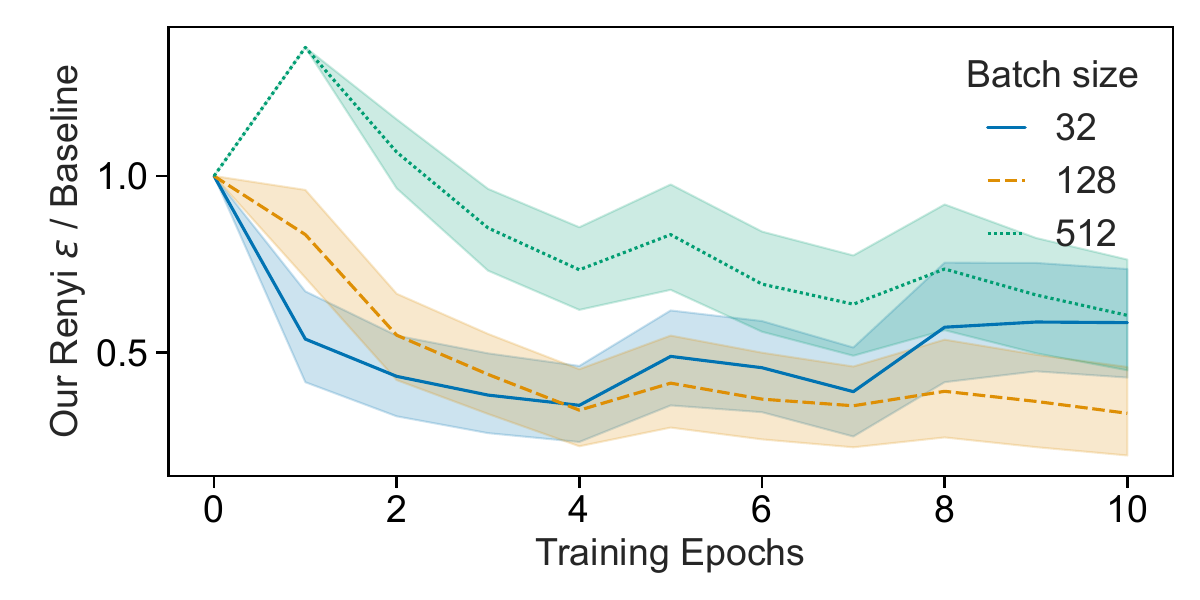}
}
\subfloat[varying architecture]
{
\includegraphics[width=0.48\linewidth]{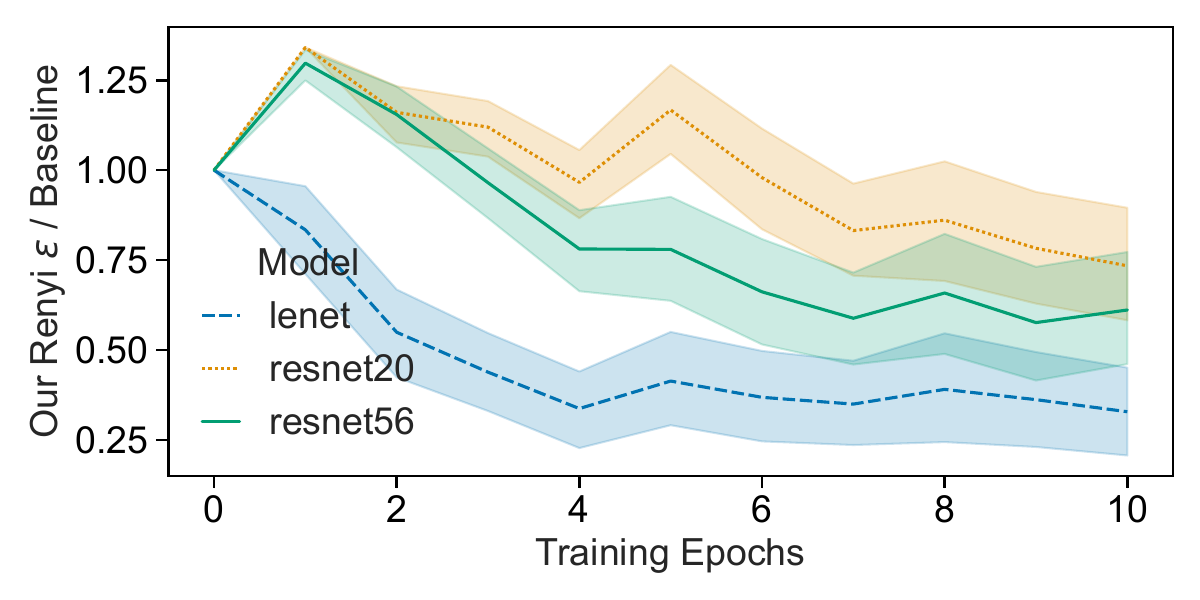}
}
\\\subfloat[varying epsilon]
{
\includegraphics[width=0.48\linewidth]{figures_camera_ready/compo_simple_eps_compo_vary_eps_fraction_curve_MNIST_eps.pdf}
}
\subfloat[varying batch size\\$\text{~~~~~}$($10^{th}$ percentile)]
{
\includegraphics[width=0.48\linewidth]{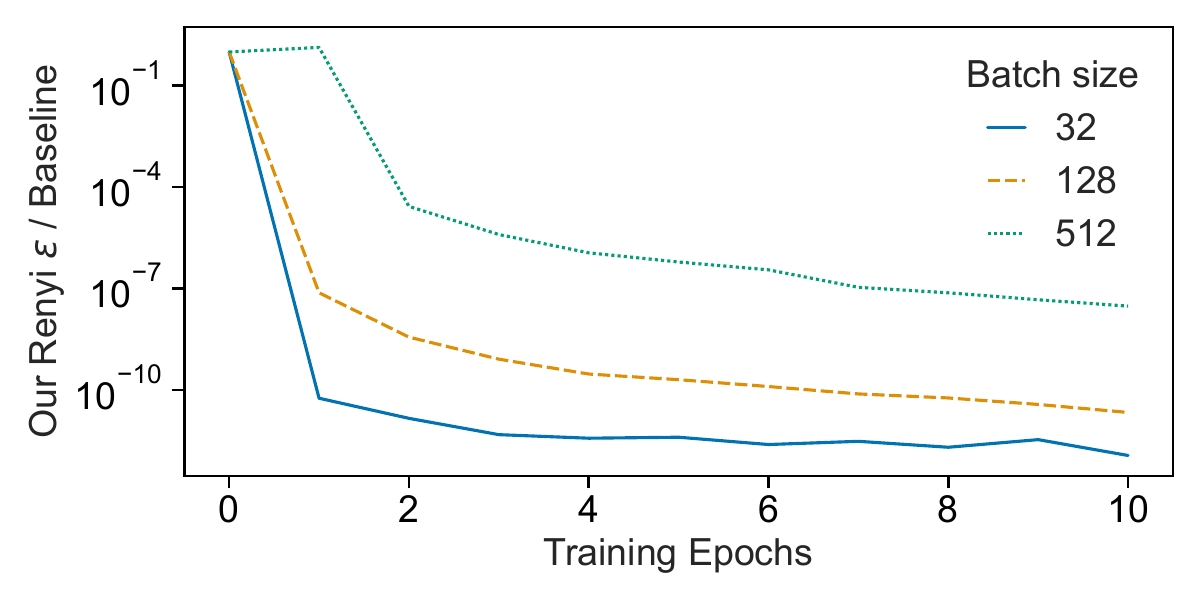}
}
\\\subfloat[varying architecture\\$\text{~~~~~}$($10^{th}$ percentile)]
{
\includegraphics[width=0.48\linewidth]{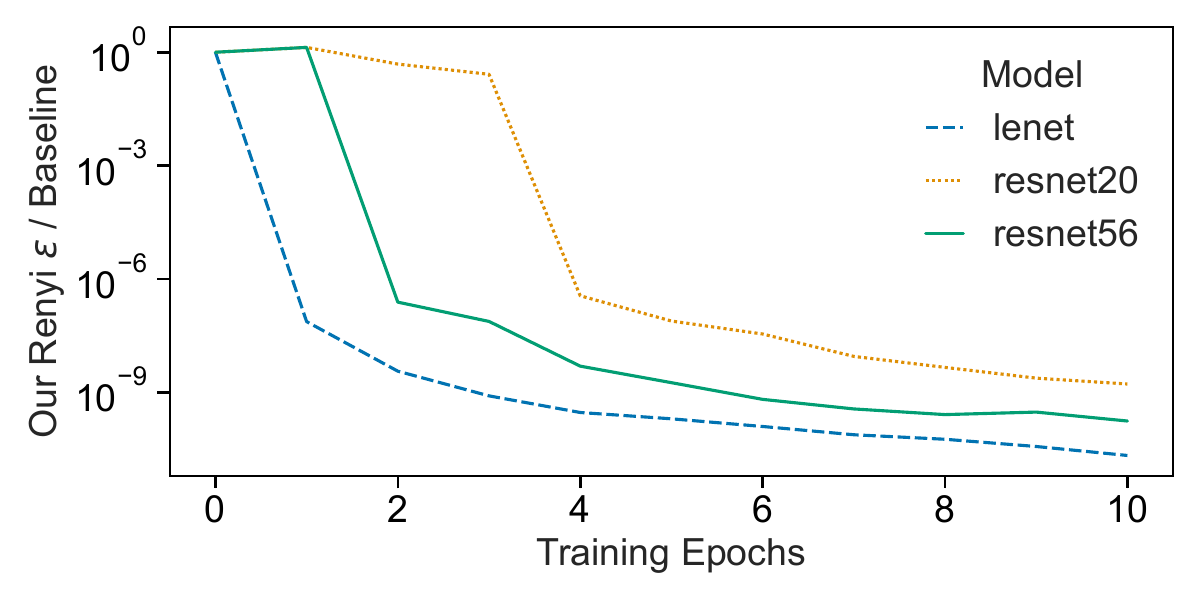}
}
\subfloat[varying epsilon\\$\text{~~~~~}$($10^{th}$ percentile)]
{
\includegraphics[width=0.48\linewidth]{figures_camera_ready/compo_simple_eps_compo_vary_eps_fraction_curve_MNIST_epspercentile10.pdf}
}
\caption{\footnotesize  Expected privacy guarantees from Theorem~\ref{thm:easy_renyi_dp} plotted as a fraction of the per-step DP-SGD guarantee over training. One can see that the ratio between our guarantee and the per-step DP-SGD guarantee (the baseline) decreases as training approaches the end, and this is consistent across different strengths of DP (i.e., $\epsilon$ set for the entire training), varying mini-batch size, and different model architectures.
}
\label{fig:renyi_simple_composition_mnist_sum}
\end{figure}

\begin{figure}[t]
\vspace{-3mm}
\centering
\subfloat[MNIST]
{
\includegraphics[width=0.48\linewidth]{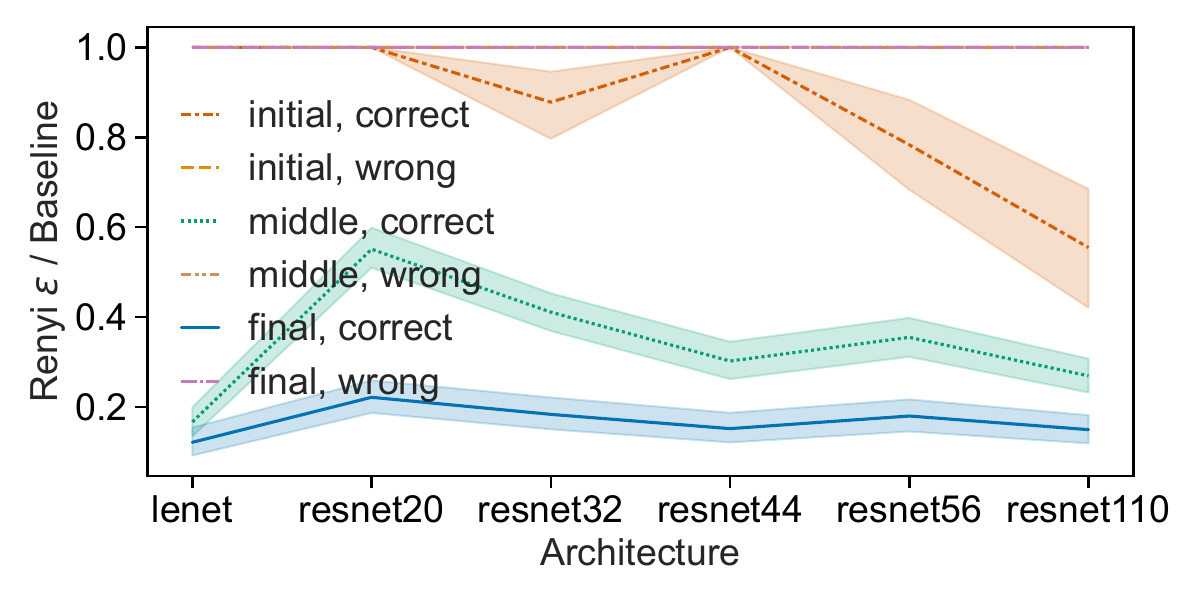}
}
\subfloat[CIFAR10]
{
\includegraphics[width=0.48\linewidth]{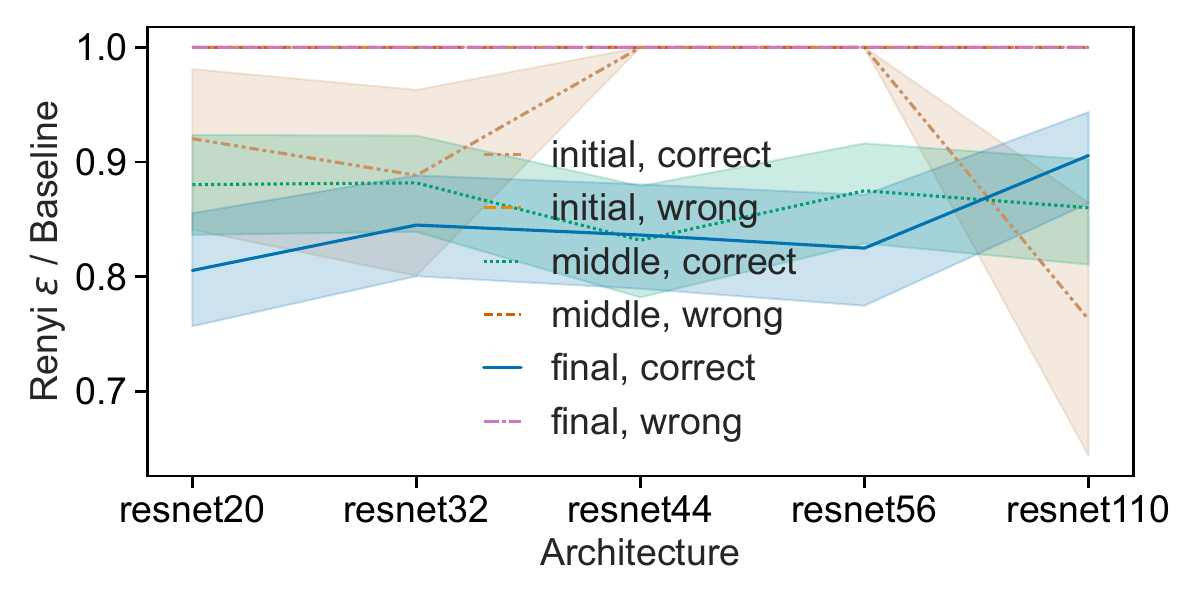}
}
\caption{\footnotesize \looseness=-1 Per-step R\'enyi-DP guarantee given by Theorem~\ref{thm:easy_renyi_dp} (divided by the baseline guarantee) plotted with respect to different model architectures trained on MNIST and CIFAR-10, at 3 stages of training. Consistently across different architectures and datasets, data points at later stages of training that are correctly classified have significantly better privacy guarantees than the baseline.
\vspace{-5mm}
}
\label{fig:renyi_simple_fraction_curve_vary_arch_mnist}
\end{figure}

\begin{figure}[t]
\centering
\subfloat[Mini-batch size = 16]
{
\includegraphics[width=0.48\linewidth]{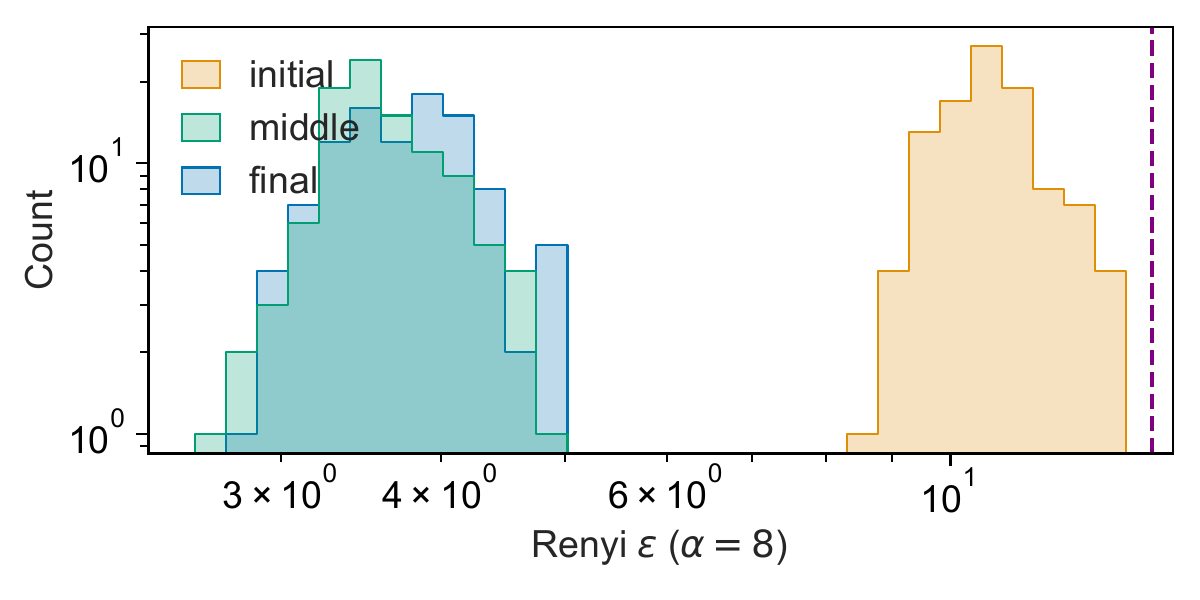}
}
\subfloat[Mini-batch size = 32]
{
\includegraphics[width=0.48\linewidth]{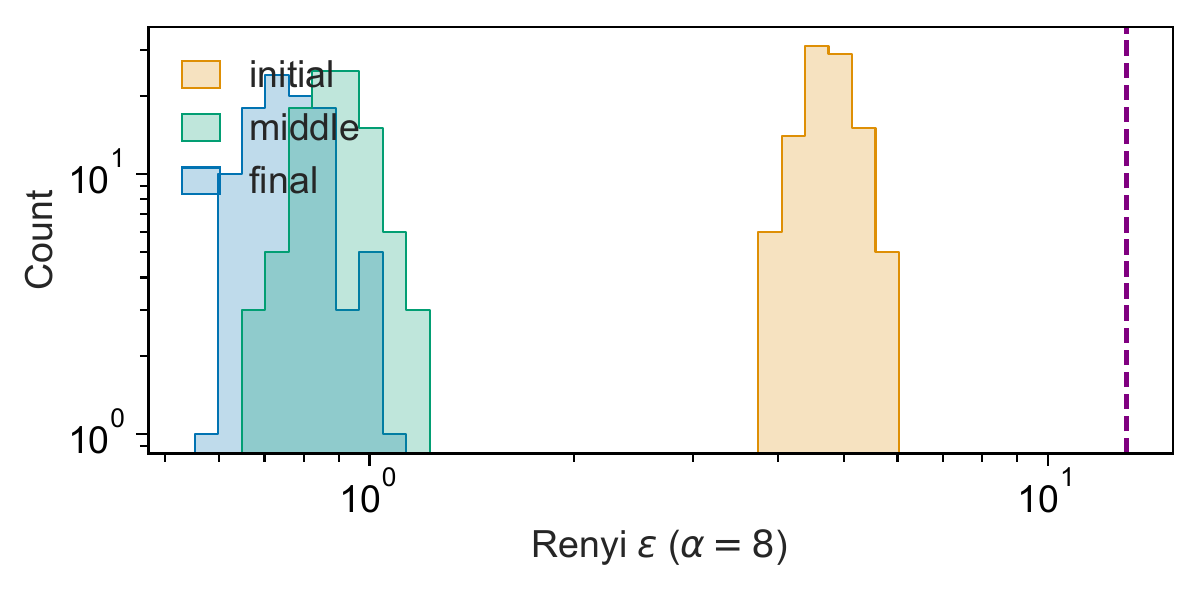}
}
\\\subfloat[Mini-batch size = 64]
{
\includegraphics[width=0.48\linewidth]{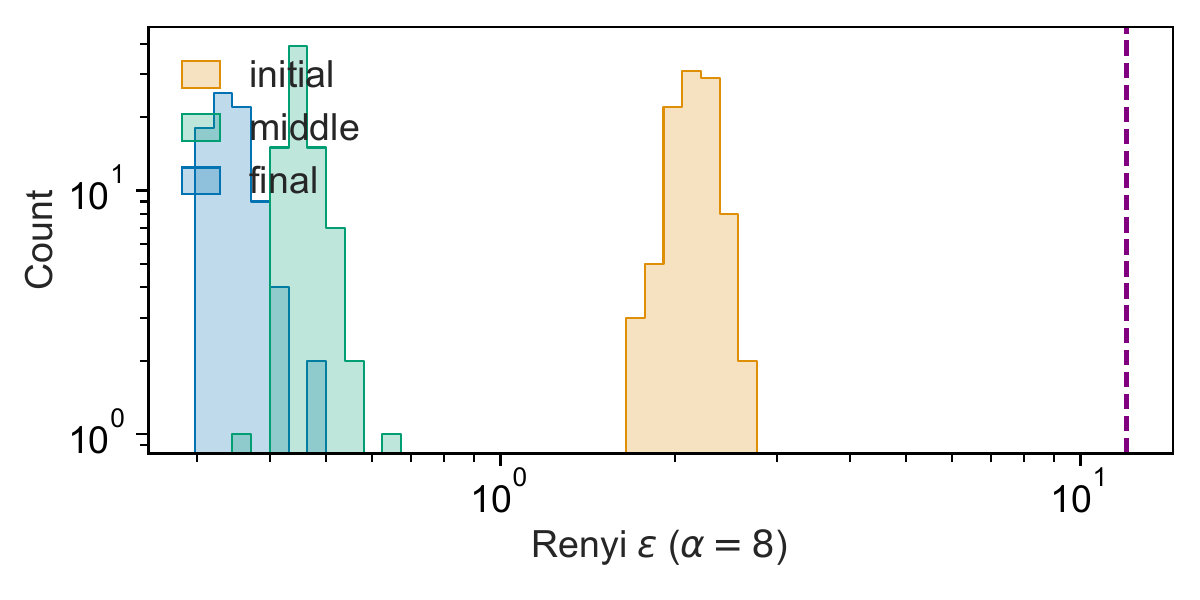}
}
\subfloat[Mini-batch size = 128]
{
\includegraphics[width=0.48\linewidth]{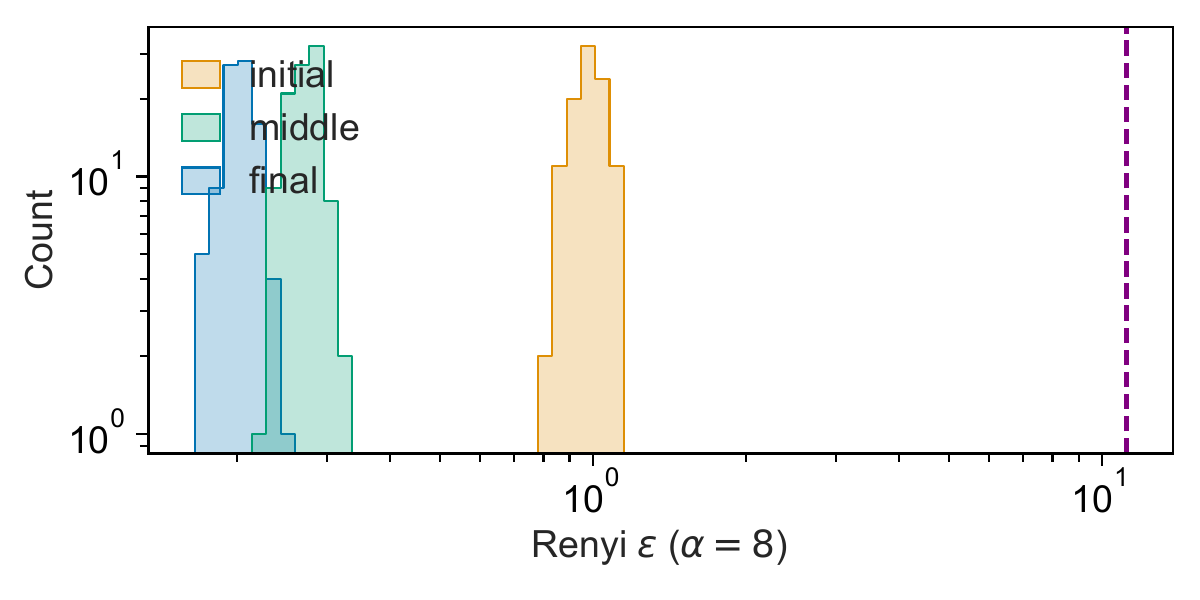}
}
\\\subfloat[Mini-batch size = 256]
{
\includegraphics[width=0.48\linewidth]{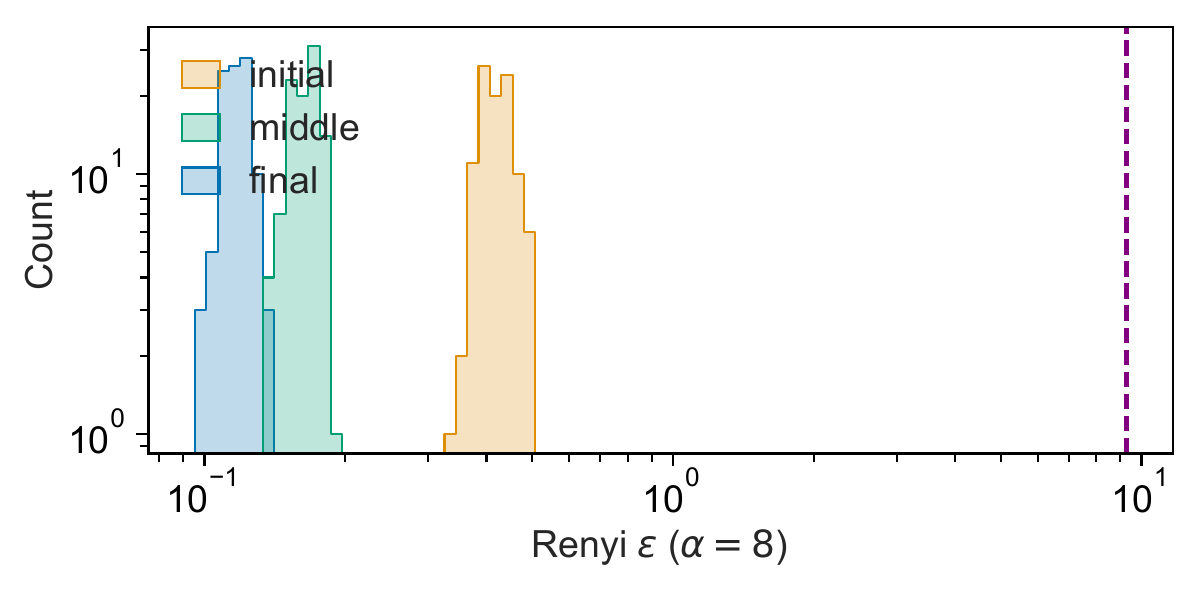}
}
\subfloat[Mini-batch size = 512]
{
\includegraphics[width=0.48\linewidth]{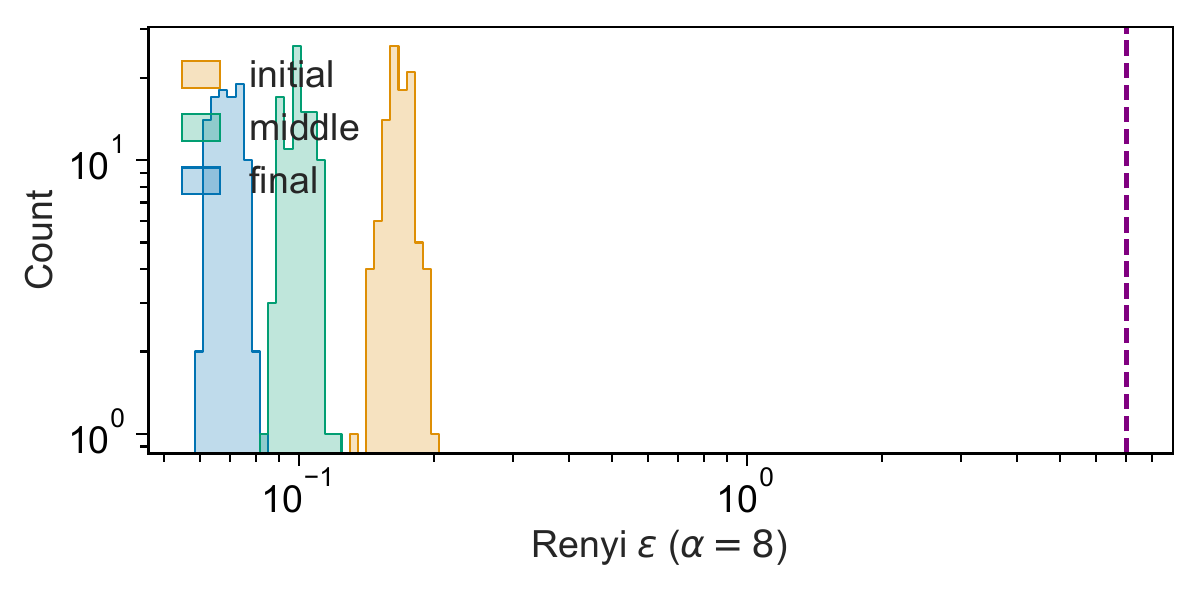}
}

\caption{\footnotesize \looseness=-1 Distribution plots (log scale) of per-step guarantees given by Theorem~\ref{thm:renyi_dp_sens} computed on LeNet-5 trained on MNIST with mean update rule and varying mini-batch sizes of $16, 32, 64, 128, 256, 512$. As specified by the legend labels, we group the plotted guarantees by whether the model is at the initial, middle, or final stage of the training. It can be seen that in all settings our guarantee is better than the baseline, which is represented by the dashed purple line. However, the guarantee distributions of points at the initial stage of training are closer to the baseline compared to the other distributions.
Additionally, since a mean update rule is used, the bounds depend on the mini-batch size, and better bounds are achieved when the mini-batch size is large. 
\vspace{-5mm}
}
\label{fig:renyi_hard_eps_distrib_bs_mnist_mean}
\end{figure}

\begin{figure}[t]
\centering
\subfloat[MNIST]
{
\includegraphics[width=0.48\linewidth]{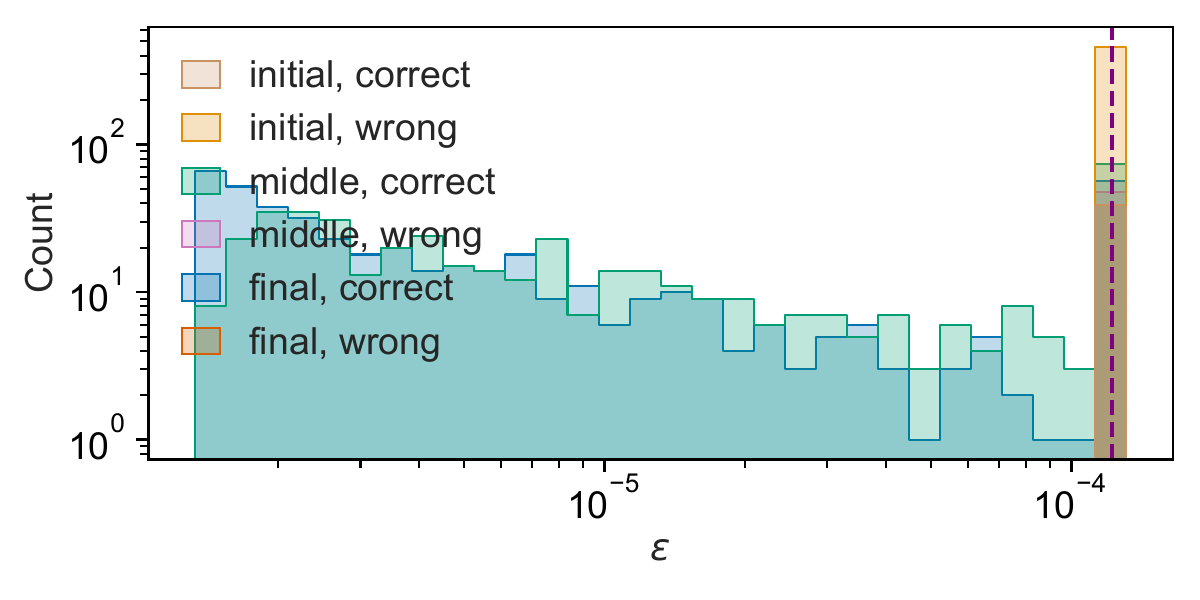}
}
\subfloat[CIFAR10]
{
\includegraphics[width=0.48\linewidth]{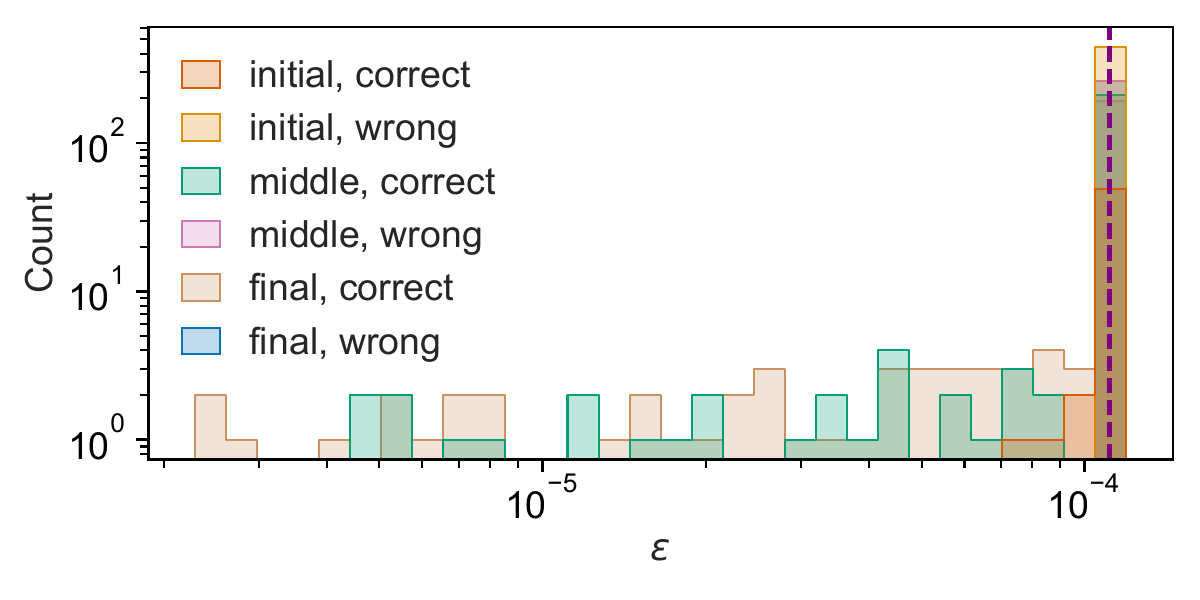}
}
\caption{\footnotesize \looseness=-1 Distribution plots of per-step guarantee given by Corollary~\ref{cor:eps_delta_sens} computed on MNIST and CIFAR10 with sum update rule with expected minibatch size $128$, for different stages of training and correctly and incorrectly classified points. It can be seen that in all settings our guarantee is better than the baseline. However, the guarantee distributions of incorrectly classified points and points at the initial stage of training are closer to the baseline compared to the other settings.
}
\label{fig:ed_eps_distrib_sum}
\end{figure}